\documentclass[twoside,11pt]{article}
\usepackage{journal}
\usepackage{natbib}
\usepackage{bbm}
\usepackage{times,epsfig,ifthen}
\usepackage{latexsym}
\usepackage{amsfonts}
\usepackage{eepic}
\usepackage{amsopn}
\usepackage{amsmath,amssymb,rotating,graphicx,rotating,float}
\usepackage[dvips,backref,pageanchor=true,plainpages=false,
pdfpagelabels,bookmarks,bookmarksnumbered,breaklinks,
pdfborder={0 0 0},  
]{hyperref}
\usepackage{cite}
\usepackage[mathscr]{eucal}

\usepackage{ulem}
\usepackage{color} 

\newcommand{\Px}{\mathcal P} % marginal distribution on \X
\newcommand{\PXY}{{\mathcal P}_{XY}} % joint distribution over \X \times \Y
\newcommand{\X}{\mathcal X} % instance space
\newcommand{\alg}{\mathcal A} % learning algorithm
\renewcommand{\H}{\mathcal H} % hypothesis space
\renewcommand{\L}{\mathcal L} % labeled set
\newcommand{\U}{\mathcal U} % unlabeled set
\newcommand{\er}{\operatorname{er}}
 
\newcommand{\Q}{\mathcal Q} 
 
\newcommand{\C}{\mathbb C} % concept space

\renewcommand{\P}{\mathbb P} % probability
\newcommand{\nats}{\mathbb{N}} % natural numbers
\newcommand{\reals}{\mathbb{R}} % real numbers
\newcommand{\E}{\mathbb E}
\newcommand{\Data}{\mathbf{\mathcal Z}}

\newcommand{\vc}{d}
\newcommand{\eps}{\varepsilon}
\newcommand{\NaiveActivizer}{Meta-Algorithm 0}
\newcommand{\BasicActivizer}{Meta-Algorithm 1}
\newcommand{\CAL}{Meta-Algorithm 2}
\newcommand{\Shattering}{Meta-Algorithm 3}
\newcommand{\RobustCAL}{Algorithm 4}
\newcommand{\RobustShattering}{Algorithm 5}
\newcommand{\diam}{{\rm diam}}
\newcommand{\argmax}{\mathop{\rm argmax}}
\newcommand{\argmin}{\mathop{\rm argmin}}

\newcommand{\conf}{\delta}
\newcommand{\mycomment}[1]{}
\newcommand{\comment}[1]{}

\newcommand{\Ball}{{\rm B}}
\newcommand{\DIS}{\mathrm{DIS}}

\newcommand{\poly}{\mathrm{poly}}
\newcommand{\polylog}{\mathrm{polylog}}
\newcommand{\Polylog}{\mathrm{Polylog}}

\newcommand{\ind}{{\mathbbm{1}}}

\renewcommand{\S}{{\cal S}}

\newcommand{\cc}{q}
\newcommand{\vrad}{\phi}
\newcommand{\init}{\tau}
\newcommand{\truV}{ V^{\star} }
\newcommand{\truL}{\L^{\star}}
\newcommand{\truT}{T^{\star}}
\newcommand{\truI}{I^{\star}}
\newcommand{\truQ}{Q^{\star}}
\newcommand{\truC}{C^{\star}}
\newcommand{\cl}{\mathrm{cl}} % closure
\newcommand{\s}{s}
\newcommand{\ActiveSelect}{{\rm ActiveSelect}}
\newcommand{\Nontrivial}{{\rm Nontrivial}}
\newcommand{\hdc}{\tilde{\theta}}
\newcommand{\dc}{\theta} % disagreement coefficient
\newcommand{\bdim}{\tilde{d}}
\newcommand{\Msize}[1]{#1^3}
\newcommand{\dprob}{\tilde{\delta}_{f}}
\newcommand{\ignore}[1]{}

\providecommand{\thmend}{}
  \renewcommand{\thmend}{%
    {\unskip\nobreak\hfil\penalty 50%
     \hskip 2em\hbox{}\nobreak\hfil$\diamond$%
     \parfillskip=0pt \finalhyphendemerits=0 \par}}

\providecommand{\thmendup}{}
  \renewcommand{\thmendup}{%
    {\unskip\nobreak\hfil\penalty 50%
     \hskip 0em\hbox{}\nobreak\hfil$\diamond$%
     \parfillskip=0pt \finalhyphendemerits=0 \par}}

\providecommand{\upthmend}[1]{}
  \renewcommand{\upthmend}[1]{%
    {\unskip\nobreak\hfil\penalty 50%
      \vskip #1\hskip 2em\hbox{}\nobreak\hfil $\diamond$%
      \parfillskip=0pt \finalhyphendemerits=0 \par}}

\newtheorem{condition}{Condition}

\newenvironment{bigboxit}{\begin{center}\begin{lrbox}{\savepar}
\begin{minipage}[h]{5.8in}
\begin{flushleft}}
{\end{flushleft}\end{minipage}\end{lrbox}\fbox{\usebox{\savepar}}
\end{center}}

\newsavebox{\savepar}

\jmlrheading{}{}{}{}{}{Steve Hanneke}

\ShortHeadings{Activized Learning}{Hanneke}
\firstpageno{1}

\begin{document}

\title{Activized Learning: Transforming Passive to Active \\with Improved Label Complexity\thanks{Some of
these (and related) results previously appeared in the author's doctoral dissertation \citep*{hanneke:thesis}.}}

\author{\name Steve Hanneke \email shanneke@stat.cmu.edu \\
\addr Department of Statistics\\
Carnegie Mellon University \\
Pittsburgh, PA 15213 USA
}

\editor{someone}

\maketitle

\begin{abstract}
We study the theoretical advantages of active learning over passive learning.
Specifically, we prove that, in noise-free classifier learning for VC classes,
any passive learning algorithm can be transformed into an active learning algorithm
with asymptotically strictly superior label complexity for all nontrivial target functions and
distributions. 
We further provide a general characterization of the
magnitudes of these improvements in terms of a novel
generalization of the disagreement coefficient.
%In particular, this supplies a concise sufficient condition for
%exponential rates of convergence in active learning.
We also extend these results
to active learning in the presence of label noise, and find that even under broad
classes of noise distributions, we can typically guarantee strict improvements over the
known results for passive learning.
\end{abstract}


\section{Introduction and Background}
\label{sec:intro}

The recent rapid growth in data sources
has spawned an equally rapid expansion in the number of potential applications of machine learning methodologies to extract useful concepts from this data.
However, in many cases, the bottleneck in the application process is the need to obtain accurate annotation of the raw data according to the target concept to be learned.
For instance, in webpage classification, it is straightforward to rapidly collect a large number of webpages, but training an accurate classifier typically requires a human
expert to examine and label a number of these webpages, which may require significant time and effort.
For this reason, it is natural to look for ways to reduce the total number of labeled examples required to train an accurate classifier.
In the traditional machine learning protocol, here referred to as \emph{passive learning}, the examples labeled by the expert are
sampled independently at random, and the emphasis is on designing learning algorithms that make the most effective use of
the number of these labeled examples available.
However, it is possible to go beyond such methods by altering the protocol itself,
allowing the learning algorithm to sequentially \emph{select} the examples
to be labeled, based on its observations of the labels of previously-selected examples;
this interactive protocol is referred to as \emph{active learning}.
The objective in designing this selection mechanism is to
focus the expert's efforts toward labeling only the most informative data for the learning process,
thus eliminating some degree of redundancy in the information content of the labeled examples.

It is now well-established that active learning can sometimes provide significant
practical and theoretical advantages over passive learning, in terms of the number of labels
required to obtain a given accuracy.  However, our current understanding of active learning
in general is still quite limited in several respects.
First, since we are lacking a complete understanding
of the potential capabilities of active learning, we are not yet sure to what standards we should
aspire for active learning algorithms to meet, and in particular this challenges our ability to characterize how a
``good'' active learning algorithm should behave.
Second, since we have yet to identify a complete set of
general principles for the design of effective active learning algorithms, in many cases the most effective
known active learning algorithms have problem-specific designs (e.g., designed specifically for linear separators, or decision trees, etc.,
under specific assumptions on the data distribution),
and it is not clear what components of their design can be abstracted and transferred to the design of active learning algorithms
for different learning problems (e.g., with different types of classifiers, or different data distributions).
Finally, we have yet to fully understand the scope of the relative benefits of active learning over passive learning,
and in particular the conditions under which such improvements are achievable, as well as a general characterization
of the potential magnitudes of these improvements.
In the present work, we take steps toward closing this gap in our understanding of the capabilities, general principles, and advantages of active learning.

Additionally, this work has a second theme, motivated by practical concerns.
To date, the machine learning community has invested decades of research into constructing solid, reliable, and well-behaved \emph{passive} learning algorithms,
and into understanding their theoretical properties.
We might hope that an equivalent amount of effort is \emph{not} required in order to discover and understand effective active learning algorithms.
In particular, rather than starting from scratch in the design and analysis of active learning algorithms,
it seems desirable to leverage this vast knowledge of passive learning, to whatever extent possible.
For instance, it may be possible to design active learning algorithms that
\emph{inherit} certain desirable behaviors or properties of a given passive learning algorithm.
In this way, we can use a given passive learning algorithm as a \emph{reference point},
and the objective is to design an active learning algorithm with performance guarantees strictly
superior to those of the passive algorithm.  Thus, if the passive learning algorithm has proven
effective in a variety of common learning problems, then the active learning algorithm should
be even better for those \emph{same} learning problems.  This approach also has the advantage
of immediately supplying us with a collection of theoretical guarantees on the performance
of the active learning algorithm: namely, improved forms of all known guarantees
on the performance of the given passive learning algorithm.

Due to its obvious practical advantages, this general line of informal thinking dominates the existing literature on empirically-tested
heuristic approaches to active learning, as most of the published heuristic active learning algorithms
make use of a passive learning algorithm as a subroutine (e.g., SVM, logistic regression, k-NN, etc.), constructing sets of
labeled examples and feeding them into the passive learning algorithm at various times during the
execution of the active learning algorithm (see the references in Section~\ref{sec:open-problems}).
Below, we take a more rigorous look at this general strategy.
We develop a reduction-style framework for studying
this approach to the design of active learning algorithms relative to a given passive learning algorithm.
We then proceed to develop and analyze a variety of such methods, to realize this approach
in a very general sense.

Specifically, we explore the following fundamental questions.

\begin{itemize}
\item Is there a general procedure that, given any passive learning algorithm,
transforms it into an active learning algorithm requiring significantly fewer labels to achieve a
given accuracy?
\item If so, how large is the reduction in the number of labels required by the resulting
active learning algorithm, compared to the number of labels required by the original passive algorithm?
\item What are sufficient conditions for an \emph{exponential} reduction in the number of labels required?
\item To what extent can these methods be made robust to imperfect or noisy labels?
\end{itemize}
In the process of exploring these questions, we find that for many interesting learning problems,
the techniques in the existing literature are not capable of realizing the full potential of active learning.
Thus, exploring this topic in generality requires us to develop novel insights and
entirely new techniques for the design of active learning algorithms.  We also
develop corresponding natural complexity quantities to characterize the performance
of such algorithms.  Several of the results we establish here are more general than any related results in the existing
literature, and in many cases the algorithms we develop use significantly fewer labels than any previously
published methods.

\subsection{Background}
\label{sec:background}

The term \emph{active learning} refers to a family of supervised learning protocols, characterized by the ability
of the learning algorithm to pose queries to a teacher, who has access to the target concept to be learned.
In practice, the teacher and queries may take a variety of forms:
a human expert, in which case the queries may be questions or annotation tasks;
nature, in which case the queries may be scientific experiments;
a computer simulation, in which case the queries may be particular parameter values or initial conditions for the simulator;
or a host of other possibilities.
In our present context, we will specifically discuss a protocol known as \emph{pool-based} active learning,
a type of sequential design based on a collection of unlabeled examples; this seems to be
the most common form of active learning in practical use today \citep*[e.g.,][]{settles:10,baldridge:09,gangadharaiah:09,hoi:06,luo:05,roy:01,tong:01,mccallum:98}. 
We will not discuss alternative models of active learning, such as \emph{online} \citep*{dekel:10} or \emph{exact} \citep*{hegedus:95}.
In the pool-based active learning setting, the learning algorithm is supplied with a large collection of unlabeled examples (the \emph{pool}),
and is allowed to select any example from the pool to request that it be labeled.  After observing the
label of this example, the algorithm can then select another unlabeled example from the pool to request that it be labeled.
This continues sequentially for a number of rounds until some halting condition is satisfied, at which time the
algorithm returns a function intended to approximately mimic and generalize the observed labeling behavior.
This setting contrasts with \emph{passive learning}, in which the learning algorithm
is supplied with a collection of \emph{labeled} examples.

Supposing the labels received agree with some true target concept, the objective
is to use this returned function to approximate the true target concept on future (previously unobserved)
data points.
The hope is that, by carefully selecting which examples should be labeled, the algorithm can
achieve improved accuracy while using fewer labels compared to passive learning.
The motivation for this setting is simple.  For many modern machine learning problems, unlabeled examples are
inexpensive and available in abundance, while annotation is time-consuming or expensive.  For instance,
this is the case in the aforementioned webpage classification problem,
where the pool would be the set of all webpages, and labeling
a webpage requires a human expert to examine the website content. \citet*{settles:10} surveys a variety of
other applications for which active learning is presently being used.
To simplify the discussion, in this work we focus specifically on \emph{binary classification}, in which there
are only two possible labels.  The results generalize naturally to multiclass classification as well.

As the above description indicates, when studying the advantages of active learning, we are primarily interested in the number of label requests
sufficient to achieve a given accuracy, a quantity referred to as the \emph{label complexity} (Definition~\ref{defn:label-complexity} below).
Although active learning has been an active topic in the machine learning literature for many years now,
our \emph{theoretical} understanding of this topic was largely lacking
until very recently.  However, within the past few years, there has been an explosion of progress.
These advances can be grouped into two categories: namely, the \emph{realizable case} and the \emph{agnostic case}.

\subsubsection{The Realizable Case}

In the realizable case, we are interested in a particularly strict scenario, where the true label of any example is
\emph{determined} by a function of the features (covariates), and where that function has a specific known form
(e.g., linear separator, decision tree, union of intervals, etc.); the set of classifiers having this known form is referred
to as the \emph{concept space}.  The natural formalization of the realizable case is
very much analogous to the well-known PAC model for passive learning \citep*{valiant:84}.  In the realizable
case, there are obvious examples of learning problems where active learning can provide a significant advantage
compared to passive learning; for instance, in the problem of learning \emph{threshold} classifiers on the real line (Example~\ref{ex:thresholds} below),
a kind of \emph{binary search} strategy for selecting which examples to request labels for naturally leads to \emph{exponential}
improvements in label complexity compared to learning from random labeled examples (passive learning).
As such, there is a natural attraction to determine how general this phenomenon is.  This leads us to think about
general-purpose learning strategies (i.e., which can be instantiated for more than merely threshold classifiers on the real line),
which exhibit this binary search behavior in various special cases.

The first such general-purpose strategy to emerge
in the literature was a particularly elegant strategy proposed by \citet*{cohn:94}, typically referred to as CAL after its
discoverers (\CAL~below).  The strategy behind CAL is the following.  The algorithm examines each example in the unlabeled pool in sequence,
and if there are two classifiers in the concept space consistent with all previously-observed labels, but which disagree
on the label of this next example, then the algorithm requests that label, and otherwise it does not.
For this reason, below we refer to the general family of algorithms inspired by CAL as
\emph{disagreement-based} methods.  Disagreement-based methods are sometimes referred to as ``mellow'' active learning,
since in some sense this is the \emph{least} we can expect from a reasonable active
learning algorithm; it never requests the label of an example whose label it can \emph{infer} from information
already available, but otherwise makes no attempt to seek out particularly informative examples to request the labels of.
That is, the notion of \emph{informativeness} implicit in disagreement-based methods is a \emph{binary} one, so that an example
is either informative or not informative, but there is no further ranking of the informativeness of examples.  The
disagreement-based strategy is quite general, and obviously leads to algorithms that are at least \emph{reasonable},
but \citet*{cohn:94} did not study the label complexity achieved by their strategy in any generality.

In a Bayesian variant of the realizable setting, \citet*{freund:97} studied an algorithm known as
Query by Committee (QBC), which in some sense represents a Bayesian variant of CAL.  However, QBC \emph{does}
distinguish between different levels of informativeness beyond simple disagreement, based on the \emph{amount}
of disagreement on a random unlabeled example.  They were able to analyze the label complexity achieved by
QBC in terms of a type of information gain, and found that when the information gain is lower bounded by a positive constant,
the algorithm achieves a label complexity exponentially smaller than the known results for passive learning.
In particular, this is the case for the threshold learning problem, and also for the problem of learning higher-dimensional
(nearly balanced) linear separators when the data satisfy a certain (uniform) distribution.  Below,
we will not discuss this analysis further, since it is for a slightly different (Bayesian) setting.  However, the
results below in our present setting do have interesting implications for the Bayesian setting as well,
as discussed in the recent work of \citet*{hanneke:11b}.

The first general analysis of the label complexity of active learning in the (non-Bayesian) realizable case came in
the breakthrough work of \citet*{dasgupta:05}.  In that work, Dasgupta proposed a quantity, called the \emph{splitting index},
to characterize the label complexities achievable by active learning.  The splitting index analysis is noteworthy for several
reasons.
First, one can show it provides nearly tight bounds on the \emph{minimax} label complexity for a given concept space and data distribution.
In particular, the analysis matches the exponential improvements known to be possible for threshold classifiers,
as well as generalizations to higher-dimensional homogeneous linear separators under near-uniform distributions
(as first established by \citet*{dasgupta:05b,dasgupta:09}).
Second, it provides a novel notion of \emph{informativeness} of an example, beyond the simple binary notion of informativeness
employed in disagreement-based methods.  Specifically, it describes the informativeness of an example in terms of the number of
\emph{pairs} of well-separated classifiers for which at least one out of each pair will definitely be contradicted, regardless
of the example's label.
Finally, unlike any other existing work on active learning (present work included), it provides an elegant description of the
\emph{trade-off} between the number of label requests and the number of unlabeled examples needed by the learning algorithm.
Another interesting byproduct of Dasgupta's work is a better understanding of the \emph{nature} of the improvements
achievable by active learning in the general case.  In particular, his work clearly illustrates the need to study the label
complexity as a quantity that varies depending on the particular target concept and data distribution.  We will see this
issue arise in many of the examples below.

Coming from a slightly different perspective, \citet*{hanneke:07a} later analyzed the label complexity of active learning
in terms of an extension of the \emph{teaching dimension} \citep*{goldman:95}.  Related quantities were previously used by
\citet*{hegedus:95} and \citet*{hellerstein:96} to tightly characterize the number of membership queries sufficient
for \emph{Exact} learning; \citet*{hanneke:07a} provided a natural generalization to the \emph{PAC} learning setting.
At this time, it is not clear how this quantity relates to the splitting index.  From a practical perspective, in some
instances it may be easier to calculate (see the work of \citet*{nowak:08} for a discussion related to this), though
in other cases the opposite seems true.

The next progress toward understanding the label complexity of active learning came in the work of \citet*{hanneke:07b},
who introduced a quantity called the \emph{disagreement coefficient} (Definition~\ref{def:disagreement-coefficient} below),
accompanied by a technique for analyzing disagreement-based
active learning algorithms.  In particular, implicit in that work, and made explicit in the later work of \citet*{hanneke:11a},
was the first general characterization of the label complexities achieved by the original CAL strategy for active learning
in the realizable case, stated in terms of the disagreement coefficient.  The results of the present work are direct descendents
of that 2007 paper, and we will discuss the disagreement coefficient, and results based on it, in substantial detail below.
Disagreement-based active learners such as CAL are known to be sometimes suboptimal relative to the splitting index analysis,
and therefore the disagreement coefficient analysis sometimes results in larger label complexity bounds than the splitting index
analysis.  However, in many cases the label complexity bounds based on the disagreement coefficient are surprisingly
good considering the simplicity of the methods.  Furthermore, as we will see below, the disagreement coefficient has the
practical benefit of often being fairly straightforward to calculate for a variety of learning problems, particularly when
there is a natural geometric interpretation of the classifiers and the data distribution is relatively smooth.
As we discuss below, it can also be used to bound the label complexity of active learning
in noisy settings.  For these reasons (simplicity of algorithms, ease of calculation, and applicability beyond the realizable case),
subsequent work on the label complexity of active learning has tended to favor the disagreement-based
approach, making use of the disagreement coefficient to bound the label complexity 
\citep*{dasgupta:07,friedman:09,beygelzimer:09,wang:09,hanneke:10a,hanneke:11a,koltchinskii:10,beygelzimer:10,mahalanabis:11,wang:11}. 
A significant part of the present paper focuses on extending and generalizing the disagreement coefficient analysis,
while still maintaining the relative ease of calculation that makes the disagreement coefficient so useful.

In addition to many positive results, \citet*{dasgupta:05} also pointed out several negative results,
even for very simple and natural learning problems.  In particular, for many problems, the minimax label complexity of active learning
will be no better than that of passive learning.  In fact, \citet*{hanneke:10a} later showed that, for a certain type of active learning algorithm
-- namely, \emph{self-verifying} algorithms, which themselves adaptively determine how many label requests they need to achieve a given accuracy --
there are even particular target concepts and data distributions for which \emph{no} active learning algorithm of that type can outperform passive learning.
Since all of the above label complexity analyses (splitting index, teaching dimension, disagreement coefficient) apply to certain respective self-verifying learning
algorithms, these negative results are also reflected in all of the existing general label complexity analyses as well.

While at first these negative results may seem discouraging, \citet*{hanneke:10a} noted that if we do not require the algorithm to be self-verifying,
instead simply measuring the number of label requests the algorithm needs to \emph{find} a good classifier, rather than the number needed
to both find a good classifier \emph{and verify} that it is indeed good, then these negative results vanish.  In fact, (shockingly) they were able to show
that for any concept space with finite VC dimension, and any fixed data distribution, for any given passive learning algorithm there is an
active learning algorithm with asymptotically superior label complexity for \emph{every} nontrivial target concept!  A positive result of this generality and strength
is certainly an exciting advance in our understanding of the advantages of active learning.  But perhaps equally exciting are the unresolved questions raised by that work,
as there are potential opportunities to strengthen, generalize, simplify, and elaborate on this result.
First, note that the above statement allows the active learning algorithm to be specialized to the particular distribution according to which the (unlabeled) data are sampled,
and indeed the active learning method used by \citet*{hanneke:10a} in their proof has a rather strong direct dependence on the data distribution
(which cannot be removed by simply replacing some calculations with data-dependent estimators).  One interesting question is whether an alternative approach might avoid
this direct distribution-dependence in the algorithm, so that the claim can be strengthened to say that the active algorithm is superior
to the passive algorithm for all nontrivial target concepts \emph{and data distributions}.  This question is interesting both theoretically,
in order to obtain the strongest possible theorem on the advantages of active learning, as well as practically, since direct access to the
distribution from which the data are sampled is typically \emph{not} available in practical learning scenarios.
A second question left open by \citet*{hanneke:10a} regards the \emph{magnitude} of the gap between the active and passive label complexities.
Specifically, although they did find particularly nasty learning problems where the label complexity of active learning will be close to that of passive learning (though always better),
they hypothesized that for most natural learning problems, the improvements over passive learning should typically be \emph{exponentially large}
(as is the case for threshold classifiers); they gave many examples to illustrate this point, but left open the problem of characterizing general
sufficient conditions for these exponential improvements to be achievable, even when they are not achievable by self-verifying algorithms.
Another question left unresolved by \citet*{hanneke:10a} is whether this type of general improvement guarantee might be realized by a computationally \emph{efficient} active learning algorithm.
Finally, they left open the question of whether such general results might be further generalized to settings that involve noisy labels.
The present work picks up where \citet*{hanneke:10a} left off in several respects, making progress on each of the above questions, in some
cases completely resolving the question.

\subsubsection{The Agnostic Case}

In addition to the above advances in our understanding of active learning in the realizable case,
there has also been wonderful progress in making these methods robust to imperfect teachers,
feature space underspecification, and model misspecification.  This general topic goes by the name
\emph{agnostic active learning}, from its roots in the agnostic PAC model \citep*{kearns:94b}.
In contrast to the realizable case, in the \emph{agnostic case}, there is not necessarily a perfect
classifier of a known form, and indeed there may even be \emph{label noise} so that there is no perfect
classifier of \emph{any} form.  Rather, we have a given set of classifiers (e.g., linear separators, or
depth-limited decision trees, etc.), and the objective is to identify a classifier whose accuracy is
not much worse than the best classifier of that type.  Agnostic learning is strictly more general,
and often more difficult, than realizable learning; this is true for both passive learning and active learning.
However, for a given agnostic learning problem, we might still hope that active learning can achieve a given
accuracy using fewer labels than required for passive learning.

The general topic of agnostic active learning got its first taste of real progress from \citet*{balcan:06,balcan:09} with
the publication of the $A^2$ (agnostic active) algorithm.  This method is a noise-robust disagreement-based
algorithm, which can be applied with essentially arbitrary types of classifiers under arbitrary noise distributions.
It is interesting both for its effectiveness and (as with CAL) its elegance.
The original work of \citet*{balcan:06,balcan:09} showed that, in some special cases (thresholds, and homogeneous linear
separators under a uniform distribution), the $A^2$ algorithm does achieve
improved label complexities compared to the known results for passive learning.

Using a different type of general active learning strategy, \citet*{hanneke:07a} found that the \emph{teaching dimension}
analysis (discussed above for the realizable case) can be extended beyond the realizable case, arriving at general bounds
on the label complexity under arbitrary noise distributions.  These bounds improve over the known results for passive
learning in many cases.  However, the algorithm requires direct access to a certain quantity that depends on the noise distribution
(namely, the noise rate, defined in Section~\ref{sec:agnostic} below), which would not be available in many real-world learning problems.

Later, \citet*{hanneke:07b} established a general characterization of the label complexities achieved by
$A^2$, expressed in terms of the disagreement coefficient.  The result holds for arbitrary types of classifiers (of finite VC dimension)
and arbitrary noise distributions, and represents the natural generalization of the aforementioned realizable-case
analysis of CAL.  In many cases, this result shows improvements over the known results for
passive learning.  Furthermore, because of the simplicity of the disagreement coefficient, the bound can be calculated
for a variety of natural learning problems.

Soon after this, \citet*{dasgupta:07} proposed a new active learning strategy, which is also effective in the agnostic setting.
Like $A^2$, the new algorithm is a noise-robust disagreement-based method.  The work of \citet*{dasgupta:07} is significant for at least two reasons.
First, they were able to establish a general label complexity bound for this method based on the disagreement coefficient.
The bound is similar in form to the previous label complexity bound for $A^2$ by \citet*{hanneke:07b},
but improves the dependence of the bound on the disagreement coefficient.
Second, the proposed method of \citet*{dasgupta:07} set a new standard for computational and aesthetic simplicity in agnostic active learning algorithms.
This work has since been followed by related methods of \citet*{beygelzimer:09} and \citet*{beygelzimer:10}.
In particular, \citet*{beygelzimer:09} develop a method capable of learning under an essentially arbitrary loss function; they also show label complexity
bounds similar to those of \citet*{dasgupta:07}, but applicable to a larger class of loss functions, and stated in terms of a generalization
of the disagreement coefficient for arbitrary loss functions.

While the above results are encouraging, the guarantees reflected in these label complexity
bounds essentially take the form of (at best) constant factor improvements;
specifically, in some cases the bounds improve the dependence on the noise rate factor (defined in Section~\ref{sec:agnostic} below),
compared to the known results for passive learning.  In fact, \citet*{kaariainen:06} showed that any label complexity bound depending on the noise
distribution only via the noise rate cannot do better than this type of constant-factor improvement.  This raised the
question of whether, with a more detailed description of the noise distribution, one can show improvements in the
\emph{asymptotic form} of the label complexity compared to passive learning.  Toward this end, \citet*{castro:08}
studied a certain refined description of the noise conditions, related to the margin conditions of \citet*{mammen:99}, which are well-studied in
the passive learning literature.  Specifically, they found that in some special cases, under certain restrictions on the
noise distribution, the asymptotic form of the label complexity \emph{can} be improved compared to passive learning,
and in some cases the improvements can even be \emph{exponential} in magnitude;
to achieve this, they developed algorithms specifically tailored to the types of classifiers they studied (threshold classifiers and boundary fragment classes).
\citet*{balcan:07} later extended this result to general homogeneous linear separators under a uniform distribution.
Following this, \citet*{hanneke:09a,hanneke:11a} generalized these results, showing that both of the published general agnostic active learning
algorithms \citep*{balcan:09,dasgupta:07} can also achieve these types of improvements in the asymptotic form of the label complexity;
he further proved general bounds on the label complexities of these methods, again based on the disagreement coefficient, which apply to
arbitrary types of classifiers, and which reflect these types of improvements (under conditions on the disagreement coefficient).
\citet*{wang:09} later bounded the label complexity of $A^2$ under
somewhat different noise conditions, in particular identifying weaker noise conditions
sufficient for these improvements
to be exponential in magnitude (again, under conditions on the disagreement coefficient).
\citet*{koltchinskii:10} has recently improved on some of Hanneke's results, refining certain
logarithmic factors and simplifying the proofs, using a slightly different algorithm
based on similar principles.
Though the present work discusses only classes of finite VC dimension, most of the above references also contain results
for various types of nonparametric classes with infinite VC dimension.

At present, all of the published bounds on the label complexity of agnostic active learning also apply to
\emph{self-verifying} algorithms.  As mentioned, in the realizable case, it is typically possible to achieve
significantly better label complexities if we do not require the active learning algorithm to be self-verifying,
since the verification of learning may be more difficult than the learning itself \citep*{hanneke:10a}.  We might wonder whether
this is also true in the agnostic case, and whether agnostic active learning algorithms that are not self-verifying
might possibly achieve significantly better label complexities than the existing label complexity bounds described above.
We investigate this in depth below.

\subsection{Summary of Contributions}
\label{subsec:contributions}

In the present work, we build on and extend the above results in a variety of ways, resolving a number of open problems.
The main contributions of this work can be summarized as follows.

\begin{itemize}
\item We formally define a notion of a universal activizer, a meta-algorithm that transforms any passive learning algorithm into an active learning algorithm with asymptotically strictly superior label complexities for all nontrivial target concepts and distributions.
\item We analyze the existing strategy of disagreement-based active learning from this perspective, precisely characterizing the conditions under which this strategy can lead to a universal activizer in the realizable case.
\item We propose a new type of active learning algorithm, based on shatterable sets, and prove that we can construct universal activizers for the realizable case based on this idea; in particular, this overcomes the issue of distribution-dependence in the existing results mentioned above.
\item We present a novel generalization of the disagreement coefficient, along with a new asymptotic bound on the label complexities achievable by active learning in the realizable case; this new bound is often significantly smaller than the existing results in the published literature.
\item We state new concise sufficient conditions for exponential improvements over passive learning to be achievable in the realizable case, including a significant weakening of known conditions in the published literature.
\item We present a new general-purpose active learning algorithm for the agnostic case, based on the aforementioned idea involving shatterable sets.
\item We prove a new asymptotic bound on the label complexities achievable by active learning in the presence of label noise (the agnostic case), often significantly smaller than any previously published results.
\item We formulate a general conjecture on the theoretical advantages of active learning over passive learning in the presence of arbitrary types of label noise.
\end{itemize}

\subsection{Outline of the Paper}
\label{subsec:outline}

The paper is organized as follows.
In Section~\ref{sec:definitions}, we introduce the basic notation used throughout,
formally define the learning protocol, and formally define the label complexity.
We also define the notion of an \emph{activizer}, which is a procedure that transforms
a passive learning algorithm into an active learning algorithm with asymptotically
superior label complexity.  In Section~\ref{sec:naive}, we review the established
technique of \emph{disagreement-based} active learning,
and prove a new result precisely characterizing the scenarios in which disagreement-based
active learning can be used to construct an activizer.  In particular, we find that in
many scenarios, disagreement-based active learning is not powerful enough to provide
the desired improvements.  In Section~\ref{sec:activizer}, we move beyond disagreement-based
active learning, developing a new type of active learning algorithm based on \emph{shatterable}
sets of points.  We apply this technique to construct a simple 3-stage procedure, which we then
prove is a universal activizer for any concept space of finite VC dimension.
In Section~\ref{sec:exponential}, we begin by reviewing the known results for bounding the
label complexity of disagreement-based active learning in terms of the disagreement
coefficient; we then develop a somewhat more involved procedure, again based on shatterable sets,
which takes full advantage of the sequential nature of active leanring.  In addition to being an
activizer, we show that this procedure often achieves dramatically superior label complexities
than achievable by passive learning.  In particular, we define a novel generalization of the
disagreement coefficient, and use it to bound the label complexity of this procedure.
This also provides us with concise sufficient conditions for obtaining
exponential improvements over passive learning.  Continuing in Section~\ref{sec:agnostic},
we extend our framework to allow for label noise (the agnostic case), and discuss the
possibility of extending the results from previous sections to these noisy learning problems.
We first review the known results for noise-robust disagreement-based active learning,
and characterizations of its label complexity in terms of the disagreement coefficient and
Mammen-Tsybakov noise parameters.  We then proceed to develop a new type of noise-robust
active learning algorithm, again based on shatterable sets, and prove bounds on its label
complexity in terms of our aforementioned generalization of the disagreement coefficient.
Additionally, we present a general conjecture concerning the existence of activizers for
certain passive learning algorithms in the agnostic case.  We conclude in
Section~\ref{sec:open-problems} with a host of enticing open problems for future investigation.

\section{Definitions and Notation}
\label{sec:definitions}

For most of the paper, we consider the following formal setting.
There is a measurable space $(\X, \mathcal{F}_{\X})$,
where $\X$ is called the \emph{instance space}; for simplicity, we
suppose this is a standard Borel space \citep*{srivastava:98}
(e.g., $\reals^{m}$ under the usual Borel $\sigma$-algebra), though
most of the results generalize.
A \emph{classifier} is any measurable function $h : \X \to \{-1,+1\}$.
There is a set $\C$ of classifiers called the \emph{concept space}.
In the \emph{realizable case}, the learning problem is characterized as follows.
There is a probability measure $\Px$ on $\X$, and a sequence $\Data_{X} = \{X_1,X_2,\ldots\}$
of independent $\X$-valued random variables, each with distribution $\Px$.
We refer to these random variables as the sequence of \emph{unlabeled examples};
although in practice, this sequence would typically be large but finite, to simplify
the discussion and focus strictly on counting labels, we will suppose this sequence is inexhaustible.  There is additionally
a special element $f \in \C$, called the \emph{target function}, and we denote by
$Y_i = f(X_i)$; we further denote by $\Data = \{(X_1,Y_1),(X_2,Y_2),\ldots\}$
the sequence of \emph{labeled examples}, and for $m \in \nats$ we denote by
$\Data_{m} = \{(X_1,Y_1),(X_2,Y_2),\ldots,(X_m,Y_m)\}$
the finite subsequence consisting
of the first $m$ elements of $\Data$.  For any classifier $h$,
we define the \emph{error rate} $\er(h) = \Px(x : h(x) \neq f(x))$. Informally, the learning
objective in the realizable case is to identify some $h$ with small $\er(h)$ using elements
from $\Data$, without direct access to $f$. 

An \emph{active learning algorithm}
$\alg$ is permitted direct access to the $\Data_{X}$
sequence (the unlabeled examples), but to gain access to the $Y_i$ values it must request
them one at a time, in a sequential manner.
Specifically, given access to the $\Data_{X}$ values, the algorithm selects any index $i\in\nats$,
requests to observe the $Y_i$ value, then having observed the value of $Y_i$, selects another index
$i^{\prime}$, observes the value of $Y_{i^{\prime}}$, etc.  The algorithm is given as input an
integer $n$, called the \emph{label budget}, and is permitted to observe at most $n$ labels total
before eventually halting and returning a classifier $\hat{h}_n = \alg(n)$; that is, by definition, an active
learning algorithm never attempts to access more than the given budget $n$ number of labels.
We will then study the values of $n$ sufficient to guarantee $\E[\er(\hat{h}_n)] \leq \eps$, for any given value
$\eps \in (0,1)$.  We refer to this as the \emph{label complexity}.  We will be particularly
interested in the asymptotic dependence on $\eps$ in the label complexity, as $\eps \to 0$.
Formally, we have the following definition.

\begin{definition}
\label{defn:label-complexity}
An active learning algorithm $\alg$ achieves label complexity $\Lambda(\cdot,\cdot,\cdot)$
if, for every target function $f$, distribution $\Px$, $\eps \in (0,1)$,
and integer $n \geq \Lambda(\eps, f, \Px)$, we have $\E\left[ \er\left( \alg(n) \right)\right] \leq \eps$.
\thmend
\end{definition}

This definition of label complexity is similar to one originally studied by \citet*{hanneke:10a}.
It has a few features worth noting.  First, the label complexity has an explicit dependence on
the target function $f$ and distribution $\Px$.  As noted by \citet*{dasgupta:05}, we need this
dependence if we are to fully understand the range of label complexities achievable by active
learning; we further illustrate this issue in the examples below.  The second feature to note
is that the label complexity, as defined here, is simply a sufficient budget size to achieve the specified accuracy.
That is, here we are asking only how many label
requests are required for the algorithm to achieve a given accuracy (in expectation).  However, as noted by \citet*{hanneke:10a},
this number might not be sufficiently large to \emph{detect} that the algorithm has indeed achieved the required accuracy
based only on the observed data.
That is, because the number of labeled examples used in active learning can be quite small, we come
across the problem that the number of labels needed to \emph{learn} a concept might be significantly
smaller than the number of labels needed to \emph{verify} that we have successfully learned the concept.
As such, this notion of label complexity is most useful in the \emph{design} of effective learning algorithms,
rather than for predicting the number of labels an algorithm should request in any particular application.
Specifically, to design effective active learning algorithms, we should generally desire small label complexity values,
so that (in the extreme case) if some algorithm $\alg$ has smaller label complexity values than some
other algorithm $\alg^{\prime}$ for \emph{all} target functions and distributions, then (all other factors being equal) we should clearly
prefer algorithm $\alg$ over algorithm $\alg^{\prime}$; this is true regardless of whether we have a means to \emph{detect}
(verify) how large the improvements offered by algorithm $\alg$ over algorithm $\alg^{\prime}$ are for any particular
application.
Thus, in our present context, this notion of label complexity plays a role analogous to concepts
such as \emph{universal consistency} or \emph{admissibility},
which are also generally useful in guiding the design of effective algorithms, but are not intended to be
informative in the context of any particular application.  See the work of \citet*{hanneke:10a} for a discussion
of this issue, as it relates to a definition of label complexity similar to that above, as well as other notions of label complexity
from the active learning literature (some of which include a verification requirement).

We will be interested in the performance of active learning algorithms, relative to the performance
of a given \emph{passive learning algorithm}.  In this context, a passive learning algorithm $\alg$
takes as input a finite sequence of labeled examples $\L \in \bigcup_{n} (\X \times \{-1,+1\})^{n}$,
and returns a classifier $\hat{h} = \alg(\L)$.
We allow both active and passive learning algorithms to be randomized:
that is, to have internal randomness, in addition to the given random data.
We define the label complexity for a passive learning algorithm as follows.

\begin{definition}
\label{defn:passive-label-complexity}
A passive learning algorithm $\alg$ achieves label complexity $\Lambda(\cdot,\cdot,\cdot)$
if, for every target function $f$, distribution $\Px$, $\eps \in (0,1)$,
and integer $n \geq \Lambda(\eps, f, \Px)$, we have $\E\left[ \er\left( \alg\left( \Data_{n} \right) \right) \right] \leq \eps$.
\thmendup
\end{definition}

Although technically some algorithms may be able to achieve a desired accuracy
without any observations, to make the general results easier to state
(namely, those in Section~\ref{sec:exponential}),
unless otherwise stated we suppose label complexities (both passive and active) take strictly
positive values, among $\nats \cup \{\infty\}$;
note that label complexities (both passive and active) can be infinite,
indicating that the corresponding algorithm might not achieve expected error rate $\eps$ for \emph{any} $n \in \nats$.
Both the passive and active label complexities are defined as a number of labels sufficient to guarantee the
\emph{expected} error rate is at most $\eps$.  It is also common in the literature to discuss the number of
label requests sufficient to guarantee the error rate is at most $\eps$ with \emph{high probability} $1-\delta$
\citep*[e.g.,][]{hanneke:10a}.
In the present work, we formulate our results in terms of the expected error rate because it simplifies the
discussion of asymptotics, in that we need only study the behavior of the label complexity as the single argument $\eps$ approaches $0$,
rather than the more complicated behavior of a function of $\eps$ and $\delta$ as both $\eps$ and $\delta$
approach $0$ at various relative rates.  However, we note that analogous results for these high-probability guarantees on the error rate
can be extracted from the proofs below without much difficulty, 
and in several places we explicitly state results of this form.

Below we employ the standard notation from asymptotic analysis,
including $O(\cdot)$, $o(\cdot)$, $\Omega(\cdot)$, $\omega(\cdot)$, $\Theta(\cdot)$, $\ll$, and $\gg$.
In all contexts below not otherwise specified, the asymptotics are always considered as
$\eps \to 0$ when considering a function of $\eps$,
and as $n \to \infty$ when considering a function of $n$; also, in any expression of the form ``$x \to 0$,''
we always mean the limit \emph{from above} (i.e., $x \downarrow 0$).
For instance, when considering nonnegative functions of $\eps$, $\lambda_{a}(\eps)$ and $\lambda_{p}(\eps)$,
the above notations are defined as follows.
We say $\lambda_{a}(\eps) = o(\lambda_{p}(\eps))$
when $\lim\limits_{\eps \to 0} \frac{\lambda_{a}(\eps)}{\lambda_{p}(\eps)} = 0$,
and this is equivalent to writing $\lambda_{p}(\eps) = \omega(\lambda_{a}(\eps))$,
$\lambda_{a}(\eps) \ll \lambda_{p}(\eps)$, or $\lambda_{p}(\eps) \gg \lambda_{a}(\eps)$.
We say $\lambda_{a}(\eps) = O(\lambda_{p}(\eps))$
when $\limsup\limits_{\eps \to 0} \frac{\lambda_{a}(\eps)}{\lambda_{p}(\eps)} < \infty$,
which can be equivalently expressed as $\lambda_{p}(\eps) = \Omega(\lambda_{a}(\eps))$.
Finally, we write $\lambda_{a}(\eps) = \Theta(\lambda_{p}(\eps))$ to mean that both
$\lambda_{a}(\eps) = O(\lambda_{p}(\eps))$ and $\lambda_{a}(\eps) = \Omega(\lambda_{p}(\eps))$ are satisfied.

Define the class of functions $\Polylog(1/\eps)$
as those $g : (0,1) \to [0,\infty)$
such that, for some $k \in [0,\infty)$, $g(\eps) = O(\log^{k}(1/\eps))$.
For a label complexity $\Lambda$, also define the set $\Nontrivial(\Lambda)$
as the collection of all pairs $(f,\Px)$ of a classifier and a distribution such that,
$\forall \eps > 0, \Lambda(\eps,f,\Px) < \infty$,
and $\forall g \in \Polylog(1/\eps)$, $\Lambda(\eps,f,\Px) = \omega(g(\eps))$.

In this context, an \emph{active meta-algorithm}
is a procedure $\alg_a$ taking
as input a passive algorithm $\alg_p$ and a label budget $n$,
such that for any passive algorithm $\alg_p$,
$\alg_a(\alg_p, \cdot)$ is an active learning algorithm.
We define an \emph{activizer} for a given passive algorithm as follows.
\begin{definition}
\label{defn:activizer}
We say an active meta-algorithm $\alg_a$ \emph{activizes} a passive algorithm $\alg_p$
for a concept space $\C$ if the following holds.  For any label complexity $\Lambda_p$ achieved by $\alg_p$,
the active learning algorithm $\alg_a(\alg_p, \cdot)$ achieves a label complexity $\Lambda_a$
such that, for every $f \in \C$ and every distribution $\Px$ on $\X$ with
$(f,\Px) \in \Nontrivial(\Lambda_{p})$,
there exists a constant $c \in [1,\infty)$ such that
\begin{equation*}
\Lambda_a( c\eps, f,\Px) = o\left( \Lambda_p(\eps, f,\Px) \right).
\end{equation*}
In this case, $\alg_a$ is called an \emph{activizer} for $\alg_p$ with respect to $\C$,
and the active learning algorithm $\alg_{a}(\alg_p, \cdot)$ is called the
$\alg_a$\emph{-activized} $\alg_p$.
\thmend
\end{definition}

We also refer to any active meta-algorithm $\alg_a$ that activizes \emph{every} passive algorithm $\alg_p$
for $\C$ as a \emph{universal activizer}
for $\C$.  One of the main contributions of this work is establishing
that such universal activizers do exist for any VC class $\C$.

A bit of explanation is in order regarding Definition~\ref{defn:activizer}.
We might interpret it as follows: an \emph{activizer} for $\alg_p$ strongly
improves (in a little-o sense) the label complexity for all \emph{nontrivial} target functions and distributions.
Here, we seek a meta-algorithm that, when given $\alg_p$ as input,
results in an active learning algorithm with strictly superior label complexities.
However, there is a sense in which some distributions $\Px$ or target functions $f$
are \emph{trivial} relative to $\alg_p$.  For instance, perhaps $\alg_p$ has a \emph{default}
classifier that it is naturally biased toward (e.g., with minimal $\Px(x : h(x)=+1)$, as in the
Closure algorithm~\citep*{auer:04}), so that when this default classifier is the target
function, $\alg_p$ achieves a constant label complexity.
In these trivial scenarios, we cannot hope to \emph{improve} over the behavior of the passive algorithm, but instead
can only hope to \emph{compete} with it.  The \emph{sense} in which we wish to compete
may be a subject of some controversy, but the implication of Definition~\ref{defn:activizer}
is that the label complexity of the activized algorithm should be strictly better than every nontrivial
upper bound on the label complexity of the passive algorithm.  For instance, if
$\Lambda_{p}(\eps,f,\Px) \in \Polylog(1/\eps)$, then we are guaranteed $\Lambda_{a}(\eps,f,\Px) \in \Polylog(1/\eps)$ as well, 
but if $\Lambda_{p}(\eps,f,\Px) = O(1)$, we are still only guaranteed $\Lambda_{a}(\eps,f,\Px) \in \Polylog(1/\eps)$.
This serves the purpose of defining a framework that can be studied without requiring too much obsession over
small additive terms in trivial scenarios, thus focusing the analyst's efforts toward
nontrivial scenarios where $\alg_p$ has relatively \emph{large} label complexity,
which are precisely the scenarios for which active learning is truly needed.
In our proofs, we find that in fact $\Polylog(1/\eps)$ can be replaced with $\log(1/\eps)$,
giving a slightly broader definition of ``nontrivial,'' for which all of the results below still hold.
Section~\ref{sec:open-problems} discusses open problems regarding this issue of trivial problems.

The definition of $\Nontrivial(\cdot)$ also only requires the activized algorithm to be effective in scenarios where the passive
learning algorithm has \emph{reasonable} behavior (i.e., finite label complexities); this is
only intended to keep with the reduction-based style of the framework, and in fact this restriction can
easily be lifted using a trick from \citet*{hanneke:10a} (aggregating the activized algorithm
with another algorithm that is always reasonable).

Finally, we also allow a constant factor $c$ loss in the $\eps$ argument to $\Lambda_{a}$.
We allow this to be an arbitrary constant, again in the interest of allowing the analyst
to focus only on the most significant aspects of the problem; for most reasonable
passive learning algorithms, we typically expect $\Lambda_{p}(\eps,f,\Px) = {\rm Poly}(1/\eps)$,
in which case $c$ can be set to $1$ by adjusting the leading constant factors of $\Lambda_{a}$.
A careful inspection of our proofs reveals that $c$ can always be set arbitrarily close to $1$ without affecting the
theorems below (and in fact, we can even get $c = (1+o(1))$, a function of $\eps$).

Throughout this work, we will adopt the usual notation for probabilities, such as $\P(\er(\hat{h}) > \eps)$,
and as usual we interpret this as measuring the corresponding event in the (implicit) underlying
probability space.  In particular, we make the usual implicit assumption that all sets involved in the
analysis are measurable; where this assumption does not hold, we may turn to outer probabilities,
though we will not make further mention of these technical details.
We will also use the notation $P^{k}(\cdot)$ to represent $k$-dimensional product measures;
for instance, for a measurable set $A \subseteq \X^{k}$, $\Px^{k}(A) = \P((X_1^{\prime},\ldots,X_k^{\prime}) \in A)$,
for independent $\Px$-distributed random variables $X_1^{\prime},\ldots,X_k^{\prime}$.
Additionally, to simplify notation,
we will adopt the convention that $\X^{0} = \{\varnothing\}$,
and $\Px^{0}(\X^{0}) = 1$.
Throughout, we will denote by $\ind_{A}(z)$ the indicator function for a set $A$,
which has the value $1$ when $z \in A$ and $0$ otherwise; additionally, at times
it will be more convenient to use the bipolar indicator function, defined as
$\ind_{A}^{\pm}(z) = 2 \ind_{A}(z) - 1$.

We will require a few additional definitions for the discussion below.
For any classifier $h : \X \to \{-1,+1\}$ and finite sequence of labeled examples $\L \in \bigcup_{m} (\X \times \{-1,+1\})^{m}$,
define the \emph{empirical error rate} $\er_{\L}(h) = |\L|^{-1} \sum_{(x,y) \in \L} \ind_{\{-y\}}(h(x))$;
for completeness, define $\er_{\emptyset}(h) = 0$.
Also, for $\L = \Data_{m}$,
the first $m$ labeled examples in the data sequence, abbreviate this as $\er_{m}(h) = \er_{\Data_{m}}(h)$.
For any distribution $P$ on $\X$, set of classifiers $\H$, classifier $h$,
and $r > 0$, define $\Ball_{\H,P}(h, r) = \{ g \in \H : P( x : h(x) \neq g(x)) \leq r\}$;
when $P = \Px$, the
distribution of the unlabeled examples, and $\Px$ is clear from the context, we abbreviate this as
$\Ball_{\H}(h,r) = \Ball_{\H,\Px}(h,r)$; furthermore, when $P = \Px$ and $\H = \C$, the concept space,
and both $\Px$ and $\C$ are clear from the context, we abbreviate this as $\Ball(h,r) = \Ball_{\C,\Px}(h,r)$.  Also, for
any set of classifiers $\H$, and any sequence of labeled examples $\L \in \bigcup_{m} (\X \times \{-1,+1\})^m$,
define $\H[\L] = \{h \in \H : \er_{\L}(h) = 0\}$;
for any $(x,y) \in \X \times \{-1,+1\}$, abbreviate
$\H[(x,y)] = \H[\{(x,y)\}] = \{h \in \H : h(x) = y\}$.

We also adopt the usual definition of ``shattering'' used in learning theory \citep*[e.g.,][]{vapnik:98}.
Specifically, for any set of classifiers $\H$, $k \in \nats$, and $S = (x_1,\ldots,x_k) \in \X^{k}$,
we say $\H$ \emph{shatters} $S$ if, $\forall (y_1,\ldots,y_k) \in \{-1,+1\}^{k}$, $\exists h \in \H$
such that $\forall i \in \{1,\ldots,k\}$, $h(x_i) = y_i$;
equivalently, $\H$ shatters $S$ if $\exists \{h_1,\ldots,h_{2^{k}}\} \subseteq \H$ such that for each $i,j \in \{1,\ldots,2^{k}\}$
with $i \neq j$, $\exists \ell \in \{1,\ldots,k\}$ with $h_{i}(x_{\ell}) \neq h_{j}(x_{\ell})$.
To simplify notation, we will also say that $\H$ shatters $\varnothing$ if and only if $\H \neq \{\}$.
As usual, we define the \emph{VC dimension} of $\C$, denoted $\vc$,
as the largest integer $k$ such that
$\exists S \in \X^{k}$ shattered by $\C$ \citep*{vapnik:98}.
To focus on nontrivial problems, we will only consider concept spaces $\C$ with $\vc > 0$ in the results below.
Generally, any such concept space $\C$ with $\vc < \infty$ is called a \emph{VC class}.

\subsection{Motivating Examples}
\label{subsec:motivating-examples}

Throughout this paper, we will repeatedly refer to a few canonical examples.
Although themselves quite toy-like, they represent the boiled-down essence
of some important distinctions between various types of learning problems.
In some sense, the process of grappling with the fundamental distinctions
raised by these types of examples has been a driving force behind much of the recent
progress in understanding the label complexity of active learning.

The first example is perhaps the most classic, and is clearly the first that
comes to mind when considering the potential for active learning to provide
strong improvements over passive learning.

\begin{example}
\label{ex:thresholds}
In the problem of learning \emph{threshold} classifiers, we consider $\X = [0,1]$
and \\$\C = \{h_{z}(x) = \ind_{[z,1]}^{\pm}(x) : z \in (0,1)\}$.
\thmend
\end{example}

There is a simple universal activizer for threshold classifiers, based on a kind of binary search.
Specifically, suppose $n \in \nats$ and that $\alg_{p}$ is any given passive learning algorithm.
Consider the points in $\{X_1,X_2,\ldots,X_{m}\}$,
for $m = 2^{n-1}$, and sort them in increasing order: $X_{(1)}, X_{(2)}, \ldots, X_{(m)}$.  Also
initialize $\ell = 0$ and $u = m+1$, and define $X_{(0)} = 0$ and $X_{(m+1)} = 1$.
Now request the label of $X_{(i)}$ for $i = \lfloor (\ell + u)/2 \rfloor$
(i.e., the median point between $\ell$ and $u$); if the label is $-1$, let $\ell = i$, and otherwise let $u = i$;
repeat this (requesting this median point, then updating $\ell$ or $u$ accordingly) until we have $u = \ell+1$.
Finally, let $\hat{z} = X_{(u)}$, construct the labeled sequence $\L = \left\{ \left(X_1, h_{\hat{z}}(X_1)\right), \ldots, \left(X_{m}, h_{\hat{z}}(X_m)\right)\right\}$,
and return the classifier $\hat{h} = \alg_{p}(\L)$.

Since each label request at least halves the set of integers between $\ell$ and $u$, the total number of label requests
is at most $\log_{2}(m)+1 = n$.
Supposing $f \in \C$ is the target function, this procedure maintains the invariant that $f(X_{(\ell)}) = -1$
and $f(X_{(u)}) = +1$.
Thus, once we reach $u = \ell+1$, since $f$ is a threshold,
it must be some $h_{z}$ with $z \in (\ell,u]$; therefore every $X_{(j)}$ with $j \leq \ell$
has $f(X_{(j)}) = -1$, and likewise every $X_{(j)}$ with $j \geq u$ has $f(X_{(j)}) = +1$;
in particular, this means $\L$ equals $\Data_{m}$, the \emph{true} labeled sequence.
But this means $\hat{h} = \alg_{p}(\Data_{m})$.  Since $n = \log_{2}(m)+1$, this
active learning algorithm will achieve an equivalent error rate to what $\alg_{p}$
achieves with $m$ labeled examples, but using only $\log_{2}(m)+1$ label requests.
In particular, this implies that if $\alg_{p}$ achieves label complexity $\Lambda_{p}$,
then this active learning algorithm achieves label complexity $\Lambda_{a}$ such
that $\Lambda_{a}(\eps, f, \Px) \leq \log_{2} \Lambda_{p}(\eps,f,\Px) + 2$;
as long as $1 \ll \Lambda_{p}(\eps,f,\Px) < \infty$, this is $o(\Lambda_{p}(\eps,f,\Px))$,
so that this procedure activizes $\alg_{p}$ for $\C$.

The second example we consider is almost equally simple (only increasing the VC dimension from $1$ to $2$),
but is far more subtle in terms of how we must approach its analysis in active learning.

\begin{example}
\label{ex:intervals}
In the problem of learning \emph{interval} classifiers, we consider $\X = [0,1]$ and
\\$\C = \{h_{[a,b]}(x) = \ind_{[a,b]}^{\pm}(x) : 0 < a \leq b < 1\}$.
\thmend
\end{example}

For the intervals problem, we can also construct a universal activizer, though slightly more complicated.
Specifically, suppose again that $n \in \nats$ and that $\alg_{p}$ is any given passive learning algorithm.
We first request the labels $\{Y_1,Y_2,\ldots, Y_{\lceil n/2 \rceil}\}$ of the first $\lceil n/2 \rceil$
examples in the sequence.  If every one of these labels is $-1$, then we immediately return the all-negative
constant classifier $\hat{h}(x) = -1$.  Otherwise, consider the points
$\{X_1,X_2,\ldots,X_{m}\}$, for $m = \max\left\{2^{\lfloor n/4 \rfloor - 1}, n\right\}$,
and sort them in increasing order $X_{(1)},X_{(2)},\ldots,X_{(m)}$.
For some value $i \in \{1,\ldots,\lceil n/2 \rceil\}$ with $Y_i = +1$, let $j_{+}$ denote the corresponding
index $j$ such that $X_{(j)} = X_{i}$.
Also initialize $\ell_1 = 0$, $u_1 = \ell_2 = j_{+}$, and $u_2 = m+1$, and define $X_{(0)} = 0$ and $X_{(m+1)} = 1$.
Now if $\ell_1 + 1 < u_1$, request the label of $X_{(i)}$ for $i = \lfloor (\ell_1+u_1)/2\rfloor$ (i.e., the median point
between $\ell_1$ and $u_1$); if the label is $-1$, let $\ell_1 = i$, and otherwise let $u_1 = i$;
repeat this (requesting this median point, then updating $\ell_1$ or $u_1$ accordingly) until we have $u_1 = \ell_1+1$.
Now if $\ell_2 +1 < u_2$, request the label of $X_{(i)}$ for $i = \lfloor (\ell_2+u_2)/2 \rfloor$ (i.e., the median point
between $\ell_2$ and $u_2$); if the label is $-1$, let $u_2 = i$, and otherwise let $\ell_2 = i$;
repeat this (requesting this median point, then updating $u_2$ or $\ell_2$ accordingly) until we have $u_2 = \ell_2+1$.
Finally, let $\hat{a} = u_1$ and $\hat{b} = \ell_2$, construct the labeled sequence
$\L = \left\{ \left(X_1, h_{[\hat{a},\hat{b}]}(X_1)\right), \ldots, \left(X_{m}, h_{[\hat{a},\hat{b}]}(X_{m})\right)\right\}$,
and return the classifier $\hat{h} = \alg_{p}(\L)$.

Since each label request in the second phase halves the set of values between either $\ell_1$ and $u_1$ or $\ell_2$ and $u_2$,
the total number of label requests is at most $\min\left\{m, \lceil n/2 \rceil + 2\log_{2}(m) + 2 \right\} \leq n$.
Suppose $f \in \C$ is the target function, and let $w(f) = \Px(x : f(x)=+1)$.  If $w(f) = 0$, then with probability $1$ the algorithm will
return the constant classifier $\hat{h}(x) = -1$, which has $\er(\hat{h}) = 0$ in this case.
Otherwise, if $w(f) > 0$, then for any $n \geq \frac{2}{w(f)} \ln \frac{1}{\eps}$, with probability at least
$1-\eps$, there exists $i \in \{1,\ldots,\lceil n/2\rceil\}$ with $Y_i = +1$.
Let $H_{+}$ denote the event that such an $i$ exists.
Supposing this is the case, the algorithm will make it into the second phase.
In this case, the procedure maintains the invariant that $f(X_{(\ell_1)}) = -1$,
$f(X_{(u_1)}) = f(X_{(\ell_2)}) = +1$, and $f(X_{(u_2)}) = -1$, where $\ell_1 < u_1 \leq \ell_2 < u_2$.
Thus, once we have $u_1 = \ell_1+1$ and $u_2 = \ell_2+1$, since $f$ is an interval, it must be
some $h_{[a,b]}$ with $a \in (\ell_1,u_1]$ and $b \in [\ell_2,u_1)$; therefore every $X_{(j)}$ with $j\leq \ell_1$ or $j \geq u_2$
has $f(X_{(j)}) = -1$, and likewise every $X_{(j)}$ with $u_1 \leq j \leq \ell_2$ has $f(X_{(j)}) = +1$;
in particular, this means $\L$ equals $\Data_{m}$, the \emph{true} labeled sequence.
But this means $\hat{h} = \alg_{p}(\Data_{m})$.  Supposing $\alg_{p}$ achieves label complexity $\Lambda_{p}$,
and that $n \geq \max\left\{ 8 + 4 \log_{2} \Lambda_{p}(\eps, f, \Px), \frac{2}{w(f)} \ln \frac{1}{\eps}\right\}$,
then $m \geq 2^{\lfloor n/4 \rfloor -1} \geq \Lambda_{p}(\eps,f,\Px)$ and
$\E\left[ \er(\hat{h}) \right]
\leq \E\left[ \er(\hat{h}) \ind_{H_{+}}\right] + (1-\P(H_{+}))
\leq \E\left[ \er( \alg_{p}(\Data_{m}))\right] + \eps
\leq 2 \eps$.
In particular, this means this active learning algorithm achieves label complexity $\Lambda_{a}$ such that,
for any $f \in \C$ with $w(f) = 0$, $\Lambda_{a}(2\eps,f,\Px) = 0$, and
for any $f \in \C$ with $w(f) > 0$,
$\Lambda_{a}(2 \eps, f, \Px)
\leq \max\left\{ 8 + 4 \log_{2} \Lambda_{p}(\eps,f,\Px), \frac{2}{w(f)} \ln\frac{1}{\eps}\right\}$.
If $(f,\Px) \in \Nontrivial(\Lambda_{p})$,
then $\frac{2}{w(f)} \ln\frac{1}{\eps} = o(\Lambda_{p}(\eps,f,\Px))$
and $8 + 4 \log_{2} \Lambda_{p}(\eps,f,\Px) = o(\Lambda_{p}(\eps,f,\Px))$,
so that $\Lambda_{a}(2\eps,f,\Px) = o(\Lambda_{p}(\eps,f,\Px))$.  Therefore, this procedure activizes $\alg_{p}$ for $\C$.

This example also brings to light some interesting phenomena in the analysis of the label complexity of active learning.
Note that unlike the thresholds example, we have a much stronger dependence on the target function in these label complexity bounds,
via the $w(f)$ quantity.  This issue is fundamental to the problem, and cannot be avoided.  In particular, when $\Px([0,x])$ is continuous,
this is the very issue that makes the \emph{minimax} label complexity for this problem
(i.e., $\min_{\Lambda_{a}} \max_{f \in \C} \Lambda_{a}(\eps,f,\Px)$) \emph{no better} than passive learning \citep*{dasgupta:05}.
Thus, this problem emphasizes the need for any informative label complexity analyses of active learning to explicitly describe
the dependence of the label complexity on the target function, as advocated by \citet*{dasgupta:05}.
This example also highlights the \emph{unverifiability} phenomenon explored by \citet*{hanneke:10a}, since in the case of $w(f) = 0$,
the error rate of the returned classifier is \emph{zero}, but (for nondegenerate $\Px$) there is no way for the algorithm to verify this fact based only on the finite
number of labels it observes.  In fact, \citet*{hanneke:10a} have shown that under continuous $\Px$, for any $f \in \C$ with $w(f) = 0$, the number of labels
required to both \emph{find} a classifier of small error rate \emph{and verify} that the error rate is small based only on observable quantities
is essentially \emph{no better} than for passive learning.

These issues are present to a small degree in the intervals example, but were easily handled in a very natural way.
The target-dependence shows up only in an initial phase of waiting for a positive example, and the always-negative
classifiers were handled by setting a \emph{default} return value.
However, we can amplify these issues so that they show up in more subtle and involved ways.
Specifically, consider the following example, 
%also 
studied by \citet*{hanneke:10a}.

\begin{example}
\label{ex:unions-of-intervals}
In the problem of learning \emph{unions of $i$ intervals}, we consider $\X = [0,1]$ and
\\$\C = \left\{h_{\mathbf{z}}(x) = \ind_{\bigcup_{j=1}^{i}[z_{2j-1},z_{2j}]}^{\pm}(x) : 0 < z_1 \leq z_2 \leq \ldots \leq z_{2i} < 1\right\}$.
\thmend
\end{example}

The challenge of this problem is that, because sometimes $z_{j} = z_{j+1}$ for some $j$ values,
we do not know how many intervals are required to minimally represent the target function: only that it is at most $i$.
This issue will be made clearer below.
We can essentially think of any effective strategy here as having two components:
one component that searches (perhaps randomly) with the purpose of identifying at least one example
from each decision region, and another component that refines our estimates of the end-points of the
regions the first component identifies.
Later, we will go through the behavior of a universal activizer for this problem in detail.

\section{Disagreement-Based Active Learning}
\label{sec:naive}

At present, perhaps the best-understood active learning algorithms are those choosing their label requests
based on disagreement among a set of remaining candidate classifiers.
The canonical algorithm of this type, a version of which we discuss below in Section~\ref{subsec:disagreement-coefficient},
was proposed by \citet*{cohn:94}.  Specifically, for any set $\H$ of classifiers, define the \emph{region of disagreement}:
%of $\H$
%
\begin{equation*}
\DIS(\H) = \left\{ x \in \X : \exists h_1,h_2 \in \H \text{ s.t. } h_1(x) \neq h_2(x)\right\}.
\end{equation*}

The basic idea of disagreement-based algorithms is that, at any given time
in the algorithm, there is a subset $V \subseteq \C$ of remaining candidates, called the \emph{version space},
which is guaranteed to contain the target $f$.  When deciding whether to request a particular label $Y_i$,
the algorithm simply checks whether $X_i \in \DIS(V)$: if so, the algorithm requests $Y_i$, and otherwise it does not.
This general strategy is reasonable, since for any $X_i \notin \DIS(V)$, the label agreed upon by $V$ must be $f(X_i)$,
so that we would get no information by requesting $Y_i$; that is, for $X_i \notin \DIS(V)$, we can accurately \emph{infer} $Y_i$
based on information already available.
This type of algorithm has recently received substantial attention, not only for its obvious elegance and simplicity,
but also because (as we discuss in Section~\ref{sec:agnostic})
there are natural ways to extend the technique to the general problem of learning with label noise and model misspecification
(the \emph{agnostic} setting).
The details of disagreement-based algorithms can vary in how they update the set $V$ and how frequently they do so, but
it turns out almost all disagreement-based algorithms share many of the same fundamental properties,
which we describe below.

\subsection{A Basic Disagreement-Based Active Learning Algorithm}
\label{subsec:basic-naive}

In Section~\ref{subsec:disagreement-coefficient}, we discuss several known results on the label complexities achievable
by these types of active learning algorithms.  However, for now let us examine a very basic algorithm of this type.
The following is intended to be a simple representative of the family of disagreement-based active learning algorithms.
It has been stripped down to the bare essentials of what makes such algorithms work.
As a result, although the gap between its label complexity and that achieved by passive learning
is not necessarily as large as those achieved by the more sophisticated disagreement-based active learning algorithms of Section~\ref{subsec:disagreement-coefficient},
it has the property that whenever those more sophisticated methods have label complexities asymptotically superior to those achieved by passive learning,
that guarantee will also be true for this simpler method, and vice versa.
The algorithm operates in only $2$ phases. In the first, it uses one batch of label requests to
reduce the version space $V$ to a subset of $\C$; in the second, it uses another batch of label requests,
this time only requesting labels for points in $\DIS(V)$.  Thus, we have isolated precisely that aspect of disagreement-based
active learning that involves improvements due to only requesting the labels of examples in the region of disagreement.
The procedure is formally defined as follows, in terms of an estimator $\hat{P}_n(\DIS(V))$ specified below.

\begin{bigboxit}
\NaiveActivizer \\ 
Input: passive algorithm $\alg_p$, label budget $n$\\
Output: classifier $\hat{h}$\\
{\vskip -2mm}\line(1,0){419}\\
0. Request the first $\lfloor n / 2 \rfloor$ labels $\{Y_1,\ldots,Y_{\lfloor n/2 \rfloor}\}$, and let $t \gets \lfloor n/2 \rfloor$\\
1. Let $V = \{ h \in \C : \er_{\lfloor n/2 \rfloor}(h) = 0\}$\\
2. Let $\hat{\Delta} \gets \hat{P}_n(\DIS(V))$\\
3. Let $\L \gets \{\}$\\
4. For $m = \lfloor n/2 \rfloor + 1, \ldots \lfloor n/2 \rfloor + \lfloor n / (4 \hat{\Delta}) \rfloor$\\
5. \quad If $X_{m} \in \DIS(V)$ and $t < n$, request the label $Y_{m}$ of $X_{m}$, and let $\hat{y} \gets Y_{m}$ and $t \gets t+1$\\
6. \quad Else let $\hat{y} \gets h(X_{m})$ for an arbitrary $h \in V$\\
7. \quad Let $\L \gets \L \cup \{(X_{m},\hat{y})\}$\\
8. Return $\alg_p(\L)$
\end{bigboxit}

\NaiveActivizer~depends on a data-dependent estimator $\hat{P}_n(\DIS(V))$ of $\Px(\DIS(V))$, which we
can define in a variety of ways using only \emph{unlabeled} examples.  In particular, for the theorems
below, we will take the following definition for $\hat{P}_n(\DIS(V))$, designed to be a confidence upper
bound on $\Px(\DIS(V))$.
Let $\U_n = \{X_{n^2+1},\ldots,X_{2n^2}\}$.
Then define
\begin{equation}
\label{eqn:pdisv-est}
\hat{P}_n(\DIS(V)) = \max\left\{ \frac{2}{n^2} \sum_{x \in \U_n} \ind_{\DIS(V)}(x), \frac{4}{n} \right\}.
\end{equation}

\NaiveActivizer~is divided into two stages: one stage where
we focus on reducing $V$, and a second stage where we construct the sample $\L$
for the passive algorithm.  This might intuitively seem somewhat wasteful, as one might
wish to use the requested labels from the first stage to augment those in the second stage
when constructing $\L$, thus feeding all of the observed labels into the passive
algorithm $\alg_{p}$.  Indeed, this can improve the label complexity in some cases
(albeit only by a constant factor); however, in order to get the \emph{general} property of being
an activizer for \emph{all} passive algorithms $\alg_{p}$, we construct the sample $\L$
so that the conditional distribution of the $\X$ components in $\L$ given $|\L|$
is $\Px^{|\L|}$, so that it is (conditionally) an i.i.d. sample, which is essential to our analysis.
The choice of the number of (unlabeled) examples to process in the second stage
guarantees (by a Chernoff bound) that the ``$t < n$'' constraint in Step 5 is redundant;
this is a trick we will employ in several of the methods below.  As explained above,
because $f \in V$, this implies that every $(x,y) \in \L$ has $y = f(x)$.

To give some basic intuition for how this algorithm behaves, consider the example of learning threshold classifiers (Example~\ref{ex:thresholds});
to simplify the explanation, for now we ignore the fact that $\hat{P}_n$ is only an estimate, as well as the ``$t < n$'' constraint in Step 5 (both of which will be addressed in the general analysis below).
In this case, suppose the target function is $f = h_{z}$.  Let $a = \max\{ X_{i} : X_{i} < z, 1 \leq i \leq \lfloor n/2 \rfloor\}$ and $b = \min\{ X_{i} : X_{i} \geq z, 1 \leq i \leq \lfloor n/2 \rfloor\}$.
Then $V = \{h_{z^{\prime}} : a < z^{\prime} \leq b\}$ and $\DIS(V) = (a,b)$, so that the second phase of the algorithm only requests labels for a number of points in
the region $(a,b)$.  With probability $1-\eps$, the probability mass in this region is at most $O(\log(1/\eps)/n)$,
so that $|\L| \geq \ell_{n,\eps} = \Omega(n^2 / \log(1/\eps))$; also, since the labels in $\L$ are all correct, and the $X_{m}$ values in $\L$
are conditionally iid (with distribution $\Px$) given $|\L|$,
we see that the conditional distribution of $\L$ given $|\L|=\ell$ is the same as the (unconditional) distribution of $\Data_{\ell}$.
In particular, if $\alg_p$ achieves label complexity $\Lambda_p$, and $\hat{h}_{n}$ is the classifier returned by \NaiveActivizer~applied to $\alg_p$,
then for any $n = \Omega\left( \sqrt{\Lambda_{p}(\eps,f,\Px) \log(1/\eps)}\right)$ chosen so that $\ell_{n,\eps} \geq \Lambda_{p}(\eps,f,\Px)$,
we have
\begin{equation*}
\E\left[ \er\left(\hat{h}_{n}\right) \right]
\leq \eps + \sup\limits_{\ell \geq \ell_{n,\eps}} \E\left[ \er\left(\alg_{p}(\Data_{\ell})\right)\right]
\leq \eps + \sup\limits_{\ell \geq \Lambda_{p}(\eps,f,\Px)} \E\left[\er\left(\alg_{p}(\Data_{\ell})\right)\right]
\leq 2\eps.
\end{equation*}
This indicates the active learning algorithm
%resulting from applying \NaiveActivizer~to $\alg_{p}$
achieves label complexity $\Lambda_{a}$ with
$\Lambda_{a}(2\eps,f,\Px) = O\left( \sqrt{\Lambda_{p}(\eps,f,\Px) \log(1/\eps)}\right)$.
In particular, if $\infty > \Lambda_{p}(\eps,f,\Px) = \omega(\log(1/\eps))$, then $\Lambda_{a}(2\eps,f,\Px) = o(\Lambda_{p}(\eps,f,\Px))$.
%and if $\Lambda_{p}(\eps,f,\Px) = O(\log(1/\eps))$, then $\Lambda_{a}(2\eps,f,\Px) = O(\log(1/\eps))$ as well.
Therefore, \NaiveActivizer~is a universal activizer for the space of threshold classifiers.

In contrast, consider the problem of learning interval classifiers (Example~\ref{ex:intervals}).
In this case, suppose the target function $f$ has $\Px(x : f(x) = +1) = 0$, and that $\Px$ is uniform in $[0,1]$.
Since (with probability one) every $Y_i = -1$, we have $V = \{ h_{[a,b]} : \{X_1,\ldots,X_{\lfloor n/2 \rfloor}\} \cap [a,b] = \emptyset\}$.
But this contains classifiers $h_{[a,a]}$ for every $a \in (0,1) \setminus \{X_1,\ldots,X_{\lfloor n/2 \rfloor}\}$, so that
$\DIS(V) = (0,1) \setminus \{X_1,\ldots,X_{\lfloor n/2 \rfloor}\}$.  Thus, $\Px(\DIS(V)) = 1$, and $|\L| = O(n)$;
that is, $\alg_{p}$ gets run with no more labeled examples than simple passive learning would use.
This indicates we should not expect \NaiveActivizer~to be a universal activizer for interval classifiers.  Below, we formalize this,
by constructing a passive learning algorithm $\alg_p$ that \NaiveActivizer~does not activize for this scenario.

\subsection{The Limiting Region of Disagreement}
\label{subsec:core}

In this subsection, we generalize the examples from the previous subsection.
Specifically, we prove that the performance of \NaiveActivizer~is intimately tied to a particular
limiting set, referred to as the \emph{disagreement core}.  A similar definition was given by \citet*{hanneke:10a} (there referred to as the \emph{boundary},
for reasons that will become clear below);
it is also related to certain quantities in the work of \citet*{hanneke:07b,hanneke:11a} described below in Section~\ref{subsec:disagreement-coefficient}.
\begin{definition}
\label{def:core}
Define the \emph{disagreement core} of a classifier $f$ with respect to a set of classifiers $\H$ and distribution $P$ as
\begin{equation*}
\partial_{\H,P} f = \lim\limits_{r \to 0} \DIS\left(\Ball_{\H,P}(f,r)\right).
\end{equation*}
\upthmend{-1.1cm}
\end{definition}
When $P=\Px$, the true distribution on $\X$, and $\Px$ is clear from the context, we abbreviate this as $\partial_{\H} f = \partial_{\H,\Px} f$;
if additionally $\H = \C$, the full concept space, which is clear from the context, we further abbreviate this as $\partial f = \partial_{\C} f = \partial_{\C,\Px} f$.

As we will see, disagreement-based algorithms often tend to focus their label requests around the disagreement core of the target function.
As such, the concept of the disagreement core will be essential in much of our discussion below.
We therefore go through a few examples to build intuition about this concept and its properties.
Perhaps the simplest example to start with is $\C$ as the class of \emph{threshold} classifiers (Example~\ref{ex:thresholds}), under $\Px$ uniform on $[0,1]$.
For any $h_{z} \in \C$ and sufficiently small $r > 0$, $\Ball(f,r) = \{ h_{z^{\prime}} : |z^{\prime} - z| \leq r\}$,
and $\DIS(\Ball(f,r)) = [z-r, z+r)$.  Therefore, $\partial h_{z} = \lim\limits_{r \to 0} \DIS(\Ball(h_{z},r)) = \lim\limits_{r \to 0} [z-r,z+r) = \{z\}$.
Thus, in this case, the disagreement core of $h_{z}$ with respect to $\C$ and $\Px$ is precisely the decision boundary of the classifier.
As a slightly more involved example, consider again the example of \emph{interval} classifiers (Example~\ref{ex:intervals}), again under $\Px$ uniform on $[0,1]$.
Now for any $h_{[a,b]} \in \C$ with $b-a > 0$, for any sufficiently small $r > 0$, $\Ball(h_{[a,b]},r) = \{ h_{[a^{\prime},b^{\prime}]} : |a-a^{\prime}| + |b-b^{\prime}| \leq r\}$,
and  $\DIS(\Ball(h_{[a,b]},r)) = [a-r,a+r) \cup (b-r,b+r]$.  Therefore, $\partial h_{[a,b]} = \lim\limits_{r \to 0} \DIS(\Ball(h_{[a,b]},r)) = \lim\limits_{r \to 0} [a-r,a+r) \cup (b-r,b+r] = \{a,b\}$.
Thus, in this case as well, the disagreement core of $h_{[a,b]}$ with respect to $\C$ and $\Px$ is again the decision boundary of the classifier.

As the above two examples illustrate, $\partial f$ often corresponds to the decision boundary of $f$ in some geometric interpretation of $\X$ and $f$.
Indeed, under fairly general conditions on $\C$ and $\Px$,
the disagreement core of $f$ does correspond to (a subset of) the set of points dividing the two label regions of $f$; for instance, \citet*{friedman:09} derives sufficient conditions, under which this is the case.
In these cases, the behavior of disagreement-based active learning algorithms can often be interpretted in the intuitive terms of seeking
label requests near the decision boundary of the target function, to refine an estimate of that boundary.
However, in some more subtle scenarios this is no longer the case, for interesting reasons.  To illustrate
this, let us continue the example of interval classifiers from above, but now consider $h_{[a,a]}$ (i.e., $h_{[a,b]}$ with $a=b$).  This time,
for any $r \in (0,1)$ we have $\Ball(h_{[a,a]},r) = \{h_{[a^{\prime},b^{\prime}]} \in \C : b^{\prime} - a^{\prime} \leq r\}$, and
$\DIS(\Ball(h_{[a,a]},r)) = (0,1)$.  Therefore, $\partial h_{[a,a]} = \lim\limits_{r \to 0} \DIS(\Ball(h_{[a,a]},r)) = \lim\limits_{r \to 0} (0,1) = (0,1)$.

This example shows that in some cases, the disagreement core does not correspond to the decision boundary of the classifier, and indeed has $\Px(\partial f) > 0$.
Intuitively, as in the above example, this typically happens when the decision surface of the classifier is in some sense
\emph{simpler} than it could be.
For instance, consider the space $\C$ of \emph{unions of two intervals} (Example~\ref{ex:unions-of-intervals} with $i=2$) under uniform $\Px$.
The classifiers $f \in \C$ with $\Px(\partial f) > 0$ are precisely those representable (up to probability zero differences)
as a single interval.  The others (with $0 < z_1 < z_2 < z_3 < z_4 < 1$) have $\partial h_{\mathbf{z}} = \{z_1,z_2,z_3,z_4\}$.
In these examples, the $f \in \C$ with $\Px(\partial f) > 0$ are not only simpler than other nearby classifiers in $\C$,
but they are also in some sense \emph{degenerate} relative to the rest of $\C$; however, it turns out this is not always
the case, as there exist scenarios $(\C,\Px)$, even with $\vc=2$, and even with \emph{countable} $\C$, for which
\emph{every} $f \in \C$ has $\Px(\partial f) > 0$; in these cases, every classifier is in some important sense \emph{simpler}
than some other subset of nearby classifiers in $\C$.

In Section~\ref{subsec:core-dis-based-theorem},
we show that the label complexity of disagreement-based active learning is intimately tied to the disagreement core.  In particular,
scenarios where $\Px(\partial f) > 0$, such as those mentioned above, lead to the conclusion that disagreement-based
methods are sometimes insufficient for activized learning.
This motivates the design of more sophisticated methods in Section~\ref{sec:activizer}, which overcome this deficiency,
along with a corresponding refinement of the definition of ``disagreement core '' in Section~\ref{subsec:sequential-activizer}
that eliminates the above issue with ``simple'' classifiers.

\subsection{Necessary and Sufficient Conditions for Disagreement-Based Activized Learning}
\label{subsec:core-dis-based-theorem}

In the specific case of \NaiveActivizer, for large $n$ we may intuitively expect it to focus its second batch of
label requests in and around the disagreement core of the target function.
Thus, whenever $\Px(\partial f) = 0$, we should expect the label requests to be quite focused, and therefore the algorithm should
achieve higher accuracy compared to passive learning.  On the other hand, if $\Px(\partial f) > 0$, then the label requests will \emph{not}
become focused beyond a constant fraction of the space, so that the improvements achieved by \NaiveActivizer~over passive
learning should be, at best, a constant factor.  This intuition is formalized in the following general theorem,
the proof of which is included in Appendix~\ref{app:naive}.

\begin{theorem}
\label{thm:naive}
For any VC class $\C$, \NaiveActivizer~is a universal activizer for $\C$ if and only if
every $f \in \C$ and distribution $\Px$ has $\Px\left(\partial_{\C,\Px} f\right) = 0$.
\thmend
\end{theorem}

While the formal proof is given in Appendix~\ref{app:naive}, the general idea is simple.
As we always have $f \in V$, any $\hat{y}$ inferred in Step 6 must equal $f(x)$,
so that all of the labels in $\L$ are correct.
Also, as $n$ grows large, classic results on passive learning imply the diameter of the
set $V$ will become small, shrinking to zero as $n \to \infty$ \citep*{vapnik:82,blumer:89}.
Therefore, as $n \to \infty$, $\DIS(V)$ should converge to a subset of $\partial f$,
so that in the case $\Px(\partial f) = 0$, we have $\hat{\Delta} \to 0$; thus $|\L| \gg n$,
which implies an asymptotic strict improvement in label complexity over the passive
algorithm $\alg_p$ that $\L$ is fed into in Step 8.
On the other hand, since $\partial f$ is defined by classifiers arbitrarily close to $f$,
it is unlikely that any finite sample of correctly labeled examples can contradict
enough classifiers to make $\DIS(V)$ significantly smaller than $\partial f$, so that
we always have $\Px(\DIS(V)) \geq \Px(\partial f)$.  Therefore, if $\Px(\partial f) > 0$, then
$\hat{\Delta}$ converges to some nonzero constant, so that $|\L| = O(n)$,
representing only a constant factor improvement in label complexity.
In fact, as is implied from this sketch (and is proven in Appendix~\ref{app:naive}),
the targets $f$ and distributions $\Px$ for
which \NaiveActivizer~achieves asymptotic strict improvements for
all passive learning algorithms (for which $f$ and $\Px$ are nontrivial)
are precisely those (and only those) for which
$\Px(\partial_{\C,\Px} f) = 0$.

There are some general conditions under which the zero-probability disagreement cores condition
of Theorem~\ref{thm:naive} will hold.  For instance, it is not difficult to show this will always
hold when $\X$ is countable; furthermore, with some effort one can show it
will hold for most classes having VC dimension one (e.g., any countable $\C$ with $d=1$).
However, as we have seen, not all spaces $\C$ satisfy this zero-probability disagreement cores property.
In particular, for the interval classifiers studied in Section~\ref{subsec:core},
we have $\Px(\partial h_{[a,a]}) = \Px((0,1)) = 1$.  Indeed, the aforementioned special cases aside,
for \emph{most} nontrivial spaces $\C$, one can construct distributions $\Px$ that
in some sense mimic the intervals problem, so that we should typically expect
disagreement-based methods will \emph{not} be activizers.
For detailed discussions of various scenarios where the $\Px(\partial_{\C,\Px} f) = 0$
condition is (or is not) satisfied for various $\C$, $\Px$, and $f$, see
the works of \citet*{hanneke:thesis,hanneke:07b,hanneke:11a,hanneke:10a,friedman:09,wang:09,wang:11}.

\section{Beyond Disagreement: A Basic Activizer}
\label{sec:activizer}

Since the zero-probability disagreement cores condition of Theorem~\ref{thm:naive} is not always satisfied,
we are left with the question of whether there could be other techniques for active learning,
beyond simple disagreement-based methods, which could activize \emph{every} passive learning
algorithm for \emph{every} VC class.  In this section, we present an
entirely new type of active learning algorithm, unlike anything in the existing literature,
and we show that indeed it is a universal activizer for any class $\C$ of finite VC dimension.

\subsection{A Basic Activizer}
\label{subsec:basic-activizer}

As mentioned, the case $\Px(\partial f) = 0$ is already handled nicely by disagreement-based methods,
since the label requests made in the second stage of \NaiveActivizer~will become focused into a small region,
and $\L$ therefore grows faster than $n$.
Thus, the primary question we are faced with is what to do when $\Px(\partial f) > 0$.  Since (loosely speaking)
we have $\DIS(V) \to \partial f$ in \NaiveActivizer, $\Px(\partial f) > 0$ corresponds to scenarios where
the label requests of \NaiveActivizer~will not become focused beyond a certain extent; specifically,
since $\Px(\DIS(V) \oplus \partial f) \to 0$ almost surely (where $\oplus$ is the symmetric difference),
\NaiveActivizer~will request labels for a constant fraction of the examples in $\L$.

On the one hand, this is definitely a major problem for disagreement-based methods, since it prevents them
from improving over passive learning in those cases.  On the other hand, if we do not restrict ourselves to
disagreement-based methods, we may actually be able to exploit properties of this scenario,
so that it works to our \emph{advantage}.  In particular, since
$\Px(\DIS(V) \oplus \partial_{\C} f) \to 0$ and $\Px(\partial_{V} f \oplus \partial_{\C} f) = 0$ (almost surely) in \NaiveActivizer,
for sufficiently large $n$ a random point $x_1$ in $\DIS(V)$ is likely to be in $\partial_{V} f$.
We can exploit this fact by using $x_1$ to split $V$ into two subsets: $V[(x_1,+1)]$ and $V[(x_1,-1)]$.
Now, if $x_1 \in \partial_{V} f$, then (by definition of the disagreement core) $\inf\limits_{h \in V[(x_1,+1)]} \er(h) = \inf\limits_{h \in V[(x_1,-1)]} \er(h) = 0$.
Therefore, for almost every point $x \notin \DIS(V[(x_1,+1)])$, the label agreed upon for $x$ by classifiers in $V[(x_1,+1)]$
should be $f(x)$.  Similarly, for almost every point $x \notin \DIS(V[(x_1,-1)])$, the label agreed upon for $x$ by classifiers
in $V[(x_1,-1)]$ should be $f(x)$.  Thus, we can accurately \emph{infer} the label of any point $x \notin \DIS(V[(x_1,+1)]) \cap \DIS(V[(x_1,-1)])$
(except perhaps a probability zero subset).  With these sets $V[(x_1,+1)]$ and $V[(x_1,-1)]$ in hand, there is no longer a need to
request the labels of points for which either of them has agreement about the label, and we can focus our label requests to the region
$\DIS(V[(x_1,+1)]) \cap \DIS(V[(x_1,-1)])$, which may be \emph{much smaller} than $\DIS(V)$.
Now if $\Px(\DIS(V[(x_1,+1)]) \cap \DIS(V[(x_1,-1)])) \to 0$, then the label requests will become focused to a shrinking region,
and by the same reasoning as for Theorem~\ref{thm:naive} we can asymptotically achieve strict improvements over passive learning
by a method analogous to \NaiveActivizer~(with changes as described above).

Already this provides a significant improvement over disagreement-based methods in many cases; indeed, in some cases (such as intervals)
this already addresses the nonzero-probability disagreement core issue in Theorem~\ref{thm:naive}.
In other cases (such as unions of two intervals), it does not completely address the issue, since for some targets
we do not have $\Px(\DIS(V[(x_1,+1)]) \cap \DIS(V[(x_1,-1)])) \to 0$.  However, by repeatedly applying this
same reasoning, we \emph{can} address the issue in full generality.  Specifically, if
$\Px(\DIS(V[(x_1,+1)]) \cap \DIS(V[(x_1,-1)])) \nrightarrow 0$, then $\DIS(V[(x_1,+1)]) \cap \DIS(V[(x_1,-1)])$ essentially converges
to a region $\partial_{\C[(x_1,+1)]} f \cap \partial_{\C[(x_1,-1)]} f$, which has nonzero probability, and is nearly equivalent to
$\partial_{V[(x_1,+1)]} f \cap \partial_{V[(x_1,-1)]} f$.  Thus, for sufficiently large $n$,
a random $x_2$ in $\DIS(V[(x_1,+1)]) \cap \DIS(V[(x_1,-1)])$ will likely be in $\partial_{V[(x_1,+1)]} f \cap \partial_{V[(x_1,-1)]} f$.
In this case, we can repeat the above argument, this time splitting $V$ into four sets
($V[(x_1,+1)][(x_2,+1)]$, $V[(x_1,+1)][(x_2,-1)]$, $V[(x_1,-1)][(x_2,+1)]$, and $V[(x_1,-1)][(x_2,-1)]$),
each with infimum error rate equal zero, so that for any point $x$ in the region of agreement of any of these
four sets, the agreed-upon label will (almost surely) be $f(x)$, so that we can infer that label.  Thus, we need
only request the labels of those points in the \emph{intersection} of all four regions of disagreement.
We can further repeat this process as many times as needed, until we get a partition of $V$ with
shrinking probability mass in the intersection of the regions of disagreement, which (as above) can
then be used to obtain 
asymptotic 
improvements over passive learning.

Note that the above argument can be written more concisely in terms of \emph{shattering}.
That is, any $x \in \DIS(V)$ is simply an $x$ such that $V$ shatters $\{x\}$;
a point $x \in \DIS(V[(x_1,+1)]) \cap \DIS(V[(x_1,-1)])$ is simply one for which
$V$ shatters $\{x_1,x\}$, and for any $x \notin \DIS(V[(x_1,+1)]) \cap \DIS(V[(x_1,-1)])$,
the label $y$ we infer about $x$ has the property that the set $V[(x,-y)]$ does not shatter $\{x_1\}$.
This continues for each repetition of the above idea, with $x$ in the intersection of the four
regions of disagreement simply being one for which $V$ shatters $\{x_1,x_2,x\}$, and so on.
In particular, this perspective makes it clear that we need only repeat this idea at most $d$ times
to get a shrinking intersection region, since no set of $d+1$ points is shatterable.
Note that there may be unobservable factors (e.g., the target function) determining the appropriate number of iterations
of this idea sufficient to have a shrinking probability of requesting a label, while maintaining the accuracy of inferred labels.
To address this, we can simply try all $\vc+1$ possibilities,
and then select one of the resulting $\vc+1$ classifiers via a 
kind of tournament of pairwise comparisons.
Also, in order to reduce the probability of a mistaken inference due to $x_1 \notin \partial_{V} f$ (or similarly for later $x_i$),
we can replace each single $x_i$ with multiple samples, and then take a majority vote over whether to
infer the label, and which label to infer if we do so; generally, we can think of this as estimating certain
probabilities, and below we write these estimators as $\hat{P}_{m}$, and discuss the details of their implementation later.
Combining \NaiveActivizer~with the above reasoning motivates a new type of active learning algorithm,
referred to as \BasicActivizer~below, and stated as follows.

\begin{bigboxit}
\BasicActivizer \\ 
Input: passive algorithm $\alg_p$, label budget $n$\\
Output: classifier $\hat{h}$\\
{\vskip -2mm}\line(1,0){419}\\
0.\phantom{0} Request the first $m_n = \lfloor n/3 \rfloor$ labels, $\left\{ Y_1, \ldots, Y_{m_n}\right\}$, and let $t \gets m_n$\\
1.\phantom{0} Let $V = \{ h \in \C : \er_{m_n}(h) = 0\}$\\
2.\phantom{0} For $k = 1,2,\ldots,d+1$\\
3.\phantom{0} \quad $\hat{\Delta}^{(k)} \gets \hat{P}_{m_n}\left(x : \hat{P}\left(S \in \X^{k-1} : V \text{ shatters } S \cup \{x\} | V \text{ shatters } S\right) \geq 1/2 \right)$ \\
4.\phantom{0} \quad Let $\L_k \gets \{\}$\\
5.\phantom{0} \quad For $m = m_n + 1, \ldots, m_n + \lfloor n / (6 \cdot 2^k \hat{\Delta}^{(k)}) \rfloor$ \\
6.\phantom{0} \qquad If $\hat{P}_{m}\left(S \in \X^{k-1} : V \text{ shatters } S \cup \{X_{m}\} | V \text{ shatters } S\right) \geq 1/2$ and $t < \lfloor 2 n / 3 \rfloor$\\
7.\phantom{0} \qquad\quad Request the label $Y_{m}$ of $X_{m}$, and let $\hat{y} \gets Y_{m}$ and $t \gets t+1$\\
8.\phantom{0} \qquad Else, let
$\hat{y} \gets \!\!\argmax\limits_{y \in \{-1,+1\}} \!\hat{P}_{m}\!\left(S \in \X^{k-1} \!:\! V[(X_{m},-y)] \text{ does not shatter } S | V \text{ shatters } S \right)$\\
9.\phantom{0} \qquad Let $\L_k \gets \L_k \cup \{(X_{m},\hat{y})\}$ \\
10. Return $\ActiveSelect(\{\alg_p(\L_1),\alg_p(\L_2),\ldots,\alg_p(\L_{d+1})\},\lfloor n/3 \rfloor, \{X_{m_n + \max_{k} |\L_k|+1},\ldots\})$
\end{bigboxit}
\begin{bigboxit}
Subroutine: $\ActiveSelect$\\
Input: set of classifiers $\{h_1,h_2,\ldots,h_{N}\}$, label budget $m$, sequence of unlabeled examples $\U$\\
Output: classifier $\hat{h}$\\
{\vskip -2mm}\line(1,0){419}\\
0. For each $j,k \in \{1,2,\ldots,N\} \text{ s.t. }  j < k$,\\
%1. \quad Take the next $\lfloor m / \binom{N}{2} \rfloor$ examples $x$ s.t. $h_j(x) \neq h_k(x)$ (if such examples exist)\\
1.\quad Let $R_{jk}$ be the first $\left\lfloor \frac{m}{j (N-j) \ln(eN)} \right\rfloor$ points in $\U \!\cap\! \{x :h_j(x) \neq h_k(x)\}$ (if such values exist)\\
2.\quad Request the labels for $R_{jk}$ and let $Q_{jk}$ be the resulting set of labeled examples\\
%3.\quad Let $m_{jk} = \er_{Q_{jk}}(h_j)$ and $m_{kj} = \er_{Q_{jk}}(h_k)$\\
3.\quad Let $m_{kj} = \er_{Q_{jk}}(h_k)$\\
%4. Return $h_{\hat{k}}$, where $\hat{k} = \argmin\limits_{k \in \{1,\ldots,N\}} \max\limits_{j \in \{1,\ldots,N\}\setminus \{k\}} m_{kj}$
4. Return $h_{\hat{k}}$, where $\hat{k} = \max \left\{ k \in \{1,\ldots,N\} : \max_{j < k} m_{kj} \leq 7 / 12\right\}$ % this version makes it a left-biased selector, saving a factor of \vc
\end{bigboxit}

\BasicActivizer~is stated as a function of three types of estimated probabilities: namely,
\begin{align*}
&\hat{P}_{m}\left(S \in \X^{k-1} : V \text{ shatters } S \cup \{x\} \Big| V \text{ shatters } S\right),
\\ &\hat{P}_{m}\left(S \in \X^{k-1} : V[(x,-y)] \text{ does not shatter } S \Big| V \text{ shatters } S\right),
\\\text{and } &\hat{P}_{m}\left(x : \hat{P}\left(S \in \X^{k-1} : V \text{ shatters } S \cup \{x\} \Big| V \text{ shatters } S\right) \geq 1/2\right).
\end{align*}
These can be defined in a variety of ways to make this a universal activizer.  Generally, the only requirement seems to be that
they converge to the appropriate respective probabilities in the limit.  For the theorem stated below regarding \BasicActivizer,
we will take the specific definitions stated in Appendix~\ref{app:hatP-definitions}.

\BasicActivizer~requests labels in three batches:
one to initially prune down the version space $V$,
a second one to construct the labeled samples $\L_k$,
and a third batch to select among the $d+1$ classifiers $\alg_p(\L_k)$ in the $\ActiveSelect$ subroutine.
As before, the choice of the number of (unlabeled) examples to process in the second batch guarantees (by a Chernoff bound)
that the ``$t < \lfloor 2n/3\rfloor$'' constraint in Step 6 is redundant.
The mechanism for requesting labels in the second batch is motivated by the reasoning outlined above,
using the shatterable sets $S$ to split $V$ into $2^{k-1}$ subsets, each of which approximates the target
with high probability (for large $n$), and then checking whether the new point $x$ is in the regions of disagreement
for all $2^{k-1}$ subsets (by testing shatterability of $S \cup \{x\}$).  To increase confidence in this test, we use many
such $S$ sets, and let them vote on whether or not to request the label (Step 6).  As mentioned, if $x$ is not in the region of
disagreement for one of these $2^{k-1}$ subsets (call it $V^{\prime}$), the agreed-upon label $y$ has the property that
$V[(x,-y)]$ does not shatter $S$ (since $V[(x,-y)]$ does not intersect with $V^{\prime}$, which represents one of the $2^{k-1}$
labelings required to shatter $S$).  Therefore, we infer that this label $y$ is the correct label of $x$, and again we
vote over many such $S$ sets to increase confidence in this choice (Step 8).  As mentioned, this reasoning leads
to correctly inferred labels in Step 8 as long as $n$ is sufficiently large \emph{and}
$\Px^{k-1}(S \in \X^{k-1} : V \text{ shatters } S) \nrightarrow 0$.
In particular, we are primarily interested in the
largest value of $k$ for which this reasoning holds, since this is the value at which the probability of
requesting a label (Step 7) shrinks to zero as $n \to \infty$.  However, since we typically cannot predict
a priori what this largest valid $k$ value will be (as it is target-dependent), we try all $d+1$ values of
$k$, to generate $d+1$ hypotheses, and then use a simple pairwise testing procedure to select among them;
note that we need at most try $d+1$ values, since $V$ definitely cannot shatter any $S \in \X^{d+1}$.
We will see that the $\ActiveSelect$ subroutine is guaranteed to select a classifier with error rate never
significantly larger than the best among the classifiers given to it (say within a factor of $2$, with high probability).
Therefore, in the present context, we need only consider whether some $k$ has a set $\L_k$ with correct
labels \emph{and} $|\L_k| \gg n$.

\subsection{Examples}
\label{subsec:basic-activizer-examples}

In the next subsection, we state a general result for \BasicActivizer.
But first, to illustrate how this procedure operates, we walk through its behavior on our usual examples;
as we did for the examples of \NaiveActivizer, to simplify the explanation, for now we will ignore the fact
that the $\hat{P}_{m}$ values are estimates, as well as the ``$t < \lfloor 2n/3 \rfloor$'' constraint of Step 6,
and the issue of effectiveness of $\ActiveSelect$;
in the proofs of the general results below, we will show that these issues do not fundamentally change the analysis.
For now, we merely focus on showing that some $k$ has $\L_k$ correctly labeled and $|\L_k| \gg n$.

For threshold classifiers (Example~\ref{ex:thresholds}), we have $d=1$.  In this case, the $k=1$ round of the algorithm is essentially identical to
\NaiveActivizer~(recall our conventions that $\X^{0} = \{\varnothing\}$, $\Px(\X^{0})=1$, and $V$ shatters $\varnothing$ iff $V \neq \{\}$),
and we therefore have $|\L_1| \gg n$, as discussed previously, so that \BasicActivizer~is a universal activizer for threshold classifiers.

Next consider interval classifiers (Example~\ref{ex:intervals}), with $\Px$ uniform on $[0,1]$; in this case, we have $d=2$.
If $f = h_{[a,b]}$ for $a < b$, then again the $k=1$ round behaves essentially the same as \NaiveActivizer,
and since we have seen $\Px(\partial h_{[a,b]})=0$ in this case, we have $|\L_1| \gg n$.
However, the behavior becomes far more interesting when $f = h_{[a,a]}$, which was precisely the case that
prevented \NaiveActivizer~from improving over passive learning.  In this case, as we know from above,
the $k=1$ round will have $|\L_1| = O(n)$, so that we need to consider larger values of $k$ to identify
improvements.  In this case, the $k=2$ round behaves as follows.  With probability $1$, the initial $\lfloor n/3 \rfloor$
labels used to define $V$ will all be negative.  Thus, $V$ is precisely the set of intervals that do not contain any
of the initial $\lfloor n/3 \rfloor$ points.  Now consider any $S = \{x_1\} \in \X^{1}$, with $x_1$ not equal to any of these
initial $\lfloor n/3 \rfloor$ points, and consider any $x \notin \{x_1, X_1,\ldots,X_{\lfloor  n/3 \rfloor}\}$.
First note that $V$ shatters $S$, since we can optionally put a small interval around $x_1$ using an element of $V$.
If there is a point $x^{\prime}$ among the initial $\lfloor n/3 \rfloor$ \emph{between} $x$ and $x_1$, then any
$h_{[a,b]} \in V$ with $x \in [a,b]$ cannot also have $x_1 \in [a,b]$, as it would also contain the observed negative
point between them.  Thus, $V$ does \emph{not} shatter $\{x_1,x\} = S \cup \{x\}$, so that this $S$ will vote to
infer (rather than request) the label of $x$ in Step 6.  Furthermore, we see that $V[(x,+1)]$ does not shatter $S$,
while $V[(x,-1)]$ does shatter $S$, so that this $S$ would also vote for the label $\hat{y} = -1$ in Step 8.
For sufficiently large $n$, with high probability, any given $x$ not equal one of the initial $\lfloor n/3 \rfloor$
should have \emph{most} (probability at least $1-O(n^{-1} \log n)$) of the possible $x_1$ values separated from it by at least one of the initial
$\lfloor n/3 \rfloor$ points, so that the outcome of the vote in Step 6 will be a decision to infer (not request) the
label, and the vote in Step 8 will be for $-1$.  Since, with probability one, every $X_m \neq a$, we have every
$Y_m = -1$, so that every point in $\L_2$ is labeled correctly.  This also indicates that, for sufficiently large $n$,
we have $\Px(x : \Px^{1}(S \in \X^{1} : V \text{ shatters } S \cup \{x\} | V \text{ shatters } S) \geq 1/2) = 0$, so that
the size of $\L_2$ is only limited by the precision of estimation in $\hat{P}_{m_n}$ in Step 3.  Thus, as long as
we implement $\hat{P}_{m_n}$ so that its value is at most $o(1)$ larger than the true probability, we can guarantee
$|\L_2| \gg n$.

The unions of $i$ intervals example (Example~\ref{ex:unions-of-intervals}), again under $\Px$ uniform on $[0,1]$, is slightly more involved;
in this case, the appropriate value of $k$ to consider for any given target depends on the minimum number of intervals necessary
to represent the target function (up to probability-zero differences).  If $j$ intervals are required for this, then the
appropriate value is $k = i-j+1$.  Specifically, suppose the target is minimally representable
as a union of $j \in \{1,\ldots,i\}$ intervals of nonzero width: $[z_1,z_2] \cup [z_3,z_4] \cup \cdots \cup [z_{2j-1},z_{2j}]$:
that is, $z_1 < z_2 < \ldots < z_{2j-1} < z_{2j}$.
Every target in $\C$ has distance zero to some classifier of this type, and will agree with that classifier on all samples with probability one,
so we lose no generality by assuming all $j$ intervals have nonzero width.
Then consider any $x \in (0,1)$ separated from each of the $z_{p}$ values by at least one of the initial $\lfloor n/3 \rfloor$ points,
and not itself equal to one of those initial points.  Further consider any $S = \{x_1,\ldots,x_{i-j}\} \in \X^{i-j}$ such that,
between any pair of elements of $S \cup \{x\} \cup \{z_{1},\ldots,z_{2j}\}$, there is at least one of the initial $\lfloor n/3 \rfloor$ points.
First note that $V$ shatters $S$, since for any $x_{\ell}$ not in one of the $[z_{2p-1},z_{2p}]$ intervals (i.e., negative), we may optionally
add an interval $[x_{\ell},x_{\ell}]$ while staying in $V$, and for any $x_{\ell}$ in one of the $[z_{2p-1},z_{2p}]$ intervals (i.e., positive), we may
optionally split $[z_{2p-1},z_{2p}]$ into two intervals to barely exclude the point $x_{\ell}$ (and a small neighborhood around it), by adding
at most one interval to the representation; thus, in total we need to add at most $i-j$ intervals to the representation, so that the largest
number of intervals used by any of these $2^{i-j}$ classifiers involved in shattering is $i$, as required; furthermore, note that one of
these $2^{i-j}$ classifiers actually requires $i$ intervals.  Now for any
such $x$ and $S = \{x_1,\ldots,x_{i-j}\}$ as above, since one of the $2^{i-j}$ classifiers in $V$ used to shatter $S$ requires $i$ intervals
to represent it, and $x$ is separated from each element of $S \cup \{z_1,\ldots,z_{2j}\}$ by a labeled example, we see that $V$ cannot shatter $S \cup \{x\}$.
Furthermore, if $f(x) = y$, then the labeled examples to the immediate left and right of $x$ are also labeled $y$, and in particular
among the $2^{i-j}$ classifiers $h$ from $V$ that shatter $S$, the one $h$ that requires $i$ intervals to represent must also have $h(x) = y$,
so that $V[(x,-y)]$ does not shatter $S$.  Thus, any set $S$ satisfying this separation property will vote to infer (rather than request)
the label of $x$ in Step 6, and will vote for the label $f(x)$ in Step 8.
Furthermore, for sufficiently large $n$, for any given $x$ with the described property, with high probability most of the sets $S \in \X^{i-j}$
will satisfy this pairwise separation property, and therefore so will most of the shatterable sets $S \in \X^{i-j}$, so that the overall outcome
of the votes will favor inferring the label of $x$, and in particular inferring the label $f(x)$ for $x$.  On the other hand, for $x$ not satisfying this property
(i.e., not separated from some $z_p$ by any of the initial $\lfloor n/3 \rfloor$ examples), for any set $S$ as above, $V$ \emph{can} shatter $S \cup \{x\}$,
since we can optionally increase or decrease $z_p$ to include or disclude $x$ from the associated interval, in addition to optionally adding the extra intervals to
shatter $S$; therefore, by the same reasoning as above, for sufficiently large $n$, any such $x$ \emph{will} satisfy the condition in Step 6, and thus have
its label requested.  Thus, for sufficiently large $n$, every example in $\L_{i-j+1}$ will be labeled correctly.
Finally, note that with probability $1$, the set of points $x$ separated from each of the $z_p$ values by at least one of the $\lfloor n/3 \rfloor$ initial points
has probability approaching $1$ as $n \to \infty$, so that again 
%as long as the estimators $\hat{P}_{m_n}$ used have values at most $o(1)$ larger than the associated true probabilities, 
we have $|\L_{i-j+1}| \gg n$.

The above examples give some intuition about the operation of this procedure.
Next, we turn to general results showing that this type of improvement generally holds.

\subsection{General Results on Activized Learning}
\label{subsec:basic-activizer-theorems}

Returning to the abstract setting, we have the following general theorem, representing one of the main results of this paper.
Its proof is included in Appendix~\ref{app:activizer}.

\begin{theorem}
\label{thm:activizer}
%{\bf (First Main Result)}
For any VC class $\C$, \BasicActivizer~is a universal activizer for $\C$.
\thmend
\end{theorem}

This result is interesting both for its strength and generality.
Recall that it means that given any passive learning algorithm $\alg_p$,
the active learning algorithm obtained by providing $\alg_p$ as input to \BasicActivizer~achieves a label
complexity that strongly dominates that of $\alg_p$ for all nontrivial distributions $\Px$ and target functions $f \in \C$.
Results of this type were not previously known.
The specific technical advance over existing results (namely, those of \citet*{hanneke:10a})
is the fact that \BasicActivizer~has no direct dependence on the distribution $\Px$; as mentioned earlier,
the (very different) approach proposed by \citet*{hanneke:10a} has a strong direct dependence on the distribution, to the
extent that the distribution-dependence in that approach cannot be removed by merely replacing certain calculations
with data-dependent estimators (as we did in \BasicActivizer).
In the proof, we actually show a somewhat more general result:
namely, that \BasicActivizer~achieves these asymptotic improvements for any target function $f$
in the \emph{closure} of $\C$ (i.e., any $f$ such that $\forall r > 0, \Ball(f,r) \neq \emptyset$).

The following corollary is one concrete implication of Theorem~\ref{thm:activizer}.

\begin{corollary}
\label{cor:activized-1IG}
For any VC class $\C$,
there exists an active learning algorithm achieving a label complexity $\Lambda_a$ such that,
for all target functions $f \in \C$ and distributions $\Px$,
\begin{equation*}
\Lambda_a(\eps, f, \Px) = o(1/\eps).
\end{equation*}
\upthmend{-1.3cm}
\end{corollary}
\begin{proof}
The \emph{one-inclusion graph} passive learning algorithm of \citet*{haussler:94} is known
to achieve label complexity at most $d / \eps$, for every target function $f \in \C$ and distribution $\Px$.
Thus, Theorem~\ref{thm:activizer} implies that the (\BasicActivizer)-activized one-inclusion graph algorithm
satisfies the claim.
\end{proof}

As a byproduct, Theorem~\ref{thm:activizer} also establishes the basic fact that there \emph{exist} activizers.
In some sense, this observation opens up a new realm for exploration: namely, characterizing the \emph{properties} that activizers can possess.
This topic includes a vast array of questions, many of which deal with whether activizers are capable of \emph{preserving} various properties of the given passive algorithm
(e.g., margin-based dimension-independence, minimaxity, admissibility, etc.).  Section~\ref{sec:open-problems} describes a variety of enticing questions of this type.
In the sections below, we will consider quantifying how large the gap in label complexity between the given passive learning algorithm
and the resulting activized algorithm can be.  We will additionally study the effects of label noise on the possibility of activized learning.

\subsection{Implementation and Efficiency}
\label{subsec:efficiency}

\BasicActivizer~typically also has certain desirable efficiency guarantees.
Specifically, suppose that for any $m$ labeled examples $Q$, there is an algorithm with $\poly(d \cdot m)$ running time that
finds some $h \in \C$ with $\er_{Q}(h) = 0$ if one exists, and otherwise returns a value indicating that no such $h$
exists in $\C$; for many concept spaces with a kind of geometric interpretation, there are known methods with this capability \citep*{khachiyan:79,karmarkar:84,valiant:84,kearns:94}.
We can use such a subroutine to create an efficient implementation of the main body of \BasicActivizer.
Specifically, rather than explicitly representing $V$ in Step 1, we can simply store the set $Q_0 = \{(X_1,Y_1),\ldots,(X_{m_n},Y_{m_n})\}$.
Then for any step in the algorithm where we need to test whether $V$ shatters a set $R$, we can simply try all $2^{|R|}$ possible
labelings of $R$, and for each one temporarily add these $|R|$ additional labeled examples to $Q_0$ and check whether there is an $h \in \C$
consistent with all of the labels.  At first, it might seem that these $2^k$ evaluations would be prohibitive; however,
supposing $\hat{P}_{m_n}$ is implemented so that it is $\Omega(1/\poly(n))$ (as it is in Appendix~\ref{app:hatP-definitions}),
note that the loop beginning at Step 5 executes a nonzero number of times only if $n/ \hat{\Delta}^{(k)} > 2^k$, so that $2^k \leq \poly(n)$;
we can easily add a condition that skips the step of calculating $\hat{\Delta}^{(k)}$ if $2^k$ exceeds this $\poly(n)$ lower bound on
$n / \hat{\Delta}^{(k)}$, so that even those shatterability tests can be skipped in this case.
Thus, for the actual occurrences of it in the algorithm, testing whether $V$ shatters $R$ requires only $\poly(n) \cdot \poly(d \cdot (|Q_0|+|R|))$ time.
The total number of times this test is performed in calculating $\hat{\Delta}^{(k)}$ (from Appendix~\ref{app:hatP-definitions}) is itself only $\poly(n)$,
and the number of iterations of the loop in Step 5 is at most $n / \hat{\Delta}^{(k)} = \poly(n)$.
Determining the label $\hat{y}$ in Step 8 can be performed in a similar fashion.
So in general, the total running time of the main body of \BasicActivizer~is $\poly(d \cdot n)$.

The only remaining question is the efficiency of the final step.
Of course, we can require $\alg_p$ to have running time polynomial in the size of its input set (and $d$).
But beyond this, we must consider the efficiency of the $\ActiveSelect$ subroutine.
This actually turns out to have some subtleties involved.
The way it is stated above is simple and elegant, but not always efficient.  Specifically,
we have no a priori bound on the number of unlabeled examples the algorithm must process before finding
a point $X_m$ where $h_j(X_m) \neq h_k(X_m)$.  Indeed, if $\Px(x : h_j(x) \neq h_k(x)) = 0$, we may effectively
need to examine the entire infinite sequence of $X_m$ values to determine this.
Fortunately, these problems can be corrected without difficulty, simply by truncating the search at a
predetermined number of points.  Specifically, rather than taking the next $\lfloor m / \binom{N}{2} \rfloor$
examples for which $h_j$ and $h_k$ disagree, simply restrict ourselves to at most this number, or at most
the number of such points among the next $M$ unlabeled examples.  In Appendix~\ref{app:activizer}, we
show that $\ActiveSelect$, as originally stated, has a high-probability ($1-\exp\{- \Omega(m)\}$) guarantee
that the classifier it selects has error rate at most twice the best of the $N$ it is given.  With the modification
to truncate the search at $M$ unlabeled examples, this guarantee is increased to
$\min_{k} \er(h_k) + \max\{\er(h_k), m / M\}$.  For the concrete guarantee of Corollary~\ref{cor:activized-1IG},
it suffices to take $M \gg m^2$.  However, to guarantee the modified $\ActiveSelect$ can still be
used in \BasicActivizer~while maintaining (the stronger) Theorem~\ref{thm:activizer}, we need $M$ at least as big as
$\Omega\left(\min\left\{\exp\left\{ m^{c} \right\}, m / \min_k \er(h_k)\right\}\right)$,
for any constant $c > 0$.  In general, if we have a $1 / \poly(n)$ lower bound on the error rate
of the classifier produced by $\alg_p$ for a given number of labeled examples as input,
we can set $M$ as above using this lower bound in place of $\min_k \er(h_k)$, resulting in an efficient
version of $\ActiveSelect$ that still guarantees Theorem~\ref{thm:activizer}.  However, it is presently
not known whether there always exist universal activizers that are efficient
(either $\poly(d\cdot n)$ or $\poly(d/\eps)$ running time) when the above assumptions on
efficiency of $\alg_{p}$ and finding $h \in \C$ with $\er_{Q}(h)=0$ hold.

\section{The Magnitudes of Improvements}
\label{sec:exponential}

In the previous section, we saw that we can always improve the label complexity of a passive learning algorithm by activizing it.
However, there remains the question of how large the gap is between the passive algorithm's label complexity and the activized algorithm's label complexity.
In the present section, we refine the above procedures, to take greater advantage of the sequential nature of active learning.
For each, we characterize the improvements it achieves relative to any given passive algorithm.

As a byproduct, this provides concise sufficient conditions for \emph{exponential} gains, addressing an open problem of \citet*{hanneke:10a}.
Specifically, consider the following definition, essentially similar to one explored by \citet*{hanneke:10a}.

\begin{definition}
\label{defn:exponential}
For a concept space $\C$ and distribution $\Px$, we say that $(\C,\Px)$ is \emph{learnable at an exponential rate} if
there exists an active learning algorithm achieving label complexity $\Lambda$ such that
$\forall f \in \C$, $\Lambda(\eps,f,\Px) \in \Polylog(1/\eps)$.
We further say $\C$ is learnable at an exponential rate if there exists an active learning algorithm achieving
label complexity $\Lambda$ such that for all distributions $\Px$ and all $f \in \C$,
$\Lambda(\eps,f,\Px) \in \Polylog(1/\eps)$.
\thmend
\end{definition}

\subsection{The Label Complexity of Disagreement-Based Active Learning}
\label{subsec:disagreement-coefficient}

As before, to establish a foundation to build upon, we begin by studying the label complexity gains achievable by disagreement-based active learning.
From above, we already know that disagreement-based active learning is not sufficient to achieve the best possible
gains; but as before, it will serve as a suitable starting place to gain intuition for how we might approach the problem of
improving \BasicActivizer~and quantifying the improvements achievable over passive learning by the resulting more
sophisticated methods.

The results on disagreement-based learning in this subsection are essentially already known,
and available in the published literature (though in a slightly less general form).
Specifically, we review (a modified version of) the method of \citet*{cohn:94}, referred to as \CAL~below,
which was historically the original disagreement-based active learning algorithm.
We then state the known results on the label complexities achievable by this method, in terms
of a quantity known as the disagreement coefficient; that result is due to \citet*{hanneke:11a,hanneke:07b}.

\subsubsection{The CAL Active Learning Algorithm}

To begin, we consider the following simple disagreement-based method, typically referred to as CAL after its discoverers \citet*{cohn:94},
though the version here is slightly modified compared to the original (see below).
It essentially represents a refinement of \NaiveActivizer~to take greater advantage of the sequential aspects of active learning.  That is,
rather than requesting only two batches of labels, as in \NaiveActivizer, this method updates the version space after every label request,
thus focusing the region of disagreement (and therefore the region in which it requests labels) 
after each label request.

\begin{bigboxit}
\CAL \\
Input: passive algorithm $\alg_p$, label budget $n$\\
Output: classifier $\hat{h}$\\
{\vskip -2mm}\line(1,0){419}\\
0.\phantom{0} $V \gets \C$, $t \gets 0$, $m \gets 0$, $\L \gets \{\}$\\
1.\phantom{0} While $t < \lceil n/2 \rceil$ and $m \leq 2^{n}$\\
2.\phantom{0} \quad $m \gets m+1$\\
3.\phantom{0} \quad If $X_m \in \DIS(V)$ \\
4.\phantom{0} \qquad Request the label $Y_{m}$ of $X_{m}$ and let $t \gets t+1$\\
5.\phantom{0} \qquad Let $V \gets V[(X_{m},Y_{m})]$\\
6.\phantom{0} Let $\hat{\Delta} \gets \hat{P}_{m}(\DIS(V))$ \\
7.\phantom{0} Do $\lfloor n / (6 \hat{\Delta}) \rfloor$ times\\
8.\phantom{0} \quad $m \gets m+1$\\
9.\phantom{0} \quad If $X_m \in \DIS(V)$ and $t < n$\\
10.                   \qquad Request the label $Y_{m}$ of $X_{m}$ and let $\hat{y} \gets Y_{m}$ and $t \gets t+1$\\
11.                   \quad Else let $\hat{y} = h(X_{m})$ for an arbitrary $h \in V$\\
12.                   \quad Let $\L \gets \L \cup \{(X_{m},\hat{y})\}$ and $V \gets V[(X_{m},\hat{y})]$ \\
13.                   Return $\alg_{p}(\L)$
\end{bigboxit}

The procedure is specified in terms of an estimator $\hat{P}_{m}$; for our purposes, we define this as in \eqref{eqn:hatPn3} of Appendix~\ref{app:hatP-definitions} (with $k=1$ there).
Every example $X_{m}$ added to the set $\L$ in Step 12 either has its label requested (Step 10) or inferred (Step 11).
By the same Chernoff bound argument
mentioned for the previous methods, we are guaranteed (with high probability) that the ``$t < n$'' constraint in Step 9
is always satisfied when $X_{m} \in \DIS(V)$.  Since we assume $f \in \C$, an inductive argument shows
that we will always have $f \in V$ as well; thus, every label requested \emph{or} inferred will agree with $f$, and therefore the
labels in $\L$ are all correct.

As with \NaiveActivizer, this method has two stages to it: one in which we focus on reducing the version space $V$,
and a second in which we focus on constructing a set of labeled examples to feed into the passive algorithm.
The original algorithm of \citet*{cohn:94} essentially used only the first stage, and simply returned any classifier in $V$ after exhausting its
budget for label requests.  Here we have added the second stage (Steps 6-13) so that we can guarantee a certain conditional independence
(given $|\L|$) among the examples fed into the passive algorithm, which is important for the general results (Theorem~\ref{thm:cal} below).
\citet*{hanneke:11a} showed that the original (simpler) algorithm achieves the (less general) label complexity bound of Corollary~\ref{cor:cal} below.

\subsubsection{Examples}

Not surprisingly, by essentially the same argument as \NaiveActivizer, one can show \CAL~satisfies the claim in Theorem~\ref{thm:naive}.
That is, \CAL~is a universal activizer for $\C$ if and only if $\Px(\partial f) = 0$ for every $\Px$ and $f \in \C$.
However, there are further results known on the label complexity achieved by \CAL.
Specifically, to illustrate the types of improvements achievable by \CAL, consider our usual toy examples;
as before, to simplify the explanation, for these examples we ignore the fact that $\hat{P}_{m}$ is only an estimate, as well as the ``$t < n$'' constraint in Step 9
(both of which will be addressed in the general results below).

First, consider threshold classifiers (Example~\ref{ex:thresholds}) under a uniform $\Px$ on $[0,1]$, and
suppose $f = h_{z} \in \C$.  Suppose the given passive algorithm has label complexity $\Lambda_{p}$.
To get expected error at most $\eps$ in \CAL, it suffices to have $|\L| \geq \Lambda_{p}(\eps/2,f,\Px)$
with probability at least $1-\eps/2$.  
Starting from any particular $V$ set obtained in the algorithm, call it $V_0$,
the set $\DIS(V_0)$ is simply the region between the largest negative example
observed so far (say $z_{\ell}$) and the smallest positive example observed so far (say $z_{r}$).
With probability at least $1-\eps/n$, at least one of the next $O(\log(n/\eps))$ examples in this $[z_{\ell},z_{r}]$
region will be in $[z_{\ell} + (1/3)(z_{r}-z_{\ell}), z_{r} - (1/3)(z_{r}-z_{\ell})]$, so that after processing
that example, we definitely have $\Px(\DIS(V)) \leq (2/3) \Px(\DIS(V_0))$.  Thus, upon reaching
Step 6, since we have made $n/2$ label requests, a union bound implies that with probability $1-\eps / 2$,
we have $\Px(\DIS(V)) \leq \exp\{- \Omega(n / \log(n/\eps))\}$,
and therefore $|\L| \geq \exp\{ \Omega(n / \log(n/\eps))\}$.
Thus, for some value
$\Lambda_{a}(\eps,f,\Px) = O(\log(\Lambda_{p}(\eps/2,f,\Px)) \log ( \log(\Lambda_{p}(\eps/2,f,\Px))/\eps))$,
any $n \geq \Lambda_{a}(\eps,f,\Px)$ gives
$|\L| \geq \Lambda_{p}(\eps/2,f,\Px)$ with probability at least $1-\eps/2$,
so that the activized algorithm achieves label complexity $\Lambda_{a}(\eps,f,\Px) \in \Polylog( \Lambda_{p}(\eps/2,f,\Px)/\eps)$.

Consider also the intervals problem (Example~\ref{ex:intervals}) under a uniform $\Px$ on $[0,1]$, and suppose $f = h_{[a,b]} \in \C$,
for $b > a$.  In this case, as with any disagreement-based algorithm, until the algorithm observes the first positive example (i.e., the first $X_{m} \in [a,b]$),
it will request the label of every example (see the reasoning above for \NaiveActivizer).  However, at every time after observing this first positive point, say $x$,
the region $\DIS(V)$ is restricted to the region between the largest negative point less than $x$ and smallest positive point, and the region between
the largest positive point and the smallest negative point larger than $x$.  For each of these two regions, the same arguments used for the threshold
problem above can be applied to show that, with probability $1-O(\eps)$, the region of disagreement is reduced by at least a constant fraction every $O(\log(n/\eps))$ label requests,
so that $|\L| \geq \exp\{\Omega(n / \log(n/\eps))\}$.  Thus, again the label complexity is of the form
$O(\log(\Lambda_{p}(\eps/2,f,\Px)) \log ( \log(\Lambda_{p}(\eps/2,f,\Px))/\eps))$, which is $\Polylog(\Lambda_{p}(\eps/2,f,\Px)/\eps)$,
though this time there is a significant (additive) target-dependent constant (roughly $\propto \frac{1}{b-a} \log(1/\eps)$), accounting for the length of the initial phase before observing any positive examples.
On the other hand, as with \emph{any} disagreement-based algorithm, when $f = h_{[a,a]}$, because the algorithm never observes a positive example,
it requests the label of every example it considers; in this case, by the same argument given for \NaiveActivizer, upon reaching Step 6 we have
$\Px(\DIS(V)) = 1$, so that $|\L| = O(n)$, and we observe no improvements for some passive algorithms $\alg_{p}$.

A similar analysis can be performed for unions of $i$ intervals under $\Px$ uniform on $[0,1]$.
In that case, we find that any $h_{\mathbf{z}} \in \C$ not representable (up to probability-zero differences) by a union of $i-1$ or fewer intervals
allows for the exponential improvements of the type observed in the previous two examples; this time, the phase of exponentially decreasing
$\Px(\DIS(V))$ only occurs after observing an example in each of the $i$ intervals and each of the $i-1$ negative regions separating the intervals,
resulting in an additive term of roughly $\propto \frac{1}{\min_{1 \leq j < 2i} z_{j+1}-z_{j}}\log(i/\eps)$ in the label complexity.
However, any $h_{\mathbf{z}} \in \C$ representable (up to probability-zero differences) by a union of $i-1$ or fewer intervals
has $\Px(\partial h_{\mathbf{z}}) = 1$, which means $|\L| = O(n)$, and therefore (as with any disagreement-based algorithm) \CAL~will not provide improvements for some passive algorithms $\alg_p$.

\subsubsection{The Disagreement Coefficient}

Toward generalizing the arguments from the above examples, consider the following definition of \citet*{hanneke:07b}. 

\begin{definition}
\label{def:disagreement-coefficient}
For $\eps \geq 0$, the \emph{disagreement coefficient} of a classifier $f$ with respect to a concept space $\C$ under a distribution $\Px$
is defined as
\begin{equation*}
\dc_{f}(\eps) = 1\lor \sup\limits_{r > \eps} \frac{\Px\left(\DIS(\Ball(f,r))\right)}{r}.
\end{equation*}
Also abbreviate $\dc_{f} = \dc_{f}(0)$.
\thmend
\end{definition}

Informally, the disagreement coefficient describes the rate of collapse of the region of disagreement,
relative to the distance from $f$.
It has been useful in characterizing the label complexities achieved by several disagreement-based
active learning algorithms \citep*{hanneke:07b,hanneke:11a,dasgupta:07,beygelzimer:09,wang:09,koltchinskii:10,beygelzimer:10},
and itself has been studied and bounded for various families of learning problems
\citep*{hanneke:07b,hanneke:11a,hanneke:10a,friedman:09,beygelzimer:09,mahalanabis:11,wang:11}.
See the paper of \citet*{hanneke:11a} for a detailed discussion of the disagreement coefficient,
including its relationships to several related quantities, as well as a variety of properties
that it satisfies that can help to bound its value for any given learning problem.
In particular, below we use the fact that, for any constant $c \in [1,\infty)$,
$\dc_{f}(\eps) \leq \dc_{f}(\eps / c) \leq c \dc_{f}(\eps)$.
Also note that $\Px(\partial f) = 0$ if and only if $\dc_{f}(\eps) = o(1/\eps)$.
See the papers of \citet*{friedman:09,mahalanabis:11} for some general conditions on $\C$ and $\Px$,
under which every $f \in \C$ has $\dc_{f} < \infty$, which (as we explain below) has particularly
interesting implications for active learning \citep*{hanneke:07b,hanneke:11a}.

To build intuition about the behavior of the disagreement coefficient,
we briefly go through its calculation for our usual toy examples from above.
The first two of these calculations are taken from \citet*{hanneke:07b}, and the last is from \citet*{hanneke:10a}.
First, consider the thresholds problem (Example~\ref{ex:thresholds}), and for simplicity suppose the distribution $\Px$
is uniform on $[0,1]$.  In this case, as in Section~\ref{subsec:core}, $\Ball(h_{z},r) = \{h_{z^{\prime}} \in \C : |z^{\prime} - z| \leq r \}$,
and $\DIS(\Ball(h_{z},r)) \subseteq [z-r,z+r)$ with equality for sufficiently small $r$.  Therefore, $\Px(\DIS(\Ball(h_{z},r))) \leq 2r$ (with
equality for small $r$), and $\dc_{h_{z}}(\eps) \leq 2$ with equality for sufficiently small $\eps$.  In particular, $\dc_{h_{z}} = 2$.

On the other hand, consider the intervals problem (Example~\ref{ex:intervals}), again under $\Px$ uniform on $[0,1]$.
This time, for $h_{[a,b]} \in \C$ with $b-a > 0$, we have for $0 < r < b-a$, $\Ball(h_{[a,b]},r) = \{h_{[a^{\prime},b^{\prime}]} \in \C : |a-a^{\prime}|+|b-b^{\prime}| \leq r\}$,
$\DIS(\Ball(h_{[a,b]},r)) \subseteq [a-r,a+r) \cup (b-r,b+r]$, and $\Px(\DIS(\Ball(h_{[a,b]},r))) \leq 4r$ (with equality for sufficiently small $r$).
But for $0 < b-a \leq r$, we have $\Ball(h_{[a,b]},r) \supseteq \{h_{[a^{\prime},a^{\prime}]} : a^{\prime} \in (0,1)\}$,
so that $\DIS(\Ball(h_{[a,b]},r)) = (0,1)$ and $\Px(\DIS(\Ball(h_{[a,b]},r))) = 1$.  Thus, we generally have
$\dc_{h_{[a,b]}}(\eps) \leq \max\left\{\frac{1}{b-a},4\right\}$, with equality for sufficiently small $\eps$.
However, this last reasoning also indicates $\forall r > 0, \Ball(h_{[a,a]},r) \supseteq \{h_{[a^{\prime},a^{\prime}]} : a^{\prime} \in (0,1)\}$,
so that $\DIS(\Ball(h_{[a,a]},r)) = (0,1)$ and $\Px(\DIS(\Ball(h_{[a,a]},r))) = 1$; therefore, $\dc_{h_{[a,a]}}(\eps) = \frac{1}{\eps}$,
the largest possible value for the disagreement coefficient; in particular, this also means $\dc_{h_{[a,a]}} = \infty$.

Finally, consider the unions of $i$ intervals problem (Example~\ref{ex:unions-of-intervals}), again under $\Px$ uniform on $[0,1]$.
First take any $h_{\mathbf{z}} \in \C$ such that any $h_{\mathbf{z}^{\prime}} \in \C$ representable as a union of $i-1$ intervals has
$\Px(\{x : h_{\mathbf{z}}(x) \neq h_{\mathbf{z}^{\prime}}(x)\}) > 0$.  Then for $0 < r < \min\limits_{1 \leq j < 2i} z_{j+1} - z_{j}$,
$\Ball(h_{\mathbf{z}},r) = \{h_{\mathbf{z}^{\prime}} \in \C : \sum\limits_{1 \leq j \leq 2i} |z_j - z_j^{\prime}| \leq r\}$, so that $\Px(\DIS(\Ball(h_{\mathbf{z}},r))) \leq 4 i r$, with equality for sufficiently small $r$.
For $r > \min\limits_{1 \leq j < 2i} z_{j+1} - z_{j}$, $\Ball(h_{\mathbf{z}},r)$ contains a set of classifiers that flips the labels (compared to $h_{\mathbf{z}}$) in that smallest region
and uses the resulting extra interval to disagree with $h_{\mathbf{z}}$ on a tiny region at an arbitrary location (either by encompassing some point with a small interval, 
or by splitting an interval into two intervals separated by a small gap).  
Thus, $\DIS(\Ball(h_{\mathbf{z}},r)) = (0,1)$, and $\Px(\DIS(h_{\mathbf{z}},r)) = 1$.
So in total, $\dc_{h_{\mathbf{z}}}(\eps) \leq \max\left\{\frac{1}{\min\limits_{1 \leq j < 2i} z_{j+1}-z_{j}}, 4 i\right\}$, with equality for sufficiently small $\eps$.
On the other hand, if $h_{\mathbf{z}} \in \C$ can be represented by a union of $i-1$ (or fewer) intervals,
then we can use the extra interval to disagree with $h_{\mathbf{z}}$ on a tiny region at an arbitrary location, while still remaining in $\Ball(h_{\mathbf{z}},r)$, so that $\DIS(\Ball(h_{\mathbf{z}},r)) = (0,1)$,
$\Px(\DIS(\Ball(h_{\mathbf{z}},r))) = 1$, and $\dc_{h_{\mathbf{z}}}(\eps) = \frac{1}{\eps}$; in particular, in this case we have $\dc_{h_{\mathbf{z}}} = \infty$.

\subsubsection{General Upper Bounds on the Label Complexity of \CAL}

As mentioned, the disagreement coefficient has implications for the label complexities achievable by disagreement-based active learning.
The intuitive reason for this is that, as the number of label requests increases, the \emph{diameter} of the version space shrinks at a predictable rate.
The disagreement coefficient then relates the diameter of the version space to the size of its region of disagreement, which in turn describes
the probability of requesting a label.  Thus, the expected frequency of label requests in the data sequence decreases at a predictable
rate related to the disagreement coefficient, so that $|\L|$ in \CAL~can be lower bounded by a function of the disagreement coefficient.
Specifically, the following result was essentially established by \citet*{hanneke:11a,hanneke:07b}, though
actually the result below is slightly more general than the original.

\begin{theorem}
\label{thm:cal}
For any VC class $\C$,
and any passive learning algorithm $\alg_p$ achieving label complexity $\Lambda_{p}$,
the active learning algorithm obtained by applying \CAL~with $\alg_{p}$ as input achieves a
label complexity $\Lambda_{a}$ that, for any distribution $\Px$ and classifier $f \in \C$,
satisfies
\begin{equation*}
\Lambda_{a}(\eps,f,\Px) = O\left(\dc_{f}\left(\Lambda_{p}(\eps/2,f,\Px)^{-1} \right) \log^{2} \frac{ \Lambda_{p}(\eps/2, f, \Px) }{\eps}\right).
\end{equation*}
\upthmend{-1.21cm}
\end{theorem}

The proof of Theorem~\ref{thm:cal} is similar to the original result of \citet*{hanneke:11a,hanneke:07b},
with only minor modifications to account for using $\alg_p$ instead of returning an arbitrary element of $V$.
The formal details are implicit in the proof of Theorem~\ref{thm:sequential-activizer} below
(%replacing ``$\bdim_f$'' with ``$1$'',
since \CAL~is essentially identical\ignore{, except for a few constants,} to the $k=1$ round of \Shattering, defined below).
We also have the following simple corollaries.

\begin{corollary}
\label{cor:cal}
For any VC class $\C$, there exists a passive learning algorithm $\alg_p$ such that,
for every $f \in \C$ and distribution $\Px$,
the active learning algorithm obtained by applying \CAL~with $\alg_p$ as input
achieves label complexity
\begin{equation*}
\Lambda_{a}(\eps,f,\Px) = O\left(\dc_{f}(\eps) \log^{2}\left( 1 / \eps \right) \right).
\end{equation*}
\upthmend{-1.25cm}
\end{corollary}
\begin{proof}
The one-inclusion graph algorithm of \citet*{haussler:94} is a passive learning algorithm achieving label complexity
$\Lambda_{p}(\eps,f,\Px) \leq d/\eps$.  Plugging this into Theorem~\ref{thm:cal}, using the fact that
$\dc_{f}(\eps/2d) \leq 2d \dc_{f}(\eps)$, and simplifying, we arrive at the result.
In fact, we will see in the proof of Theorem~\ref{thm:sequential-activizer} that incurring this extra constant factor of $d$ is not actually necessary.
\end{proof}

\begin{corollary}
\label{cor:cal-exponential}
For any VC class $\C$ and distribution $\Px$, if $\forall f \in \C$, $\dc_f < \infty$,
then $(\C,\Px)$ is learnable at an exponential rate.  If this is true for all $\Px$, then $\C$ is learnable at an exponential rate.
\thmend
\end{corollary}
\begin{proof}
The first claim follows directly from Corollary~\ref{cor:cal}, since $\dc_{f}(\eps) \leq \dc_{f}$.
The second claim then follows from the fact that \CAL~is adaptive to $\Px$ (has no direct dependence on $\Px$ except
via the data).
\end{proof}

Aside from the disagreement coefficient and $\Lambda_p$ terms, the other constant factors hidden in the big-O in Theorem~\ref{thm:cal}
are only $\C$-dependent (i.e., independent of $f$ and $\Px$).
As mentioned, if we are only interested in achieving the label complexity bound of Corollary~\ref{cor:cal}, 
we can obtain this result more directly by the simpler
original algorithm of \citet*{cohn:94} via the analysis of \citet*{hanneke:11a,hanneke:07b}.

\subsubsection{General Lower Bounds on the Label Complexity of \CAL}
\label{subsubsec:cal-lower}

It is also possible to prove a kind of \emph{lower bound} on the label complexity of
\CAL~in terms of the disagreement coefficient,
so that the dependence on the disagreement coefficient in Theorem~\ref{thm:cal} is unavoidable.
Specifically, there are two simple observations that intuitively explain the possibility of such lower bounds.
The first observation is that the expected number of label requests \CAL~makes
among the first $\lceil 1/r \rceil$ unlabeled examples is at least $\Px(\DIS(\Ball(f,r))) / (2r)$ (assuming it does not halt first).
Similarly, the second observation is that, to arrive at a
region of disagreement with expected probability mass less than
$\Px(\DIS(\Ball(f,r))) / 2$, \CAL~requires a budget $n$ of size at least $\Px(\DIS(\Ball(f,r))) / (2r)$.
These observations are formalized in Appendix~\ref{app:cal} 
as Lemmas~\ref{lem:cal-queries-lower} and \ref{lem:cal-dis-lower}.
Noting that, for unbounded $\dc_{f}(\eps)$, $\Px(\DIS(\Ball(f,\eps))) / \eps \neq o\left(\dc_{f}(\eps)\right)$,
the relevance of these observations in the context of deriving lower bounds based on the disagreement coefficient becomes clear.
In particular, we can use the latter of these insights to arrive at the following theorem,
which essentially complements Theorem~\ref{thm:cal},
showing that it cannot generally be improved 
beyond reducing the constants and logarithmic factors,
without altering the algorithm or
introducing additional $\alg_p$-dependent quantities in the label complexity bound.
The proof is included in Appendix~\ref{app:cal}.

\begin{theorem}
\label{thm:cal-lower}
For any set of classifiers $\C$, $f \in \C$, distribution $\Px$,
and nonincreasing 
function $\lambda : (0,1) \to \nats$,
there exists a passive learning algorihtm $\alg_p$
achieving a label complexity $\Lambda_{p}$ with $\Lambda_{p}(\eps,f,\Px) = \lambda(\eps)$
for all $\eps > 0$,
such that if \CAL, with $\alg_p$ as its argument, achieves label complexity $\Lambda_a$,
then
\begin{equation*}
\Lambda_{a}(\eps,f,\Px) \neq o\left( \dc_{f}\left( \Lambda_{p}(2 \eps, f, \Px)^{-1}\right) \right).
\end{equation*}
\upthmend{-1.25cm}
\end{theorem}

Recall that there are many natural learning problems for which $\dc_{f} = \infty$,
and indeed where $\dc_{f}(\eps) = \Omega(1/\eps)$:
for instance, intervals with $f = h_{[a,a]}$ under uniform $\Px$,
or unions of $i$ intervals under uniform $\Px$ with $f$ representable as $i-1$ or fewer intervals.
Thus, since we have just seen that the improvements gained by disagreement-based methods are
well-characterized by the disagreement coefficient,
if we would like to achieve exponential improvements over passive learning for these problems,
we will need to move beyond these disagreement-based methods.
In the subsections that follow, we will use an alternative algorithm and analysis, and prove a general result that is always
at least as good as Theorem~\ref{thm:cal} (in a big-O sense), and often significantly better (in a little-o sense).  In particular, it leads to a sufficient
condition for learnability at an exponential rate, strictly more general than that of Corollary~\ref{cor:cal-exponential}.

\subsection{An Improved Activizer}
\label{subsec:sequential-activizer}

In this subsection, we define a new active learning method based on shattering, as in \BasicActivizer,
but which also takes fuller advantage of the sequential aspect of active learning, as in \CAL.
We will see that this algorithm can be analyzed in a manner analogous to the disagreement coefficient
analysis of \CAL, leading to a new and often dramatically-improved label complexity bound.
Specifically, consider the following meta-algorithm.

\begin{bigboxit}
\Shattering \\ 
Input: passive algorithm $\alg_p$, label budget $n$\\
Output: classifier $\hat{h}$\\
{\vskip -2mm}\line(1,0){419}\\
0.\phantom{0} $V \gets V_{0} = \C$, $T_0 \gets \lceil 2n / 3 \rceil$, $t \gets 0$, $m \gets 0$\\
1.\phantom{0} For $k = 1,2,\ldots,d+1$\\
2.\phantom{0} \quad Let $\L_{k} \gets \{\}$, $T_k \gets T_{k-1} - t$, and let $t \gets 0$\\
3.\phantom{0} \quad While $t < \lceil T_k / 4 \rceil$ and $m \leq k \cdot 2^{n}$ \\
4.\phantom{0} \qquad $m \gets m + 1$ \\
5.\phantom{0} \qquad If $\hat{P}_{m}\left(S \in \X^{k-1} : V \text{ shatters } S \cup \{X_{m}\} | V \text{ shatters } S\right) \geq 1/2$\\
6.\phantom{0} \qquad\quad Request the label $Y_{m}$ of $X_{m}$, and let $\hat{y} \gets Y_{m}$ and $t \gets t+1$\\
7.\phantom{0} \qquad Else let
$\hat{y} \gets\!\! \argmax\limits_{y \in \{-1,+1\}} \!\hat{P}_{m}\!\left(S \in \X^{k-1} \!:\! V[(X_{m},-y)] \text{ does not shatter } S | V \text{ shatters } S \right)$\\
8.\phantom{0} \qquad Let $V \gets V_{m} = V_{m-1}\left[\left(X_{m}, \hat{y}\right)\right]$\\
9.\phantom{0} \quad $\hat{\Delta}^{(k)} \gets \hat{P}_{m}\left(x : \hat{P}\left(S \in \X^{k-1} : V \text{ shatters } S \cup \{x\} | V \text{ shatters } S\right) \geq 1/2 \right)$\\
10. \quad Do $\lfloor T_k / (3 \hat{\Delta}^{(k)}) \rfloor$ times\\
11. \qquad $m \gets m+1$ \\
12. \qquad If $\hat{P}_{m}\left(S \in \X^{k-1} : V \text{ shatters } S \cup \{X_{m}\} | V \text{ shatters } S\right) \geq 1/2$ and $t < \lfloor 3 T_k / 4 \rfloor$ \\
13. \qquad\quad Request the label $Y_{m}$ of $X_{m}$, and let $\hat{y} \gets Y_{m}$ and $t \gets t+1$\\
14. \qquad Else, let
$\hat{y} \gets\!\! \argmax\limits_{y \in \{-1,+1\}} \!\hat{P}_{m}\!\left(S \in \X^{k-1} \!:\! V[(X_{m},-y)] \text{ does not shatter } S | V \text{ shatters } S \right)$\\
15. \qquad Let $\L_k \gets \L_k \cup \left\{\left(X_{m},\hat{y}\right)\right\}$ and $V \gets V_{m} = V_{m-1}\left[\left(X_{m}, \hat{y}\right)\right]$\\
16. Return $\ActiveSelect(\{\alg_p(\L_1),\alg_p(\L_2),\ldots,\alg_p(\L_{d+1})\},\lfloor n/3 \rfloor, \{X_{m+1},X_{m+2},\ldots\})$
\end{bigboxit}

As before, the procedure is specified in terms of estimators $\hat{P}_{m}$.
Again, these can be defined in a variety of ways, as long as they
converge (at a fast enough rate) to their respective true probabilities.
For the results below, we will use the definitions given in Appendix~\ref{app:hatP-definitions}:
i.e., the same definitions used in \BasicActivizer.  
Following the same argument as for \BasicActivizer, one can show that \Shattering~is a universal activizer for $\C$, for any VC class $\C$.
However, we can also obtain more detailed results in terms of a generalization of the disagreement coefficient given below.

As with \BasicActivizer, this procedure has three main components: one in which we focus on reducing the
version space $V$, one in which we focus on collecting a (conditionally) i.i.d. sample to feed into $\alg_{p}$, and one
in which we select from among the $d+1$ executions of $\alg_{p}$.
However, unlike \BasicActivizer, here the first stage is also broken up based on the
value of $k$, so that each $k$ has its own first and second stages, rather than sharing
a single first stage.  Again, the choice of the number of (unlabeled) examples processed in each second stage
guarantees (by a Chernoff bound) that the ``$t < \lfloor 3T_k / 4 \rfloor$'' constraint in Step 12 is redundant.
Depending on the type of label complexity result we wish to prove, this multistage architecture is sometimes avoidable.
In particular, as with Corollary~\ref{cor:cal} above, to directly achieve the label complexity bound in
Corollary~\ref{cor:sequential-activizer} below, we can use a much simpler approach that replaces Steps 9-16,
instead simply returning an arbitrary element of $V$ upon termination.

Within each value of $k$, \Shattering~behaves analogous to \CAL,
requesting the label of an example only if it cannot infer the label from known information, and updating the version
space $V$ after every label request; however, unlike \CAL, for values of $k > 1$, the mechanism for inferring
a label is based on shatterable sets, as in \BasicActivizer, and is motivated by the same argument of splitting
$V$ into subsets containing arbitrarily good classifiers (see the discussion in Section~\ref{subsec:basic-activizer}).
Also unlike \CAL, even the inferred labels can be used to reduce the set $V$ (Steps 8 and 15), since they are not only correct
but also potentially informative in the sense that $x \in \DIS(V)$.
As with \BasicActivizer, the key to obtaining improvement guarantees is that some value of $k$ has $|\L_{k}| \gg n$,
while maintaining that all of the labels in $\L_{k}$ are correct; $\ActiveSelect$ then guarantees the overall performance
is not too much worse than that obtained by $\alg_{p}(\L_{k})$ for this value of $k$.

To build intuition about the behavior of \Shattering, let us consider our usual toy examples,
again under a uniform distribution $\Px$ on $[0,1]$; as before, for simplicity we ignore the fact that $\hat{P}_{m}$
is only an estimate, as well as the constraint on $t$ in Step 12 and the effectiveness of $\ActiveSelect$,
all of which will be addressed in the general analysis.
First, for the behavior of the algorithm for thresholds and nonzero-width intervals, we may simply
refer to the discussion of \CAL, since the $k=1$ round of \Shattering~is essentially identical
to \CAL; in this case, we have already seen that $|\L_{1}|$ grows as $\exp\{\Omega(n / \log(n/\eps))\}$ for thresholds, and
does so for nonzero-width intervals after some initial period of slow growth related to the width of the target interval
(i.e., the period before finding the first positive example).  As with \BasicActivizer, for zero-width intervals,
we must look to the $k=2$ round of \Shattering~to find improvements.  Also as with \BasicActivizer,
for sufficiently large $n$, every $X_{m}$ processed in the $k=2$ round will have its label inferred (correctly)
in Step 7 or 14 (i.e., it does not request any labels).  But this means we reach Step 9 with $m = 2 \cdot 2^{n}+1$;
furthermore, in these circumstances the definition of $\hat{P}_{m}$ from Appendix~\ref{app:hatP-definitions} guarantees (for sufficiently large $n$)
that $\hat{\Delta}^{(2)} = 2/m$, so that $|\L_2| \propto n \cdot m = \Omega\left( n \cdot 2^{n}\right)$.  Thus, we expect
the label complexity gains to be \emph{exponentially improved} compared to $\alg_{p}$.

For a more involved example, consider unions of 2 intervals (Example~\ref{ex:unions-of-intervals}), under uniform $\Px$ on $[0,1]$,
and suppose $f = h_{(a,b,a,b)}$ for $b-a > 0$; that is, the target function is representable as a single nonzero-width interval $[a,b] \subset (0,1)$.
As we have seen, $\partial f = (0,1)$ in this case, so that disagreement-based methods are ineffective at improving over passive.  This also means
the $k=1$ round of \Shattering~will not provide improvements (i.e., $|\L_1| = O(n)$).  However, consider the $k=2$ round.
As discussed in Section~\ref{subsec:basic-activizer-examples}, for sufficiently large $n$, after the first round ($k=1$) the set $V$
is such that any label we infer in the $k=2$ round will be correct.  Thus, it suffices to determine how large the set $\L_2$ becomes.
By the same reasoning as in Section~\ref{subsec:basic-activizer-examples}, for sufficiently large $n$, the examples $X_m$ whose
labels are requested in Step 6 are precisely those \emph{not} separated from both $a$ and $b$ by at least one of the $m-1$ examples already
processed (since $V$ is consistent with the labels of all $m-1$ of those examples).  But this is the same set of points \CAL~would query for
the \emph{intervals} example in Section~\ref{subsec:disagreement-coefficient}; thus, the same argument used there implies that in this
problem we have $|\L_2| \geq \exp\{ \Omega(n / \log(n/\eps)) \}$ with probability $1-\eps/2$, which means we should
expect a label complexity of $O\left( \log( \Lambda_{p}(\eps/2,f,\Px)) \log(\log(\Lambda_{p}(\eps/2,f,\Px))/\eps) \right)$,
where $\Lambda_{p}$ is the label complexity of $\alg_{p}$.
For the case $f = h_{(a,a,a,a)}$, $k=3$ is the relevant round, and the analysis goes similarly to the $h_{[a,a]}$ scenario
for intervals above.
Unions of $i > 2$ intervals can be studied analogously,
with the appropriate value of $k$ to analyze being determined by the number of intervals required to represent the
target up to probability-zero differences (see the discussion in Section~\ref{subsec:basic-activizer-examples}).

\subsection{Beyond the Disagreement Coefficient}
\label{subsec:generalization-of-disagreement-coefficient}

In this subsection, we introduce a new quantity,
a generalization of the disagreement coefficient, which we will later use to provide a
general characterization of the improvements achievable by \Shattering, analogous to how the
disagreement coefficient characterized the improvements achievable by \CAL~in Theorem~\ref{thm:cal}.
First, let us define the following generalization of the disagreement core.
\begin{definition}
\label{def:k-dim-core}
For an integer $k \geq 0$, define the \emph{$k$-dimensional shatter core} of a classifier $f$ with respect to a set of classifiers $\H$ and distribution $P$
as
\begin{equation*}
\partial_{\H,P}^{k} f = \lim\limits_{r \to 0} \left\{ S \in \X^k : \Ball_{\H,P}(f,r) \text{ shatters } S\right\}.
\end{equation*}
\upthmend{-1.25cm}
\end{definition}
As before, when $P=\Px$, and $\Px$ is clear from the context, we will abbreviate $\partial^{k}_{\H} f = \partial^{k}_{\H,\Px} f$,
and when we also intend $\H=\C$, the \emph{full} concept space,
and $\C$ is clearly defined in the given context, we further abbreviate $\partial^{k} f = \partial^{k}_{\C} f = \partial^{k}_{\C,\Px} f$.
We have the following definition, which will play a key role in the label complexity bounds below.
\begin{definition}
\label{def:higher-dim-coefficient}
For any concept space $\C$, distribution $\Px$, and classifier $f$, $\forall k \in \nats$, $\forall \eps \geq 0$, define
\begin{equation*}
\dc_{f}^{(k)}(\eps) = 1\lor \sup\limits_{r > \eps} \frac{\Px^{k}\left( S \in \X^{k} : \Ball(f,r) \text{ shatters } S \right)}{r}.
\end{equation*}
Then define
\begin{equation*}
\bdim_f = \min\left\{ k \in \nats : \Px^{k}\left( \partial^{k} f\right) = 0\right\}
\end{equation*}
and  
\begin{equation*}
\hdc_{f}(\eps) = \dc_{f}^{(\bdim_f)}(\eps).
\end{equation*}
Also abbreviate $\dc_f^{(k)} = \dc_f^{(k)}(0)$ and $\hdc_f = \hdc_f(0)$.
\thmend
\end{definition}

We might refer to the quantity $\dc_{f}^{(k)}(\eps)$ as the order-$k$ (or $k$-dimensional) disagreement coefficient, as it
%(unfortunately, ``shatter coefficient'' is already taken :-). 
represents a direct generalization of the disagreement coefficient $\dc_{f}(\eps)$.
However, rather than merely measuring the rate of collapse of the probability of \emph{disagreement} (one-dimensional shatterability),
$\dc_{f}^{(k)}(\eps)$ measures the rate of collapse of the probability of \emph{$k$-dimensional shatterability}.
In particular, we have $\hdc_{f}(\eps) = \dc_{f}^{(\bdim_f)}(\eps) \leq
\dc_{f}^{(1)}(\eps) = \dc_{f}(\eps)$, so that this new quantity is never
larger than the disagreement coefficient.
However, unlike the disagreement coefficient, we \emph{always} have
$\hdc_f(\eps) = o(1/\eps)$ for VC classes $\C$.
In fact, we could equivalently define $\hdc_{f}(\eps)$
as the value of $\dc_{f}^{(k)}(\eps)$ for the smallest $k$ with $\dc_{f}^{(k)}(\eps) = o(1/\eps)$.
Additionally, we will see below that there are many interesting cases where $\dc_f = \infty$ (even $\dc_{f}(\eps) = \Omega(1/\eps)$) but $\hdc_f < \infty$
(e.g, intervals with a zero-width target, or unions of $i$ intervals where the target is representable as a union of $i-1$ or fewer intervals).
As was the case for $\dc_{f}$, we will see that showing $\hdc_f < \infty$ for a given learning problem has interesting implications for the label complexity of active learning (Corollary~\ref{cor:exponential} below).
In the process, we have also defined the quantity $\bdim_{f}$, which may itself be of independent interest in the asymptotic analysis of learning in general.
For VC classes, $\bdim_{f}$ always exists, and in fact is at most $d+1$ (since $\C$ cannot shatter any $d+1$ points).  When $d=\infty$, the quantity $\bdim_f$
might not be defined (or defined as $\infty$), in which case $\hdc_{f}(\eps)$ is also not defined; in this work we restrict our discussion to VC classes,
so that this issue never comes up; Section~\ref{sec:open-problems} discusses possible extensions to classes of infinite VC dimension.

We should mention that the restriction of $\hdc_{f}(\eps) \geq 1$ in the definition is only for
convenience, as it simplifies the theorem statements and proofs below.  It is not fundamental to the
definition, and can be removed (at the expense of slightly more complicated theorem statements).
In fact, this only makes a difference to the value of $\hdc_{f}(\eps)$ in some
(seemingly unusual) degenerate cases.  The same is true of $\dc_{f}(\eps)$ in Definition~\ref{def:disagreement-coefficient}.

The process of calculating $\hdc_{f}(\eps)$ is quite similar to that for the disagreement coefficient;
we are interested in describing $\Ball(f,r)$, and specifically the variety of behaviors of elements of $\Ball(f,r)$ on points in $\X$,
in this case with respect to shattering.
To illustrate the calculation of $\hdc_{f}(\eps)$, consider our usual toy examples, again under $\Px$ uniform on $[0,1]$.
For the thresholds example (Example~\ref{ex:thresholds}), we have $\bdim_f = 1$, so that
$\hdc_{f}(\eps) = \dc_{f}^{(1)}(\eps) = \dc_{f}(\eps)$, which
we have seen is equal $2$ for small $\eps$.  Similarly, for the intervals example (Example~\ref{ex:intervals}),
any $f = h_{[a,b]} \in \C$ with $b-a>0$ has $\bdim_{f} = 1$,
so that $\hdc_{f}(\eps) = \dc_{f}^{(1)}(\eps) = \dc_{f}(\eps)$, which
for sufficiently small $\eps$, is equal $\max\left\{ \frac{1}{b-a}, 4\right\}$.  Thus, for these
two examples, $\hdc_{f}(\eps) = \dc_{f}(\eps)$.  However, continuing the
intervals example, consider $f = h_{[a,a]} \in \C$.  In this case, we have seen $\partial^{1} f = \partial f = (0,1)$,
so that $\Px(\partial^{1} f) = 1 > 0$.  For any $x_1,x_2 \in (0,1)$ with $0 < |x_1-x_2| \leq r$, $\Ball(f,r)$ can shatter
$(x_1,x_2)$, specifically using the classifiers $\{h_{[x_1,x_2]}, h_{[x_1,x_1]}, h_{[x_2,x_2]}, h_{[x_3,x_3]}\}$ for any $x_3 \in (0,1) \setminus \{x_1,x_2\}$.
However, for any $x_1,x_2 \in (0,1)$ with $|x_1-x_2| > r$, no element of
$\Ball(f,r)$ classifies both as $+1$ (as it would need width greater than $r$, and thus would have distance
from $h_{[a,a]}$ greater than $r$).  Therefore, $\{S \in \X^2 : \Ball(f,r) \text{ shatters } S\} = \{ (x_1,x_2) \in (0,1)^2 : 0 < |x_1-x_2| \leq r\}$;
this latter set has probability $2r(1-r) + r^2 = (2-r)\cdot r$, which shrinks to $0$ as $r \to 0$.  Therefore,
$\bdim_f = 2$.  Furthermore, this shows $\hdc_{f}(\eps) = \dc_{f}^{(2)}(\eps) = \sup_{r > \eps} (2-r) = 2-\eps \leq 2$.
Contrasting this with $\dc_{f}(\eps) = 1/\eps$, we see $\hdc_{f}(\eps)$
is significantly smaller than the disagreement coefficient;
in particular, $\hdc_{f} = 2 < \infty$, while $\dc_{f} = \infty$.

Consider also the space of unions of $i$ intervals (Example~\ref{ex:unions-of-intervals}) under $\Px$ uniform on $[0,1]$.
In this case, we have already seen that, for any $f = h_{\mathbf{z}} \in \C$ not representable (up to probability-zero differences) by a uinon of $i-1$ or fewer intervals,
we have $\Px(\partial^{1} f) = \Px(\partial f) = 0$, so that $\bdim_{f} = 1$, and $\hdc_{f} = \dc_{f}^{(1)} = \dc_{f} = \max\left\{\frac{1}{\min\limits_{1 \leq p < 2i} z_{p+1}-z_{p}}, 4i\right\}$.
To generalize this, suppose $f = h_{\mathbf{z}}$ is minimally representable as a union of any number $j \leq i$ of intervals of nonzero width: $[z_1,z_2] \cup [z_3,z_4] \cup \cdots \cup [z_{2j-1},z_{2j}]$, with $0 < z_1 < z_2 < \cdots < z_{2j} < 1$.
For our purposes, this is fully general, since every element of $\C$ has distance zero to some $h_{\mathbf{z}}$ of this type, 
and $\hdc_{h} = \hdc_{h^{\prime}}$ for any $h,h^{\prime}$ with $\Px(x : h(x) \neq h^{\prime}(x))=0$.
Now for any $k < i-j+1$, and any $S = (x_1,\ldots,x_k) \in \X^{k}$ with all elements distinct and no elements equal any of the $z_{p}$ values,
the set $\Ball(f,r)$ can shatter $S$, as follows.  Begin with the intervals $[z_{2p-1},z_{2p}]$ as above, and modify the classifier in the following way for each labeling of $S$.
For any of the $x_{\ell}$ values we wish to label $+1$, if it is already in an interval $[z_{2p-1},z_{2p}]$, we do nothing; if it is not in one of the $[z_{2p-1},z_{2p}]$ intervals, we add the interval $[x_{\ell},x_{\ell}]$ to the classifier.
For any of the $x_{\ell}$ values we wish to label $-1$, if it is not in any interval $[z_{2p-1},z_{2p}]$, we do nothing; if it is in some interval $[z_{2p-1},z_{2p}]$, 
we split the interval by setting to $-1$ the labels in a small region $(x_{\ell}-\gamma,x_{\ell}+\gamma)$,
for $\gamma < \min\{r/k, z_{2p} - z_{2p-1}\}$ chosen small enough so that $(x_{\ell}-\gamma,x_{\ell}+\gamma)$ does not contain any other element of $S$.
These operations add at most $k$ new intervals to the minimal representation of the classifier as a union of intervals, 
which therefore has at most $j+k \leq i$ intervals.  Furthermore, the classifier disagrees with $f$ on a set of size at most $r$,
so that it is contained in $\Ball(f,r)$.  We therefore have $\Px^{k}( S \in \X^{k} : \Ball(f,r) \text{ shatters } S) = 1$.
However, note that for $0 < r < \min\limits_{1\leq p < 2j} z_{p+1} - z_{p}$, 
for any $k$ and $S \in \X^{k}$ with all elements of $S \cup \{z_p : 1 \leq p \leq 2j\}$ separated by a distance greater than $r$,
classifying the points in $S$ opposite to $f$ while remaining $r$-close to $f$ requires us to increase
to a minimum of $j+k$ intervals.
Thus, for $k = i-j+1$, any $S = (x_1,\ldots,x_k) \in \X^k$
with $\min\limits_{y_1,y_2 \in S \cup \{z_p\}_p : y_1 \neq y_2} |y_1-y_2| > r$ is \emph{not} shatterable by $\Ball(f,r)$.
We therefore have $\{S \in \X^{k} : \Ball(f,r) \text{ shatters } S\} \subseteq \left\{S \in \X^{k} : \min\limits_{y_1,y_2 \in S \cup \{z_p\}_p : y_1 \neq y_2} |y_1-y_2| \leq r\right\}$.
For $r < \min\limits_{1 \leq p < 2j}z_{p+1} - z_{p}$,
we can bound the probability of this latter set by considering sampling the points $x_{\ell}$ sequentially; the probability the $\ell^{\rm{th}}$ point
is within $r$ of one of $x_{1},\ldots,x_{\ell-1},z_1,\ldots,z_{2j}$ is at most $2r(2j+\ell-1)$,
so (by a union bound) the probability any of the $k$ points $x_1,\ldots,x_{k}$ is within $r$ of any other or any of $z_1,\ldots,z_{2j}$ is at most
$\sum_{\ell=1}^{k} 2r (2j+\ell-1) = 2r\left(2jk + \binom{k}{2}\right) = (1+i-j)(i+3j)r$.  Since this approaches zero as $r \to 0$, we have $\bdim_f = i-j+1$.
Furthermore, this analysis shows $\hdc_{f} = \dc_{f}^{(i-j+1)} \leq \max\left\{ \frac{1}{\min\limits_{1 \leq p < 2j} z_{p+1}-z_{p}}, (1+i-j)(i+3j)\right\}$.
In fact, careful further inspection reveals that this upper bound is tight (i.e., this is the exact value of $\hdc_{f}$).
Recalling that $\dc_{f}(\eps) = 1/\eps$ for $j < i$, we see that again $\hdc_{f}(\eps)$ is significantly smaller than the disagreement coefficient;
in particular, $\hdc_{f} < \infty$ while $\dc_{f} = \infty$.

Of course, for the quantity $\hdc_{f}(\eps)$ to be truly useful,
we need to be able to describe its behavior for families of learning problems beyond these simple toy problems.
Fortunately, as with the disagreement coefficient, for learning problems with simple ``geometric'' interpretations,
one can typically bound the value of $\hdc_{f}$ without too much difficulty.
For instance, consider $\X$ the surface of a unit hypersphere in $p$-dimensional Euclidean space (with $p \geq 3$),
with $\Px$ uniform on $\X$, and $\C$ the space of linear separators:
$\C = \{h_{\mathbf{w},b}(\mathbf{x}) = \ind_{[0,\infty)}^{\pm}( \mathbf{w} \cdot \mathbf{x} + b) : \mathbf{w} \in \reals^{p}, b \in \reals\}$.
\citet*{hanneke:10a} proved that $(\C,\Px)$ is learnable at an exponential rate, by a specialized argument for this space.
In the process, they established that for any $f \in \C$ with $\Px(x : f(x) = +1) \in (0,1)$,
$\dc_{f} < \infty$; in fact, a similar argument shows $\dc_{f} \leq 4 \pi \sqrt{p} / \min_{y} \Px(x : f(x)=y)$.
Thus, in this case, $\bdim_{f} = 1$, and $\hdc_{f} = \dc_{f} < \infty$.
However, consider $f \in \C$ with $\Px(x : f(x) = y) = 1$, for some $y \in \{-1,+1\}$.
In this case, every $h \in \C$ with $\Px(x : h(x) = -y) \leq r$
has $\Px(x : h(x) \neq f(x)) \leq r$ and is therefore contained in $\Ball(f,r)$.
In particular, for any $x \in \X$, there is such an $h$ that disagrees with $f$
on only a small spherical cap containing $x$, so that $\DIS(\Ball(f,r)) = \X$ for all $r > 0$.
But this means $\partial f = \X$, which implies $\dc_{f}(\eps) = 1/\eps$ and $\bdim_f > 1$.
However, let us examine the value of $\dc_{f}^{(2)}$.
Let $A_{p} = \frac{2 \pi^{p/2}}{\Gamma\left(\frac{p}{2}\right)}$ denote the surface area of the unit sphere in $\reals^{p}$,
and let $C_{p}(z) = \frac{1}{2} A_{p} I_{2z-z^2}\left(\frac{p-1}{2},\frac{1}{2}\right)$ denote the surface area of a spherical cap of height $z$
\citep*{li:11}, where $I_{x}(a,b) = \frac{\Gamma(a+b)}{\Gamma(a)\Gamma(b)} \int_{0}^{x} t^{a-1} (1-t)^{b-1} {\rm d}t$
is the regularized incomplete beta function.  In particular, since
$\sqrt{\frac{p}{12}} \leq \frac{\Gamma\left(\frac{p}{2}\right)}{\Gamma\left(\frac{p-1}{2}\right)\Gamma\left(\frac{1}{2}\right)}\leq \frac{1}{2} \sqrt{p-2}$,
the probability mass $\frac{C_{p}(z)}{A_{p}} = \frac{1}{2} \frac{\Gamma\left(\frac{p}{2}\right)}{\Gamma\left(\frac{p-1}{2}\right)\Gamma\left(\frac{1}{2}\right)} \int_{0}^{2z-z^2} t^{\frac{p-3}{2}} (1-t)^{- \frac{1}{2}} {\rm d}t$
contained in a spherical cap of height $z$ satisfies
\begin{equation}
\label{eqn:cap-lower}
\frac{C_{p}(z)}{A_p}
%\geq \frac{1}{2} \sqrt{\frac{p}{12}} \int_{0}^{2z-z^2} t^{\frac{p-3}{2}} (1-t)^{-\frac{1}{2}} {\rm d}t 
\geq \frac{1}{2} \sqrt{\frac{p}{12}} \int_{0}^{2z-z^2} t^{\frac{p-3}{2}} {\rm d}t
= \sqrt{\frac{p}{12}} \frac{(2z-z^2)^{\frac{p-1}{2}}}{p-1}
\geq \frac{(2z-z^2)^{\frac{p-1}{2}}}{\sqrt{12 p}},
\end{equation}
and letting $\bar{z} = \min\{z,1/2\}$, also satisfies
\begin{align}
\frac{C_{p}(z)}{A_p}
&\leq \frac{2 C_{p}\left(\bar{z}\right)}{A_p}
\leq \frac{1}{2} \sqrt{p-2} \int_{0}^{2\bar{z} - \bar{z}^2} t^{\frac{p-3}{2}} (1-t)^{-\frac{1}{2}} {\rm d}t \notag
\\ &\leq \sqrt{p-2} \int_{0}^{2z - z^2} t^{\frac{p-3}{2}}{\rm d}t
= \frac{2 \sqrt{p-2}}{p-1} (2z-z^2)^{\frac{p-1}{2}}
\leq \frac{(2z-z^2)^{\frac{p-1}{2}}}{\sqrt{p/6}}
\leq \frac{(2z)^{\frac{p-1}{2}}}{\sqrt{p/6}}. \label{eqn:cap-upper}
\end{align}
Consider any linear separator $h \in \Ball(f,r)$ for $r < 1/2$,
and let $z(h)$ denote the height of the spherical cap where $h(x) = -y$.
Then \eqref{eqn:cap-lower} indicates the probability mass in this region is
at least $\frac{(2z(h)-z(h)^2)^{\frac{p-1}{2}}}{\sqrt{12 p}}$.
Since $h \in \Ball(f,r)$, we know this probability mass is at most $r$,
and we therefore have
$2z(h)-z(h)^2 \leq \left( \sqrt{12 p} r \right)^{\frac{2}{p-1}}$.
Now for any $x_1 \in \X$, the set of $x_2 \in \X$ for which $\Ball(f,r)$ shatters $(x_1,x_2)$
is equivalent to the set $\DIS(\{h \in \Ball(f,r) : h(x_1) = -y\})$.  But if $h(x_1) = -y$, then $x_1$
is in the aforementioned spherical cap associated with $h$.
A little trigonometry reveals that, for any spherical cap of height $z(h)$, any two points on the surface
of this cap are within distance
$2\sqrt{2z(h) - z(h)^2} \leq 2 \left( \sqrt{12 p} r \right)^{\frac{1}{p-1}}$
of each other.
Thus, for any point $x_2$ further than
$2 \left( \sqrt{12 p} r \right)^{\frac{1}{p-1}}$
from $x_1$, it must be outside the spherical
cap associated with $h$, which means $h(x_2) = y$.  But this is true for every $h \in \Ball(f,r)$ with
$h(x_1) = -y$, so that $\DIS(\{h \in \Ball(f,r) : h(x_1) = -y\})$ is contained in the spherical cap of
all elements of $\X$ within distance
$2 \left( \sqrt{12 p} r \right)^{\frac{1}{p-1}}$
of $x_1$;
a little more trigonometry reveals that the height of this spherical cap is
$2 \left( \sqrt{12 p} r \right)^{\frac{2}{p-1}}$.
Then \eqref{eqn:cap-upper} indicates the probability mass in this region
is at most $\frac{2^{p-1} \sqrt{12 p} r}{\sqrt{p/6}} = 2^{p} \sqrt{18} r$.
Thus,
$\Px^{2}((x_1,x_2) : \Ball(f,r) \text{ shatters } (x_1,x_2)) = \int \Px(\DIS(\{h \in \Ball(f,r) : h(x_1) = -y\})) \Px({\rm d}x_1) \leq 2^{p} \sqrt{18} r$.
In particular, since this approaches zero as $r\to 0$, we have $\bdim_f = 2$.
This also shows that
$\hdc_{f} = \dc_{f}^{(2)} \leq 2^{p} \sqrt{18}$,
a finite constant (albeit a rather large one).
%we will see that finite values of $\hdc_{f}$ are particularly interesting for active learning (Corollary~\ref{cor:exponential} below).
Following similar reasoning, using the opposite inequalities as appropriate, and taking $r$ sufficiently small,
one can also show $\hdc_{f} \geq 2^{p} / (12 \sqrt{2})$.

\subsection{Bounds on the Label Complexity of Activized Learning}
\label{subsec:bound}

We have seen above that in the context of several examples,
\Shattering~can offer significant advantages in label complexity over any given
passive learning algorithm, and indeed also over disagreement-based active learning in many cases.  In this subsection, we
present a general result characterizing the magnitudes of these improvements over passive learning, in terms of $\hdc_{f}(\eps)$.
Specifically, we have the following general theorem, along with two immediate corollaries.
The proof is included in Appendix~\ref{app:exponential},

\newpage
\begin{theorem}
\label{thm:sequential-activizer}
For any VC class $\C$,
and any passive learning algorithm $\alg_p$ achieving label complexity $\Lambda_{p}$,
the (\Shattering)-activized $\alg_p$ algorithm achieves
a label complexity $\Lambda_a$ that, for any distribution $\Px$ and classifier $f \in \C$, satisfies
\begin{equation*}
\Lambda_a(\eps, f, \Px) = O\left( \hdc_{f} \left( \Lambda_{p}(\eps/4,f,\Px)^{-1}\right) \log^2 \frac{\Lambda_{p}(\eps/4,f,\Px)}{\eps} \right).
\end{equation*}
\upthmend{-1.1cm}
\end{theorem}

\begin{corollary}
\label{cor:sequential-activizer}
For any VC class $\C$, there exists a passive learning algorithm $\alg_p$ such that,
for every $f \in \C$ and distributions $\Px$,
the (\Shattering)-activized $\alg_p$ algorithm achieves label complexity
\begin{equation*}
\Lambda_{a}(\eps,f,\Px) = O\left( \hdc_f ( \eps ) \log^2(1/\eps) \right).
\end{equation*}
\upthmend{-1.15cm}
\end{corollary}
\begin{proof}
The one-inclusion graph algorithm of \citet*{haussler:94} is a passive learning algorithm achieving label complexity
$\Lambda_{p}(\eps,f,\Px) \leq d/\eps$.  Plugging this into Theorem~\ref{thm:sequential-activizer}, using the
fact that $\hdc_{f}(\eps / 4d) \leq 4 d \hdc_{f}(\eps)$, and simplifying,
we arrive at the result.  In fact, in the proof of Theorem~\ref{thm:sequential-activizer}, we see that incurring this extra
constant factor of $d$ is not actually necessary.
\end{proof}

\begin{corollary}
\label{cor:exponential}
For any VC class $\C$ and distribution $\Px$,
if $\forall f \in \C, \hdc_f < \infty$,
then $(\C,\Px)$ is learnable at an exponential rate.
If this is true for all $\Px$, then $\C$ is learnable at an exponential rate.
\thmend
\end{corollary}
\begin{proof}
The first claim follows directly from Corollary~\ref{cor:sequential-activizer}, since $\hdc_f(\eps) \leq \hdc_f$.
The second claim then follows from the fact that \Shattering~is adaptive to $\Px$ (has no direct dependence on $\Px$ except via the data).
\end{proof}

Actually, in the proof we arrive at a somewhat more general result, in that the bound of
Theorem~\ref{thm:sequential-activizer} actually holds for any target function $f$ in the ``closure''
of $\C$: that is, any $f$ such that $\forall r > 0, \Ball(f,r) \neq \emptyset$.
As previously mentioned, if our goal is only to obtain the label complexity bound of Corollary~\ref{cor:sequential-activizer} by a direct approach,
then we can use a simpler procedure (which cuts out Steps 9-16, instead returning an arbitrary element of $V$),
analogous to how the analysis of the original algorithm of \citet*{cohn:94} by \citet*{hanneke:11a} obtains
the label complexity bound of Corollary~\ref{cor:cal}
(see also \RobustShattering~below).
However, the general result of Theorem~\ref{thm:sequential-activizer} is interesting in that it applies to any passive algorithm.

Inspecting the proof, we see that it is also possible to state a result that separates the probability of success from the achieved error rate,
similar to the PAC model of \citet*{valiant:84} and the analysis of active learning by \citet*{hanneke:10a}.  Specifically, suppose $\alg_{p}$
is a passive learning algorithm such that, $\forall \eps,\conf \in (0,1)$, there is a value $\lambda(\eps,\conf,f,\Px) \in \nats$
such that $\forall n \geq \lambda(\eps,\conf,f,\Px)$,
$\P\left( \er\left( \alg_{p}(\Data_{n})\right) > \eps\right) \leq \conf$.
Suppose $\hat{h}_{n}$ is the classifier returned by the (\Shattering)-activized $\alg_p$ with label budget $n$.
Then for some $(\C,\Px,f)$-dependent constant $c \in [1,\infty)$, 
$\forall \eps,\conf \in (0,e^{-3})$,
letting $\lambda = \lambda(\eps/2,\conf/2,f,\Px)$,
\begin{equation*}
\forall n \geq c \hdc_{f}\left( \lambda^{-1}\right) \log^2 \left( \lambda / \conf\right),
~~\P\left( \er\left(\hat{h}_{n}\right) > \eps\right) \leq \conf.
\end{equation*}
For instance, if $\alg_{p}$ is an empirical risk minimization algorithm, then this is $\propto \hdc_{f}(\eps) \polylog\left( \frac{1}{\eps \conf}\right)$.

\subsection{Limitations and Potential Improvements}
\label{subsec:exp-alternatives}

Theorem~\ref{thm:sequential-activizer} and its corollaries represent significant improvements
over most known results for the label complexity of active learning, and in particular over Theorem~\ref{thm:cal}
and its corollaries.
As for whether this also represents the best possible label complexity gains achievable by any active learning
algorithm, the answer is mixed.
As with any algorithm and analysis, \Shattering, Theorem~\ref{thm:sequential-activizer}, and corollaries,
represent one set of solutions in a spectrum that trades strength of performance guarantees with simplicity.
As such, there are several possible modifications
one might make, which could potentially improve the performance guarantees.
Here we sketch a few such possibilities.

Even with \Shattering~as-is, various improvements to the bound of Theorem~\ref{thm:sequential-activizer} should be possible,
simply by being more careful in the analysis.
For instance, as mentioned, \Shattering~is a \emph{universal activizer} for any VC class $\C$,
so in particular we know that whenever
$\hdc_f(\eps) \neq o\left(1 / \left(\eps \log(1/\eps)\right)\right)$,
the above bound is not tight (see the work of \citet*{hanneke:10a} for a construction leading to such $\hdc_{f}(\eps)$ values),
and indeed any bound of the form $\hdc_f(\eps) \polylog(1/\eps)$
will not be tight in that case.  Again, a more refined analysis may close this gap.

Another type of potential improvement is in the constant factors.
Specifically, in the case when $\hdc_f < \infty$,
if we are only interested in \emph{asymptotic} label complexity guarantees in
Corollary~\ref{cor:sequential-activizer},
we can replace ``$\sup\limits_{r > 0}$'' in Definition~\ref{def:higher-dim-coefficient}
with ``$\limsup\limits_{r \to 0}$,'' which can sometimes be significantly smaller and/or easier to study.
This is true for the disagreement coefficient in Corollary~\ref{cor:cal} as well.
Additionally, the proof (in Appendix~\ref{app:exponential}) reveals that there are significant $(\C,\Px,f)$-dependent
constant factors other than $\hdc_{f}(\eps)$, and it is quite likely that these can be improved by a more
careful analysis of \Shattering~(or in some cases, possibly an improved definition of the estimators $\hat{P}_{m}$).

However, even with such refinements to improve the results, the approach of using $\hdc_{f}$ to prove
learnability at an exponential rate has limits.  For instance, it is known that any \emph{countable} $\C$ is learnable at an
exponential rate \citep*{hanneke:10a}.  However, there are countable VC classes $\C$ for which
$\hdc_{f} = \infty$ for some elements of $\C$ (e.g., take the tree-paths concept space
of \citet*{hanneke:10a}, except instead of all infinite-depth paths from the root, take all of the finite-depth
paths from the root, but keep one infinite-depth path $f$; for this modified space $\C$, which is countable,
every $h \in \C$ has $\bdim_{h} = 1$, and for that one infinite-depth $f$ we have $\hdc_{f} = \infty$).
% in fact, \Shattering~does not achieve exponential rates for this problem, when given the closure algorithm as input

Inspecting the proof reveals that it is possible to make the results slightly sharper by replacing $\hdc_{f}(r_0)$ (for $r_0$ as in the results above) with a somewhat
more complicated quantity: namely,
\begin{equation}
\label{eqn:tighter-hdc}
\min_{k < \bdim_{f}} \sup_{r > r_0} r^{-1} \cdot \Px\left( x \in \X : \Px^{k}\left( S\in \X^{k} : \Ball(f,r) \text{ shatters } S \cup \{x\} \right) \geq \P\left(\partial^{k} f\right) / 16\right).
\end{equation}
This quantity can be bounded in terms of $\hdc_{f}(r_0)$ via Markov's inequality, but is sometimes smaller. 

As for improving \Shattering~itself, there are several possibilities.
One immediate improvement one can make is to repace the condition in Steps 5 and 12 by
$\min_{1\leq j \leq k}\hat{P}_{m}(S \in \X^{j-1} : V \text{ shatters } S \cup \{X_{m}\} | V \text{ shatters } S) \geq 1/2$,
likewise replacing the corresponding quantity in Step 9, and substituting in Steps 7 and 14 the quantity
$\max_{1 \leq j \leq k} \hat{P}_{m}(S \in \X^{j-1} : V[(X_{m},-y)]$ does not shatter $S | V \text{ shatters } S)$;
in particular, the results stated for \Shattering~remain valid with this substitution, requiring only minor modifications
to the proofs.  However, it is not clear what gains in theoretical guarantees this achieves.

Additionally, there are various
quantities in this procedure that can be altered almost arbitrarily, allowing room for fine-tuning.
Specifically, the $2/3$ in Step 0 and $1/3$ in Step 16 can be set to arbitrary constants summing to $1$.
Likewise, the $1/4$ in Step 3, $1/3$ in Step 10, and $3 / 4$ in Step 12 can be changed to
any constants in $(0,1)$, possibly depending on $k$, such that the sum of the first
two is strictly less than the third.
Also, the $1/2$ in Steps 5, 9, and 12 can be set to any constant in $(0,1)$.
Furthermore, the $k \cdot 2^{n}$ in Step 3 only prevents infinite looping, and can be set
to any function growing superlinearly in $n$, though to get the largest possible improvements
it should at least grow exponentially in $n$;  typically, \emph{any} active learning algorithm
capable of exponential improvements over reasonable passive learning algorithms will
require access to a number of unlabeled examples exponential in $n$, and \Shattering~is
no exception to this.

One major issue in the design of the procedure is an inherent trade-off between the achieved label
complexity and the number of unlabeled examples used by the algorithm.  This is noteworthy both
because of the practical concerns of gathering such large quantities of unlabeled data, and also for
computational efficiency reasons.  In contrast to disagreement-based methods, the design of the
estimators used in \Shattering~introduces such a trade-off, though in contrast to the splitting
index analysis of \citet*{dasgupta:05}, the trade-off here seems only in the constant factors.  The choice of these $\hat{P}_{m}$
estimators, both in their definition in Appendix~\ref{app:hatP-definitions}, and indeed in the very quantities they estimate,
is such that we can (if desired) limit the number of unlabeled examples the main body of the algorithm uses
(the actual number it needs to achieve Theorem~\ref{thm:sequential-activizer} can be extracted from the proofs in Appendix~\ref{app:sequential-activizer}).
However, if the number of unlabeled examples used by the algorithm is not a limiting factor, we can
suggest more effective quantities.  Specifically, following the original motivation for using shatterable sets,
we might consider a greedily-constructed distribution over the set
$\{ S \in \X^{j} : V \text{ shatters } S, 1 \leq j < k, \text{ and either } j=k-1\text{ or }$ $\Px(s : V \text{ shatters } S \cup \{s\}) = 0\}$.
We can construct the distribution implicitly, via the following generative model.
First we set $S = \{\}$.  Then repeat the following.
If $|S| = k-1$ or $\Px(s \in \X : V \text{ shatters } S \cup \{s\})=0$, output $S$;
otherwise, sample $s$ according to the conditional distribution of $X$ given that $V \text{ shatters } S \cup \{X\}$.
If we denote this distribution (over $S$) as $\tilde{\Px}_{k}$, then replacing the estimator
$\hat{P}_{m}\left(S \in \X^{k-1} : V \text{ shatters } S \cup \{X_{m}\} | V \text{ shatters } S\right)$
in \Shattering~with an appropriately constructed estimator of
$\tilde{\Px}_{k}\left( S : V \text{ shatters } S \cup \{X_{m}\} \right)$
(and similarly replacing the other estimators)
can lead to some improvements in the constant factors of the label complexity.
However, such a modification can also dramatically increase the number of unlabeled examples
required by the algorithm, since determining whether
$\Px(s \in \X : V \text{ shatters } S \cup \{s\}) \approx 0$ can be costly.

Unlike \BasicActivizer, there remain serious efficiency concerns surrounding \Shattering.
If we knew the value of $\bdim_f$ and $\bdim_f \leq c\log_{2}(d)$ for some constant $c$,
then we could potentially design an efficient version of \Shattering~still achieving Corollary~\ref{cor:sequential-activizer}.
Specifically, suppose we can find a classifier in $\C$ consistent with any given sample,
or determine that no such classifier exists, in time polynomial in the sample size (and $d$),
and also that $\alg_{p}$ efficiently returns a classifier in $\C$ consistent with the sample it is given.
Then replacing the loop of Step 1 by simply running with $k=\bdim_{f}$ and returning $\alg_{p}(\L_{\bdim_f})$,
the algorithm becomes efficient, in the sense that with high probability, its running time is ${\rm {\poly}}(d/\eps)$,
where $\eps$ is the error rate guarantee from inverting the label complexity at the value of $n$ given to the algorithm.
To be clear, in some cases we may obtain
values $m \propto \exp\{ \Omega(n)\}$, but the error rate guaranteed by $\alg_{p}$ is $\tilde{O}(1/m)$ in these cases,
so that we still have $m$ polynomial in $d/\eps$.
However, in the absence of this access to $\bdim_f$, the values of $k > \bdim_f$ in \Shattering~may
reach values of $m$ much larger than ${\rm {\poly}}(d/\eps)$, since the error rates obtained from these
$\alg_{p}(\L_{k})$ evaluations are not guaranteed to be better than the $\alg_{p}(\L_{\bdim_f})$ evaluations,
and yet we may have $|\L_{k}| \gg |\L_{\bdim_f}|$.
Thus, there remains a challenging problem of obtaining the results above (Theorem~\ref{thm:sequential-activizer}
and Corollary~\ref{cor:sequential-activizer}) via an efficient algorithm, adaptive to the value of $\bdim_f$.

\section{Toward Agnostic Activized Learning}
\label{sec:agnostic}

The previous sections addressed learning in the \emph{realizable} case,
where there is a perfect classifier $f \in \C$ (i.e., $\er(f) = 0$).  
To move beyond these scenarios, to problems in which $f$ is not a perfect
classifier (i.e., stochastic labels) or not well-approximated by $\C$,
requires a change in technique to make the algorithms more robust to such
issues.  As we will see in Subsection~\ref{subsec:agnostic-counterexample},
the results we can prove in this more general setting are not quite as strong as those of the previous
sections, but in some ways they are more interesting, both from a practical
perspective, as we expect real learning problems to involve imperfect teachers
or underspecified instance representations, and also from a theoretical perspective,
as the class of problems addressed is significantly more general than those
encompassed by the realizable case above.

In this context, we will be largely interested in more general versions of the same
types of questions as above, such as whether one can activize a given passive
learning algorithm, in this case guaranteeing strictly improved label complexities
for all nontrivial joint distributions over $\X \times \{-1,+1\}$.
In Subsection~\ref{subsec:activized-erm}, we present a general conjecture regarding
this type of strong domination.  At the same time,
to approach such questions, we will also need to focus on developing techniques to make
the algorithms robust to label noise.  For this, we will use a natural generalization
of techniques developed for noise-robust disagreement-based active learning,
analogous to how we generalized \CAL~to arrive at \Shattering~above.
For this purpose, as well as for the sake of comparison, we will review the known
techniques and results for disagreement-based agnsotic active learning in
Subsection~\ref{subsec:disagreement-based-agnostic}.  We then extend these
techniques in Subsection~\ref{subsec:robust-pseudo-activizer} to develop a new
type of agnostic active learning algorithm, based on shatterable sets, which relates
to the disagreement-based agnostic active learning algorithms in a way analogous to how
\Shattering~relates to \CAL.  Furthermore, we present a bound on the
label complexities achieved by this method, representing a natural generalization of
both Corollary~\ref{cor:sequential-activizer} and the known results on
disagreement-based agnostic active learning \citep*{hanneke:11a}.

Although we present several new results, in some sense this section is less about what we know
and more about what we do not yet know.  As such, we will focus less on presenting a complete and elegant theory,
and more on identifying potentially promising directions for exploration.  In particular,
Subsection~\ref{subsec:noise-conditions} sketches out some interesting directions,
which could potentially lead to a resolution of the aforementioned general conjecture from Subsection~\ref{subsec:activized-erm}.

\subsection{Definitions and Notation}
\label{subsec:agnostic-definitions}

In this setting, there is a joint distribution $\PXY$ on $\X \times \{-1,+1\}$,
with marginal distribution $\Px$ on $\X$.
For any classifier $h$, we denote
by $\er(h) = \PXY((x,y) : h(x) \neq y)$.
Also, denote by
$\nu^*(\PXY) = \inf\limits_{h : \X \to \{-1,+1\}}  \er(h)$
the \emph{Bayes error rate},
or simply $\nu^*$ when $\PXY$ is clear from the context;
also define the conditional
label distribution $\eta(x; \PXY) = \P(Y = +1 | X = x)$,
where $(X,Y) \sim \PXY$,
or $\eta(x) = \eta(x ; \PXY)$
when $\PXY$ is clear from the context.
For a given concept space $\C$, denote $\nu(\C ; \PXY) = \inf\limits_{h \in \C} \er(h)$,
called the \emph{noise rate} of $\C$; when $\C$ and/or $\PXY$ is clear from
the context, we may abbreviate $\nu = \nu(\C) = \nu(\C ; \PXY)$.
For $\H \subseteq \C$, the \emph{diameter} is defined as
$\diam(\H ; \Px) = \sup\limits_{h_1,h_2 \in \H} \Px( x : h_1(x) \neq h_2(x))$.
Also, for any $\eps > 0$, define the $\eps$-minimal set
$\C(\eps ; \PXY) = \{ h \in \C : \er(h) \leq \nu + \eps\}$.
For any set of classifiers $\H$, define the \emph{closure}, denoted $\cl(\H ; \Px)$,
as the set of all measurable $h : \X \to \{-1,+1\}$ such that $\forall r > 0, \Ball_{\H,\Px}(h,r) \neq \emptyset$.
When $\PXY$ is clear from the context, we will simply refer to $\C(\eps) = \C(\eps ; \PXY)$,
and when $\Px$ is clear, we write $\diam(\H) = \diam(\H ; \Px)$
and $\cl(\H) = \cl(\H ; \Px)$.

In the noisy setting, rather than being a \emph{perfect} classifier, we will let
$f$ denote an arbitrary element of $\cl(\C; \Px)$ with $\er(f) = \nu(\C;\PXY)$:
that is, $f \in \bigcap\limits_{\eps > 0} \cl\left( \C(\eps ; \PXY) ; \Px\right)$.
Such a classifier must exist, since $\cl(\C)$ is \emph{compact} in the pseudo-metric
$\rho(h,g) = \int |h - g| d\Px \propto \Px(x : h(x) \neq g(x))$ (in the usual sense
of the equivalence classes being compact in the $\rho$-induced metric).  This can
be seen by recalling that $\C$ is totally bounded \citep*{haussler:92},
and thus so is $\cl(\C)$, and that $\cl(\C)$ is a closed subset of $\mathcal{L}^{1}(\Px)$,
which is complete \citep*{dudley:02}, so $\cl(\C)$ is also complete \citep*{munkres:00}.
Total boundedness and completeness together imply compactness \citep*{munkres:00},
and this implies the existence of $f$ since monotone sequences of nonempty closed subsets
of a compact space have a nonempty limit set \citep*{munkres:00}.

As before, in the learning problem there is a sequence
$\Data = \{(X_1,Y_1),(X_2,Y_2),\ldots\}$,
where the $(X_i,Y_i)$
are independent
and identically distributed, and we denote by $\Data_{m} = \{(X_i,Y_i)\}_{i=1}^{m}$.
As before, the $X_i \sim \Px$, but rather than having
each $Y_i$ value determined as a function of $X_i$, instead we have each pair
$(X_i,Y_i) \sim \PXY$.  
The learning protocol is defined identically as above; that is, the algorithm has
direct access to the $X_i$ values, but must request the $Y_i$ (label) values one at a time,
sequentially, and can request at most $n$ total labels, where $n$ is a budget provided
as input to the algorithm.  The label complexity is now defined just as before
(Definition~\ref{defn:label-complexity}), but generalized by replacing
$(f,\Px)$ with the joint distribution $\PXY$.  Specifically, we have the following formal definition,
which will be used throughout this section (and the corresponding appendices).
\begin{definition}
\label{defn:agnostic-label-complexity}
An active learning algorithm $\alg$ achieves label complexity $\Lambda(\cdot,\cdot)$
if, for any joint distribution $\PXY$, for any $\eps \in (0,1)$
and any integer $n \geq \Lambda(\eps, \PXY)$, we have $\E\left[ \er\left( \alg(n) \right) \right] \leq \eps$.
\thmend
\end{definition}
However, because there may not be any classifier with error rate less than
any arbitrary $\eps \in (0,1)$, our objective changes here to
achieving error rate at most $\nu + \eps$ for any given $\eps \in (0,1)$.
Thus, we are interested in the quantity $\Lambda(\nu + \eps, \PXY)$,
and will be particularly interested in this quantity's asymptotic dependence on $\eps$,
as $\eps \to 0$.  In particular, $\Lambda(\eps, \PXY)$ may often be infinite for $\eps < \nu$.

The label complexity for passive learning can be generalized analogously, again replacing $(f,\Px)$ by $\PXY$ in
Definition~\ref{defn:passive-label-complexity} as follows.
\begin{definition}
\label{defn:agnostic-passive-label-complexity}
A passive learning algorithm $\alg$ achieves label complexity $\Lambda(\cdot,\cdot)$
if, for any joint distribution $\PXY$, for any $\eps \in (0,1)$ and any integer $n \geq \Lambda(\eps, \PXY)$,
we have $\E\left[ \er\left( \alg\left( \Data_{n} \right) \right) \right] \leq \eps$.
\thmend
\end{definition}

For any label complexity $\Lambda$ in the agnostic case, define the set $\Nontrivial(\Lambda ; \C)$
as the set of all distributions $\PXY$ on $\X \times \{-1,+1\}$ such that
$\forall \eps > 0, \Lambda(\nu+\eps, \PXY) < \infty$,
and $\forall g \in \Polylog(1/\eps)$, $\Lambda(\nu+\eps, \PXY) = \omega(g(\eps))$.
In this context, we can define an \emph{activizer} for a given passive algorithm as follows.

\newpage
\begin{definition}
\label{defn:agnostic-activizer}
We say an active meta-algorithm $\alg_a$ \emph{activizes} a passive algorithm $\alg_p$
for $\C$ in the agnostic case if the following holds.
For any label complexity $\Lambda_p$ achieved by $\alg_p$,
the active learning algorithm $\alg_a(\alg_p, \cdot)$ achieves a label complexity $\Lambda_a$
such that, for every distribution $\PXY \in \Nontrivial(\Lambda_{p} ; \C)$,
there exists a constant $c \in [1,\infty)$ such that
\begin{equation*}
\Lambda_a(\nu + c\eps, \PXY) = o\left( \Lambda_p(\nu + \eps, \PXY) \right).
\end{equation*}
In this case, $\alg_a$ is called an \emph{activizer} for $\alg_p$ with respect to $\C$ in the agnostic case,
and the active learning algorithm $\alg_{a}(\alg_p, \cdot)$ is called the
$\alg_a$\emph{-activized} $\alg_p$.
\thmend
\end{definition}

\subsection{A Negative Result}
\label{subsec:agnostic-counterexample}

First, the bad news: we cannot generally hope for universal activizers for VC classes in the agnostic case.
In fact, there even exist passive algorithms that \emph{cannot be activized}, even by any specialized
active learning algorithm.

Specifically, consider again Example~\ref{ex:thresholds}, where $\X = [0,1]$
and $\C$ is the class of threshold classifiers,
and let $\check{\alg}_p$ be a passive learning algorithm that behaves as follows.
Given $n$ points $\Data_{n} = \left\{(X_1,Y_1),(X_2,Y_2),\ldots,(X_n,Y_n)\right\}$,
$\check{\alg}_p(\Data_{n})$ returns the classifier $h_{\hat{z}} \in \C$, where $\hat{z} = \frac{1-2\hat{\eta}_{0}}{1-\hat{\eta}_{0}}$
and $\hat{\eta}_{0} = \left(\frac{\left| \left\{ i \in \{1,\ldots,n\} : X_{i} = 0, Y_{i} = +1\right\}\right|}{\left| \left\{ i \in \{1,\ldots,n\} : X_{i} = 0\right\}\right|} \lor \frac{1}{8}\right) \land \frac{3}{8}$,
taking
$\hat{\eta}_{0} = 1/8$
if $\{ i \in \{1,\ldots,n\} : X_{i} = 0\} = \emptyset$.
For most distributions $\PXY$, this algorithm clearly would not behave ``reasonably,'' in that
its error rate would be quite large;  in particular, in the realizable case, the algorithm's
worst-case expected error rate does not converge to zero as $n\to\infty$.
However, for certain distributions $\PXY$ engineered specifically for this algorithm, it has near-optimal behavior in a strong sense.
Specifically, we have the following result, the proof of which is included in Appendix~\ref{app:agnostic-counterexample}.

\begin{theorem}
\label{thm:agnostic-counterexample}
There is no activizer for $\check{\alg}_p$ with respect to the space of threshold classifiers in the agnostic case.
\thmend
\end{theorem}

Recall that threshold classifiers were, in some sense, one of the simplest scenarios for activized learning in the realizable case.
Also, since threshold-like problems are embedded in most ``geometric'' concept spaces, this indicates we should generally
not expect there to exist activizers for arbitrary passive algorithms in the agnostic case.  However, this leaves open the
question of whether certain families of passive learning algorithms can be activized in the agnostic case, a topic we turn to next.

\subsection{A Conjecture: Activized Empirical Risk Minimization}
\label{subsec:activized-erm}

The counterexample above is interesting, in that it exposes the limits on
generality in the agnostic setting.  However, the passive algorithm that cannot be activized there is in
many ways not very reasonable, in that it has suboptimal worst-case expected excess error rate
(among other deficiencies).
It may therefore be more interesting to ask whether some family of ``reasonable'' passive learning algorithms
can be activized in the agnostic case.  It seems that, unlike $\check{\alg}_{p}$ above, certain passive learning algorithms
should not have too peculiar a dependence on the label noise, so that they use $Y_i$ to help determine $f(X_i)$ and that is all.
In such cases, any $Y_i$ value for which we can already infer the value $f(X_i)$ should simply be ignored as redundant information,
so that we needn't request such values.
While this discussion is admittedly vague, consider the following formal conjecture.

Recall that an \emph{empirical risk minimization} algorithm for $\C$ is a type of
passive learning algorithm $\alg$, characterized by the fact that for any
set $\L \in \bigcup_{m} (\X \times \{-1,+1\})^m$, $\alg(\L) \in \argmin\limits_{h \in \C} \er_{\L}(h)$.

\begin{conjecture}
\label{conj:activized-erm}
For any VC class, there exists an active meta-algorithm $\alg_a$
and an empirical risk minimization algorithm $\alg_p$ for $\C$
such that $\alg_a$ \emph{activizes} $\alg_p$ for $\C$ in the agnostic case.
\thmend
\end{conjecture}

Resolution of this conjecture would be interesting for a variety of reasons.
If the conjecture is correct, it means that the vast (and growing) literature on the label complexity
of empirical risk minimization has direct implications for the potential performance
of active learning under the same conditions.  We might also expect activized empirical risk minimization
to be quite effective in practical applications.

While this conjecture remains open at this time,
the remainder of this section might be viewed as partial evidence in its favor, as we show that
active learning is able to achieve improvements over the known bounds on the label complexity of
passive learning in many cases.

\subsection{Low Noise Conditions}
\label{subsec:tsybakov}

In the subsections below, we will be interested in stating bounds on the label complexity of active learning,
analogous to those of Theorem~\ref{thm:cal} and Theorem~\ref{thm:sequential-activizer}, but for learning
with label noise.  As in the realizable case, we should expect such bounds to have some explicit dependence
on the distribution $\PXY$.  Initially, one might hope that we could state interesting label complexity
bounds purely in terms of a simple quantity such as $\nu(\C; \PXY)$.  However, it is known that any
label complexity bound for a nontrivial $\C$  (for either passive or active) depending on $\PXY$ only
via $\nu(\C; \PXY)$ will be $\Omega\left(\eps^{-2}\right)$ when $\nu(\C;\PXY) > 0$ \citep*{kaariainen:06}.
Since passive learning can achieve a $\PXY$-independent $O\left(\eps^{-2}\right)$ label complexity bound
for any VC class \citep*{alexander:84}, we will need to discuss label complexity bounds that depend on $\PXY$ via
more detailed quantities than merely $\nu(\C; \PXY)$ if we are to characterize the improvements of active learning
over passive.

In this subsection, we review an index commonly used to describe certain
properties of $\PXY$ relative to $\C$: namely, the Mammen-Tsybakov margin conditions \citep*{mammen:99,tsybakov:04,koltchinskii:06}.
Specifically, we have the following formal condition from \citet*{koltchinskii:06}.

\begin{condition}
\label{con:tsybakov}
There exist constants $\mu, \kappa \in [1,\infty)$ such that $\forall \eps > 0$, $\diam(\C(\eps; \PXY) ; \Px) \leq \mu \cdot \eps^{\frac{1}{\kappa}}$.
\thmend
\end{condition}

This condition has recently been studied in depth in the passive learning literature,
as it can be used to characterize scenarios where the label complexity of passive learning
is \emph{between} the worst-case $\Theta(1/\eps^2)$ and the realizable case $\Theta(1/\eps)$
\citep*[e.g.,][]{mammen:99,tsybakov:04,koltchinskii:06,massart:06}.  The condition is implied by a
variety of interesting special cases.  For instance, it is satisfied when 
\begin{equation*}
\exists \mu^{\prime},\kappa \in [1,\infty) \text{ s.t. } \forall h \in \C, \er(h) - \nu(\C;\PXY) \geq \mu^{\prime} \cdot \Px(x : h(x) \neq f(x))^{\kappa}.
\end{equation*}
It is also satisfied when $\nu(\C;\PXY) = \nu^*(\PXY)$ and
\begin{equation*}
\exists \mu^{\prime\prime}, \alpha \in (0,\infty) \text{ s.t. } \forall \eps > 0, \Px( x : |\eta(x; \PXY) - 1/2| \leq \eps) \leq \mu^{\prime\prime} \cdot \eps^{\alpha},
\end{equation*}
where $\kappa$ and $\mu$ are functions of $\alpha$ and $\mu^{\prime\prime}$ \citep*{mammen:99,tsybakov:04};
in particular, $\kappa = (1+\alpha)/\alpha$.  Special cases of this condition have also been studied in depth;
for instance, \emph{bounded noise} conditions, wherein $\nu(\C;\PXY) = \nu^*(\PXY)$ and $\forall x, |\eta(x; \PXY) -1/2| > c$
for some constant $c > 0$ \citep*[e.g.,][]{gine:06,massart:06}, are a special case of Condition~\ref{con:tsybakov} with $\kappa = 1$.

Condition~\ref{con:tsybakov} can be interpretted in a variety of ways, depending on the context.
For instance, in certain concept spaces with a geometric interpretation, it can often be realized as a
kind of \emph{large margin} condition, under some condition
relating the noisiness of a point's label to its distance from the optimal decision surface.  That is,
if the magnitude of noise ($1/2-|\eta(x;\PXY) - 1/2|$) for a given point depends inversely on its distance
from the optimal decision surface, so that points closer to the decision surface have noisier labels,
a small value of $\kappa$ in Condition~\ref{con:tsybakov} will occur if the distribution $\Px$
has \emph{low density} near the optimal decision surface (assuming $\nu(\C;\PXY) = \nu^*(\PXY)$)
\citep*[e.g.,][]{dekel:10}.
On the other hand, when there is \emph{high} density near the optimal decision surface, the value of
$\kappa$ may be determined by how quickly $\eta(x; \PXY)$ changes as $x$ approaches the decision
boundary \citep*{castro:08}.
See the works of \citet*{mammen:99,tsybakov:04,koltchinskii:06,massart:06,castro:08,dekel:10,bartlett:06}
for further interpretations of Condition~\ref{con:tsybakov}.

In the context of passive learning, one natural method to study is that of \emph{empirical risk minimization}.
Recall that a passive learning algorithm $\alg$ is called an empirical risk minimization algorithm for $\C$
if it returns a classifier from $\C$ making the minimum number of mistakes on
the labeled sample it is given as input.  It is known that for any VC class $\C$, for any $\PXY$ satisfying
Condition~\ref{con:tsybakov} for finite $\mu$ and $\kappa$, every empirical risk minimization
algorithm for $\C$ achieves a label complexity
\begin{equation}
\label{eqn:passive-tsybakov}
\Lambda(\nu + \eps, \PXY) = O\left( \eps^{\frac{1}{\kappa} - 2} \cdot \log\frac{1}{\eps}\right).
\end{equation}
This follows from the works of \citet*{koltchinskii:06} and \citet*{massart:06}.
Furthermore, for nontrivial concept spaces, one can show that
$\inf_{\Lambda} \sup_{\PXY} \Lambda(\nu + \eps ; \PXY) = \Omega\left( \eps^{\frac{1}{\kappa}-2} \right)$,
where the supremum ranges over all $\PXY$ satisfying
Condition~\ref{con:tsybakov} for the given $\mu$ and $\kappa$ values,
and the infimum ranges over all label complexities achievable by passive learning algorithms \citep*{castro:08,hanneke:11a};
that is, the bound \eqref{eqn:passive-tsybakov} cannot be significantly improved
by any passive algorithm, without allowing the label complexity to have a more refined dependence on $\PXY$
than afforded by Condition~\ref{con:tsybakov}.

In the context of active learning, a variety of results are presently known, which in some cases show improvements over \eqref{eqn:passive-tsybakov}.
Specifically, for any VC class $\C$ and any $\PXY$ satisfying Condition~\ref{con:tsybakov},
a certain noise-robust disagreement-based active learning algorithm achieves label complexity
\begin{equation}
\label{eqn:hanneke10b-bound}
\Lambda(\nu + \eps, \PXY) = O\left( \dc_{f}\left(\eps^{\frac{1}{\kappa}}\right) \cdot \eps^{\frac{2}{\kappa}-2} \cdot \log^{2}\frac{1}{\eps}\right).
\end{equation}
This general result was established by \citet*{hanneke:11a} (analyzing the algorithm of \citet*{dasgupta:07}),
generalizing earlier $\C$-specific results by \citet*{castro:08} and \citet*{balcan:07},
and was later simplified and refined in some cases by \citet*{koltchinskii:10}.
Comparing this to \eqref{eqn:passive-tsybakov}, when $\dc_{f} < \infty$ this is an improvement over passive learning by a factor of $\eps^{\frac{1}{\kappa}} \cdot \log (1/\eps)$.
Note that this generalizes the label complexity bound of Corollary~\ref{cor:cal} above, since the realizable case entails Condition~\ref{con:tsybakov} with $\kappa = \mu/2 = 1$.
It is also known that this type of improvement is essentially the best we can hope for when we describe $\PXY$ purely in terms of the parameters of Condition~\ref{con:tsybakov}.
Specifically, for any nontrivial concept space $\C$,
$\inf_{\Lambda} \sup_{\PXY} \Lambda(\nu + \eps,\PXY) = \Omega\left( \max\left\{\eps^{\frac{2}{\kappa}-2}, \log\frac{1}{\eps}\right\} \right)$,
where 
%again 
the supremum ranges over all $\PXY$ satisfying Condition~\ref{con:tsybakov} for the given $\mu$ and $\kappa$ values,
and the infimum ranges over all label complexities achievable by active learning 
algorithms 
\citep*{hanneke:11a,castro:08}.

In the following subsection, we review the established techniques and results for disagreement-based agnostic active learning;
the algorithm presented there is slightly different from that originally analyzed by \citet*{hanneke:11a},
but the label complexity bounds of \citet*{hanneke:11a} hold for this new algorithm as well.
We follow this in Subsection~\ref{subsec:robust-tsybakov} with a new agnostic active learning method that goes beyond
disagreement-based learning, again generalizing the notion of disagreement to the notion of shatterability; this can
be viewed as analogous to the generalization of \CAL~represented by \Shattering, and as in that case the resulting
label complexity bound replaces $\dc_{f}(\cdot)$ with $\hdc_{f}(\cdot)$.

For both passive and active learning, results under Condition~\ref{con:tsybakov} are also known for more general scenarios
than VC classes: namely, entropy conditions \citep*{mammen:99,tsybakov:04,koltchinskii:06,koltchinskii:08,massart:06,castro:08,hanneke:11a,koltchinskii:10}.
For a nonparametric class known as \emph{boundary fragments}, \citet*{castro:08} find that active learning sometimes offers
advantages over passive learning, under a special case of Condition~\ref{con:tsybakov}.  Furthermore, \citet*{hanneke:11a}
shows a general result on the label complexity achievable by disagreement-based agnostic active learning, which sometimes exhibits an
improved dependence on the parameters of Condition~\ref{con:tsybakov} under conditions on the disagreement coefficient
and certain entropy conditions for $(\C,\Px)$ \citep*[see also][]{koltchinskii:10}.
These results will not play a role in the discussion below, as in the present work we restrict
ourselves strictly to VC classes, leaving more general results for future investigations.

\subsection{Disagreement-Based Agnostic Active Learning}
\label{subsec:disagreement-based-agnostic}

Unlike the realizable case, here in the agnostic case we cannot eliminate a classifier from the version space after making merely a single mistake, since even the best classifier is potentially imperfect.
Rather, we take a collection of samples with labels, and eliminate those classifiers making significantly more mistakes relative to some others in the version space.
This is the basic idea underlying most of the known agnostic active learning algorithms, including those discussed in the present work.
The precise meaning of ``significantly more,'' sufficient to guarantee the version space always contains some good classifier, is
typically determined by established bounds on the deviation of excess empirical error rates from excess true error rates, taken from the passive learning literature.

The following disagreement-based algorithm is slightly different from any in the existing literature, but is similar in style to a method of \citet*{beygelzimer:09};
it also bares resemblence to the algorithms of \citet*{koltchinskii:10,dasgupta:07,balcan:06,balcan:09}.
It should be considered as representative of the family of disagreement-based agnostic active learning algorithms,
and all results below concerning it have analogous results for variants of these other disagreement-based methods.

\begin{bigboxit}
\RobustCAL \\
Input: label budget $n$, confidence parameter $\conf$ \\
Output: classifier $\hat{h}$\\
{\vskip -2mm}\line(1,0){419}\\
0. $m \gets 0$, $i \gets 0$, $V_{0} \gets \C$, $\L_{1} \gets \emptyset$ \\
1. While $t < n$ and $m \leq 2^{n}$\\
2. \quad $m \gets m+1$\\
3. \quad If $X_{m} \in \DIS\left(V_{i}\right)$\\
4. \qquad Request the label $Y_{m}$ of $X_{m}$, and let $\L_{i +1} \gets \L_{i+1} \cup \{(X_{m},Y_{m})\}$ and $t \gets t+1$\\
5. \quad Else let $\hat{y}$ be the label agreed upon by classifiers in $V_i$, and $\L_{i+1} \gets \L_{i+1} \cup \{(X_{m},\hat{y})\}$\\
6. \quad If $m = 2^{i+1}$ \\
7. \qquad $V_{i+1} \gets \left\{ h \in V_{i} : \er_{\L_{i+1}}(h) - \min\limits_{h^{\prime} \in V_{i}} \er_{\L_{i+1}}(h^{\prime}) \leq \hat{U}_{i+1}\left(V_{i}, \conf\right)\right\}$ \\
8. \qquad $i \gets i + 1$, and then $\L_{i+1} \gets \emptyset$\\
9. Return any $\hat{h} \in V_{i}$
\end{bigboxit}

The algorithm is specified in terms of an estimator, $\hat{U}_{i}$.
The definition of $\hat{U}_{i}$ should typically be based on generalization bounds known for passive learning.
Inspired by the work of \citet*{koltchinskii:06} and applications thereof in active learning
\citep*{hanneke:11a,koltchinskii:10}, we will take a definition of $\hat{U}_{i}$ based on a
data-dependent Rademacher complexity, as follows.
Let $\xi_1, \xi_2,\ldots$ denote a sequence of independent Rademacher random variables (i.e., uniform in $\{-1,+1\}$),
also independent from all other random variables in the algorithm (i.e., $\Data$).
Then for any set $\H \subseteq \C$, define
\begin{align}
\hat{R}_{i}(\H) & = \sup_{h_1,h_2 \in \H} 2^{-i} \sum_{m = 2^{i-1}+1}^{2^{i}} \xi_{m} \cdot (h_1(X_m) - h_2(X_m)), \notag
\\ \hat{D}_{i}(\H) & = \sup_{h_1,h_2 \in \H} 2^{-i} \sum_{m = 2^{i-1}+1}^{2^{i}} |h_1(X_m) - h_2(X_m)|, \notag
\\ \hat{U}_{i}(\H,\conf) & = 12 \hat{R}_{i}(\H) + 34 \sqrt{\hat{D}_{i}(\H) \frac{\ln (32 i^2 / \conf)}{2^{i-1}}} + \frac{752 \ln(32 i^2 /\conf)}{2^{i-1}}. \label{eqn:hatU-defn}
\end{align}

\RobustCAL~operates by repeatedly doubling the sample size $|\L_{i+1}|$, while only requesting the labels of the points in the region of disagreement of the version space.
Each time it doubles the size of the sample $\L_{i+1}$, it updates the version space by eliminating any classifiers that make significantly more mistakes on $\L_{i+1}$
relative to others in the version space.  Since the labels of the examples we infer in Step 5 are agreed upon by all elements of the version space,
the \emph{difference} of empirical error rates in Step 7 is identical to the difference of empirical error rates under the \emph{true} labels.  This allows us to
use established results on deviations of excess empirical error rates from excess true error rates to judge suboptimality of some of the classifiers in
the version space in Step 7, thus reducing the version space.

As with \CAL, for computational feasibility, the sets $V_i$ and $\DIS(V_i)$ in \RobustCAL~can be represented implicitly by
a set of constraints imposed by previous rounds of the loop.  Also, the update to $\L_{i+1}$ in Step 5
is included only to make Step 7 somewhat simpler or more intuitive; it can be be removed without
altering the behavior of the algorithm, as long as we compensate by multiplying $\er_{\L_{i+1}}$ by an
appropriate renormalization constant in Step 7: namely, $2^{-i} |\L_{i+1}|$.

We have the following result about the label complexity of \RobustCAL; it is representative
of the type of theorem one can prove about disagreement-based
active learning under Condition~\ref{con:tsybakov}.

\begin{lemma}
\label{lem:dis-based-tsybakov}
Let $\C$ be a VC class and suppose the joint distribution $\PXY$ on $\X \times \{-1,+1\}$ satisfies Condition~\ref{con:tsybakov}
for finite parameters $\mu$ and $\kappa$.  There is a $(\C,\PXY)$-dependent constant $c \in (0,\infty)$ such that, for any
$\eps, \conf \in (0,e^{-3})$, and any integer
\begin{equation*}
n \geq c \cdot \dc_{f}\left( \eps^{\frac{1}{\kappa}} \right) \cdot \eps^{\frac{2}{\kappa}-2} \cdot \log^{2}\frac{1}{\eps\conf},
\end{equation*}
if $\hat{h}_n$ is the output of \RobustCAL~when run with label budget $n$ and confidence parameter $\conf$, then
on an event of probability at least $1-\conf$,
\begin{equation*}
\er\left(\hat{h}_n\right) \leq \nu + \eps.
\end{equation*}
\upthmend{-1.4cm}
\end{lemma}

The proof of this result is essentially similar to the proof by \citet*{hanneke:11a},
combined with some simplifying ideas from \citet*{koltchinskii:10}.
It is also implicit in the proof of Lemma~\ref{lem:robust-tsybakov} below (by replacing ``$\bdim_f$'' with ``$1$'' in the proof).
The details are omitted.
This result leads immediately to the following implication concerning the label complexity.

\begin{theorem}
\label{thm:dis-based-tsybakov}
Let $\C$ be a VC class and suppose the joint distribution $\PXY$ on $\X \times \{-1,+1\}$
satisfies Condition~\ref{con:tsybakov} for finite parameters $\mu, \kappa \in (1,\infty)$.
With an appropriate $(n,\kappa)$-dependent setting of $\conf$, \RobustCAL~achieves a label
complexity $\Lambda_{a}$ with
\begin{equation*}
\Lambda_{a}(\nu + \eps,\PXY) = O\left( \dc_{f}\left(\eps^{\frac{1}{\kappa}}\right) \cdot \eps^{\frac{2}{\kappa}-2} \cdot \log^{2}\frac{1}{\eps}\right).
\end{equation*}
\upthmend{-1.35cm}
\end{theorem}
\begin{proof}
Taking $\conf = n^{-\frac{\kappa}{2\kappa-2}}$, the result follows by simple algebra.
\end{proof}

We should note that it is possible to design a kind of wrapper to adaptively determine an appropriate $\conf$ value,
so that the algorithm achieves the label complexity guarantee of Theorem~\ref{thm:dis-based-tsybakov} without requiring
any explicit dependence on the noise parameter $\kappa$.  Specifically, one can use an idea similar to the model selection
procedure of \citet*{hanneke:11a} for this purpose.  However, as our focus in this work is on moving beyond
disagreement-based active learning, we do not include the details of such a procedure here.

Note that Theorem~\ref{thm:dis-based-tsybakov} represents an improvement over the known results for passive learning (namely, \eqref{eqn:passive-tsybakov})
whenever $\dc_{f}(\eps)$ is small, and in particular this gap can be large when $\dc_{f} < \infty$.
The results of Lemma~\ref{lem:dis-based-tsybakov} and Theorem~\ref{thm:dis-based-tsybakov} represent the state-of-the-art (up to logarithmic factors)
in our understanding of the label complexity of agnostic active learning for VC classes. Thus, any significant improvement over these would advance
our understanding of the fundamental capabilities of active learning in the presence of label noise.  Next, we provide such an improvement.

\subsection{A New Type of Agnostic Active Learning Algorithm Based on Shatterable Sets}
\label{subsec:robust-pseudo-activizer}

\RobustCAL~and Theorem~\ref{thm:dis-based-tsybakov} represent natural extensions of \CAL~and Theorem~\ref{thm:cal} to the agnostic setting.
As such, they not only benefit from the advantages of those methods (small $\dc_{f}(\eps)$ implies improved label complexity),
but also suffer the same disadvantages ($\Px(\partial f) > 0$ implies no strong improvements over passive).  It is therefore natural to
investigate whether the improvements offered by \Shattering~and the corresponding Theorem~\ref{thm:sequential-activizer}
can be extended to the agnostic setting in a similar way.  In particular, as was possible for Theorem~\ref{thm:sequential-activizer}
with respect to Theorem~\ref{thm:cal}, we might wonder whether it is possible to replace $\dc_{f}\left(\eps^{\frac{1}{\kappa}}\right)$
in Theorem~\ref{thm:dis-based-tsybakov} with $\hdc_{f}\left(\eps^{\frac{1}{\kappa}}\right)$ by a modification of \RobustCAL~analogous
to the modification of \CAL~embodied in \Shattering.  As we have seen, $\hdc_{f}\left(\eps^{\frac{1}{\kappa}}\right)$
is often significantly smaller in its asymptotic dependence on $\eps$, compared to $\dc_{f}\left(\eps^{\frac{1}{\kappa}}\right)$,
in many cases even bounded by a finite constant when $\dc_{f}\left(\eps^{\frac{1}{\kappa}}\right)$ is not.  This would therefore
represent a significant improvement over the known results for active learning under Condition~\ref{con:tsybakov}.
Toward this end, consider the following algorithm.

\begin{bigboxit}
\RobustShattering \\ 
Input: label budget $n$, confidence parameter $\conf$\\
Output: classifier $\hat{h}$\\
{\vskip -2mm}\line(1,0){419}\\
0. $m \gets 0$, $i_{0} \gets 0$, $V_{0} \gets \C$\\ 
1. For $k = 1, 2, \ldots, d+1$\\
2.\quad $t \gets 0$, $i_{k} \gets i_{k-1}$, $m \gets 2^{i_{k}}$, $V_{i_{k}+1} \gets V_{i_{k}}$, $\L_{i_k + 1} \gets \emptyset$\\
3.\quad While $t < \left\lfloor 2^{-k} n \right\rfloor$ and $m \leq k \cdot 2^{n}$\\
4.\qquad $m \gets m+1$\\
5.\qquad If $\hat{P}_{4m}\left(S \in \X^{k-1} : V_{i_k+1} \text{ shatters } S \cup \{X_{m}\} | V_{i_k+1} \text{ shatters } S\right) \geq 1/2$\\
6.\qquad\quad Request the label $Y_{m}$ of $X_{m}$, and let $\L_{i_k +1} \gets \L_{i_k+1} \cup \{(X_{m},Y_{m})\}$ and $t \gets t+1$\\
7.\qquad Else
$\hat{y} \!\gets\!\!\! \argmax\limits_{y \in \{-1,+1\}} \!\!\hat{P}_{4m}\!\!\left(S \in \X^{k-1} \!:\! V_{i_k +1}[(X_{m},\!-y)] \text{ does not shatter } S | V_{i_{k}+1} \text{ shatters } S \right)$\\
8.\qquad\qquad $\L_{i_k+1} \gets \L_{i_k+1} \cup \{(X_{m},\hat{y})\}$ and $V_{i_k+1} \gets V_{i_k+1}[(X_{m},\hat{y})]$\\
9.\qquad If $m = 2^{i_k+1}$ \\
10.\qquad\quad $V_{i_k+1} \gets \left\{ h \in V_{i_k+1} : \er_{\L_{i_k+1}}(h) - \min\limits_{h^{\prime} \in V_{i_k+1}} \er_{\L_{i_k+1}}(h^{\prime}) \leq \hat{U}_{i_k+1}\left(V_{i_k}, \conf\right)\right\}$ \\
11.\qquad\quad $i_k \gets i_k + 1$, then $V_{i_k+1} \gets V_{i_k}$, and $\L_{i_k+1} \gets \emptyset$\\
12. Return any $\hat{h} \in V_{i_{d+1}+1}$
\end{bigboxit}

For the $\argmax$ in Step 7, we break ties in favor of a $\hat{y}$ value with $V_{i_k+1}[(X_m,\hat{y})] \neq \emptyset$ to maintain the
invariant that $V_{i_k + 1} \neq \emptyset$ (see the proof of Lemma~\ref{lem:robust-good-labels});
when both $y$ values satisfy this, we may break ties arbitrarily.
The procedure is specified in terms of several estimators.
The $\hat{P}_{4m}$ estimators, as usual, are defined in Appendix~\ref{app:hatP-definitions}.
For $\hat{U}_{i}$, we again use the definition \eqref{eqn:hatU-defn} above,
based on a data-dependent Rademacher complexity.

\RobustShattering~is largely based on the same principles as \RobustCAL, combined with \Shattering.
As in \RobustCAL, the algorithm proceeds by repeatedly doubling the size of a labeled sample $\L_{i+1}$,
while only requesting a subset of the labels in $\L_{i+1}$, inferring the others.  As before, it updates
the version space every time it doubles the size of the sample $\L_{i+1}$, and the update eliminates
classifiers from the version space that make significantly more mistakes on $\L_{i+1}$ compared to others
in the version space.  In \RobustCAL, this is guaranteed to be effective, since the classifiers
in the version space agree on all of the inferred labels, so that the differences of empirical error rates
remain equal to the \emph{true} differences of empirical error rates (i.e., under the true $Y_m$ labels
for all elements of $\L_{i+1}$); thus, the established results from the passive learning literature bounding
the deviations of excess empirical error rates from excess true error rates can be applied, showing that this does not
eliminate the best classifiers.  In \RobustShattering, the situation is somewhat more subtle, but the principle
remains the same.  In this case, we \emph{enforce} that the classifiers in the version space agree on
the inferred labels in $\L_{i+1}$ by explicitly removing the disagreeing classifiers in Step 8.  Thus,
as long as Step 8 does not eliminate all of the good classifiers, then neither will Step 10.  To argue that
Step 8 does not eliminate all good classifiers, we appeal to the same reasoning as for \BasicActivizer~and
\Shattering.  That is, for $k \leq \bdim_{f}$ and sufficiently large $n$, as long as there exist
good classifiers in the version space, the labels $\hat{y}$ inferred in Step 7 will agree with some good
classifiers, and thus Step 8 will not eliminate all good classifiers.  However, for $k > \bdim_{f}$, the
labels $\hat{y}$ in Step 7 have no such guarantees, so that we are only guaranteed that \emph{some}
classifier in the version space is not eliminated.  Thus, determining guarantees on the error rate of this algorithm
hinges on bounding the worst excess error rate among all classifiers in the version space at the conclusion of
the $k=\bdim_{f}$ round.  This is essentially determined by the size of $\L_{i_{k}}$ at the conclusion of that round,
which itself is largely determined by how frequently the algorithm requests labels during this $k=\bdim_{f}$ round.
Thus, once again the analysis rests on bounding the rate at which the frequency of label requests shrinks in the $k=\bdim_{f}$
round, which determines the rate of growth of $|\L_{i_{k}}|$, and thus the final guarantee on the excess error rate.

As before, for computational feasibility, we can maintain the sets $V_{i}$ implicitly
as a set of constraints imposed by the previous updates,
so that we may perform the various calculations required for the estimators $\hat{P}$
as constrained optimizations.
Also, the update to $\L_{i_k+1}$ in Step 8 is merely included to make the algorithm statement
and the proofs somewhat more elegant; it can be omitted, as long as we compensate with an
appropriate renormalization of the $\er_{\L_{i_k+1}}$ values in Step 10 (i.e., multiplying by $2^{-i_k} |\L_{i_k+1}|$).
Additionally, the same potential improvements we proposed in Section~\ref{subsec:exp-alternatives}
for \Shattering~can be made to \RobustShattering~as well, again with only minor modifications to the proofs.

We should note that this is certainly not the only reasonable way to extend \Shattering~to the agnostic setting.
For instance, another natural extension of \BasicActivizer~to the agnostic setting, based on
a completely different idea, appears in the author's doctoral dissertation \citep*{hanneke:thesis}; 
that method can be improved in a natural way to take advantage of the sequential
aspect of active learning, yielding an agnostic extension of \Shattering~differing from \RobustShattering~in
several interesting ways.

In the next subsection, we will see that the label complexities achieved by  \RobustShattering~are often significantly
better than the known results for passive learning.  In fact, they are often significantly better than the presently-known
results for any \emph{active} learning algorithms in the published literature.

\subsection{Improved Label Complexity Bounds for Active Learning with Noise}
\label{subsec:robust-tsybakov}

Under Condition~\ref{con:tsybakov}, we can extend Lemma~\ref{lem:dis-based-tsybakov} and Theorem~\ref{thm:dis-based-tsybakov}
in an analogous way to how Theorem~\ref{thm:sequential-activizer} extends Theorem~\ref{thm:cal}.
Specifically, we have the following result, the proof of which is included in Appendix~\ref{app:robust-tsybakov}.

\newpage
\begin{lemma}
\label{lem:robust-tsybakov}
Let $\C$ be a VC class and suppose the joint distribution $\PXY$ on $\X \times \{-1,+1\}$
satisfies Condition~\ref{con:tsybakov} for finite parameters $\mu$ and $\kappa$.
There is a $(\C,\PXY)$-dependent constant $c \in (0,\infty)$ such that,
for any $\eps, \conf \in \left(0,e^{-3}\right)$, and any integer
\begin{equation*}
n  \geq c \cdot \hdc_{f}\left(\eps^{\frac{1}{\kappa}}\right) \cdot \eps^{\frac{2}{\kappa}-2} \cdot \log^{2}\frac{1}{\eps \conf},
\end{equation*}
if $\hat{h}_n$ is the output of \RobustShattering~when run with label budget $n$
and confidence parameter $\conf$,
then on an event of probability at least $1-\conf$,
\begin{equation*}
\er\left(\hat{h}_n\right) \leq \nu + \eps.
\end{equation*}
\upthmend{-1.3cm}
\end{lemma}

This has the following implication for the label complexity of \RobustShattering.

\begin{theorem}
\label{thm:robust-tsybakov}
Let $\C$ be a VC class and suppose the joint distribution $\PXY$ on $\X \times \{-1,+1\}$
satisfies Condition~\ref{con:tsybakov} for finite parameters $\mu, \kappa \in (1,\infty)$.
With an appropriate $(n,\kappa)$-dependent setting of $\conf$,
\RobustShattering~achieves a label complexity $\Lambda_{a}$ with
\begin{equation*}
\Lambda_{a}(\nu + \eps, \PXY) = O\left(\hdc_{f}\left(\eps^{\frac{1}{\kappa}}\right) \cdot \eps^{\frac{2}{\kappa} - 2} \cdot \log^{2} \frac{1}{\eps}\right).
\end{equation*}
\upthmend{-1.3cm}
\end{theorem}
\begin{proof}
Taking $\conf = n^{-\frac{\kappa}{2\kappa-2}}$, the result follows by simple algebra.
\end{proof}

Theorem~\ref{thm:robust-tsybakov} represents an interesting generalization beyond the realizable case,
and beyond the disagreement coefficient analysis.
Note that if $\hdc_{f}(\eps) = o\left(\eps^{-1} \log^{-2}(1/\eps)\right)$, Theorem~\ref{thm:robust-tsybakov}
represents an improvement over the known results for passive learning \citep*{massart:06}.  As we always have
$\hdc_{f}(\eps) = o\left(\eps^{-1}\right)$, we should typically expect such improvements for all
but the most extreme learning problems.  Recall that $\dc_{f}(\eps)$ is often \emph{not} $o\left(\eps^{-1}\right)$,
so that Theorem~\ref{thm:robust-tsybakov} is often a much stronger statement than Theorem~\ref{thm:dis-based-tsybakov}.
In particular, this is a significant improvement over the known results for passive learning whenever $\hdc_{f} < \infty$,
and an equally significant improvement over Theorem~\ref{thm:dis-based-tsybakov} whenever $\hdc_{f} < \infty$
but $\dc_{f}(\eps) = \Omega(1/\eps)$ (see above for examples of this).
However, note that unlike \Shattering, \RobustShattering~is \emph{not}
an activizer.  Indeed, it is not clear (to the author) how to modify the algorithm to make it a universal activizer
(even for the realizable case), while maintaining the guarantees of Theorem~\ref{thm:robust-tsybakov}.

As with Theorem~\ref{thm:sequential-activizer} and Corollary~\ref{cor:sequential-activizer},
\RobustShattering~and Theorem~\ref{thm:robust-tsybakov} can potentially be improved in a variety
of ways, as outlined in Section~\ref{subsec:exp-alternatives}.  In particular,
Theorem~\ref{thm:robust-tsybakov} can be made slightly sharper in some cases by replacing
$\hdc_{f}\left( \eps^{\frac{1}{\kappa}} \right)$ with the sometimes-smaller (though more complicated)
quantity \eqref{eqn:tighter-hdc} (with $r_0 = \eps^{\frac{1}{\kappa}}$).

\subsection{Beyond Condition~\ref{con:tsybakov}}
\label{subsec:noise-conditions}

While Theorem~\ref{thm:robust-tsybakov} represents an improvement over the known results for agnostic active learning,
Condition~\ref{con:tsybakov} is not fully general, and disallows many important and interesting scenarios.
In particular, one key property of Condition~\ref{con:tsybakov}, heavily exploited in the label complexity proofs for both
passive learning and disagreement-based active learning, is that it implies $\diam(\C(\eps)) \to 0$ as $\eps \to 0$.
In scenarios where this shrinking diameter condition is not satisfied, the existing proofs of \eqref{eqn:passive-tsybakov} for passive learning break down,
and furthermore, the disagreement-based algorithms themselves cease to give significant improvements over passive learning, for essentially the
same reasons leading to the ``only if'' part of Theorem~\ref{thm:naive} (i.e., the sampling region never focuses beyond some nonzero-probability region).
Even more alarming (at first glance) is the fact that this same problem
can sometimes be observed for the $k=\bdim_{f}$ round of \RobustShattering; that is,
$\Px\left( x : \Px^{\bdim_{f}-1}( S \in \X^{\bdim_{f}-1} : V_{i_{\bdim_f}+1} \text{ shatters } S \cup \{x\} | V_{i_{\bdim_f}+1} \text{ shatters } S) \geq 1/2\right)$
is no longer guaranteed to approach $0$ as the budget $n$ increases (as it \emph{does} when $\diam(\C(\eps)) \to 0$).

Thus, if we wish to approach an understanding of improvements achievable by active learning in general,
we must come to terms with scenarios where $\diam(\C(\eps))$ does not shrink to zero.
Toward this goal, it will be helpful to partition the distributions into two distinct categories, which we
will refer to as the \emph{benign noise} case and the \emph{misspecified model} case.
The $\PXY$ in the benign noise case are characterized by the property that $\nu(\C; \PXY) = \nu^*(\PXY)$;
this is in some ways similar to the realizable case, in that $\C$ can approximate an optimal classifier,
except that the labels are stochastic.  In the benign noise case, the only reason $\diam(\C(\eps))$
would not shrink to zero is if there is a nonzero probability set of points $x$
with $\eta(x) = 1/2$; that is, there are at least two classifiers achieving the Bayes error rate,
and they are at nonzero distance from each other, which must mean they disagree on some points
that have equal probability of either label occurring.

Interestingly, it seems that in the benign noise case, $\diam(\C(\eps)) \nrightarrow 0$ might not be a problem for
algorithms based on shatterable sets, such as \RobustShattering.  In particular, \RobustShattering~appears to continue exhibiting reasonable
behavior in such scenarios.
That is, even if there is a nonshrinking probability that the query condition
in Step 5 is satisfied for $k = \bdim_f$, on any given sequence $\Data$ there must be \emph{some} smallest value of $k$ for which
this probability \emph{does} shrink as $n\to\infty$.  For this value of $k$, we should expect to observe good behavior from the algorithm,
in that (for sufficiently large $n$) the inferred labels in Step 7 will tend to agree with \emph{some} optimal
classifier.  Thus, the algorithm addresses the problem of multiple optimal classifiers by effectively
\emph{selecting} one of the optimal classifiers.

To illustrate this phenomenon, consider learning with respect to the space of threshold
classifiers (Example~\ref{ex:thresholds}) with $\Px$ uniform in $[0,1]$,
and let $(X,Y) \sim \PXY$ satisfy $\P(Y=+1 | X) = 0$
for $X < 1/3$, $\P(Y=+1|X) = 1/2$ for $1/3 \leq X < 2/3$, and $\P(Y=+1|X) = 1$ for $2/3 \leq X$.
As we know from above, $\bdim_f = 1$ here.  However, in this scenario we have
$\DIS(\C(\eps)) \to [1/3,2/3]$ as $\eps \to 0$.
Thus, \RobustCAL~never focuses its queries beyond
a constant fraction of $\X$, and therefore cannot improve over certain passive learning algorithms in terms of the
asymptotic dependence of its label complexity on $\eps$ (assuming a worst-case choice of $\hat{h}$ in Step 9).
However, for $k=2$ in \RobustShattering,
every $X_m$ will be assigned a label $\hat{y}$ in Step 7 (since no $2$ points are shattered); furthermore,
for sufficiently large $n$ we have (with high probability) $\DIS(V_{i_1})$ not too much larger than $[1/3,2/3]$,
so that most points in $\DIS(V_{i_1})$ can be labeled either $+1$ or $-1$ by some optimal
classifier.  For us, this has two implications.  First, the $S \in [1/3,2/3]^1$ will (with high probability)
dominate the votes for $\hat{y}$ in Step 7, so that the $\hat{y}$ inferred for any $X_{m} \notin [1/3,2/3]$
will agree with all of the optimal classifiers.  Second, the inferred labels $\hat{y}$ for $X_{m} \in [1/3,2/3]$ will
definitely agree with \emph{some} optimal classifier.
Since we also impose the $h(X_m) = \hat{y}$ constraint for $V_{i_{2}+1}$ in Step 8,
the inferred $\hat{y}$ labels must all be consistent with the \emph{same} optimal
classifier, so that $V_{i_2 + 1}$ will quickly converge to within a small neighborhood around that classifier,  
without any further label requests.  Note, however, that the particular optimal classifier
the algorithm converges to will be a random variable, determined by the particular
sequence of data points processed by the algorithm; thus, it cannot be determined a priori,
which significantly complicates any general attempt to analyze the label complexity achieved by
the algorithm for arbitrary $\C$ and $\PXY$ satisfying the benign noise condition.  In particular,
for some $\C$ and $\PXY$, even this minimal $k$ for which convergence occurs may be a nondeterministic random variable.
At this time, it is not entirely clear how general this phenomenon
is (i.e., \RobustShattering~providing improvements over certain passive algorithms even for benign noise distributions with $\diam(\C(\eps)) \nrightarrow 0$),
nor how to characterize the label complexity achieved by \RobustShattering~in general benign noise settings where
$\diam(\C(\eps)) \nrightarrow 0$.

However, as mentioned earlier, there are other natural ways to generalize \Shattering~to handle noise,
some of which have more predictable behavior in the general benign noise setting.
In particular, the original thesis work of \citet*{hanneke:thesis} explores a technique for active learning with benign noise,
which unlike \RobustShattering, only uses the \emph{requested} labels,
not the inferred labels, and as a consequence never eliminates any optimal classifier from $V$.
Because of this fact, the sampling region for each $k$ converges to a predictable limiting region,
so that we have an accurate \emph{a priori} characterization of the algorithm's behavior.
However, it is not immediately clear (to the author) whether this alternative technique might lead
to a method achieving results similar to Theorem~\ref{thm:robust-tsybakov}.

In contrast to the benign noise case, in the misspecified model case we have $\nu(\C ; \PXY) > \nu^*(\PXY)$.  In this case,
if the diameter does not shrink, it is because of the existence of two classifiers $h_1, h_2 \in \cl(\C)$ achieving
error rate $\nu(\C; \PXY)$, with $\Px(x : h_1(x) \neq h_2(x)) > 0$.
However, unlike above, since they do not achieve the Bayes error rate, it is possible that a significant
fraction of the set of points they disagree on may have $\eta(x) \neq 1/2$.  Intuitively, this
makes the active learning problem more difficult, as there is a worry that a method such as
\RobustShattering~might infer the label $h_2(x)$ for some point $x$ when in fact $h_1(x)$ is better for that particular $x$,
and vice versa for the points $x$ where $h_2(x)$ would be better, thus getting the worst of both and potentially
doubling the error rate in the process.  However, it turns out that, for the purpose of exploring Conjecture~\ref{conj:activized-erm},
we can circumvent all of these issues by noting that there is a trivial solution to the
misspecified model case.  Specifically, since in our present context we are only interested in the label complexity
for achieving error rate better than $\nu + \eps$, we can simply turn to any algorithm that
asymptotically achieves an error rate strictly better than $\nu$ \citep[e.g.,][]{devroye:96},
in which case the algorithm should require only a finite constant number of labels to achieve an expected error
rate better than $\nu$.  To make the algorithm effective for the general case, we simply
split our budget in three: one part for an active learning algorithm, such as \RobustShattering, for the benign noise case,
one part for the method above handling the misspecified model case, and one part to select among their
outputs.  The full details of such a procedure are specified in Appendix~\ref{app:misspecified-model-trivial}, along with
a proof of its performance guarantees, which are summarized as follows.

\begin{theorem}
\label{thm:misspecified-model-trivial}
Fix any concept space $\C$.
Suppose there exists an active learning algorithm $\alg_a$ achieving a label complexity $\Lambda_a$.
Then there exists an active learning algorithm $\alg_a^{\prime}$ achieving a label complexity $\Lambda_a^{\prime}$ such that,
for any distribution $\PXY$ on $\X \times \{-1,+1\}$, there exists a function $\lambda(\eps) \in \Polylog(1/\eps)$ such that
\begin{equation*}
\Lambda_a^{\prime}(\nu + \eps, \PXY) \leq
\begin{cases}
\max\left\{ 2 \Lambda_a(\nu + \eps/2, \PXY), \lambda(\eps)\right\}, & \text{ in the benign noise case}\\
\lambda(\eps), & \text{ in the misspecified model case}
\end{cases}.
\end{equation*}
\thmend
\end{theorem}

The main point of Theorem~\ref{thm:misspecified-model-trivial} is that, for our purposes, we can safely ignore the
misspecified model case (as its solution is a trivial extension), and focus entirely on the performance of algorithms
for the benign noise case.  In particular, for any label complexity $\Lambda_{p}$, every $\PXY \in \Nontrivial(\Lambda_{p};\C)$
in the misspecified model case has $\Lambda_{a}^{\prime}(\nu+\eps,\PXY) = o(\Lambda_{p}(\nu+\eps,\PXY))$,
for $\Lambda_{a}^{\prime}$ as in Theorem~\ref{thm:misspecified-model-trivial}.
Thus, if there exists an active meta-algorithm achieving the strong improvement
guarantees of an activizer for some passive learning algorithm $\alg_{p}$ (Definition~\ref{defn:agnostic-activizer})
for all distributions $\PXY$ in the benign noise case, then there exists an activizer for $\alg_{p}$ with respect to $\C$
in the agnostic case.

\section{Open Problems}
\label{sec:open-problems}

In some sense, this work raises more questions than it answers.
Here, we list several problems that remain open at this time.
Resolving any of these problems would make a significant contribution
to our understanding of the fundamental capabilities of active learning.

\begin{itemize}

\item We have established the existence of universal activizers for VC classes in the realizable case.
However, we have not made any serious attempt to characterize the properties that such activizers can possess.
In particular, as mentioned, it would be interesting to know whether activizers exist that \emph{preserve} certain
favorable properties of the given passive learning algorithm.
For instance, we know that some passive learning algorithms (say, for linear separators)
achieve a label complexity that is independent of the dimensionality of the space $\X$,
under a large margin condition on $f$ and $\Px$ \citep*{balcan:06c}.
Is there an activizer for such algorithms that preserves this large-margin-based
dimension-independence in the label complexity?
Similarly, there are passive algorithms whose label complexity has a weak dependence on dimensionality,
due to sparsity considerations \citep*{bunea:07,wang:07}.
Is there an activizer for these algorithms that preserves this sparsity-based weak dependence on dimension?
Is there an activizer that preserves adaptiveness to the dimension of the manifold to which $\Px$ is restricted?
What about an activizer that is \emph{sparsistent} \citep*{rocha:09},
given any sparsistent passive learning algorithm as input?
Is there an activizer that preserves admissibility, in that given any admissible passive learning algorithm,
the activized algorithm is an admissible active learning algorithm?
Is there an activizer that, given any minimax optimal passive learning algorithm as input,
produces a minimax optimal active learning algorithm?
What about preserving other notions of optimality, or other properties?

\item There may be some waste in the above activizers, since the label requests used in their initial
phase (reducing the version space) are not used by the passive algorithm to produce the final
classifier.  This guarantees the examples fed into the passive
algorithm are conditionally independent given the number of examples.
Intuitively, this seems necessary for the general results, since any dependence among
the examples fed to the passive algorithm could influence its label complexity.  However,
it is not clear (to the author) how dramatic this effect can be, nor whether a simpler strategy
(e.g., slightly randomizing the budget of label requests) might yield a similar effect while
allowing a single-stage approach where all labels are used in the passive algorithm.
It seems intuitively clear that some special types of passive algorithms should
be able to use the full set of examples, from both phases, while still maintaining the strict
improvements guaranteed in the main theorems above.  What general properties must such
passive algorithms possess?

\item As previously mentioned, the vast majority of empirically-tested \emph{heuristic} active learning algorithms
in the published literature are designed in a reduction style, using a well-known passive learning algorithm
as a subroutine, constructing sets of labeled examples and feeding them into the passive learning algorithm
at various points in the execution of the active learning algorithm
\citep*[e.g.,][]{abe:98,mccallum:98,schohn:00,campbell:00,tong:01,roy:01,muslea:02,lindenbaum:04,mitra:04,small:06,schein:07,har-peled:07,beygelzimer:09}.
However, rather than including some examples whose labels are requested and other examples whose labels are
\emph{inferred} in the sets of labeled examples given to the passive learning algorithm
(as in our rigorous methods above), these heuristic methods typically only input
to the passive algorithm the examples whose labels were \emph{requested}.
We should expect that meta-algorithms of this type could not be \emph{universal} activizers,
but perhaps there do exist meta-algorithms of this type that are activizers for every passive
learning algorithm of some special type.  What are some general conditions on the
passive learning algorithm so that some meta-algorithm of this type (i.e., feeding in only the
\emph{requested} labels) can activize every passive learning algorithm satisfying those conditions?

\item As discussed earlier, the definition of ``activizer'' is based on a trade-off between the strength of
claimed improvements for nontrivial scenarios, and ease of analysis within the framework.
There are two natural questions regarding the possibility of stronger notions of ``activizer.''
In Definition~\ref{defn:activizer} we allow a constant factor $c$ loss in the $\eps$ argument
of the label complexity.  In most scenarios, this loss is inconsequential (e.g., typically
$\Lambda_{p}(\eps/c,f,\Px) = O(\Lambda_{p}(\eps,f,\Px))$), but one can construct
scenarios where it does make a difference.  In our proofs, we see that it is possible to achieve $c = 3$; in fact,
a careful inspection of the proofs reveals we can even get $c = (1+o(1))$, a function of $\eps$,
converging to $1$.  However, whether there exist universal activizers for every VC class that have
$c=1$ remains an open question.

A second question regards our notion of ``nontrivial problems.''  In Definition~\ref{defn:activizer},
we have chosen to think of any target and distribution with label complexity growing faster than $\Polylog(1/\eps)$ as \emph{nontrivial},
and do not require the activized algorithm to improve over the underlying passive algorithm
for scenarios that are trivial for the passive algorithm.  As mentioned, Definition~\ref{defn:activizer}
does have implications for the label complexities of these problems, as the label complexity of
the activized algorithm will improve over every nontrivial upper bound on the label complexity of
the passive algorithm.  However, in order to allow for various operations in the meta-algorithm
that may introduce additive $\Polylog(1/\eps)$ terms due to exponentially small failure probabilities,
such as the test that selects among hypotheses in $\ActiveSelect$, we do not require the activized
algorithm to achieve the same \emph{order} of label complexity in trivial scenarios.
For instance, there may be cases in which a passive algorithm achieves $O(1)$
label complexity for a particular $(f,\Px)$, but its activized counterpart has $\Theta(\log(1/\eps))$
label complexity.
The intention is to define a framework that focuses on nontrivial scenarios,
where passive learning uses prohibitively many labels, rather than
one that requires us to obsess over extra additive logarithmic terms.
Nonetheless, there is a question of whether these losses in the label complexities of
trivial problems are necessary to gain the improvements in the label complexities of
nontrivial problems.
There is also the question of how much the definition of ``nontrivial'' can be relaxed.
Specifically, we have the following question: to what extent can we relax the notion of
``nontrivial'' in Definition~\ref{defn:activizer}, while still maintaining the existence of
universal activizers for VC classes?
We see from our proofs that we can at least replace $\Polylog(1/\eps)$
with $\log(1/\eps)$.  However, it is not clear whether we can go further than this in the realizable case (e.g., to say ``nontrivial'' means $\omega(1)$).
When there is noise, it is clear that we cannot relax the notion of ``nontrivial'' beyond replacing $\Polylog(1/\eps)$ with $\log(1/\eps)$.
Specifically,  whenever $\DIS(\C) \neq \emptyset$,
for any label complexity $\Lambda_{a}$ achieved by an active learning algorithm, there must be
some $\PXY$ with $\Lambda_{a}(\nu + \eps, \PXY) = \Omega(\log(1/\eps))$, even with the
support of $\Px$ restricted to a \emph{single point} $x \in \DIS(\C)$; the proof of this is
via a reduction from sequential hypothesis testing for whether a coin has bias $\alpha$ or $1-\alpha$,
for some $\alpha \in (0,1/2)$.  Since passive learning via empirical risk minimization can achieve
label complexity $\Lambda_{p}(\nu+\eps,\PXY) = O(\log(1/\eps))$ whenever the support of $\Px$
is restricted to a single point, we cannot further relax the notion of ``nontrivial,'' while preserving
the possibility of a positive outcome for Conjecture~\ref{conj:activized-erm}.
It is interesting to note that this entire issue vanishes if we are only
interested in methods that achieve error at most $\eps$ with probability at least $1-\conf$, where $\conf \in (0,1)$
is some acceptable constant failure probability, as in the work of \citet*{hanneke:10a}; in this case,
we can simply take ``nontrivial'' to mean $\omega(1)$ label complexity, and both \BasicActivizer~and \Shattering~remain
universal activizers under this alternative definition, and achieve $O(1)$ label complexity in trivial scenarios.

\item Another interesting question concerns efficiency.  Suppose there exists an algorithm to find an element of $\C$
consistent with any labeled sequence $\L$ in time polynomial in $|\L|$ and $d$, and that $\alg_{p}(\L)$ has running
time polynomial in $|\L|$ and $d$.  Under these conditions, is there an activizer for $\alg_{p}$ capable of achieving
an error rate smaller than any $\eps$ in running time polynomial in $1/\eps$ and $d$, given some appropriately
large budget $n$?  Recall that if we knew the value of $\bdim_f$ and $\bdim_f \leq c \log d$, then \BasicActivizer~could be made efficient,
as discussed above.  Therefore, this question is largely focused on the issue of adapting to the value of $\bdim_f$.
Another related question is whether there is an efficient active learning algorithm achieving the label complexity
bound of Corollary~\ref{cor:activized-1IG} or Corollary~\ref{cor:sequential-activizer}.

\item One question that comes up in the results above is the minimum number of \emph{batches} of label requests
necessary for a universal activizer.  In \NaiveActivizer~and Theorem~\ref{thm:naive}, we saw that sometimes
two batches are sufficient: one to reduce the version space, and another to construct the labeled sample by
requesting only those points in the region of disagreement.  We certainly cannot use fewer than
two batches in a universal activizer, for any nontrivial concept space, so that this represents the minimum.
However, to get a universal activizer for \emph{every} concept space, we increased the number of batches to
\emph{three} in \BasicActivizer.  The question is whether this increase is really necessary.
Is there always a universal activizer using only \emph{two} batches of label requests, for every VC class $\C$?

\item For some $\C$, the learning process in the above methods might be viewed in two components: one component that performs active learning as usual (say, disagreement-based) under the assumption that the target function is very simple,
and another component that searches for signs that the target function is in fact more complex.  Thus, for some natural classes such as linear separators, it would be interesting to find
simpler, more specialized methods, which explicitly execute these two components.  For instance, for the first component, we might consider the usual margin-based active learning
methods, which query near a current guess of the separator \citep*{dasgupta:05b,dasgupta:09,balcan:07}, except that we bias toward simple hypotheses via a regularization penalty in the optimization that defines how we update the
separator in response to a query.  The second component might then be a simple random search for points whose correct classification requires larger values of the regularization term.

\item Can we construct universal activizers for some concept spaces with infinite VC dimension?
What about under some constraints on the distribution $\Px$ or $\PXY$ (e.g., the usual entropy conditions
\citep*{van-der-Vaart:96})?
It seems we can still run \BasicActivizer, \Shattering, and \RobustShattering~in this case,
except we should increase the number of rounds (values of $k$) as a function of $n$;
this may continue to have reasonable behavior even in some cases where $\bdim_f = \infty$,
especially when $\Px^{k}(\partial^{k} f) \to 0$ as $k \to \infty$.  However, it is not clear whether they will continue
to guarantee the strict improvements over passive learning in the realizable case,
nor what label complexity guarantees they will achieve.  One specific question is whether there is a method
always achieving label complexity
$o\left(\eps^{\frac{1-\rho}{\kappa}-2}\right)$,
where $\rho$ is from the entropy conditions \citep*{van-der-Vaart:96} and $\kappa$ is from
Condition~\ref{con:tsybakov}.  This would be an improvement over the known results for passive learning \citep*{mammen:99,tsybakov:04,koltchinskii:06}.
Another related question is whether we can improve over the known results for active learning in these scenarios.
Specifically, \citet*{hanneke:11a} proved a bound of
$\tilde{O}\left(\dc_{f}\left(\eps^{\frac{1}{\kappa}}\right) \eps^{\frac{2-\rho}{\kappa}-2}\right)$
on the label complexity of a certain disagreement-based active learning method, under entropy
conditions and Condition~\ref{con:tsybakov}.  Do there exist active learning methods achieving
asymptotically smaller label complexities than this, in particular improving the
$\dc_{f}\left(\eps^{\frac{1}{\kappa}}\right)$ factor?
The quantity $\hdc_{f}\left(\eps^{\frac{1}{\kappa}}\right)$ is no longer defined when $\bdim_{f}=\infty$,
so this might not be a direct extension of Theorem~\ref{thm:robust-tsybakov}, but
we could perhaps use the sequence of  $\dc_{f}^{(k)}\left(\eps^{\frac{1}{\kappa}}\right)$ values in some other way
to replace $\dc_{f}\left(\eps^{\frac{1}{\kappa}}\right)$ in this case.

\item There is also a question about generalizing this approach to label spaces other than $\{-1,+1\}$, and possibly other loss functions.
It should be straightforward to extend these results to the setting of multiclass classification.
However, it is not clear what the implications would be for general structured prediction problems, where
the label space may be quite large (even infinite), and the loss function involves a notion of \emph{distance} between labels.
From a practical perspective, this question is particularly interesting, since problems with more complicated label spaces
are often the scenarios where active learning is most needed, as it takes substantial time or effort to label each example.
At this time, there are no published theoretical results on the label complexity improvements achievable for general structured
prediction problems.

\item All of the claims in this work also hold when $\alg_p$ is a \emph{semi-supervised} passive learning algorithm,
simply by withholding a set of unlabeled data points in a preprocessing step, and feeding them into the
passive algorithm along with the labeled set generated by the activizer.  However, it is not clear whether
further claims are possible when activizing a semi-supervised algorithm, for instance by taking into
account specific details of the learning bias used by the particular semi-supervised algorithm (e.g.,
a cluster assumption).

\item The splitting index analysis of \citet*{dasgupta:05} has the interesting feature of characterizing a
\emph{trade-off} between the number of label requests and the number of unlabeled examples used
by the active learning algorithm.  In the present work, we do not characterize any such trade-off.
Indeed, the algorithms do not really have any parameter to adjust the number of unlabeled examples
they use (aside from the precision of the $\hat{P}$ estimators), so that they simply use as many as they need and then halt.  This is true in both the realizable
case and in the agnostic case.  It would be interesting to try to modify these algorithms and their analysis
so that, when there are more unlabeled examples available than would be used by the above methods,
the algorithms can take advantage of this in a way that can be reflected in improved label complexity bounds,
and when there are fewer unlabeled examples available, the algorithms can alter their behavior to compensate
for this, at the cost of an increased label complexity.  This would be interesting both for the realizable and
agnostic cases.  In fact, in the agnostic case, there are no known methods that exhibit this type of trade-off.

\item Finally, as mentioned in the previous section, there is a serious question concerning what types of algorithms can
be activized in the agnostic case, and how large the improvements in label complexity will be.  In particular,
Conjecture~\ref{conj:activized-erm} hypothesizes that for any VC class, we can activize some
empirical risk minimization algorithm in the agnostic case.  Resolving this conjecture (either positively or negatively)
should significantly advance our understanding of the capabilities of active learning compared to passive learning.
\end{itemize}

\appendix

\section{Proofs Related to Section~\ref{sec:naive}: Disagreement-Based Learning}
\label{app:naive}

The following result follows from a theorem of \citet*{anthony:99},
based on the classic results of \citet*{vapnik:82} (with slightly better constant factors);
see also the work of \citet*{blumer:89}.

\begin{lemma}
\label{lem:VinB} 
For any VC class $\C$, $m \in \nats$, and classifier $f$ such that $\forall r > 0, \Ball(f,r) \neq \emptyset$,
let $\truV_{m} = \{h \in \C : \forall i \leq m, h(X_i) = f(X_i)\}$;
for any $\conf \in (0,1)$, there is an event $H_{m}(\conf)$
with $\P\left(H_{m}(\conf)\right) \geq 1 - \conf$ such that,
on $H_{m}(\conf)$, $\truV_{m} \subseteq \Ball(f,\vrad(m ; \conf))$,
where
\begin{equation*}
\vrad(m ; \conf) = 2 \frac{d \ln \frac{2 e \max\{m, d\}}{d} + \ln(2/\conf)}{m}.
\end{equation*}
\upthmend{-1.25cm}\end{lemma}
A fact we will use repeatedly is that, for any $N(\eps) = \omega(\log(1/\eps))$, we have $\vrad(N(\eps) ; \eps) = o(1)$.

\begin{lemma}
\label{lem:naive-estimator}
For $\hat{P}_n(\DIS(V))$ from \eqref{eqn:pdisv-est},
on an event $J_n$ with $\P(J_n) \geq 1 - 2 \cdot \exp\{-n/4\}$,
\begin{equation*}
\max\left\{\Px(\DIS(V)), 4/n\right\} \leq \hat{P}_n(\DIS(V)) \leq \max\left\{4\Px(\DIS(V)), 8/n \right\}.
\end{equation*}
\upthmend{-1.25cm}
\end{lemma}
\begin{proof}
Note that the sequence $\U_n$ from \eqref{eqn:pdisv-est} is independent from both $V$ and $\L$.
By a Chernoff bound, on an event $J_n$ with $\P(J_n) \geq 1 - 2 \cdot \exp\{-n/4\}$,
\begin{align*}
                 &\Px(\DIS(V)) > 2/n \implies \frac{\Px(\DIS(V))}{\frac{1}{n^2}\sum_{x \in \U_n} \ind_{\DIS(V)}(x)} \in [1/2,2],\\
\text{and } & \Px(\DIS(V)) \leq 2/n \implies \frac{1}{n^2}\sum_{x \in \U_n} \ind_{\DIS(V)}(x) \leq 4/n.
\end{align*}
This immediately implies the stated result.
\end{proof}

\begin{lemma}
\label{lem:inverse-little-o}
Let $\lambda : (0,1) \to (0,\infty)$ and $L : \nats \times (0,1) \to [0,\infty)$
be such that $\lambda(\eps) = \omega(1)$,
$L(n,\eps)$ is $0$ at $n=1$ and is diverging as $n\to\infty$ for every $\eps \in (0,1)$,
and for any $\nats$-valued $N(\eps) = \omega(\lambda(\eps))$, $L(N(\eps),\eps) = \omega(N(\eps))$.
Let $L^{-1}(m;\eps) = \max\left\{n \in \nats : L(n,\eps) < m\right\}$, for any $m \in (0,\infty)$.
Then for any $\Lambda(\eps) = \omega(\lambda(\eps))$,
$L^{-1}(\Lambda(\eps);\eps) = o\left(\Lambda(\eps)\right)$.
\thmend
\end{lemma}
\begin{proof}
First note that $L^{-1}$ is well-defined and finite, due to the facts that $L(n,\eps)$ can be $0$ and is diverging in $n$.
Let $\Lambda(\eps) = \omega(\lambda(\eps))$.
It is fairly straightforward to show $L^{-1}(\Lambda(\eps);\eps) \neq \Omega(\Lambda(\eps))$, but the stronger $o(\Lambda(\eps))$ result takes slightly more work.
Let $\bar{L}(n,\eps) = \min\left\{L(n,\eps), n^2 / \lambda(\eps)\right\}$ for every $n \in \nats$ and $\eps \in (0,1)$,
and let $\bar{L}^{-1}(m;\eps) = \max\left\{n \in \nats : \bar{L}(n,\eps) < m\right\}$.  We will first prove the result for $\bar{L}$.

Note that by definition of $\bar{L}^{-1}$, we know
\begin{equation*}
\left(\bar{L}^{-1}\left(\Lambda(\eps);\eps\right)+1\right)^2 / \lambda(\eps) \geq \bar{L}\left(\bar{L}^{-1}\left(\Lambda(\eps);\eps\right) + 1,\eps\right) \geq \Lambda(\eps) = \omega(\lambda(\eps)),
\end{equation*}
which implies $\bar{L}^{-1}\left(\Lambda(\eps);\eps\right) = \omega(\lambda(\eps))$.
But, by definition of $\bar{L}^{-1}$ and the condition on $L$,
\begin{equation*}
\Lambda(\eps) > \bar{L}\left(\bar{L}^{-1}\left(\Lambda(\eps);\eps\right),\eps\right) = \omega\left(\bar{L}^{-1}\left(\Lambda(\eps);\eps\right)\right).
\end{equation*}
Since $\bar{L}^{-1}(m;\eps) \geq L^{-1}(m;\eps)$ for all $m$, this implies $\Lambda(\eps) = \omega\left(L^{-1}\left(\Lambda(\eps);\eps\right)\right)$,
or equivalently $L^{-1}\left(\Lambda(\eps);\eps\right) = o\left(\Lambda(\eps)\right)$.
\end{proof}

\begin{lemma}
\label{lem:naive-improvements}
For any VC class $\C$ and passive algorithm $\alg_p$,
if $\alg_p$ achieves label complexity $\Lambda_p$,
then \NaiveActivizer, with $\alg_p$ as its argument, achieves a label complexity $\Lambda_a$
such that, for every $f \in \C$ and distribution $\Px$ over $\X$,
if $\Px(\partial_{\C,\Px} f) = 0$ and $\infty > \Lambda_{p}(\eps,f,\Px) = \omega(\log(1/\eps))$,
then $\Lambda_a(2\eps,f,\Px) = o\left(\Lambda_p(\eps,f,\Px)\right)$.
\thmend
\end{lemma}
\begin{proof}
This proof follows similar lines to a proof of a related result of \citet*{hanneke:10a}.
Suppose $\alg_{p}$ achieves a label complexity $\Lambda_{p}$,
and that $f \in \C$ and distribution $\Px$ satisfy $\infty > \Lambda_{p}(\eps,f,\Px) = \omega(\log(1/\eps))$
and $\Px(\partial_{\C,\Px} f) = 0$.  Let $\eps \in (0,1)$.
For $n \in \nats$,
let $\Delta_n(\eps) = \Px(\DIS(\Ball(f,\vrad(\lfloor n / 2 \rfloor; \eps/2))))$,
$L(n;\eps) = \left\lfloor n / \max\{32/n, 16 \Delta_n(\eps)\}\right\rfloor$,
and for $m \in (0,\infty)$ let $L^{-1}(m;\eps) = \max\left\{n \in \nats : L(n;\eps) < m\right\}$.
Suppose
\begin{equation*}
n \geq \max\Big\{12 \ln(6/\eps), 1+L^{-1}\left(\Lambda_p(\eps,f,\Px);\eps\right)\Big\}.
\end{equation*}
Consider running \NaiveActivizer~with $\alg_{p}$ and $n$ as arguments, while $f$ is the target function and $\Px$ is the data distribution.
Let $V$ and $\L$ be as in \NaiveActivizer, and let $\hat{h}_{n} = \alg_{p}(\L)$ denote the classifier returned at the end.

By Lemma~\ref{lem:VinB}, on the event $H_{\lfloor n / 2\rfloor}(\eps/2)$,
$V \subseteq \Ball(f,\vrad(\lfloor n / 2\rfloor; \eps/2))$, so that $\Px(\DIS(V)) \leq \Delta_n(\eps)$.
Letting $\U = \{X_{\lfloor n/2 \rfloor + 1}, \ldots, X_{\lfloor n/2 \rfloor + \lfloor n / (4 \hat{\Delta}) \rfloor}\}$, by Lemma~\ref{lem:naive-estimator}, on $H_{\lfloor n/2\rfloor}(\eps/2) \cap J_n$ we have
\begin{equation}
\label{eqn:naive-U-bound}
\left\lfloor n / \max\left\{ 32 / n, 16 \Delta_n(\eps)\right\}\right\rfloor \leq |\U| \leq \left\lfloor n / \max\left\{4 \Px(\DIS(V)), 16/n\right\} \right\rfloor.
\end{equation}
By a Chernoff bound, for an event $K_{n}$ with $\P(K_n) \geq 1 - \exp\{-n / 12\}$,
on $H_{\lfloor n/2 \rfloor}(\eps/2) \cap J_n \cap K_n$, $|\U \cap \DIS(V)| \leq 2 \Px(\DIS(V)) \cdot \lfloor n / \max\{4\Px(\DIS(V)),16/n\}\rfloor \leq \lceil n / 2 \rceil$.
Defining the event $G_n(\eps) = H_{\lfloor n/2 \rfloor}(\eps/2) \cap J_n \cap K_n$,
we see that on $G_n(\eps)$, every time $X_{m} \in \DIS(V)$ in Step 5 of \NaiveActivizer, we have $t < n$;
therefore, since $f \in V$ implies that the inferred labels in Step 6 are correct as well,
we have that on $G_n(\eps)$,
\begin{equation}
\label{eqn:naive-all-good-labels}
\forall (x,\hat{y}) \in \L, \hat{y} = f(x).
\end{equation}
Noting that
\begin{equation*}
\P\left(G_n(\eps)^{c}\right) \leq \P\left( H_{\lfloor n/2 \rfloor}(\eps/2)^{c}\right) + \P\left(J_n^{c}\right) + \P\left(K_n^{c}\right) \leq \eps/2 + 2\cdot \exp\left\{-n/4\right\} + \exp\{-n / 12\} \leq \eps,
\end{equation*}
we have
\begin{align}
&\E\left[ \er\left(\hat{h}_{n}\right)\right] \notag
\\ & \leq \E\left[ \ind_{G_{n}(\eps)} \ind\left[ |\L| \geq \Lambda_{p}(\eps,f,\Px)\right] \er\left(\hat{h}_{n}\right)\right]
+ \P\left(G_{n}(\eps) \cap \left\{ |\L| < \Lambda_{p}(\eps,f,\Px)\right\}\right) + \P\left(G_{n}(\eps)^{c}\right) \notag
\\ & \leq \E\left[ \ind_{G_{n}(\eps)} \ind\left[ |\L| \geq \Lambda_{p}(\eps,f,\Px)\right] \er\left(\alg_{p}(\L)\right)\right]
+ \P\left(G_{n}(\eps) \cap \left\{ |\L| < \Lambda_{p}(\eps,f,\Px)\right\}\right) + \eps. \label{eqn:naive-G-breakout}
\end{align}
On $G_{n}(\eps)$, \eqref{eqn:naive-U-bound} implies $|\L| \geq L(n;\eps)$,
and we chose $n$ large enough so that $L(n;\eps) \geq \Lambda_{p}(\eps,f,\Px)$.
Thus, the second term in \eqref{eqn:naive-G-breakout} is zero, and we have
\begin{align}
\E\left[ \er\left(\hat{h}_{n}\right)\right]
& \leq \E\left[ \ind_{G_{n}(\eps)} \ind\left[ |\L| \geq \Lambda_{p}(\eps,f,\Px)\right] \er\left(\alg_{p}\left(\L\right)\right)\right] + \eps \notag
\\ & = \E\left[ \E\left[ \ind_{G_{n}(\eps)} \er\left(\alg_{p}\left(\L\right)\right) \Big| |\L|\right] \ind\left[ |\L| \geq \Lambda_{p}(\eps,f,\Px)\right]\right] + \eps. \label{eqn:naive-conditional-L-size}
\end{align}
For any $\ell \in \nats$ with $\P(|\L|=\ell) > 0$, the conditional of $\U | \{|\U| = \ell\}$ is a product distribution $\Px^{\ell}$;
that is, the samples in $\U$ are conditionally independent and identically distributed with distribution $\Px$,
which is the same as the distribution of $\{X_1,X_2,\ldots,X_{\ell}\}$.
Therefore, for any such $\ell$ with $\ell \geq \Lambda_p(\eps,f,\Px)$, by \eqref{eqn:naive-all-good-labels} we have
\begin{equation*}
\E\left[ \ind_{G_{n}(\eps)} \er\left(\alg_{p}\left(\L\right)\right) \Big| \left\{|\L| = \ell\right\}\right] \leq \E\left[ \er\left(\alg_{p}\left(\Data_{\ell}\right)\right)\right] \leq \eps.
\end{equation*}
In particular, this means \eqref{eqn:naive-conditional-L-size} is at most $2\eps$.
This implies \NaiveActivizer, with $\alg_{p}$ as its argument, achieves a label complexity $\Lambda_{a}$ such that
\begin{equation*}
\Lambda_{a}(2\eps,f,\Px) \leq \max\Big\{12 \ln(6/\eps), 1+L^{-1}\left(\Lambda_p(\eps,f,\Px);\eps\right)\Big\}.
\end{equation*}

Since $\Lambda_{p}(\eps,f,\Px) = \omega(\log(1/\eps)) \Rightarrow 12 \ln(6/\eps) = o(\Lambda_{p}(\eps,f,\Px))$,
it remains only to show that $L^{-1}\left(\Lambda_{p}(\eps,f,\Px);\eps\right) = o(\Lambda_{p}(\eps,f,\Px))$.
Note that $\forall \eps \in (0,1)$, $L(1;\eps) = 0$ and $L(n;\eps)$ is diverging in $n$.
Furthermore, by the assumption $\Px(\partial_{\C,\Px} f) = 0$, we know that for any $N(\eps) = \omega(\log(1/\eps))$,
we have $\Delta_{N(\eps)}(\eps) = o(1)$ (by continuity of probability measures),
which implies $L(N(\eps);\eps) = \omega(N(\eps))$.  Thus, since $\Lambda_{p}(\eps,f,\Px) = \omega(\log(1/\eps))$,
Lemma~\ref{lem:inverse-little-o} implies $L^{-1}\left(\Lambda_{p}(\eps,f,\Px);\eps\right) = o\left(\Lambda_{p}(\eps,f,\Px)\right)$, as desired.
\end{proof}

\begin{lemma}
\label{lem:naive-no-improvements}
For any VC class $\C$, target function $f \in \C$, and distribution $\Px$,
if $\Px(\partial_{\C,\Px} f) > 0$, then there exists a passive
learning algorithm $\alg_p$ achieving a label complexity $\Lambda_p$
such that $(f,\Px) \in \Nontrivial(\Lambda_{p})$,
and for any label complexity $\Lambda_a$ achieved by
running \NaiveActivizer~with $\alg_p$ as its argument,
and any constant $c \in (0,\infty)$,
\begin{equation*}
\Lambda_a(c \eps,f,\Px) \neq o(\Lambda_p(\eps,f,\Px)).
\end{equation*}
\upthmend{-1.25cm}
\end{lemma}
\begin{proof}
The proof can be broken down into three essential claims.
First, it follows from Lemma~\ref{lem:Vshat-to-Boundaries} below
that, on an event $H^{\prime}$ of probability one,
$\Px(\partial_{V} f) \geq \Px(\partial_{\C} f)$;
since $\Px(\DIS(V)) \geq \Px(\partial_{V} f)$, we have
$\Px(\DIS(V)) \geq \Px(\partial_{\C} f)$ on $H^{\prime}$.

The second claim is that on $H^{\prime} \cap J_n$, $|\L| = O(n)$.
This follows from Lemma~\ref{lem:naive-estimator} and our first claim by noting that, on $H^{\prime} \cap J_n$,
$|\L| = \left\lfloor  n / (4 \hat{\Delta}) \right\rfloor \leq n / (4 \Px(\DIS(V))) \leq n / (4 \Px(\partial_{\C} f))$.

Finally, we construct a passive algorithm $\alg_p$ whose label complexity
is not significantly improved when $|\L| = O(n)$.
There is a fairly obvious randomized $\alg_p$ with this property
(simply returning $-f$ with probability $1/|\L|$, and otherwise $f$); however,
we can even satisfy the property with a deterministic $\alg_p$, as follows.
Let $\H_f = \{h_i\}_{i=1}^{\infty}$ be any sequence of classifiers
(not necessarily in $\C$) with $0 < \Px(x : h_i(x) \neq f(x))$ strictly decreasing to $0$,
(say with $h_1 = - f$).
We know such a sequence must exist since $\Px(\partial_{\C} f) > 0$.
Now define, for nonempty $S$,
\begin{equation*}
\alg_p(S) = \argmin\limits_{h_i \in \H_f} \Px(x : h_i(x) \neq f(x)) + 2\ind_{[0,1/|S|)}(\Px(x : h_i(x) \neq f(x))).
\end{equation*}
$\alg_p$ is constructed so that, in the special case that this particular $f$ is the target function
and this particular $\Px$ is the data distribution,
$\alg_p(S)$ returns the $h_i \in \H_f$ with minimal $\er(h_i)$ such that $\er(h_i) \geq 1/|S|$.
For completeness, let $\alg_p(\emptyset) = h_1$.
Define $\eps_i = \er(h_i) = \Px(x : h_i(x) \neq f(x))$.

Now let $\hat{h}_n$ be the returned classifier from running \NaiveActivizer~with $\alg_p$ and $n$ as inputs,
let $\Lambda_p$ be the (minimal) label complexity achieved by $\alg_p$,
and let $\Lambda_a$ be the (minimal) label complexity achieved by \NaiveActivizer~with $\alg_p$ as input.
Take any $c \in (0,\infty)$, and $i$ sufficiently large so that $\eps_{i-1} < 1/2$.
Then we know that for any $\eps \in [\eps_{i},\eps_{i-1})$, $\Lambda_p(\eps,f,\Px) = \lceil 1 / \eps_{i} \rceil$.
In particular, $\Lambda_{p}(\eps,f,\Px) \geq 1/\eps$, so that $(f,\Px) \in \Nontrivial(\Lambda_{p})$.
Also, by Markov's inequality and the above results on $|\L|$,
\begin{align*}
\E[\er(\hat{h}_n)] \geq \E\left[ \frac{1}{|\L|} \right] &\geq \frac{ 4 \Px(\partial_{\C} f) }{n} \P\left( \frac{1}{|\L|} > \frac{4 \Px(\partial_{\C} f)}{n}\right)
\\ &\geq \frac{4\Px(\partial_{\C} f) }{n} \P(H^{\prime} \cap J_n) \geq \frac{4 \Px(\partial_{\C} f)}{n} \left( 1 - 2 \cdot \exp\{-n/4\}\right).
\end{align*}
This implies that for $4 \ln(4) < n < \frac{2\Px(\partial_{\C} f)}{c \eps_i}$, we have $\E\left[\er(\hat{h}_n)\right] > c \eps_i$,
so that for all sufficiently large $i$,
\begin{equation*}
\Lambda_a(c \eps_i, f,\Px) \geq \frac{2\Px(\partial_{\C} f)}{c \eps_i} \geq \frac{\Px(\partial_{\C} f)}{c} \left\lceil \frac{1}{\eps_i} \right\rceil = \frac{\Px(\partial_{\C} f)}{c} \Lambda_p(\eps_i,f,\Px).
\end{equation*}
Since this happens for all sufficiently large $i$, and thus for arbitrarily small $\eps_i$ values,
we have
\begin{equation*}
\Lambda_a(c \eps, f, \Px) \neq o\left(\Lambda_p(\eps,f,\Px)\right).
\end{equation*}
\end{proof}

\begin{proof}[Theorem~\ref{thm:naive}]
Theorem~\ref{thm:naive} now follows directly from Lemmas~\ref{lem:naive-improvements} and \ref{lem:naive-no-improvements},
corresponding to the ``if'' and ``only if'' parts of the claim, respectively.
\end{proof}

\section{Proofs Related to Section~\ref{sec:activizer}: Basic Activizer}
\label{app:activizer}

In this section, we provide detailed definitions, lemmas and proofs related to \BasicActivizer.

In fact, we will develop slightly more general results here.
Specifically, we fix an arbitrary constant $\gamma \in (0,1)$,
and will prove the result for a
family of meta-algorithms parameterized by the value $\gamma$,
used as the threshold in Steps 3 and 6 of \BasicActivizer, which were set to $1/2$ above to
simplify the algorithm.  Thus, setting $\gamma = 1/2$ in the statements
below will give the stated theorem.

Throughout this section, we will assume $\C$ is a VC class with VC
dimension $\vc$, and let $\Px$ denote the (arbitrary) marginal
distribution of $X_i$ ($\forall i$).
We also fix an arbitrary classifier $f \in \cl(\C)$, where (as in Section~\ref{sec:agnostic})
$\cl(\C) = \{h : \forall r > 0, \Ball(h,r) \neq \emptyset\}$ denotes the closure of $\C$.
In the present context, $f$ corresponds to the target function when running \BasicActivizer.
Thus, we will study the behavior of \BasicActivizer~for this fixed $f$ and $\Px$;
since they are chosen arbitrarily, to establish Theorem~\ref{thm:activizer} it will suffice to
prove that for any passive $\alg_p$, \BasicActivizer~with
$\alg_p$ as input achieves superior label complexity compared to $\alg_p$ for this $f$ and $\Px$.
In fact, because here we only assume $f \in \cl(\C)$ (rather than $f \in \C$), we actually
end up proving a slightly more general version of Theorem~\ref{thm:activizer}.  But more importantly,
this relaxation to $\cl(\C)$ will also make the lemmas developed below
more useful for subsequent proofs: namely, those in Appendix~\ref{app:robust-tsybakov}.
For this same reason, many of the lemmas of this section are substantially more general than is
necessary for the proof of Theorem~\ref{thm:activizer}; the more general versions will be used
in the proofs of results in later sections.

For any $m \in \nats$, we define
$\truV_m = \left\{h \in \C : \forall i \leq m, h(X_i) = f(X_i)\right\}$.
Additionally, for $\H \subseteq \C$, and an integer $k \geq 0$,
we will adopt the notation
\begin{align*}
\S^{k}(\H) &= \left\{S \in \X^{k} : \H \text{ shatters } S\right\},\\
\bar{\S}^{k}(\H) & = \X^{k} \setminus \S^{k}(\H),
\end{align*}
and as in Section~\ref{sec:exponential}, we define the $k$-dimensional shatter core of $f$
with respect to $\H$ (and $\Px$) as
\begin{equation*}
\partial_{\H}^{k} f = \lim\limits_{r \to 0} \S^{k}\left(\Ball_{\H}(f,r)\right),
\end{equation*}
and further define
\begin{equation*}
\bar{\partial}_{\H}^{k} f = \X^{k} \setminus \partial_{\H}^{k} f.
\end{equation*}
Also as in Section~\ref{sec:exponential}, define
\begin{equation*}
\bdim_f = \min\left\{ k \in \nats : \Px^{k}\left( \partial^{k}_{\C} f \right) = 0 \right\}.
\end{equation*}
For convenience, we also define the abbreviation
\begin{equation*}
\dprob = \Px^{\bdim_f-1}\left( \partial_{\C}^{\bdim_f-1} f\right).
\end{equation*}
Also, recall that we are using the convention that $\X^0 = \{\varnothing\}$,
$\Px^{0}(\X^{0}) = 1$, and we say a set of classifiers $\H$ shatters $\varnothing$ iff $\H \neq \{\}$.
In particular, $\S^{0}(\H) \neq \{\}$ iff $\H \neq \{\}$, and $\partial_{\H}^{0} f \neq \{\}$ iff $\inf_{h \in \H} \Px(x : h(x) \neq f(x)) = 0$.
For any measurable sets $S_1,S_2 \subseteq \X^{k}$ with $\Px^{k}(S_2) > 0$,
as usual we define $\Px^{k}(S_1 | S_2) = \Px^{k}(S_1 \cap S_2) / \Px^{k}(S_2)$;
in the situation where $\Px^{k}(S_2) = 0$, it will be convenient to define
$\Px^{k}(S_1 | S_2) = 0$.
We use the definition of $\er(h)$ from above, and additionally define the \emph{conditional} error rate
$\er(h | S) = \Px(\{x : h(x) \neq f(x)\} | S)$ for any measurable $S \subseteq \X$.
We also adopt the usual short-hand for equalities and inequalities involving conditional expectations and probabilities given random variables,
wherein for instance, we write $\E[X|Y] = Z$ to mean that there is a version of $\E[X|Y]$ that is everywhere equal to $Z$,
so that in particular, any version of $\E[X|Y]$ equals $Z$ almost everywhere \citep*[see e.g.,][]{ash:00}.

\subsection{Definition of Estimators for \BasicActivizer}
\label{app:hatP-definitions}

While the estimated probabilities used in \BasicActivizer~can be defined in a variety of ways to make it a universal activizer,
in the statement of Theorem~\ref{thm:activizer} above and proof thereof below, we take the following specific definitions.
After the definition, we discuss alternative possibilities.

Though it is a slight twist on the formal model, it will greatly simplify our discussion
below to suppose we have access to two independent sequences of i.i.d.~unlabeled examples
$W_1 = \{w_1,w_2,\ldots\}$ and $W_2 = \{w_1^{\prime}, w_2^{\prime},\ldots\}$,
also independent from the main sequence $\{X_1,X_2,\ldots\}$, with $w_i, w_i^{\prime} \sim \Px$.
Since the data sequence $\{X_1,X_2,\ldots\}$ is i.i.d., this is
distributionally equivalent to supposing we partition the data sequence
in a preprocessing step, into three subsequences,
alternatingly assigning each data point to either $\Data_{X}^{\prime}$, $W_1$, or $W_2$.
Then, if we suppose $\Data_{X}^{\prime} = \{X_1^{\prime}, X_2^{\prime},\ldots\}$,
and we replace all references to $X_i$ with $X_i^{\prime}$ in the algorithms and
results, we obtain the equivalent statements holding for the model as originally
stated.  Thus, supposing the existence of these $W_i$ sequences simply serves
to simplify notation, and does not represent a further assumption on top of the
previously stated framework.

For each $k \geq 2$, we partition $W_2$ into subsets of size $k-1$, as follows.  For $i \in \nats$, let
\begin{equation*}
S_i^{(k)} = \{w^{\prime}_{ 1 + (i-1)(k-1)}, \ldots, w^{\prime}_{i(k-1)}\}.
\end{equation*}

We define the $\hat{P}_{m}$ estimators in terms of three types of functions, defined below.
For any $\H \subseteq \C$, $x \in \X$, $y \in \{-1,+1\}$, $m \in \nats$, we define
\begin{align}
&\hat{P}_{m}\left(S \in \X^{k-1} : \H \text{ shatters } S \cup \{x\} | \H \text{ shatters } S\right) && = \hat{\Delta}_{m}^{(k)}(x,W_2,\H), \label{eqn:hatPn1}\\
&\hat{P}_{m}\left(S \in \X^{k-1} : \H[(x,-y)] \text{ does not shatter } S | \H \text{ shatters } S\right) && = \hat{\Gamma}_{m}^{(k)}(x,y,W_2,\H), \label{eqn:hatPn2}\\
&\hat{P}_{m}\left(x : \hat{P}\left(S \in \X^{k-1} : \H \text{ shatters } S \cup \{x\} | \H \text{ shatters } S\right) \geq \gamma\right) && = \hat{\Delta}_{m}^{(k)}(W_1,W_2,\H). \label{eqn:hatPn3}
\end{align}
The quantities $\hat{\Delta}_{m}^{(k)}(x,W_2,\H)$, $\hat{\Gamma}_{m}^{(k)}(x,y,W_2,\H)$, and $\hat{\Delta}_{m}^{(k)}(W_1,W_2,\H)$ are specified as follows.

For $k = 1$, $\hat{\Gamma}_{m}^{(1)}(x,y,W_2,\H)$
is simply an indicator for whether every $h \in \H$ has $h(x) = y$, while
$\hat{\Delta}_{m}^{(1)}(x,W_2,\H)$ is an indicator for whether $x \in \DIS(\H)$.
Formally, they are defined as follows.
\begin{align*}
\hat{\Gamma}_{m}^{(1)}(x,y,W_2, \H) & = \ind_{\bigcap\limits_{h \in \H} \{h(x)\}}(y).\\
\hat{\Delta}_{m}^{(1)}(x,W_2, \H) & = \ind_{\DIS(\H)}(x).
\end{align*}
For $k \geq 2$, we first define
\begin{equation*}
M_m^{(k)}(\H) = \max\left\{ 1, \sum_{i=1}^{\Msize{m}} \ind_{\S^{k-1}(\H)}\left(S_i^{(k)}\right)\right\}.
\end{equation*}
Then we take the following definitions for $\hat{\Gamma}^{(k)}$ and $\hat{\Delta}^{(k)}$.
\begin{align}
\hat{\Gamma}_{m}^{(k)}(x,y,W_2,\H) & = \frac{1}{M_{m}^{(k)}(\H)} \sum\limits_{i=1}^{\Msize{m}} \ind_{\bar{\S}^{k-1}\left(\H[(x,-y)]\right)}\left(S^{(k)}_i\right) \ind_{\S^{k-1}(\H)}\left(S_i^{(k)}\right). \label{eqn:gamma-defn}\\
\hat{\Delta}_{m}^{(k)}(x,W_2,\H) & = \frac{1}{M_{m}^{(k)}(\H)} \sum_{i=1}^{\Msize{m}} \ind_{\S^{k}(\H)}\left( S_i^{(k)} \cup \{x\} \right). \label{eqn:hat-delta-defn}
\end{align}

For the remaining estimator, for any $k$ we generally define
\begin{equation*}
\hat{\Delta}_{m}^{(k)}(W_1,W_2,\H) = \frac{2}{m} + \frac{1}{m^3} \sum\limits_{i=1}^{m^3} \ind_{[\gamma/4,\infty)}\left(\hat{\Delta}_{m}^{(k)}(w_i, W_2, \H) \right).
\end{equation*}

The above definitions will be used in the proofs below.
However, there are certainly viable alternative definitions one can consider, some of which
may have interesting theoretical properties.  In general, one has the same sorts of trade-offs
present whenever estimating a conditional probability.  For instance, we could
replace ``$\Msize{m}$'' in \eqref{eqn:gamma-defn} and \eqref{eqn:hat-delta-defn} by
$\min\left\{ \ell \in \nats : M_{\ell}^{(k)}(\H) = \Msize{m}\right\}$, and then normalize by $\Msize{m}$
instead of $M_{m}^{(k)}(\H)$; this would give us $\Msize{m}$ samples
from the conditional distribution with which to estimate the conditional probability.
The advantages of this approach would be its simplicity or elegance, and possibly some
improvement in the constant factors in the label complexity bounds below.  On the other
hand, the drawback of this alternative definition would be that we do not know a priori
how many unlabeled samples we will need to process in order to calculate it; indeed, for some values of $k$ and $\H$,
we expect $\Px^{k-1}\left( \S^{k-1}(\H) \right) = 0$, so that $M_{\ell}^{(k)}(\H)$ is bounded, and
we might technically need to examine the entire sequence to distinguish this case
from the case of very small $\Px^{k-1}\left( \S^{k-1}(\H) \right)$.  Of course, these practical
issues can be addressed with small modifications, but only at the expense of complicating
the analysis, thus losing the elegance factor.  For these reasons, we have opted for the
slightly looser and less elegant, but more practical, definitions above in \eqref{eqn:gamma-defn}
and \eqref{eqn:hat-delta-defn}.

\subsection{Proof of Theorem~\ref{thm:activizer}}
\label{app:activizer-proof}

At a high level, the structure of the proof is the following.
The primary components of the proof are three lemmas:
\ref{lem:active-select}, \ref{lem:good-labels}, and \ref{lem:label-everything}.
Setting aside, for a moment, the fact that we are using the $\hat{P}_{m}$ estimators
rather than the actual probability values they estimate,
Lemma~\ref{lem:label-everything} indicates that the number of data points
in $\L_{\bdim_f}$ grows superlinearly in $n$ (the number of label requests), while
Lemma~\ref{lem:good-labels} guarantees that the labels of these points are correct,
and Lemma~\ref{lem:active-select} tells us that the classifier returned in the end
is never much worse than $\alg_{p}(\L_{\bdim_f})$.  These three factors combine
to prove the result.  The rest of the proof is composed of supporting lemmas
and details regarding the $\hat{P}_{m}$ estimators.
Specifically, Lemmas~\ref{lem:Vshat-to-Boundaries} and
\ref{lem:converging-conditional} serve a supporting role, with the purpose of showing
that the set of $V$-shatterable $k$-tuples converges to the $k$-dimensional
shatter core (up to probability-zero differences).
The other lemmas below (\ref{lem:basic-Mk-lower-bound} -- \ref{lem:empirical-works-too})
are needed primarily to extend the above basic idea
to the actual scenario where the $\hat{P}_{m}$ estimators are used as surrogates
for the probability values.
Additionally, a sub-case of Lemma~\ref{lem:empirical-works-too} is needed in order to
guarantee the label request budget will not be reached prematurely.
Again, in many cases we prove a more general lemma than is required for its use
in the proof of Theorem~\ref{thm:activizer};
these more general results will be needed in subsequent proofs:
namely, in the proofs of Theorem~\ref{thm:sequential-activizer} and Lemma~\ref{lem:robust-tsybakov}.

We begin with a lemma concerning the $\ActiveSelect$ subroutine.

\begin{lemma}
\label{lem:active-select}
For any $k^{*},M,N \in \nats$ with $k^{*} \leq N$, 
and $N$ classifiers $\{h_1,h_2,\ldots,h_N\}$ (themselves possibly random variables, independent from $\{X_{M},X_{M+1},\ldots\}$),
%the procedure 
$\ActiveSelect(\{h_1,h_2,\ldots,h_N\},$ $m,$ $\{X_{M},X_{M+1},\ldots\})$ makes at most $m$ label requests,
and if $h_{\hat{k}}$ is the classifier it outputs, then
with probability at least $1- e N \cdot \exp\left\{-m / \left( 72 k^{*} N \ln(eN)\right)\right\}$,
we have $\er(h_{\hat{k}}) \leq 2 \er(h_{k^{*}})$.
\thmend\end{lemma}
\begin{proof}
This proof is essentially identical to a similar result of \citet*{hanneke:10a},
but is included here for completeness.

Let $M_{k} = \left\lfloor \frac{m}{k (N - k) \ln(eN)}\right\rfloor$.
First note that the total number of label requests in
$\ActiveSelect$ is at most $m$, since summing up the sizes
of the batches of label requests made in all executions of Step 2
yields 
\begin{equation*}
\sum_{j=1}^{N-1} \sum_{k=j+1}^{N} \left\lfloor \frac{m}{j (N-j) \ln(eN)} \right\rfloor 
%\leq \sum_{j=1}^{N-1} \sum_{k=j+1}^{N} \frac{m}{j (N-j) \ln(eN)}
\leq \sum_{j=1}^{N-1} \frac{m}{j \ln(eN)}
\leq m.
\end{equation*}

Let $k^{**} = \argmin_{k \in \{1,\ldots,k^{*}\}} \er(h_k)$.
Now for any $j \in \{1,2,\ldots,k^{**}-1\}$ with $\Px(x : h_j(x)\neq h_{k^{**}}(x))>0$,
the law of large numbers implies that with probability one we will find
at least $M_{j}$ examples
remaining in the sequence for which $h_j(x) \neq h_{k^{**}}(x)$,
and since
$\er(h_{k^{**}}|\{x : h_j(x)\neq h_{k^{**}}(x)\}) \leq 1/2$,
Hoeffding's inequality implies that
$\P\left(m_{k^{**} j} > 7/12\right) \leq \exp\left\{- M_{j} / 72 \right\} \leq \exp\left\{1 - m / \left(72 k^{*} N \ln(eN)\right)\right\}$.
A union bound implies
\begin{equation*}
\P\left(\max_{j < k^{**}} m_{k^{**} j} > 7/12 \right)\leq k^{**} \cdot \exp\left\{1 - m / \left(72 k^{*} N \ln(eN)\right)\right\}.
\end{equation*}
In particular, note that when $\max_{j < k^{**}} m_{k^{**} j} \leq 7 / 12$, we must have $\hat{k} \geq k^{**}$.

Now suppose $j \in \{k^{**}+1,\ldots,N\}$ has $\er(h_j) > 2 \er(h_{k^{**}})$.
In particular, this implies
$\er(h_j | \{x : h_{k^{**}}(x)\neq h_{j}(x)\}) > 2/3$
and 
$\Px(x : h_j(x) \neq h_{k^{**}}(x)) > 0$,
which again means (with probability one) we will find at least $M_{k^{**}}$ examples in the sequence for which $h_{j}(x) \neq h_{k^{**}}(x)$.
By Hoeffding's inequality, we have that
\begin{equation*}
\P\left(m_{j k^{**}} \leq 7/12 \right) \leq \exp\left\{- M_{k^{**}} / 72\right\} \leq \exp\left\{ 1 - m / \left(72 k^{*} N \ln(eN)\right)\right\}.
\end{equation*}
By a union bound, we have that
\begin{multline*}
\P\left(\exists j > k^{**} : \er(h_j) > 2 \er(h_{k^{**}}) \text{ and } m_{j k^{**}} \leq 7/12\right) 
\\ \leq \left(N-k^{**}\right) \cdot \exp\left\{1 - m/\left(72 k^{*} N \ln(eN) \right)\right\}.
\end{multline*}
In particular, when $\hat{k} \geq k^{**}$, and $m_{j k^{**}} > 7/12$ for all $j > k^{**}$ with $\er(h_j) > 2 \er(h_{k^{**}})$, 
it must be true that $\er(h_{\hat{k}}) \leq 2 \er(h_{k^{**}}) \leq 2 \er(h_{k^*})$.

So, by a union bound, with probability $\geq 1- eN \cdot \exp\left\{-m/\left(72 k^{*} N \ln(eN)\right)\right\}$,
the $\hat{k}$ chosen by $\ActiveSelect$ has $\er(h_{\hat{k}}) \leq 2 \er(h_{k^{*}})$.
\end{proof}

The next two lemmas describe the limiting behavior of $\S^{k}(\truV_m)$.
In particular, we see that its limiting value is precisely $\partial_{\C}^{k} f$ (up to probability-zero differences).
Lemma~\ref{lem:Vshat-to-Boundaries} establishes that $\S^{k}(\truV_m)$ does not decrease below $\partial_{\C}^{k} f$ (except for a probability-zero set),
and Lemma~\ref{lem:converging-conditional} establishes that its limit is not larger than $\partial_{\C}^{k} f$ (again, except for a probability-zero set).

\begin{lemma}
\label{lem:Vshat-to-Boundaries}
There is an event $H^{\prime}$ with $\P(H^{\prime}) = 1$
such that on $H^{\prime}$, $\forall m \in \nats$, $\forall k \in \{0,\ldots,\bdim_f-1\}$,
for any $\H$ with $\truV_m \subseteq \H \subseteq \C$,
\begin{equation*}
\Px^{k}\left( \S^{k}(\H) \Big| \partial^{k}_{\C} f\right)
= \Px^{k}\left( \partial^{k}_{\H} f \Big| \partial^{k}_{\C} f\right) = 1,
\end{equation*}
and
\begin{equation*}
\forall i \in \nats, \ind_{\partial_{\H}^{k} f} \left(S_i^{(k+1)}\right) = \ind_{\partial_{\C}^{k} f}\left(S_i^{(k+1)}\right).
\end{equation*}
Also, on $H^{\prime}$, every such $\H$ has
$\Px^{k}\Big(\partial^{k}_{\H} f\Big) = \Px^{k}\Big(\partial^{k}_{\C} f\Big)$,
and $M_{\ell}^{(k)}(\H) \to \infty$ as $\ell \to \infty$.
\thmend\end{lemma}
\begin{proof}
We will show the first claim for the set $\truV_m$, and the result will then hold for $\H$ by monotonicity.
In particular, we will show this for any fixed $k \in \{0,\ldots,\bdim_f-1\}$ and $m \in \nats$,
and the existence of $H^\prime$ then holds by a union bound.
Fix any set $S \in \partial^{k}_{\C} f$.
Suppose $\Ball_{\truV_m}(f,r)$ does not shatter $S$ for some $r > 0$.
There is an infinite sequence of sets
$\{\{h^{(i)}_1,h^{(i)}_2,\ldots,h^{(i)}_{2^{k}}\}\}_i$ with $\forall j \leq 2^{k}$, $\Px(x : h_j^{(i)}(x) \neq f(x)) \downarrow 0$,
such that each $\{h^{(i)}_1,\ldots,h^{(i)}_{2^{k}}\} \subseteq \Ball(f,r)$ and shatters $S$.
Since $\Ball_{\truV_m}(f,r)$ does not shatter $S$,
\begin{equation*}
1 = \inf\limits_{i}\ind\left[\exists j : h^{(i)}_j \notin \Ball_{\truV_m}(f,r)\right] =
\inf\limits_{i}\ind\left[\exists j : h^{(i)}_j\left(\Data_{m}\right) \neq f\left(\Data_{m}\right)\right].
\end{equation*}
But
\begin{align*}
\P\left(\inf\limits_i \ind\left[\exists j :  h^{(i)}_j\left(\Data_{m}\right)\neq f\left(\Data_{m}\right)\right] = 1\right)
%& = \E\left[\inf\limits_i \ind\left[\exists j : h^{(i)}_j\left(\Data_{m}\right)\neq f\left(\Data_{m}\right)\right]\right]
& \leq \inf\limits_i \P\left(\exists j :  h^{(i)}_j\left(\Data_{m}\right)\neq f\left(\Data_{m}\right)\right)
%\\ \leq \inf\limits_i \E\left[\ind\left[\exists j : h^{(i)}_j\left(\Data_{m}\right)\neq f\left(\Data_{m}\right)\right]\right]
\\ \leq \lim\limits_{i\rightarrow \infty} \sum\limits_{j \leq 2^{k}} m \Px\left(x : h^{(i)}_j(x)\neq f(x)\right)
& = \sum\limits_{j \leq 2^{k}} m \lim\limits_{i\rightarrow \infty} \Px\left(x : h^{(i)}_j(x)\neq f(x)\right) = 0,
\end{align*}
where the second inequality follows from the union bound.
Therefore, $\forall r > 0$, \\ $\P\left(S \notin \S^{k}\left(\Ball_{\truV_m}(f,r)\right)\right) = 0$.
%(where $\truV_m$ is the random variable in the probability, with distribution induced as a function of $\Data_m$).
Furthermore, since $\bar{\S}^{k}\left(\Ball_{\truV_m}(f,r)\right)$ is monotonic in $r$,
the dominated convergence theorem give us that
\begin{equation*}
\P\left( S \notin \partial^{k}_{\truV_m} f \right)
= \E\left[\lim\limits_{r \to 0}\ind_{\bar{\S}^{k}(\Ball_{\truV_m}(f,r))}(S) \right] = \lim\limits_{r \to 0}\P\left(S \notin \S^{k}\left(\Ball_{\truV_m}(f,r)\right)\right) = 0.
\end{equation*}
This implies that (letting $\mathbf{S} \sim \Px^{k}$ be independent from $\truV_{m}$)
\begin{align*}
\P\left(\Px^{k}\left( \bar{\partial}^{k}_{\truV_m} f \Big| \partial^{k}_{\C} f \right)>0\right)&
= \P\left(\Px^{k}\left( \bar{\partial}^{k}_{\truV_m} f \cap \partial^{k}_{\C} f \right)>0\right)\\
&= \lim_{\xi \to 0} \P\left(\Px^{k}\left( \bar{\partial}^{k}_{\truV_m} f \cap \partial^{k}_{\C} f \right)>\xi\right).\\
& \leq \lim_{\xi \to 0} \frac{1}{\xi}\E\left[\Px^{k}\left(\bar{\partial}^{k}_{\truV_m} f \cap \partial^{k}_{\C} f  \right)\right] ~~~~~~~~~~~~~\text{ (Markov)}\\
& = \lim_{\xi \to 0} \frac{1}{\xi}\E\left[\ind_{\partial^{k}_{\C} f} ( {\bf S} ) \P\left({\bf S} \notin \partial^{k}_{\truV_m} f \Big| \mathbf{S}\right)\right] ~~~~\text{ (Fubini)}\\
& = \lim_{\xi \to 0} 0 = 0.
\end{align*}
This establishes the first claim for $\truV_m$, on an event of probability $1$,
and monotonicity extends the claim to any $\H \supseteq \truV_m$.
Also note that, on this event,
\begin{equation*}
\Px^{k}\left(\partial^{k}_{\H} f\right) \geq \Px^{k}\left(\partial^{k}_{\H} f \cap \partial^{k}_{\C} f\right) = \Px^{k}\left(\partial^{k}_{\H} f \Big| \partial^{k}_{\C} f\right) \Px^{k}\left(\partial^{k}_{\C} f\right) = \Px^{k}\left(\partial^{k}_{\C} f \right),
\end{equation*}
where the last equality follows from the first claim.
Noting that for $\H \subseteq \C$, $\partial^{k}_{\H} f \subseteq \partial^{k}_{\C} f$, we must have
\begin{equation*}
\Px^{k}\left(\partial^{k}_{\H} f \right) = \Px^{k}\left(\partial^{k}_{\C} f\right).
\end{equation*}
This establishes the third claim.
From the first claim, for any given value of $i \in \nats$ the second claim holds for
$S_i^{(k+1)}$ (with $\H=\truV_m$) on an additional event of probability $1$; taking a union bound over all
$i \in \nats$ extends this claim to every $S_i^{(k)}$ on an event of probability $1$.
Monotonicity then implies
\begin{equation*}
\ind_{\partial_{\C}^{k} f}\left(S_i^{(k+1)}\right)
= \ind_{\partial_{\truV_m}^{k} f}\left(S_i^{(k+1)}\right)
\leq \ind_{\partial_{\H}^{k} f}\left(S_i^{(k+1)}\right)
\leq \ind_{\partial_{\C}^{k} f} \left(S_i^{(k+1)}\right),
\end{equation*}
extending the result to general $\H$.
Also, as $k < \bdim_f$, we know $\Px^{k}\left(\partial^{k}_{\C} f\right) > 0$,
and since we also know $\truV_m$ is independent from $W_2$,
the strong law of large numbers implies the final claim (for $\truV_m$)
on an additional event of probability $1$; again, monotonicity extends this claim to any $\H \supseteq \truV_m$.
Intersecting the above events over values $m \in \nats$ and $k < \bdim_f$ gives the event $H^{\prime}$,
and as each of the above events has probability $1$ and there are countably many such events,
a union bound implies $\P(H^{\prime}) = 1$.
\end{proof}

Note that one specific implication of Lemma~\ref{lem:Vshat-to-Boundaries}, obtained by taking $k=0$,
is that on $H^{\prime}$, $\truV_m \neq \emptyset$ (even if $f \in \cl(\C) \setminus \C$).  This is because,
for $f \in \cl(\C)$, we have $\partial^{0}_{\C} f = \X^{0}$ so that $\Px^{0}\left(\partial^{0}_{\C} f\right) = 1$,
which means $\Px^{0}\left(\partial^{0}_{\truV_{m}} f\right) = 1$ (on $H^{\prime}$), so that we must have
$\partial^{0}_{\truV_{m}} f = \X^{0}$, which implies $\truV_{m} \neq \emptyset$.  In particular, this
also means $f \in \cl\left( \truV_{m} \right)$.

\begin{lemma}
\label{lem:converging-conditional}
There is a monotonic function $\cc(r) = o(1)$ (as $r \to 0$) such that,
on event $H^{\prime}$, for any $k \in \left\{0,\ldots,\bdim_f-1\right\}$, $m \in \nats$, $r > 0$,
and set $\H$ such that $\truV_m \subseteq \H \subseteq \Ball(f,r)$,
\begin{equation*}
\Px^{k}\left(\bar{\partial}^{k}_{\C} f \Big| \S^{k}\left(\H\right)\right) \leq \cc(r).
\end{equation*}
In particular, for $\init \in \nats$ and $\delta > 0$, on $H_{\init}(\delta) \cap H^{\prime}$ (defined above),
every $m \geq \init$ and $k \in \left\{0,\ldots,\bdim_f-1\right\}$ has
$\Px^{k}\left(\bar{\partial}^{k}_{\C} f \Big| \S^{k}\left(\truV_m\right)\right) \leq \cc(\vrad(\init;\delta))$.
\thmend\end{lemma}
\begin{proof}
Fix any $k \in \left\{0,\ldots,\bdim_f-1\right\}$.
By Lemma~\ref{lem:Vshat-to-Boundaries},
we know that on event $H^\prime$,
\begin{align*}
\Px^{k}\left(\bar{\partial}^{k}_{\C} f\Big| \S^{k}\left(\H\right)\right)
& = \frac{\Px^{k}\left(\bar{\partial}^{k}_{\C} f \cap \S^{k}\left(\H\right)\right)}
{\Px^{k}\left(\S^{k}\left(\H\right)\right)}
\leq \frac{\Px^{k}\left(\bar{\partial}^{k}_{\C} f \cap \S^{k}\left(\H\right)\right)}
{\Px^{k}\left(\partial_{\H}^{k} f\right)}
\\ & = \frac{\Px^{k}\left(\bar{\partial}^{k}_{\C} f \cap \S^{k}\left(\H\right)\right)}
{\Px^{k}\left(\partial^{k}_{\C} f\right)} %\text{ (by Lemma~\ref{lem:Vshat-to-Boundaries})}
\leq \frac{\Px^{k}\left(\bar{\partial}^{k}_{\C} f \cap \S^{k}\left( \Ball\left(f, r \right)\right)\right)}
{\Px^{k}\left(\partial^{k}_{\C} f\right)}. %\text{ (by Lemma~\ref{lem:VinB})}.
\end{align*}
Define $\cc_{k}(r)$ as this latter quantity.
Since
$\Px^{k}\left(\bar{\partial}^{k}_{\C} f \cap \S^{k}\left(\Ball(f,r)\right)\right)$
is monotonic in $r$,
\begin{equation*}
\lim\limits_{r\to 0} \frac{\Px^{k}\left(\bar{\partial}^{k}_{\C} f \cap \S^{k}\left(\Ball(f,r)\right)\right)}
{\Px^{k}\left(\partial^{k}_{\C} f\right)}
%= \frac{\E\left[\ind_{\bar{\partial}^{k}_{\C} f}\left( \mathbf{S} \right) \lim\limits_{r \to 0} \ind_{\S^{k}\left(\Ball(f,r)\right)}\left(\mathbf{S}\right)\right]}
= \frac{\Px^{k}\left(\bar{\partial}^{k}_{\C} f \cap \lim\limits_{r \to 0} \S^{k}\left(\Ball(f,r)\right)\right)}
{\Px^{k}\left(\partial^{k}_{\C} f\right)}
= \frac{\Px^{k}\left(\bar{\partial}^{k}_{\C} f \cap \partial^{k}_{\C} f\right)}{\Px^{k}\left(\partial^{k}_{\C} f\right)} = 0.
\end{equation*}
This proves $\cc_{k}(r) = o(1)$.
Defining
\begin{equation*}
\cc(r) = \max\left\{\cc_{k}(r) : k \in \left\{0,1,\ldots,\bdim_f-1\right\}\right\} = o(1)
\end{equation*}
completes the proof of the first claim.

For the final claim, simply recall that by Lemma~\ref{lem:VinB},
on $H_{\init}(\delta)$, every $m \geq \init$ has
$\truV_m \subseteq \truV_{\init} \subseteq \Ball(f,\vrad(\init;\delta))$.
\end{proof}

\begin{lemma}
\label{lem:good-labels}
For $\zeta \in (0,1)$, define
\begin{equation*}
r_{\zeta} = \sup\left\{ r \in (0,1) : \cc(r) < \zeta\right\} / 2.
\end{equation*}
On $H^{\prime}$, $\forall k \in \left\{0,\ldots,\bdim_f-1\right\}$, $\forall \zeta \in (0,1)$,
$\forall m \in \nats$, for any set $\H$ such that $\truV_m \subseteq \H \subseteq \Ball(f,r_{\zeta})$,
\begin{multline}
\Px\left( x :  \Px^{k}\left(\bar{\S}^{k}\left(\H[(x,f(x))]\right) \Big| \S^{k}\left(\H\right)\right) > \zeta\right)
\\ = \Px\left( x : \Px^{k}\left(\bar{\S}^{k}\left(\H[(x,f(x))]\right) \Big| \partial^{k}_{\H} f \right) > \zeta\right)
= 0. \label{eqn:x-zeta}
\end{multline}
In particular, for $\delta \in (0,1)$, defining $\init(\zeta;\delta) = \min\left\{ \init \in \nats : \sup\limits_{m \geq \init} \vrad(m;\delta) \leq r_{\zeta}\right\}$,
for any $\init \geq \init(\zeta;\delta)$, and any $m \geq \init$, on $H_{\init}(\delta) \cap H^{\prime}$, \eqref{eqn:x-zeta} holds
for $\H = \truV_m$.
\thmend\end{lemma}
\begin{proof}
Fix $k, m, \H$ as described above,
and suppose $\cc = \Px^{k}\left(\bar{\partial}^{k}_{\C} f | \S^{k}(\H)\right) < \zeta$;
by Lemma~\ref{lem:converging-conditional}, this happens on $H^{\prime}$.
Since, $\partial^{k}_{\H} f \subseteq \S^{k}(\H)$, we have that $\forall x \in \X$,
\begin{align*}
\Px^{k}\left( \bar{\S}^{k}\left(\H[(x,f(x))]\right) \Big| \S^{k}(\H)\right)
& = \Px^{k}\left( \bar{\S}^{k}\left(\H[(x,f(x))]\right) \Big| \partial^{k}_{\H} f\right) \Px^{k}\left(\partial^{k}_{\H} f \Big| \S^{k}(\H)\right)
\\ & + \Px^{k}\left(\bar{\S}^{k}\left(\H[(x,f(x))]\right) \Big| \S^{k}(\H) \cap \bar{\partial}^{k}_{\H} f\right) \Px^{k}\left( \bar{\partial}^{k}_{\H} f \Big| \S^{k}(\H)\right).
\end{align*}
Since all probability values are bounded by $1$, we have
\begin{equation}
\Px^{k}\left( \bar{\S}^{k}\left(\H[(x,f(x))]\right) \Big| \S^{k}(\H)\right) \leq
\Px^{k}\left( \bar{\S}^{k}\left(\H[(x,f(x))]\right) \Big| \partial^{k}_{\H} f\right)
+ \Px^{k}\left( \bar{\partial}^{k}_{\H} f \Big| \S^{k}(\H)\right). \label{eqn:Vshat-sum}
\end{equation}
Isolating the right-most term in \eqref{eqn:Vshat-sum}, by basic properties of probabilities we have
\begin{align}
&\Px^{k}\left( \bar{\partial}^{k}_{\H} f \Big| \S^{k}(\H)\right) \notag \\
&= \Px^{k}\left(\bar{\partial}^{k}_{\H} f \Big| \S^{k}(\H) \cap \bar{\partial}^{k}_{\C} f\right)\Px^{k}\left(\bar{\partial}^{k}_{\C} f \Big| \S^{k}(\H)\right)
+ \Px^{k}\left( \bar{\partial}^{k}_{\H} f \Big| \S^{k}(\H) \cap \partial^{k}_{\C} f\right) \Px^{k}\left( \partial^{k}_{\C} f \Big| \S^{k}(\H)\right) \notag \\
&\leq \Px^{k}\left(\bar{\partial}^{k}_{\C} f \Big| \S^{k}(\H)\right) + \Px^{k}\left( \bar{\partial}^{k}_{\H} f \Big| \S^{k}(\H) \cap \partial^{k}_{\C} f\right). \label{eqn:Vshat-second-sum}
\end{align}
By assumption, the left term in \eqref{eqn:Vshat-second-sum} equals $q$.
Examining the right term in \eqref{eqn:Vshat-second-sum}, we see that
\begin{align}
\Px^{k}\left( \bar{\partial}^{k}_{\H} f \Big| \S^{k}(\H) \cap \partial^{k}_{\C} f\right)
&= \Px^{k}\left( \S^{k}(\H) \cap \bar{\partial}^{k}_{\H} f \Big| \partial^{k}_{\C} f\right) / \Px^{k}\left(\S^{k}(\H) \Big| \partial^{k}_{\C} f\right) \notag \\
&\leq \Px^{k}\left( \bar{\partial}^{k}_{\H} f \Big| \partial^{k}_{\C} f\right) / \Px^{k}\left(\partial^{k}_{\H} f \Big| \partial^{k}_{\C} f\right). \label{eqn:Vshat-ratio}
\end{align}
By Lemma~\ref{lem:Vshat-to-Boundaries}, on $H^{\prime}$ the denominator in \eqref{eqn:Vshat-ratio} is $1$ and the numerator is $0$.
Thus, combining this fact with \eqref{eqn:Vshat-sum} and \eqref{eqn:Vshat-second-sum}, we have that on $H^{\prime}$,
\begin{equation}
\Px\!\left( x \!:\! \Px^{k}\!\left( \bar{\S}^{k}\!\left(\H[(x,f(x))]\right) \Big| \S^{k}\!\left(\H\right) \right) > \zeta\right) \leq
\Px\!\left( x \!:\! \Px^{k}\!\left( \bar{\S}^{k}\!\left(\H[(x,f(x))]\right) \Big| \partial^{k}_{\H}f\right) > \zeta - \cc\right). \label{eqn:zeta-cc}
\end{equation}
Note that proving the right side of \eqref{eqn:zeta-cc} equals zero will suffice to establish the result,
since it upper bounds \emph{both} the first expression of \eqref{eqn:x-zeta}
(as just established) \emph{and} the second expression of \eqref{eqn:x-zeta}
(by monotonicity of measures).
Letting $X \sim \Px$ be independent from the other random variables ($\Data, W_1, W_2$),
by Markov's inequality, the right side of \eqref{eqn:zeta-cc} is at most
\begin{equation*}
\frac{1}{\zeta - \cc}\E\left[\Px^{k}\left(\bar{\S}^{k}\left(\H[(X,f(X))]\right) \Big| \partial^{k}_{\H} f \right) \Big| \H\right]
= \frac{\E\left[ \Px^{k}\left( \bar{\S}^{k}\left(\H[(X,f(X))]\right) \cap \partial^{k}_{\H} f\right) \Big| \H\right]}{(\zeta - \cc)\Px^{k}\left(\partial^{k}_{\H} f\right)},
\end{equation*}
and by Fubini's theorem, this is (letting $\mathbf{S} \sim \Px^{k}$ be independent from the other random variables)
\begin{equation*}
\frac{\E\left[ \ind_{\partial^{k}_{\H} f}(\mathbf{S}) \Px\left(x : \mathbf{S} \notin \S^{k}\left(\H[(x,f(x))]\right)\right) \Big| \H\right]}{(\zeta - \cc) \Px^{k}\left( \partial^{k}_{\H} f\right)}.
\end{equation*}
Lemma~\ref{lem:Vshat-to-Boundaries} implies this equals
\begin{equation}
\frac{\E\left[ \ind_{\partial^{k}_{\H} f}(\mathbf{S}) \Px\left(x : \mathbf{S} \notin \S^{k}\left(\H[(x,f(x))]\right)\right) \Big| \H\right]}{(\zeta - \cc) \Px^{k}\left( \partial^{k}_{\C} f\right)}. \label{eqn:Vxh-not-shat}
\end{equation}

For any fixed $S \in \partial^{k}_{\H} f$, there is an infinite sequence of sets
\begin{equation*}
\left\{\left\{h^{(i)}_1,h^{(i)}_2,\ldots,h^{(i)}_{2^{k}}\right\}\right\}_{i\in\nats}
\end{equation*}
with $\forall j \leq 2^{k}$,
$\Px\left(x : h_j^{(i)}(x) \neq f(x)\right) \downarrow 0$,
such that each $\left\{h^{(i)}_1,\ldots,h^{(i)}_{2^{k}}\right\} \subseteq \H$ and shatters $S$.
If $\H[(x,f(x))]$ does not shatter $S$, then
\begin{equation*}
1 = \inf\limits_{i}\ind\left[\exists j : h^{(i)}_j \notin \H[(x,f(x))]\right] =
\inf\limits_{i}\ind\left[\exists j : h^{(i)}_j(x) \neq f(x)\right].
\end{equation*}
In particular,
\begin{align*}
\Px\left(x : S \notin \S^{k}\left(\H[(x,f(x))]\right)\right)
&\leq \Px\left(x : \inf\limits_i \ind\left[\exists j : h^{(i)}_j(x) \neq f(x)\right]=1\right)
%= \E\left[\inf\limits_i \ind\left[\exists j : h^{(i)}_j(X) \neq f(X)\right] \Big| \H\right]
\\ = \Px\left(\bigcap_i \left\{ x : \exists j : h^{(i)}_j(x) \neq f(x)\right\} \right)
& \leq \inf\limits_i \Px\left(x : \exists j \text{ s.t. } h^{(i)}_j(x)\neq f(x)\right)
\\ \leq \lim\limits_{i\rightarrow \infty} \sum\limits_{j \leq 2^{k}} \Px\left(x : h^{(i)}_j(x) \neq f(x)\right)
&= \sum\limits_{j \leq 2^{k}} \lim\limits_{i\rightarrow \infty} \Px\left(x : h^{(i)}_j(x) \neq f(x)\right) = 0.
\end{align*}
Thus \eqref{eqn:Vxh-not-shat} is zero, which establishes the result.

The final claim is then implied by Lemma~\ref{lem:VinB} and monotonicity of $\truV_m$ in $m$:
that is, on $H_{\init}(\delta)$, $\truV_m \subseteq \truV_{\init} \subseteq \Ball(f,\vrad(\init;\delta)) \subseteq \Ball(f,r_{\zeta})$.
\end{proof}

\begin{lemma}
\label{lem:label-everything}
For any $\zeta \in (0,1)$, there are values $\left\{\Delta_n^{(\zeta)}(\eps) : n \in \nats, \eps \in (0,1)\right\}$ such that,
for any $n \in \nats$ and $\eps > 0$, on event $H_{\lfloor n/3\rfloor}(\eps/2) \cap H^{\prime}$,
letting $V = \truV_{\lfloor n/3 \rfloor}$,
\begin{equation*}
\Px\left(x : \Px^{\bdim_f-1}\left(S \in \X^{\bdim_f-1} : S \cup \{x\} \in \S^{\bdim_f}(V) \Big| \S^{\bdim_f - 1}(V) \right) \geq \zeta\right) \leq \Delta_n^{(\zeta)}(\eps),
\end{equation*}
and for any $\nats$-valued $N(\eps) = \omega(\log(1/\eps))$, $\Delta_{N(\eps)}^{(\zeta)}(\eps) = o(1)$.
\thmend\end{lemma}
\begin{proof}
Throughout, we suppose the event $H_{\lfloor n / 3 \rfloor}(\eps/2) \cap H^{\prime}$, and fix some $\zeta \in (0,1)$.
We have $\forall x$,
\begin{align}
&  \Px^{\bdim_f-1}\left( S \in \X^{\bdim_f - 1} : S \cup \{x\} \in \S^{\bdim_f}(V) \Big| \S^{\bdim_f -1}(V)\right) \notag
\\ & = \Px^{\bdim_f-1}\left( S \in \X^{\bdim_f - 1} : S \cup \{x\} \in \S^{\bdim_f}(V) \Big| \S^{\bdim_f -1}(V) \cap \partial_{\C}^{\bdim_f-1}f\right)\Px^{\bdim_f-1}\left(\partial_{\C}^{\bdim_f-1}f\Big|\S^{\bdim_f-1}(V)\right) \notag
\\ & + \Px^{\bdim_f-1}\left( S \in \X^{\bdim_f - 1} : S \cup \{x\} \in \S^{\bdim_f}(V) \Big| \S^{\bdim_f -1}(V) \cap \bar{\partial}_{\C}^{\bdim_f-1}f\right) \Px^{\bdim_f-1}\left(\bar{\partial}_{\C}^{\bdim_f-1}f\Big|\S^{\bdim_f-1}(V)\right) \notag
\\ & \leq \Px^{\bdim_f-1}\!\left( S \!\in\! \X^{\bdim_f - 1} : S \cup \{x\} \in \S^{\bdim_f}(V) \Big| \S^{\bdim_f - 1}(V) \cap \partial_{\C}^{\bdim_f-1} \! f\right)
\!+\! \Px^{\bdim_f-1}\!\left(\bar{\partial}_{\C}^{\bdim_f-1}\! f\Big|\S^{\bdim_f-1}(V)\right)\!. \label{eqn:label-everything-split}
\end{align}
By Lemma~\ref{lem:Vshat-to-Boundaries}, the left term in \eqref{eqn:label-everything-split} equals
\begin{align*}
& \Px^{\bdim_f-1}\left( S \in \X^{\bdim_f - 1} : S \cup \{x\} \in \S^{\bdim_f}(V) \Big| \S^{\bdim_f - 1}(V) \cap \partial_{\C}^{\bdim_f-1}f\right) \Px^{\bdim_f-1}\left(\S^{\bdim_f-1}(V) \Big| \partial_{\C}^{\bdim_f-1} f\right)
\\ & = \Px^{\bdim_f-1}\left( S \in \X^{\bdim_f - 1} : S \cup \{x\} \in \S^{\bdim_f}(V) \Big| \partial_{\C}^{\bdim_f-1}f\right),
\end{align*}
and by Lemma~\ref{lem:converging-conditional}, the right term in \eqref{eqn:label-everything-split} is at most $\cc(\vrad(\lfloor n/3\rfloor;\eps/2))$.
Thus, we have
\begin{align}
& \Px\left(x : \Px^{\bdim_f-1}\left(S \in \X^{\bdim_f-1} : S \cup \{x\} \in \S^{\bdim_f}(V) \Big| \S^{\bdim_f - 1}(V) \right) \geq \zeta\right) \notag
\\ &\leq \Px\left( x : \Px^{\bdim_f-1}\left( S \in \X^{\bdim_f - 1} : S \cup \{x\} \in \S^{\bdim_f}(V) \Big| \partial_{\C}^{\bdim_f-1}f\right) \geq \zeta - \cc(\vrad(\lfloor n/3 \rfloor;\eps/2))\right). \label{eqn:label-everything-cc}
\end{align}
For $n < 3 \init(\zeta/2;\eps/2)$ (for $\init(\cdot;\cdot)$ defined in Lemma~\ref{lem:good-labels}), we define $\Delta_{n}^{(\zeta)}(\eps) = 1$.
Otherwise, suppose $n \geq 3 \init(\zeta/2;\eps/2)$, so that $\cc(\vrad(\lfloor n/3 \rfloor;\eps/2)) < \zeta/2$, and thus \eqref{eqn:label-everything-cc} is at most
\begin{equation*}
\Px\left( x : \Px^{\bdim_f-1}\left( S \in \X^{\bdim_f - 1} : S \cup \{x\} \in \S^{\bdim_f}(V) \Big| \partial_{\C}^{\bdim_f-1}f\right) \geq \zeta /2 \right).
\end{equation*}
By Lemma~\ref{lem:VinB}, this is at most
\begin{equation*}
\Px\left( x : \Px^{\bdim_f-1}\left( S \in \X^{\bdim_f - 1} : S \cup \{x\} \in \S^{\bdim_f}\left(\Ball(f,\vrad(\lfloor n / 3 \rfloor;\eps/2))\right) \Big| \partial_{\C}^{\bdim_f-1}f\right) \geq \zeta/2\right).
\end{equation*}
Letting $X \sim \Px$, by Markov's inequality this is at most
\begin{align}
&\frac{2}{\zeta} \E\left[ \Px^{\bdim_f-1}\left( S \in \X^{\bdim_f - 1} : S \cup \{X\} \in \S^{\bdim_f}\left(\Ball(f,\vrad(\lfloor n / 3 \rfloor;\eps/2))\right) \Big| \partial_{\C}^{\bdim_f-1}f\right) \right] \notag\\
&=\frac{2}{\zeta \dprob} \Px^{\bdim_f}\left( S \cup \{x\} \in \X^{\bdim_f} : S \cup \{x\} \in \S^{\bdim_f}\left(\Ball(f,\vrad(\lfloor n / 3 \rfloor;\eps/2))\right) \text{ and } S \in \partial_{\C}^{\bdim_f-1}f\right) \notag\\
&\leq \frac{2}{\zeta \dprob}\Px^{\bdim_f}\left( \S^{\bdim_f}\left(\Ball(f,\vrad(\lfloor n / 3 \rfloor;\eps/2))\right)\right). \label{eqn:label-everything-ratio}
\end{align}
Thus, defining $\Delta_{n}^{(\zeta)}(\eps)$ as \eqref{eqn:label-everything-ratio} for $n \geq 3 \init(\zeta/2;\eps/2)$ establishes the first claim.

It remains only to prove the second claim.
Let $N(\eps) = \omega(\log(1/\eps))$.  Since
$\init(\zeta/2;\eps/2) \leq \left\lceil \frac{4}{r_{\zeta/2}} \left( \vc \ln\left(\frac{4e}{r_{\zeta/2}}\right) + \ln\left(\frac{4}{\eps}\right)\right)\right\rceil = O(\log(1/\eps))$,
we have that for all sufficiently small $\eps > 0$, $N(\eps) \geq 3 \init(\zeta/2;\eps/2)$, so that $\Delta_{N(\eps)}^{(\zeta)}(\eps)$ equals \eqref{eqn:label-everything-ratio} (with $n = N(\eps)$).
Furthermore, since $\dprob > 0$, $\Px^{\bdim_f}\left( \partial_{\C}^{\bdim_f} f\right) = 0$, and $\vrad(\lfloor N(\eps)/3\rfloor;\eps/2) = o(1)$,
by continuity of probability measures we know \eqref{eqn:label-everything-ratio} is $o(1)$ when $n = N(\eps)$,
so that we generally have $\Delta_{N(\eps)}^{(\zeta)}(\eps) = o(1)$.
\end{proof}

For any $m \in \nats$, define
\begin{equation*}
\tilde{M}(m) = \Msize{m} \dprob / 2.
\end{equation*}

\begin{lemma}
\label{lem:basic-Mk-lower-bound}
There is a $(\C,\Px,f)$-dependent constant $c^{(i)} \in (0,\infty)$ such that,
for any $\init \in \nats$ there is an event $H_{\init}^{(i)} \subseteq H^{\prime}$ with
\begin{equation*}
\P\left(H_{\init}^{(i)}\right) \geq 1 - c^{(i)} \cdot \exp\left\{ - \tilde{M}(\init) / 4 \right\}
\end{equation*}
such that on $H_{\init}^{(i)}$,
if $\bdim_f \geq 2$, then $\forall k \in \left\{2,\ldots,\bdim_f\right\}$,
$\forall m \geq \init$, $\forall \ell \in \nats$, for any set $\H$ such that $\truV_{\ell} \subseteq \H \subseteq \C$,
\begin{equation*}
M^{(k)}_{m}\left(\H\right) \geq \tilde{M}(m).
\end{equation*}
\upthmend{-1.1cm}
\end{lemma}
\begin{proof} 
On $H^{\prime}$, Lemma~\ref{lem:Vshat-to-Boundaries} implies every $\ind_{\S^{k-1}(\H)}\left(S_i^{(k)}\right) \geq \ind_{\partial_{\H}^{k-1} f}\left(S_i^{(k)}\right) = \ind_{\partial_{\C}^{k-1} f}\left(S_i^{(k)}\right)$,
so we focus on showing $\left| \left\{ S_i^{(k)} : i \leq \Msize{m} \right\} \cap \partial_{\C}^{k-1} f \right| \geq \tilde{M}(m)$ on an appropriate event.
We know
\begin{align*}
&\P\left( \forall k \in \left\{2,\ldots,\bdim_f\right\}, \forall m \geq \init, \left| \left\{ S_i^{(k)} : i \leq \Msize{m}\right\} \cap \partial_{\C}^{k-1} f \right| \geq \tilde{M}(m) \right)
\\ & = 1 - \P\left( \exists k \in \left\{2,\ldots,\bdim_f\right\}, m \geq \init : \left| \left\{ S_i^{(k)} : i \leq \Msize{m}\right\} \cap \partial_{\C}^{k-1} f \right| < \tilde{M}(m) \right)
\\ & \geq 1 - \sum_{m \geq \init} \sum_{k=2}^{\bdim_f} \P\left( \left| \left\{ S_i^{(k)} : i \leq \Msize{m}\right\} \cap \partial_{\C}^{k-1} f \right| < \tilde{M}(m) \right),
\end{align*}
where the last line follows by a union bound.  Thus, we will focus on bounding
\begin{equation}
\label{eqn:basic-Mk-bigsum}
\sum_{m \geq \init} \sum_{k=2}^{\bdim_f} \P\left( \left| \left\{ S_i^{(k)} : i \leq \Msize{m}\right\} \cap \partial_{\C}^{k-1} f \right| < \tilde{M}(m) \right).
\end{equation}
Fix any $k \in \left\{2,\ldots,\bdim_f\right\}$, and integer $m \geq \init$.
Since
\begin{equation*}
\E\left[ \left| \left\{ S_i^{(k)} : i \leq \Msize{m}\right\} \cap \partial_{\C}^{k-1} f \right| \right] = \Px^{k-1}\left(\partial_{\C}^{k-1} f\right) \Msize{m} \geq \dprob \Msize{m},
\end{equation*}
a Chernoff bound implies that
\begin{align*}
\P\left( \left| \left\{ S_i^{(k)} : i \leq \Msize{m}\right\} \cap \partial_{\C}^{k-1} f \right| < \tilde{M}(m) \right)
& \leq \exp\left\{ - \Msize{m} \Px^{k-1}\left(\partial_{\C}^{k-1} f\right) / 8\right\}
\\ & \leq \exp\left\{ - \Msize{m} \dprob / 8\right\}.
\end{align*}
Thus, we have that \eqref{eqn:basic-Mk-bigsum} is at most
\begin{align*}
\sum_{m \geq \init} \sum_{k=2}^{\bdim_f}  &\exp\left\{ - \Msize{m} \dprob / 8\right\}
\leq \sum_{m \geq \init} \bdim_f \cdot \exp\left\{ - \Msize{m} \dprob / 8\right\}
\leq \sum_{m \geq \Msize{\init}} \bdim_f \cdot \exp\left\{ - m \dprob / 8\right\}
\\ &\leq \bdim_f \cdot \exp\left\{- \tilde{M}(\init) / 4\right\} + \bdim_f \cdot \int_{\Msize{\init}}^{\infty} \exp\left\{ - x \dprob / 8\right\} {\rm d}x
\\ &= \bdim_f \cdot \left(1 + 8 / \dprob \right) \cdot \exp\left\{ - \tilde{M}(\init) / 4 \right\}
\\ &\leq \left(9 \bdim_f / \dprob \right) \cdot \exp\left\{- \tilde{M}(\init) / 4 \right\}.
\end{align*}
Note that since $\P(H^{\prime}) = 1$, defining
\begin{equation*}
H_{\init}^{(i)} = \left\{ \forall k \in \left\{2,\ldots,\bdim_f\right\}, \forall m \geq \init, \left| \left\{ S_i^{(k)} : i \leq \Msize{m}\right\} \cap \partial_{\C}^{k-1} f \right| \geq \tilde{M}(m) \right\} \cap H^{\prime}
\end{equation*}
has the required properties.
\end{proof}

\begin{lemma}
\label{lem:Mball-core}
For any $\init \in \nats$, there is an event $G^{(i)}_{\init}$ with
\begin{equation*}
\P\left(H_{\init}^{(i)} \setminus G^{(i)}_{\init}\right) \leq \left(121 \bdim_f / \dprob\right) \cdot \exp\left\{ -\tilde{M}(\init) / 60\right\}
\end{equation*}
such that, on $G_{\init}^{(i)}$, if $\bdim_f \geq 2$, then for every integer $\s \geq \init$ and $k \in\left\{2,\ldots,\bdim_f\right\}$,
$\forall r \in \left(0,r_{1/6}\right]$,
\begin{equation*}
M_{\s}^{(k)}\left(\Ball\left(f,r\right)\right) \leq (3/2) \left| \left\{ S_i^{(k)} : i \leq \Msize{\s} \right\} \cap \partial_{\C}^{k-1} f\right|.
\end{equation*}
\upthmend{-1.15cm}
\end{lemma}
\begin{proof}
Fix integers $\s \geq \init$ and $k \in\left\{2,\ldots, \bdim_f\right\}$, and let $r = r_{1/6}$.
Define the set $\hat{\S}^{k-1} = \left\{S_i^{(k)} : i \leq \Msize{\s}\right\} \cap \S^{k-1}\left(\Ball\left(f,r\right)\right)$.
Note $\left|\hat{\S}^{k-1}\right| = M^{(k)}_{\s}\left(\Ball\left(f,r\right)\right)$
and the elements of $\hat{\S}^{k-1}$ are conditionally i.i.d. given $M^{(k)}_{\s}\left(\Ball\left(f,r\right)\right)$,
each with conditional distribution equivalent to the conditional
$S^{(k)}_1 \Big| \left\{S^{(k)}_1 \in \S^{k-1}\left(\Ball\left(f,r\right)\right)\right\}$.
In particular,
$\E\left[\big|\hat{\S}^{k-1} \cap \partial_{\C}^{k-1} f\big| \Big| M_{s}^{(k)}\left(\Ball\left(f,r\right)\right)\right] = \Px^{k-1}\left(\partial_{\C}^{k-1} f \Big| \S^{k-1}\left(\Ball\left(f,r\right)\right)\right) M_{s}^{(k)}\left(\Ball\left(f,r\right)\right)$.
Define the event
\begin{equation*}
G_{\init}^{(i)}(k,\s) = \left\{\left|\hat{\S}^{k-1}\right| \leq (3/2) \left| \hat{\S}^{k-1} \cap \partial_{\C}^{k-1} f \right|\right\}.
\end{equation*}
By Lemma~\ref{lem:converging-conditional} (indeed by definition of $\cc(r)$ and $r_{1/6}$) we have
\begin{align}
& 1 - \P\left(G_{\init}^{(i)}(k,\s) \Big| M^{(k)}_{\s}\left(\Ball\left(f,r\right)\right)\right) \notag
\\ & = \P\left( \big| \hat{\S}^{k-1} \cap \partial_{\C}^{k-1} f\big| < (2/3) M^{(k)}_{\s}\left(\Ball\left(f,r\right)\right) \Big | M^{(k)}_{\s}\left(\Ball\left(f,r\right)\right)\right) \notag
\\ & \leq \P\left( \big| \hat{\S}^{k-1} \cap \partial_{\C}^{k-1} f\big| < (4/5) \left(1-\cc\left(r\right)\right) M^{(k)}_{\s}\left(\Ball\left(f,r\right)\right) \Big| M^{(k)}_{\s}\left(\Ball\left(f,r\right)\right) \right) \notag
\\ & \leq \P\left( \big| \hat{\S}^{k-1} \cap \partial_{\C}^{k-1} f\big| < (4/5) \Px^{k-1}\left( \partial_{\C}^{k-1} f \Big| \S^{k-1}\left(\Ball\left(f,r\right)\right)\right) M^{(k)}_{\s}\!\left(\Ball\left(f,r\right)\right) \Big| M^{(k)}_{\s}\!\left(\Ball\left(f,r\right)\right) \right)\! . \label{eqn:Hiiksr-prebound}
\end{align}
By a Chernoff bound, \eqref{eqn:Hiiksr-prebound} is at most
\begin{multline*}
\exp\left\{ - M^{(k)}_{\s}\left(\Ball\left(f,r\right)\right) \Px^{k-1}\left(\partial_{\C}^{k-1} f \Big| \S^{k-1}\left(\Ball\left(f,r\right)\right)\right) / 50\right\}
\\ \leq \exp\left\{ - M^{(k)}_{\s}\left(\Ball\left(f,r\right)\right) \left(1-\cc\left(r\right)\right) / 50\right\}
\leq \exp\left\{ - M^{(k)}_{\s}\left(\Ball\left(f,r\right)\right) / 60\right\}.
\end{multline*}
Thus, by Lemma~\ref{lem:basic-Mk-lower-bound},
\begin{align*}
& \P\left(H_{\init}^{(i)} \setminus G_{\init}^{(i)}(k,\s)\right)
 \leq \P\left(\left\{ M_{\s}^{(k)}\left(\Ball\left(f,r\right)\right) \geq \tilde{M}(\s)\right\} \setminus G_{\init}^{(i)}(k,\s)\right)
\\ & = \E\left[ \left(1 - \P\left(G_{\init}^{(i)}(k,\s) \Big| M^{(k)}_{\s}\left(\Ball\left(f,r\right)\right)\right)\right) \ind_{\left[\tilde{M}(\s), \infty\right)}\left(M^{(k)}_{\s}\left(\Ball\left(f,r\right)\right)\right)\right]
\\ & \leq \E\left[ \exp\left\{ - M^{(k)}_{\s}\left(\Ball\left(f,r\right)\right) / 60 \right\} \ind_{\left[\tilde{M}(\s),\infty\right)}\left(M^{(k)}_{\s}\left(\Ball\left(f,r\right)\right)\right)\right]
\leq \exp\left\{ - \tilde{M}(\s) / 60 \right\}.
\end{align*}
Now defining $G_{\init}^{(i)} = \bigcap_{\s \geq \init} \bigcap_{k=2}^{\bdim_f} G^{(i)}_{\init}(k,\s)$, a union bound implies
\begin{align*}
\P\left(H_{\init}^{(i)} \setminus G_{\init}^{(i)}\right)
& \leq \sum_{\s \geq \init} \bdim_f \cdot \exp\left\{ - \tilde{M}(\s) / 60 \right\}
\\ & \leq \bdim_f \left( \exp\left\{ - \tilde{M}(\init) / 60 \right\} + \int_{\Msize{\init}}^{\infty} \exp\left\{- x \dprob / 120\right\} {\rm d}x\right)
\\ & = \bdim_f \left(1 + 120 / \dprob \right) \cdot \exp\left\{ - \tilde{M}(\init) / 60\right\}
\\ & \leq \left(121 \bdim_f / \dprob \right) \cdot \exp\left\{ -\tilde{M}(\init) / 60\right\}.
\end{align*}
This completes the proof for $r = r_{1/6}$.
Monotonicity extends the result to any $r \in \left(0,r_{1/6}\right]$.
\end{proof}

\begin{lemma}
\label{lem:kstar-good-labels}
There exist $(\C,\Px,f,\gamma)$-dependent constants $\init^{*} \in \nats$ and $c^{(ii)} \in (0,\infty)$ such that,
for any integer $\init \geq \init^{*}$,
there is an event $H_{\init}^{(ii)} \subseteq G_{\init}^{(i)}$ with
\begin{equation}
\label{eqn:Hii-prob-bound}
\P\left(H^{(i)}_{\init} \setminus H_{\init}^{(ii)}\right) \leq c^{(ii)} \cdot \exp\left\{ - \tilde{M}(\init)^{1/3} / 60 \right\}
\end{equation}
such that, on $H_{\init}^{(i)} \cap H_{\init}^{(ii)}$,
$\forall \s,m,\ell,k \in \nats$ with $\ell < m$ and $k \leq \bdim_f$,
for any set of classifiers $\H$ with $\truV_{\ell} \subseteq \H$,
if either $k=1$, or $\s \geq \init$ and $\H \subseteq \Ball(f,r_{(1-\gamma)/6})$, then
\begin{equation*}
\hat{\Delta}_{\s}^{(k)}\left(X_m, W_2, \H\right) < \gamma \implies \hat{\Gamma}_{\s}^{(k)}\left(X_m, -f(X_m), W_2, \H\right) < \hat{\Gamma}_{\s}^{(k)}\left(X_m,f(X_m),W_2,\H\right).
\end{equation*}
In particular, for $\delta \in (0,1)$ and $\init \geq \max\{\init((1-\gamma)/6;\delta),\init^{*}\}$, on $H_{\init}(\delta) \cap H_{\init}^{(i)} \cap H_{\init}^{(ii)}$,
this is true for $\H = \truV_{\ell}$ for every $k,\ell,m,\s \in \nats$ satisfying $\init \leq \ell < m$, $\init \leq \s$, and $k \leq \bdim_f$.
\thmend
\end{lemma}
\begin{proof}
Let $\init^{*} = (6 / (1-\gamma)) \cdot \left(2 / \dprob\right)^{1/3}$, and 
consider any $\init,k,\ell,m,\s,\H$ as described above.
If $k = 1$, the result clearly holds.  In particular, Lemma~\ref{lem:Vshat-to-Boundaries} implies that on $H_{\init}^{(i)}$,
$\H[(X_m,f(X_m))] \supseteq \truV_m \neq \emptyset$, so that some $h \in \H$ has $h(X_m) = f(X_m)$, and therefore
\begin{equation*}
\hat{\Gamma}_{\s}^{(1)}\left(X_{m}, -f(X_{m}), W_2, \H\right) = \ind_{\bigcap\limits_{h \in \H}\{h(X_{m})\}}(-f(X_{m})) = 0,
\end{equation*}
and since
$\hat{\Delta}_{\s}^{(1)}\left(X_{m}, W_2, \H\right) = \ind_{\DIS\left(\H\right)}(X_{m})$,
if $\hat{\Delta}_{\s}^{(1)}\left(X_{m}, W_2, \H\right) < \gamma$, then since $\gamma < 1$ we have $X_m \notin \DIS(\H)$, so that
\begin{equation*}
\hat{\Gamma}_{\s}^{(1)}\left(X_{m},f(X_{m}),W_2,\H\right) = \ind_{\bigcap\limits_{h \in \H}\{h(X_{m})\}}(f(X_{m})) = 1.
\end{equation*}

Otherwise, suppose $2 \leq k \leq \bdim_f$.
Note that on $H_{\init}^{(i)} \cap G_{\init}^{(i)}$, $\forall m \in \nats$,
and any $\H$ with $\truV_{\ell} \subseteq \H \subseteq \Ball(f,r_{(1-\gamma)/6})$ for some $\ell \in \nats$,
\begin{align*}
&\hat{\Gamma}_{\s}^{(k)}\left(X_{m},-f(X_{m}),W_2,\H\right) &
\\ & = \frac{1}{M_{\s}^{(k)}(\H)}\sum_{i=1}^{\Msize{\s}} \ind_{\bar{\S}^{k-1}(\H[(X_m,f(X_m))])}\left(S_i^{(k)}\right) \ind_{\S^{k-1}(\H)}\left(S_i^{(k)}\right) &
\\ & \leq \frac{1}{\left|\left\{ S_i^{(k)} : i \leq \Msize{\s} \right\} \cap \partial_{\H}^{k-1} f\right|} \sum_{i=1}^{\Msize{\s}} \ind_{\bar{\S}^{k-1}\left(\truV_m\right)}\left(S_i^{(k)}\right) \ind_{\S^{k-1}\left(\Ball(f,r_{(1-\gamma)/6})\right)}\left(S_i^{(k)}\right) &\text{ (monotonicity)}
\\ & \leq \frac{1}{\left|\left\{ S_i^{(k)} : i \leq \Msize{\s} \right\} \cap \partial_{\H}^{k-1} f\right|} \sum_{i=1}^{\Msize{\s}} \ind_{\bar{\partial}^{k-1}_{\truV_m} f}\left(S_i^{(k)}\right) \ind_{\S^{k-1}\left(\Ball(f,r_{(1-\gamma)/6})\right)}\left(S_i^{(k)}\right) &\text{ (monotonicity)}
\\ & = \frac{1}{\left|\left\{ S_i^{(k)} : i \leq \Msize{\s} \right\} \cap \partial_{\C}^{k-1} f\right|} \sum_{i=1}^{\Msize{\s}} \ind_{\bar{\partial}^{k-1}_{\C} f}\left(S_i^{(k)}\right) \ind_{\S^{k-1}\left(\Ball(f,r_{(1-\gamma)/6})\right)}\left(S_i^{(k)}\right) &\text{ (Lemma~\ref{lem:Vshat-to-Boundaries})}
\\ & \leq \frac{3}{2 M_{\s}^{(k)}(\Ball(f,r_{(1-\gamma)/6}))} \sum_{i=1}^{\Msize{\s}} \ind_{\bar{\partial}^{k-1}_{\C} f}\left(S_i^{(k)}\right) \ind_{\S^{k-1}\left(\Ball(f,r_{(1-\gamma)/6})\right)}\left(S_i^{(k)}\right). &\text{ (Lemma~\ref{lem:Mball-core})}
\end{align*}
For brevity, let $\hat{\Gamma}$ denote this last quantity, and let $M_{ks} = M_{s}^{(k)}\left(\Ball\left(f, r_{(1-\gamma)/6}\right)\right)$.
By Hoeffding's inequality, we have
\begin{equation*}
\P\left( (2/3)\hat{\Gamma} > \Px^{k-1}\left( \bar{\partial}^{k-1}_{\C} f \Big| \S^{k-1}\left(\Ball\left(f,r_{(1-\gamma)/6}\right)\right)\right)  + M_{ks}^{-1/3} \Bigg| M_{ks}\right)
 \leq \exp\left\{- 2 M_{ks}^{1/3} \right\}.
\end{equation*}
Thus, by Lemmas \ref{lem:converging-conditional}, \ref{lem:basic-Mk-lower-bound} and \ref{lem:Mball-core},
\begin{align*}
&\P\left( \left\{(2/3)\hat{\Gamma}_{\s}^{(k)}\left(X_{m},-f(X_{m}),W_2,\H\right) > \cc\left(r_{(1-\gamma)/6}\right) + \tilde{M}(\s)^{-1/3}\right\} \cap H_{\init}^{(i)} \cap G_{\init}^{(i)}\right) \\
&\leq \P\left( \left\{(2/3)\hat{\Gamma} > \Px^{k-1}\left( \bar{\partial}^{k-1}_{\C} f \Big| \S^{k-1}\left(\Ball\left(f,r_{(1-\gamma)/6}\right)\right)\right) +\tilde{M}(\s)^{-1/3}\right\} \cap H_{\init}^{(i)}\right) \\
& \leq \P\left( \left\{(2/3)\hat{\Gamma} > \Px^{k-1}\left( \bar{\partial}^{k-1}_{\C} f \Big| \S^{k-1}\left(\Ball\left(f,r_{(1-\gamma)/6}\right)\right)\right) + M_{ks}^{-1/3}\right\} \cap \{M_{ks} \geq \tilde{M}(\s)\} \right)\\
& = \E\left[\P\left( (2/3)\hat{\Gamma} > \Px^{k-1}\left( \bar{\partial}^{k-1}_{\C} f\Big| \S^{k-1}\left(\Ball\left(f,r_{(1-\gamma)/6}\right)\right)\right) + M_{ks}^{-1/3} \Bigg| M_{ks}\right) \ind_{[\tilde{M}(\s),\infty)}\left(M_{ks}\right)\right]\\
& \leq \E\left[ \exp\left\{- 2 M_{ks}^{1/3}\right\} \ind_{[\tilde{M}(\s),\infty)}\left(M_{ks}\right)\right] \leq \exp\left\{- 2 \tilde{M}(\s)^{1/3}\right\}.
\end{align*}
Thus, there is an event $H_{\init}^{(ii)}(k,\s)$ with
$\P\left(H_{\init}^{(i)} \cap G_{\init}^{(ii)} \setminus H_{\init}^{(ii)}(k,\s)\right) \leq \exp\left\{ - 2 \tilde{M}(\s)^{1/3} \right\}$
such that
\begin{equation*}
\hat{\Gamma}_{\s}^{(k)}\left(X_{m},-f(X_{m}), W_2, \H\right) \leq (3/2)\left(\cc\left(r_{(1-\gamma)/6}\right) + \tilde{M}(\s)^{-1/3}\right)
\end{equation*}
holds for these particular values of $k$ and $\s$.

To extend to the full range of values, we simply take
$H_{\init}^{(ii)} = G_{\init}^{(i)} \cap \bigcap_{\s \geq \init} \bigcap_{k \leq \bdim_f} H_{\init}^{(ii)}(k,\s)$.
Since $\init \geq (2 / \dprob)^{1/3}$, we have $\tilde{M}(\init) \geq 1$, 
so a union bound implies
\begin{align*}
\P&\left(H_{\init}^{(i)} \cap G_{\init}^{(i)} \setminus H_{\init}^{(ii)}\right)
\leq \sum_{\s \geq \init} \bdim_f \cdot \exp\left\{ - 2 \tilde{M}(\s)^{1/3}\right\}
\\ & \leq \bdim_f \cdot \left( \exp\left\{ - 2 \tilde{M}(\init)^{1/3} \right\} + \int_{\init}^{\infty} \exp\left\{ - 2 \tilde{M}(x)^{1/3}\right\}{\rm d}x\right)
\\ & = \bdim_f \left( 1 + 2^{-2/3} \dprob^{-1/3} \right) \cdot \exp\left\{ - 2 \tilde{M}(\init)^{1/3} \right\} 
\leq 2 \bdim_f \dprob^{-1/3} \cdot \exp\left\{ - 2 \tilde{M}(\init)^{1/3} \right\}.
\end{align*}
Then Lemma~\ref{lem:Mball-core} and a union bound imply
\begin{align*}
\P\left(H_{\init}^{(i)} \setminus H_{\init}^{(ii)}\right)
&\leq 2 \bdim_f \dprob^{-1/3} \cdot \exp\left\{ - 2 \tilde{M}(\init)^{1/3}\right\} + 121 \bdim_f \dprob^{-1} \cdot \exp\left\{- \tilde{M}(\init) / 60\right\}
\\ &\leq 123 \bdim_f \dprob^{-1} \cdot \exp\left\{ - \tilde{M}(\init)^{1/3} / 60\right\}.
\end{align*}

On $H_{\init}^{(i)} \cap H_{\init}^{(ii)}$, every such $\s,m,\ell,k$ and $\H$ satisfy
\begin{align}
\hat{\Gamma}_{\s}^{(k)}\left(X_{m},-f(X_{m}), W_2, \H\right)
& \leq (3/2)\left(\cc(r_{(1-\gamma)/6}) + \tilde{M}(\s)^{-1/3}\right) \notag
\\ & < (3/2) \left( (1-\gamma)/6 + (1-\gamma)/6\right) = (1-\gamma)/2, \label{eqn:half-sized}
\end{align}
where the second inequality follows by definition of $r_{(1-\gamma)/6}$ and $\s \geq \init \geq \init^{*}$.

If $\hat{\Delta}_{\s}^{(k)}\left(X_m, W_2, \H\right) < \gamma$, then
\begin{equation}
1-\gamma < 1-\hat{\Delta}_{\s}^{(k)}\left(X_m, W_2, \H\right) = \frac{1}{M^{(k)}_{\s}\left(\H\right)} \sum_{i=1}^{\Msize{\s}} \ind_{\S^{k-1}\left(\H\right)}\left(S_i^{(k)}\right) \ind_{\bar{\S}^{k}\left(\H\right)}\left(S_i^{(k)} \cup \{X_m\}\right). \label{eqn:full-sized}
\end{equation}
Finally, noting that we always have
\begin{equation*}
\ind_{\bar{\S}^{k}(\H)}\left(S_i^{(k)} \cup \{X_m\}\right) \leq \ind_{\bar{\S}^{k-1}(\H[(X_m,f(X_m))])}\left(S_i^{(k)}\right) + \ind_{\bar{\S}^{k-1}(\H[(X_m,-f(X_m))])}\left(S_i^{(k)}\right),
\end{equation*}
we have that, on the event $H_{\init}^{(i)} \cap H_{\init}^{(ii)}$,
if $\hat{\Delta}_{\s}^{(k)}\left(X_m,W_2,\H\right) < \gamma$,
then
\begin{align*}
&\hat{\Gamma}_{\s}^{(k)}\left(X_m, -f(X_m), W_2, \H\right) &
\\ &< (1-\gamma)/2 = -(1-\gamma)/2 + (1-\gamma) & \text{ by \eqref{eqn:half-sized}}
\\ & < -(1-\gamma)/2 + \frac{1}{M^{(k)}_{\s}\left(\H\right)} \sum_{i=1}^{\Msize{\s}} \ind_{\S^{k-1}\left(\H\right)}\left(S_i^{(k)}\right) \ind_{\bar{\S}^{k}\left(\H\right)}\left(S_i^{(k)} \cup \{X_m\}\right) & \text{ by \eqref{eqn:full-sized}}
\\ & \leq -(1-\gamma)/2 + \frac{1}{M^{(k)}_{\s}\left(\H\right)} \sum_{i=1}^{\Msize{\s}} \ind_{\S^{k-1}\left(\H\right)}\left(S_i^{(k)}\right) \ind_{\bar{\S}^{k-1}\left(\H[(X_m,f(X_m))]\right)}\left(S_i^{(k)}\right) &
\\ & \phantom{\leq -(1-\gamma)/2 } ~+ \frac{1}{M^{(k)}_{\s}\left(\H\right)} \sum_{i=1}^{\Msize{\s}} \ind_{\S^{k-1}\left(\H\right)}\left(S_i^{(k)}\right) \ind_{\bar{\S}^{k-1}\left(\H[(X_m,-f(X_m))]\right)}\left(S_i^{(k)}\right) &
\\ & = -(1-\gamma)/2 + \hat{\Gamma}_{\s}^{(k)}\left(X_m, -f(X_m), W_2, \H\right) + \hat{\Gamma}_{\s}^{(k)}\left(X_m, f(X_m), W_2, \H\right) &
\\ & < \hat{\Gamma}_{\s}^{(k)}\left(X_m, f(X_m), W_2, \H\right). & \text{ by \eqref{eqn:half-sized}}
\end{align*}

The final claim in the lemma statement is then implied by Lemma~\ref{lem:VinB},
since $\truV_{\ell} \subseteq \truV_{\init} \subseteq \Ball\left(f,\vrad(\init;\delta)\right) \subseteq \Ball\left(f,r_{(1-\gamma)/6}\right)$
on $H_{\init}(\delta)$.
\end{proof}

For any $k,\ell,m \in \nats$, and any $x \in \X$, define
\begin{align*}
\hat{p}_x(k,\ell,m) &= \hat{\Delta}_{m}^{(k)}\left(x, W_2, \truV_{\ell}\right)
\\ p_x(k,\ell) &= \Px^{k-1}\left( S \in \X^{k-1} : S \cup \{x\} \in \S^{k}\left(\truV_{\ell}\right) \Big| \S^{k-1}\left( \truV_{\ell} \right)\right).
\end{align*}

\begin{lemma}
\label{lem:empirical-inner}
For any $\zeta \in (0,1)$, there is a $(\C,\Px,f,\zeta)$-dependent constant $c^{(iii)}(\zeta) \in (0,\infty)$ such that,
for any $\init \in \nats$, there is an event $H_{\init}^{(iii)}(\zeta)$ with
\begin{equation*}
\P\left(H_{\init}^{(i)} \setminus H_{\init}^{(iii)}(\zeta) \right) \leq c^{(iii)}(\zeta) \cdot \exp\left\{ - \zeta^2 \tilde{M}(\init) \right\}
\end{equation*}
such that on $H_{\init}^{(i)} \cap H_{\init}^{(iii)}(\zeta)$,
$\forall k, \ell, m \in \nats$ with $\init \leq \ell \leq m$ and $k \leq \bdim_f$,
for any $x \in \X$,
\begin{equation*}
\Px\left(x : \left| p_x(k,\ell) - \hat{p}_x(k,\ell,m) \right| > \zeta\right) \leq \exp\left\{ - \zeta^2 \tilde{M}(m)\right\}.
\end{equation*}
\upthmend{-1.5cm}
\end{lemma}
\begin{proof}
Fix any $k,\ell,m \in \nats$ with $\tau \leq \ell \leq m$ and $k \leq \bdim_f$.
Recall our convention that $\X^{0} = \{\varnothing\}$ and $\Px^{0}\left(\X^{0}\right) =1$;
thus, if $k=1$,
$\hat{p}_x(k,\ell,m)  = \ind_{\DIS\left(\truV_{\ell}\right)}(x) = \ind_{\S^{1}\left(\truV_{\ell}\right)}(x) = p_{x}(k,\ell)$,
so the result clearly holds for $k=1$.

For the remaining case, suppose $2 \leq k \leq \bdim_f$.
To simplify notation, let $\tilde{m} = M_m^{(k)}\left(\truV_{\ell}\right)$, $X = X_{\ell+1}$, $p_x = p_x(k,\ell)$ and $\hat{p}_x = \hat{p}_x(k,\ell,m)$.
Consider the event
\begin{equation*}
H^{(iii)}(k,\ell,m,\zeta) = \left\{ \Px\left(x : \left| p_x - \hat{p}_x \right| > \zeta\right) \leq \exp\left\{- \zeta^2 \tilde{M}(m) \right\}\right\}.
\end{equation*}
We have
\begin{align}
& \P\left( H_{\init}^{(i)} \setminus H^{(iii)}(k,\ell,m,\zeta) \Big| \truV_{\ell} \right) \label{eqn:inner-original-conditional}
\\ & \leq \P\left( \left\{\tilde{m} \geq \tilde{M}(m)\right\} \setminus H^{(iii)}(k,\ell,m,\zeta)\Big| \truV_{\ell} \right) \text{ (by Lemma~\ref{lem:basic-Mk-lower-bound})} \notag
\\ & = \P\left( \left\{\tilde{m} \geq \tilde{M}(m)\right\} \cap \left\{ \P\left( e^{s \tilde{m} \left| p_{X} - \hat{p}_{X}\right|} > e^{s \tilde{m} \zeta} \Big| W_2, \truV_{\ell}\right) > e^{- \zeta^2 \tilde{M}(m)} \right\} \Big| \truV_{\ell}\right), \label{eqn:inner-pre-chernoff}
\end{align}
for any value $s > 0$.  Proceeding as in Chernoff's bounding technique, by Markov's inequality \eqref{eqn:inner-pre-chernoff} is at most
\begin{align}
& \P\left( \left\{\tilde{m} \geq \tilde{M}(m)\right\} \cap \left\{ e^{-s \tilde{m} \zeta} \E\left[ e^{s \tilde{m} \left| p_{X} - \hat{p}_{X}\right|} \Big| W_2, \truV_{\ell}\right] > e^{- \zeta^2 \tilde{M}(m)} \right\} \Big| \truV_{\ell}\right) \notag
\\ & \leq \P\left( \left\{\tilde{m} \geq \tilde{M}(m)\right\} \cap \left\{ e^{-s \tilde{m} \zeta} \E\left[ e^{s \tilde{m} \left( p_{X} - \hat{p}_{X}\right)} + e^{s \tilde{m} \left( \hat{p}_{X} - p_{X}\right)}\Big| W_2, \truV_{\ell}\right] > e^{- \zeta^2 \tilde{M}(m)} \right\} \Big| \truV_{\ell}\right) \notag
\\ & = \E\!\left[ \ind_{[\tilde{M}(m),\infty)}\left(\tilde{m}\right) \P\left(e^{-s \tilde{m} \zeta} \E\!\left[ e^{s \tilde{m} \left( p_{X} - \hat{p}_{X}\right)} \!+ e^{s \tilde{m} \left( \hat{p}_{X} - p_{X}\right)}\Big| W_2, \truV_{\ell}\right] > e^{- \zeta^2 \tilde{M}(m)} \Big| \tilde{m}, \truV_{\ell}\right) \Bigg| \truV_{\ell}\right] \notag
\end{align}
By Markov's inequality, this is at most
\begin{align}
& \E\left[ \ind_{[\tilde{M}(m),\infty)}\left(\tilde{m}\right) e^{\zeta^2 \tilde{M}(m)} \E\left[ e^{-s \tilde{m} \zeta} \E\left[ e^{s \tilde{m} \left( p_{X} - \hat{p}_{X}\right)} + e^{s \tilde{m} \left( \hat{p}_{X} - p_{X}\right)}\Big| W_2, \truV_{\ell}\right] \Big| \tilde{m}, \truV_{\ell}\right] \Bigg| \truV_{\ell}\right] \notag
\\ & = \E\left[ \ind_{[\tilde{M}(m),\infty)}\left(\tilde{m}\right) e^{\zeta^2 \tilde{M}(m)} e^{-s \tilde{m} \zeta} \E\left[ e^{s \tilde{m} \left( p_{X} - \hat{p}_{X}\right)} + e^{s \tilde{m} \left( \hat{p}_{X} - p_{X}\right)}\Big| \tilde{m}, \truV_{\ell}\right] \Bigg| \truV_{\ell}\right] \notag
\\ & = \E\left[ \ind_{[\tilde{M}(m),\infty)}\left(\tilde{m}\right) e^{\zeta^2 \tilde{M}(m)} e^{-s \tilde{m} \zeta} \E\left[ \E\left[e^{s \tilde{m} \left( p_{X} - \hat{p}_{X}\right)} + e^{s \tilde{m} \left( \hat{p}_{X} - p_{X}\right)} \Big| X, \tilde{m}, \truV_{\ell} \right] \Big| \tilde{m}, \truV_{\ell}\right] \Bigg| \truV_{\ell}\right]. \label{eqn:inner-pre-product}
\end{align}
The conditional distribution of $\tilde{m} \hat{p}_{X}$ given $\left(X, \tilde{m}, \truV_{\ell}\right)$ is ${\rm {Binomial}}\left(\tilde{m}, p_{X}\right)$,
so letting $\mathbf{B}_{1}(p_{X})$, $\mathbf{B}_{2}(p_{X})$, $\ldots$ denote a sequence of random variables, conditionally independent with distribution ${\rm {Bernoulli}}(p_{X})$ given $(X,\tilde{m},\truV_{\ell})$,
we have
\begin{align}
& \E\left[e^{s \tilde{m} \left( p_{X} - \hat{p}_{X}\right)} + e^{s \tilde{m} \left( \hat{p}_{X} - p_{X}\right)} \Big| X, \tilde{m}, \truV_{\ell} \right] \notag
\\ & = \E\left[e^{s \tilde{m} \left( p_{X} - \hat{p}_{X}\right)} \Big| X, \tilde{m}, \truV_{\ell} \right] + \E\left[e^{s \tilde{m} \left( \hat{p}_{X} - p_{X}\right)} \Big| X, \tilde{m}, \truV_{\ell} \right] \notag
\\ & = \E\left[\prod_{i=1}^{\tilde{m}} e^{s \left( p_{X} - \mathbf{B}_{i}(p_{X})\right)} \Big| X, \tilde{m}, \truV_{\ell} \right] + \E\left[\prod_{i=1}^{\tilde{m}} e^{s \left( \mathbf{B}_{i}(p_{X}) - p_{X}\right)} \Big| X, \tilde{m}, \truV_{\ell} \right] \notag
\\ & = \E\left[e^{s \left( p_{X} - \mathbf{B}_{1}(p_{X})\right)} \Big| X, \tilde{m}, \truV_{\ell} \right]^{\tilde{m}} + \E\left[e^{s \left( \mathbf{B}_{1}(p_{X}) - p_{X}\right)} \Big| X, \tilde{m}, \truV_{\ell} \right]^{\tilde{m}}. \label{eqn:m-in-exp}
\end{align}
It is known that for $\mathbf{B} \sim {\rm {Bernoulli}(p)}$, $\E\left[ e^{s (\mathbf{B} - p)} \right]$ and $\E\left[ e^{s (p - \mathbf{B})} \right]$ are at most $e^{s^2 / 8}$ \citep*[see e.g., Lemma 8.1 of][]{devroye:96}.
Thus, taking $s = 4 \zeta$, \eqref{eqn:m-in-exp} is at most $2 e^{2 \tilde{m} \zeta^2 }$,
and \eqref{eqn:inner-pre-product} is at most
\begin{align*}
\E\left[ \ind_{[\tilde{M}(m),\infty)}\left(\tilde{m}\right) 2 e^{\zeta^2 \tilde{M}(m)} e^{-4 \tilde{m} \zeta^2} e^{2 \tilde{m} \zeta^2} \Big| \truV_{\ell}\right]
&= \E\left[ \ind_{[\tilde{M}(m),\infty)}\left(\tilde{m}\right) 2 e^{\zeta^2 \tilde{M}(m)} e^{-2 \tilde{m} \zeta^2}\Big| \truV_{\ell}\right]
\\ & \leq 2 \exp\left\{- \zeta^2 \tilde{M}(m) \right\}.
\end{align*}
Since this bound holds for \eqref{eqn:inner-original-conditional}, the law of total probability implies
\begin{equation*}
\P\left(H_{\init}^{(i)} \setminus H^{(iii)}(k,\ell,m,\zeta) \right) = \E\left[\P\left( H_{\init}^{(i)} \setminus H^{(iii)}(k,\ell,m,\zeta) \Big| \truV_{\ell}\right)\right] \leq 2 \cdot \exp\left\{ - \zeta^2 \tilde{M}(m) \right\}.
\end{equation*}
Defining $H_{\init}^{(iii)}(\zeta) = \bigcap_{\ell \geq \init} \bigcap _{m \geq \ell} \bigcap_{k = 2}^{\bdim_f} H^{(iii)}(k,\ell,m,\zeta)$,
we have the required property for the claimed ranges of $k$, $\ell$ and $m$, and a union bound implies
\begin{align*}
\P&\left(H_{\init}^{(i)} \setminus H_{\init}^{(iii)}(\zeta) \right)
 \leq \sum_{\ell \geq \init} \sum_{m \geq \ell} 2 \bdim_f \cdot \exp\left\{ - \zeta^2 \tilde{M}(m) \right\}
\\ & \leq 2 \bdim_f \cdot \sum_{\ell \geq \init} \left( \exp\left\{ - \zeta^2 \tilde{M}(\ell) \right\} + \int_{\Msize{\ell}}^{\infty} \exp\left\{ - x \zeta^2 \dprob / 2 \right\} {\rm d}x \right)
\\ & = 2 \bdim_f \cdot \sum_{\ell \geq \init} \left(1 + 2 \zeta^{-2} \dprob^{-1}\right) \cdot \exp\left\{ - \zeta^2 \tilde{M}(\ell) \right\}
\\ & \leq 2 \bdim_f \cdot \left(1 + 2 \zeta^{-2} \dprob^{-1}\right) \cdot \left( \exp\left\{ - \zeta^2 \tilde{M}(\init) \right\} + \int_{\Msize{\init}}^{\infty} \exp\left\{ - x \zeta^2 \dprob / 2 \right\} {\rm d}x\right)
\\ & = 2 \bdim_f \cdot \left(1 + 2 \zeta^{-2} \dprob^{-1}\right)^2 \cdot \exp\left\{ - \zeta^2 \tilde{M}(\init) \right\}
\\ & \leq 18 \bdim_f \zeta^{-4} \dprob^{-2} \cdot \exp\left\{ - \zeta^2 \tilde{M}(\init) \right\}.
\end{align*}
\end{proof}

For $k, \ell, m \in \nats$ and $\zeta \in (0,1)$, define
\begin{equation}
\label{eqn:bar-delta-defn}
\bar{p}_{\zeta}\left(k,\ell,m\right) = \Px\left( x : \hat{p}_{x}\left(k,\ell,m\right) \geq \zeta \right).
\end{equation}

\begin{lemma}
\label{lem:basic-bar-delta-bound}
For any $\alpha,\zeta,\delta \in (0,1)$,
$\beta \in \big(0, 1-\sqrt{\alpha}\big]$, and integer $\init \geq \init(\beta;\delta)$,
on $H_{\init}(\delta) \cap H_{\init}^{(i)} \cap H_{\init}^{(iii)}(\beta\zeta)$,
for any $k, \ell, \ell^{\prime}, m \in \nats$ with
$\init \leq \ell \leq \ell^{\prime} \leq m$ and $k \leq \bdim_f$,
\begin{equation}
\label{eqn:basic-bar-delta-bound}
\bar{p}_{\zeta}(k,\ell^{\prime},m)
\leq \Px\left(x : p_{x}(k,\ell) \geq \alpha\zeta\right) + \exp\left\{- \beta^2 \zeta^2 \tilde{M}(m)\right\}.
\end{equation}
\thmend
\end{lemma}
\begin{proof}
Fix any $\alpha,\zeta,\delta \in (0,1)$, $\beta \in \big(0,1-\sqrt{\alpha}\big]$,
$\init, k, \ell, \ell^{\prime}, m \in \nats$ with
$\init(\beta;\delta) \leq \init \leq \ell \leq \ell^{\prime} \leq m$ and $k \leq \bdim_f$.

If $k = 1$, the result clearly holds.  In particular, we have
\begin{equation*}
\bar{p}_{\zeta}(1,\ell^{\prime},m)
 = \Px\left(\DIS\left(\truV_{\ell^{\prime}}\right)\right)
\leq \Px\left(\DIS\left(\truV_{\ell}\right)\right)
= \Px\left(x : p_{x}(1,\ell) \geq \alpha\zeta \right).
\end{equation*}
Otherwise, suppose $2 \leq k \leq \bdim_f$.
By a union bound,
\begin{align}
\bar{p}_{\zeta}(k,\ell^{\prime},m)
&= \Px\left(x : \hat{p}_x(k,\ell^{\prime},m) \geq \zeta \right) \notag
\\ &\leq \Px\left( x : p_x(k,\ell^{\prime}) \geq \sqrt{\alpha} \zeta \right)  + \Px\left( x : \left| p_x(k,\ell^{\prime}) - \hat{p}_x(k,\ell^{\prime},m) \right| > (1-\sqrt{\alpha})\zeta \right). \label{eqn:basic-bar-bound-two-terms}
\end{align}
Since
\begin{equation*}
\Px\left(x : \left| p_x(k,\ell^{\prime}) - \hat{p}_x(k,\ell^{\prime},m) \right| > (1-\sqrt{\alpha}) \zeta\right) \leq \Px\left( x : \left| p_x(k,\ell^{\prime}) - \hat{p}_x(k,\ell^{\prime},m) \right| > \beta \zeta\right),
\end{equation*}
Lemma \ref{lem:empirical-inner} implies that,
on $H_{\init}^{(i)} \cap H_{\init}^{(iii)}(\beta\zeta)$,
\begin{equation}
\label{eqn:basic-bar-empirical-inner}
\Px\left( x : \left| p_x(k,\ell^{\prime}) - \hat{p}_x(k,\ell^{\prime},m) \right| > (1-\sqrt{\alpha})\zeta\right) \leq \exp\left\{- \beta^2 \zeta^2 \tilde{M}(m)\right\}.
\end{equation}
It remains only to examine the first term on the right side of \eqref{eqn:basic-bar-bound-two-terms}.
For this, if $\Px^{k-1}\left(\S^{k-1}\left(\truV_{\ell^{\prime}}\right)\right) = 0$, then the first term is $0$ by our aforementioned convention, and thus \eqref{eqn:basic-bar-delta-bound} holds;
otherwise, since
\begin{equation*}
\forall x \in \X, \left\{S \in \X^{k-1} : S \cup \{x\} \in \S^{k}\left(\truV_{\ell^{\prime}}\right)\right\} \subseteq \S^{k-1}\left(\truV_{\ell^{\prime}}\right),
\end{equation*}
we have
\begin{align}
& \Px\left( x : p_x(k,\ell^{\prime}) \geq \sqrt{\alpha} \zeta \right)
= \Px\left( x : \Px^{k-1}\left( S \in \X^{k-1} : S \cup \{x\} \in \S^{k}\left(\truV_{\ell^{\prime}}\right) \Big| \S^{k-1}\left(\truV_{\ell^{\prime}}\right)\right) \geq \sqrt{\alpha} \zeta \right) \notag
\\ & = \Px\left( x : \Px^{k-1}\left( S \in \X^{k-1} : S \cup \{x\} \in \S^{k}\left(\truV_{\ell^{\prime}}\right)\right) \geq \sqrt{\alpha} \zeta \Px^{k-1}\left(\S^{k-1}\left(\truV_{\ell^{\prime}}\right)\right)\right). \label{eqn:basic-bar-break-conditional}
\end{align}
By Lemma~\ref{lem:Vshat-to-Boundaries} and monotonicity,
on $H_{\init}^{(i)} \subseteq H^{\prime}$, \eqref{eqn:basic-bar-break-conditional} is at most
\begin{equation*}
\Px\left( x : \Px^{k-1}\left( S \in \X^{k-1} : S \cup \{x\} \in \S^{k}\left(\truV_{\ell^{\prime}}\right)\right) \geq \sqrt{\alpha} \zeta \Px^{k-1}\left(\partial^{k-1}_{\C} f \right)\right),
\end{equation*}
and monotonicity implies this is at most
\begin{equation}
\label{eqn:basic-bar-ell-swap}
\Px\left( x : \Px^{k-1}\left( S \in \X^{k-1} : S \cup \{x\} \in \S^{k}\left(\truV_{\ell}\right)\right) \geq \sqrt{\alpha} \zeta \Px^{k-1}\left(\partial^{k-1}_{\C} f \right)\right).
\end{equation}
By Lemma~\ref{lem:converging-conditional},
for $\init \geq \init(\beta;\delta)$, on $H_{\init}(\delta) \cap H_{\init}^{(i)}$,
\begin{equation*}
\Px^{k-1}\left( \bar{\partial}^{k-1}_{\C} f \big| \S^{k-1}\left(\truV_{\ell}\right)\right) \leq \cc(\vrad(\init;\delta)) < \beta \leq 1-\sqrt{\alpha},
\end{equation*}
which implies
\begin{multline*}
\Px^{k-1}\left( \partial^{k-1}_{\C} f\right)
 \geq \Px^{k-1}\left( \partial^{k-1}_{\C} f \cap \S^{k-1}\left(\truV_{\ell}\right)\right)
\\  = \left(1 - \Px^{k-1}\left( \bar{\partial}^{k-1}_{\C} f \Big| \S^{k-1}\left(\truV_{\ell}\right)\right)\right) \Px^{k-1}\left(\S^{k-1}\left(\truV_{\ell}\right)\right) 
\geq \sqrt{\alpha} \Px^{k-1}\left(\S^{k-1}\left(\truV_{\ell}\right)\right).
\end{multline*}
Altogether, for $\init \geq \init(\beta;\delta)$, on $H_{\init}(\delta) \cap H_{\init}^{(i)}$, \eqref{eqn:basic-bar-ell-swap} is at most
\begin{equation*}
\Px\left( x : \Px^{k-1}\!\left( S \!\in\! \X^{k-1} : S \!\cup\! \{x\} \!\in\! \S^{k}\left(\truV_{\ell}\right)\right) \geq \alpha\zeta \Px^{k-1}\!\left(\S^{k-1}\left(\truV_{\ell}\right) \right)\right)
= \Px\left( x : p_{x}(k,\ell) \geq \alpha \zeta\right)\! ,
\end{equation*}
which, combined with \eqref{eqn:basic-bar-bound-two-terms} and \eqref{eqn:basic-bar-empirical-inner}, establishes \eqref{eqn:basic-bar-delta-bound}.
\end{proof}

\begin{lemma}
\label{lem:basic-empirical-bound}
There are events $\left\{ H_{\init}^{(iv)} : \init \in \nats\right\}$ with
\begin{equation*}
\P\left(H_{\init}^{(iv)}\right) \geq 1 - 3 \bdim_f \cdot \exp\left\{ - 2 \init \right\}
\end{equation*}
such that, for any $\xi \in (0, \gamma/16]$, $\delta \in (0,1)$, and integer $\init \geq \init^{(iv)}(\xi;\delta)$,
where $\init^{(iv)}(\xi;\delta) = \max\left\{\init(4\xi/\gamma;\delta), \left(\frac{4}{\dprob \xi^2}\ln\left(\frac{4}{\dprob \xi^2}\right)\right)^{1/3}\right\}$,
on $H_{\init}(\delta) \cap H_{\init}^{(i)} \cap H_{\init}^{(iii)}(\xi) \cap H_{\init}^{(iv)}$,
$\forall k \in \left\{1,\ldots,\bdim_f\right\}$, $\forall \ell \in \nats$ with $\ell \geq \init$,
\begin{align}
\Px\Big(x : p_{x}(k,\ell)\geq \gamma/2\Big) + \exp\left\{-\gamma^2 \tilde{M}(\ell) / 256\right\}
 &\leq \hat{\Delta}_{\ell}^{(k)}\left(W_1, W_2, \truV_{\ell}\right) \label{eqn:empirical-bound-first-ineq}
\\ &\leq \Px\left(x : p_{x}(k,\ell) \geq \gamma/8\right) + 4 \ell^{-1}. \label{eqn:empirical-bound-second-ineq}
\end{align}
\thmend
\end{lemma}
\begin{proof}
For any $k,\ell \in \nats$, by Hoeffding's inequality and the law of total probability, on an event $G^{(iv)}(k,\ell)$ with
$\P\left(G^{(iv)}(k,\ell)\right) \geq 1 - 2 \exp\left\{ - 2 \ell\right\}$,
we have
\begin{equation}
\label{eqn:empirical-bound-deviation}
\left| \bar{p}_{\gamma/4}(k,\ell,\ell) - \ell^{-3} \sum_{i=1}^{\ell^3} \ind_{[\gamma/4,\infty)}\left(\hat{\Delta}_{\ell}^{(k)}\left(w_i, W_2, \truV_{\ell}\right)\right) \right| \leq \ell^{-1}.
\end{equation}
Define the event $H_{\init}^{(iv)} = \bigcap_{\ell \geq \init} \bigcap_{k = 1}^{\bdim_f} G^{(iv)}(k,\ell)$.
By a union bound, we have
\begin{align*}
1-\P\left( H_{\init}^{(iv)} \right)
&\leq 2 \bdim_f \cdot \sum_{\ell \geq \init} \exp\left\{- 2 \ell \right\}
\\ & \leq 2 \bdim_f \cdot \left( \exp\left\{ - 2 \init \right\} + \int_{\init}^{\infty} \exp\left\{ - 2 x \right\} {\rm d}x\right)
= 3 \bdim_f \cdot \exp\left\{ - 2 \init \right\}.
\end{align*}
Now fix any $\ell \geq \init$ and $k \in \left\{1,\ldots,\bdim_f\right\}$.
By a union bound,
\begin{equation}
\label{eqn:empirical-bound-two-terms}
\Px\left(x : p_x(k,\ell) \geq \gamma/2\right) \leq \Px\left( x : \hat{p}_x(k,\ell,\ell) \geq \gamma/4\right) + \Px\left(x : \left| p_x(k,\ell) - \hat{p}_x(k,\ell,\ell) \right| > \gamma/4\right).
\end{equation}
By Lemma~\ref{lem:empirical-inner}, on $H_{\init}^{(i)} \cap H_{\init}^{(iii)}(\xi)$,
\begin{equation}
\label{eqn:empirical-bound-second-term}
\Px\left(x : \left| p_x(k,\ell) - \hat{p}_x(k,\ell,\ell) \right| > \gamma/4\right)
\leq \Px\left(x : \left| p_x(k,\ell) - \hat{p}_x(k,\ell,\ell) \right| > \xi\right)
\leq \exp\left\{- \xi^2 \tilde{M}(\ell)\right\}.
\end{equation}
Also, on $H_{\init}^{(iv)}$, \eqref{eqn:empirical-bound-deviation} implies
\begin{align}
\Px\left( x : \hat{p}_x(k,\ell,\ell) \geq \gamma/4 \right)
& = \bar{p}_{\gamma/4}(k,\ell,\ell) \notag
\\ & \leq \ell^{-1} + \ell^{-3} \sum_{i=1}^{\ell^3} \ind_{[\gamma/4,\infty)}\left(\hat{\Delta}_{\ell}^{(k)}\left(w_i,W_2,\truV_{\ell}\right)\right) \notag
\\ & = \hat{\Delta}_{\ell}^{(k)}\left(W_1, W_2, \truV_{\ell}\right) - \ell^{-1}. \label{eqn:empirical-bound-first-term}
\end{align}
Combining \eqref{eqn:empirical-bound-two-terms} with \eqref{eqn:empirical-bound-second-term} and \eqref{eqn:empirical-bound-first-term}
yields
\begin{equation}
\label{eqn:empirical-bound-unsimplified}
\Px\left( x : p_x(k,\ell) \geq \gamma/2\right) \leq \hat{\Delta}_{\ell}^{(k)}\left(W_1,W_2,\truV_{\ell}\right) - \ell^{-1} + \exp\left\{- \xi^2 \tilde{M}(\ell) \right\}.
\end{equation}
For $\init \geq \init^{(iv)}(\xi;\delta)$, $\exp\left\{-\xi^2 \tilde{M}(\ell)\right\} -\ell^{-1} \leq -\exp\left\{-\gamma^2 \tilde{M}(\ell) / 256\right\}$,
so that \eqref{eqn:empirical-bound-unsimplified} implies the first inequality of the lemma: namely \eqref{eqn:empirical-bound-first-ineq}.

For the second inequality (i.e., \eqref{eqn:empirical-bound-second-ineq}),
on $H_{\init}^{(iv)}$, \eqref{eqn:empirical-bound-deviation} implies we have
\begin{equation}
\label{eqn:empirical-bound-second-ineq-a}
\hat{\Delta}_{\ell}^{(k)}\left(W_1,W_2,\truV_{\ell}\right)
\leq \bar{p}_{\gamma/4}(k,\ell,\ell) + 3\ell^{-1}.
\end{equation}
Also, by Lemma~\ref{lem:basic-bar-delta-bound} (with $\alpha = 1/2$, $\zeta = \gamma/4$, $\beta = \xi/\zeta < 1-\sqrt{\alpha}$), for $\init \geq \init^{(iv)}(\xi;\delta)$,
on $H_{\init}(\delta) \cap H_{\init}^{(i)} \cap H_{\init}^{(iii)}(\xi)$,
\begin{equation}
\label{eqn:empirical-bound-second-ineq-b}
\bar{p}_{\gamma/4}(k,\ell,\ell) \leq \Px\left(x : p_x(k,\ell) \geq \gamma/8\right) + \exp\left\{- \xi^2 \tilde{M}(\ell)\right\}.
\end{equation}
Thus, combining \eqref{eqn:empirical-bound-second-ineq-a} with \eqref{eqn:empirical-bound-second-ineq-b} yields
\begin{equation*}
\hat{\Delta}_{\ell}^{(k)}\left(W_1,W_2,\truV_{\ell}\right) \leq \Px\left(x : p_x(k,\ell) \geq \gamma/8\right) + 3\ell^{-1} + \exp\left\{- \xi^2 \tilde{M}(\ell)\right\}.
\end{equation*}
For $\init \geq \init^{(iv)}(\xi;\delta)$, we have $\exp\left\{- \xi^2 \tilde{M}(\ell)\right\} \leq \ell^{-1}$,
which establishes \eqref{eqn:empirical-bound-second-ineq}.
\end{proof}

For $n \in \nats$ and $k \in \{1,\ldots,d+1\}$, define the set
\begin{equation*}
\U_{n}^{(k)} = \left\{ m_n + 1, \ldots, m_n + \left\lfloor n / \left( 6 \cdot 2^{k} \hat{\Delta}_{m_n}^{(k)}(W_1, W_2, V)\right)\right\rfloor \right\},
\end{equation*}
where $m_n = \lfloor n/3 \rfloor$;
$\U_{n}^{(k)}$ represents the set of indices processed in the inner loop of \BasicActivizer~for the specified value of $k$.

\begin{lemma}
\label{lem:empirical-works-too}
There are $(f,\C,\Px,\gamma)$-dependent constants $\hat{c}_1, \hat{c}_2 \in (0,\infty)$
such that, for any $\eps \in (0,1)$ and integer $n \geq \hat{c}_1 \ln(\hat{c}_2 / \eps)$,
on an event $\hat{H}_{n}(\eps)$ with
\begin{equation}
\label{eqn:hatH-bound}
\P(\hat{H}_{n}(\eps)) \geq 1- (3/4)\eps,
\end{equation}
we have, for $V = \truV_{m_n}$,
\begin{equation}
\label{eqn:kstar-query-bound}
\forall k \in \left\{1,\ldots,\bdim_f\right\}, \left|\left\{m \in \U_{n}^{(k)} : \hat{\Delta}_{m}^{(k)}(X_{m},W_2,V) \geq \gamma\right\}\right| \leq \left\lfloor n/\left(3\cdot 2^{k}\right)\right\rfloor,
\end{equation}
\begin{equation}
\label{eqn:kstar-label-everything}
\hat{\Delta}_{m_n}^{(\bdim_f)}(W_1,W_2,V) \leq \Delta_{n}^{(\gamma/8)}(\eps) + 4 m_n^{-1},
\end{equation}
and $\forall m \in \U_{n}^{(\bdim_f)}$,
\begin{equation}
\label{eqn:kstar-good-labels}
\hat{\Delta}_{m}^{(\bdim_f)}(X_{m},W_2,V) < \gamma \Rightarrow \hat{\Gamma}_{m}^{(\bdim_f)}(X_{m},-f(X_{m}),W_2,V) < \hat{\Gamma}_{m}^{(\bdim_f)}(X_{m},f(X_{m}),W_2,V).
\end{equation}
\thmend\end{lemma}
\begin{proof}
Suppose $n \geq \hat{c}_1 \ln(\hat{c}_2 / \eps)$, where
$\hat{c}_1 = \max\left\{\frac{2^{\bdim_f + 12}}{\dprob \gamma^2}, \frac{24}{r_{(1/16)}}, \frac{24}{r_{(1-\gamma)/6}}, 3\init^{*}\right\}$
and
$\hat{c}_2 = \max\left\{4\left(c^{(i)}+c^{(ii)}+c^{(iii)}(\gamma/16)+6\bdim_{f}\right), 4 \left(\frac{4e}{r_{(1/16)}}\right)^{\vc}, 4 \left(\frac{4e}{r_{(1-\gamma)/6}}\right)^{\vc}\right\}$.
In particular, we have chosen $\hat{c}_1$ and $\hat{c}_2$ large enough so that
\begin{equation*}
m_n \geq \max\left\{\init(1/16;\eps/2), \init^{(iv)}(\gamma/16;\eps/2), \init((1-\gamma)/6;\eps/2), \init^{*}\right\}.
\end{equation*}

We begin with~\eqref{eqn:kstar-query-bound}.
By Lemmas~\ref{lem:basic-bar-delta-bound} and \ref{lem:basic-empirical-bound},
on the event
\begin{equation*}
\hat{H}_n^{(1)}(\eps) = H_{m_n}(\eps/2) \cap H_{m_n}^{(i)} \cap H_{m_n}^{(iii)}(\gamma/16) \cap H_{m_n}^{(iv)},
\end{equation*}
$\forall m \in \U_{n}^{(k)}, \forall k \in \left\{1,\ldots,\bdim_f\right\}$,
\begin{align}
\bar{p}_{\gamma}\left(k, m_n, m\right)
&\leq \Px\left( x : p_{x}(k, m_n) \geq \gamma / 2\right) + \exp\left\{-\gamma^2 \tilde{M}(m) / 256\right\} \notag
\\ &\leq \Px\left( x : p_{x}(k, m_n) \geq \gamma / 2\right) + \exp\left\{- \gamma^2 \tilde{M}(m_n) / 256\right\}
\leq \hat{\Delta}_{m_n}^{(k)}\left(W_1, W_2, V\right). \label{eqn:empirical-works-bar-bound}
\end{align}
Recall that $\left\{ X_{m} : m \in \U_{n}^{(k)}\right\}$ is a sample of size
$\left\lfloor n / (6 \cdot 2^{k} \hat{\Delta}_{m_n}^{(k)}(W_1,W_2,V))\right\rfloor$, conditionally
i.i.d. (given $(W_1,W_2,V)$) with conditional distributions $\Px$.
Thus, $\forall k \in \left\{1,\ldots,\bdim_f\right\}$,
on $\hat{H}_n^{(1)}(\eps)$,
\begin{align}
& \P\left( \left|\left\{ m \in \U_{n}^{(k)} : \hat{\Delta}_{m}^{(k)}\left(X_m,W_2,V\right) \geq \gamma\right\}\right| > n / \left( 3 \cdot 2^{k} \right) \Bigg| W_1,W_2,V\right)\notag
\\ & \leq \P\left( \left|\left\{ m \in \U_{n}^{(k)} : \hat{\Delta}_{m}^{(k)}\left(X_m,W_2,V\right) \geq \gamma\right\}\right| > 2\left| \U_{n}^{(k)} \right| \hat{\Delta}_{m_n}^{(k)}(W_1,W_2,V) \Bigg| W_1,W_2,V\right)\notag
\\ & \leq \P\left( \mathbf{B}\left(|\U_{n}^{(k)}|, \hat{\Delta}_{m_n}^{(k)}(W_1,W_2,V)\right) > 2 \left| \U_{n}^{(k)} \right| \hat{\Delta}_{m_n}^{(k)}(W_1,W_2,V) \Bigg| W_1,W_2,V\right), \label{eqn:empirical-works-pre-chernoff}
\end{align}
where this last inequality follows from \eqref{eqn:empirical-works-bar-bound},
and $\mathbf{B}(u,p) \sim \text{Binomial}(u,p)$ is independent of $W_1,W_2,V$ (for any fixed $u$ and $p$).
By a Chernoff bound, \eqref{eqn:empirical-works-pre-chernoff} is at most
\begin{equation*}
\exp\left\{- \left\lfloor n / \left( 6 \cdot 2^{k} \hat{\Delta}_{m_n}^{(k)}(W_1,W_2,V)\right)\right\rfloor \hat{\Delta}_{m_n}^{(k)}(W_1,W_2,V) / 3 \right\} \leq \exp\left\{ 1 - n / \left(18 \cdot 2^k\right)\right\}.
\end{equation*}
By the law of total probability and a union bound, there exists an event $\hat{H}_n^{(2)}$ with
\begin{equation*}
\P\left(\hat{H}_n^{(1)}(\eps) \setminus \hat{H}_n^{(2)}\right) \leq \bdim_f \cdot \exp\left\{1 - n / \left(18 \cdot 2^{\bdim_f}\right)\right\}
\end{equation*}
such that, on $\hat{H}_n^{(1)}(\eps) \cap \hat{H}_n^{(2)}$, \eqref{eqn:kstar-query-bound} holds.

Next, by Lemma~\ref{lem:basic-empirical-bound}, on $\hat{H}_{n}^{(1)}(\eps)$,
\begin{equation*}
\hat{\Delta}_{m_n}^{(\bdim_f)}(W_1,W_2,V) \leq \Px\left( x : p_{x}\left(\bdim_f, m_n\right) \geq \gamma/8\right) + 4 m_n^{-1},
\end{equation*}
and by Lemma~\ref{lem:label-everything}, on $\hat{H}_{n}^{(1)}(\eps)$, this is at most
$\Delta_n^{(\gamma/8)}(\eps) + 4 m_n^{-1}$,
which establishes \eqref{eqn:kstar-label-everything}.

Finally, Lemma~\ref{lem:kstar-good-labels} implies that on $\hat{H}_{n}^{(1)}(\eps) \cap H_{m_n}^{(ii)}$, $\forall m \in \U_{n}^{(\bdim_f)}$, \eqref{eqn:kstar-good-labels} holds.

Thus, defining
\begin{equation*}
\hat{H}_{n}(\eps) = \hat{H}_n^{(1)}(\eps) \cap \hat{H}_n^{(2)} \cap H_{m_n}^{(ii)},
\end{equation*}
it remains only to establish \eqref{eqn:hatH-bound}.
By a union bound, we have
\begin{align}
1 - \P\left(\hat{H}_n\right) &\leq
\left(1-\P\left(H_{m_n}(\eps/2)\right)\right)
+ \left(1-\P\left(H_{m_n}^{(i)}\right)\right)
+ \P\left(H_{m_n}^{(i)} \setminus H_{m_n}^{(ii)}\right) \notag
\\ & + \P\left(H_{m_n}^{(i)} \setminus H_{m_n}^{(iii)}(\gamma/16)\right)
+ \left(1 - \P\left(H_{m_n}^{(iv)}\right)\right)
+ \P\left(\hat{H}_n^{(1)}(\eps) \setminus \hat{H}_n^{(2)}\right). \notag
\\ & \leq \eps/2
+ c^{(i)} \cdot \exp\left\{ - \tilde{M}(m_n) / 4\right\}
+ c^{(ii)}\cdot \exp\left\{ - \tilde{M}(m_n)^{1/3} / 60\right\} \notag
\\ &+ c^{(iii)}(\gamma/16)\cdot \exp\left\{-\tilde{M}(m_n) \gamma^2 / 256\right\}
+ 3 \bdim_f\cdot \exp\left\{-2 m_n\right\} \notag
\\ &+ \bdim_f \cdot \exp\left\{1 - n / \left(18 \cdot 2^{\bdim_f} \right)\right\} \notag
\\ & \leq \eps/2 + \left(c^{(i)} + c^{(ii)} + c^{(iii)}(\gamma/16) + 6 \bdim_f \right) \cdot \exp\left\{- n \dprob \gamma^2 2^{-\bdim_f-12} \right\}. \label{eqn:hatH-exp-bound}
\end{align}
We have chosen $n$ large enough so that \eqref{eqn:hatH-exp-bound} is at most $(3/4)\eps$, which establishes \eqref{eqn:hatH-bound}.
\end{proof}

The following result is a slightly stronger version of Theorem~\ref{thm:activizer}.

\begin{lemma}
\label{lem:activizer}
For any passive learning algorithm $\alg_{p}$,
if $\alg_{p}$ achieves a label complexity $\Lambda_{p}$ with
$\infty > \Lambda_{p}(\eps,f,\Px) = \omega(\log(1/\eps))$,
then \BasicActivizer, with $\alg_{p}$ as its argument,
achieves a label complexity $\Lambda_{a}$ such that
$\Lambda_{a}(3\eps,f,\Px) = o(\Lambda_{p}(\eps,f,\Px))$.
\thmend
\end{lemma}
\begin{proof}
Suppose $\alg_{p}$ achieves label complexity $\Lambda_{p}$ with $\infty > \Lambda_{p}(\eps,f,\Px) = \omega(\log(1/\eps))$.
Let $\eps \in (0,1)$,
define $L(n;\eps) = \left\lfloor n / \left(6\cdot 2^{\bdim_f} \left(\Delta_{n}^{(\gamma/8)}(\eps)+4m_n^{-1}\right) \right) \right\rfloor$ (for any $n \in \nats$),
and let $L^{-1}(m;\eps) = \max\left\{ n \in \nats : L(n;\eps) < m\right\}$ (for any $m \in (0,\infty)$).
Define 
\begin{eqnarray*}
c_1 = \max\left\{\hat{c}_1, 2 \cdot 6^3(\vc+1) \bdim_f \ln(e (\vc+1))\right\} & \text{ and } & c_2 = \max\left\{\hat{c}_2, 4e (\vc+1) \right\},
\end{eqnarray*}
and suppose
\begin{equation*}
n \geq \max\Big\{c_1 \ln(c_2 / \eps), 1 + L^{-1}\left(\Lambda_{p}(\eps,f,\Px);\eps\right)\Big\}.
\end{equation*}
Consider running \BasicActivizer~with $\alg_{p}$ and $n$ as inputs, while $f$ is the target function and $\Px$ is the data distribution.

Letting $\hat{h}_{n}$ denote the classifier returned from \BasicActivizer,
Lemma~\ref{lem:active-select} implies that on an event $\hat{E}_{n}$ with
$\P(\hat{E}_{n}) \geq 1 - e (\vc+1) \cdot \exp\left\{- \lfloor n/3 \rfloor / (72 \bdim_f (\vc+1) \ln(e (\vc+1)))\right\} \geq 1-\eps/4$,
we have
\begin{equation*}
\er(\hat{h}_{n}) \leq 2 \er\left(\alg_{p}\left(\L_{\bdim_f}\right)\right).
\end{equation*}
By a union bound, the event $\hat{G}_{n}(\eps) = \hat{E}_{n} \cap \hat{H}_{n}(\eps)$
has $\P\left(\hat{G}_{n}(\eps)\right) \geq 1-\eps$.
Thus,
\begin{align}
\E\left[\er\left(\hat{h}_{n}\right)\right]
&\leq \E\left[ \ind_{\hat{G}_{n}(\eps)} \ind\left[|\L_{\bdim_f}| \geq \Lambda_{p}(\eps,f,\Px)\right] \er\left(\hat{h}_{n}\right)\right] \notag
\\ & {\hskip 4cm} + \P\left(\hat{G}_{n}(\eps) \cap \left\{|\L_{\bdim_f}| < \Lambda_{p}(\eps,f,\Px)\right\}\right) + \P\left(\hat{G}_{n}(\eps)^{c}\right) \notag
\\ & \leq \E\left[ \ind_{\hat{G}_{n}(\eps)} \ind\left[|\L_{\bdim_f}| \geq \Lambda_{p}(\eps,f,\Px)\right] 2 \er\left(\alg_{p}\left(\L_{\bdim_f}\right)\right)\right] \notag
\\ & {\hskip 4cm} + \P\left(\hat{G}_{n}(\eps) \cap \left\{|\L_{\bdim_f}| < \Lambda_{p}(\eps,f,\Px)\right\}\right) + \eps. \label{eqn:activizer-1}
\end{align}
On $\hat{G}_{n}(\eps)$, \eqref{eqn:kstar-label-everything} of Lemma~\ref{lem:empirical-works-too}
implies $|\L_{\bdim_f}| \geq L(n;\eps)$, and we chose $n$ large enough so that $L(n;\eps) \geq \Lambda_{p}(\eps,f,\Px)$.
Thus, the second term in \eqref{eqn:activizer-1} is zero, and we have
\begin{align}
\E\left[\er\left(\hat{h}_{n}\right)\right] & \leq 2 \cdot \E\left[ \ind_{\hat{G}_{n}(\eps)} \ind\left[|\L_{\bdim_f}| \geq \Lambda_{p}(\eps,f,\Px)\right] \er\left(\alg_{p}\left(\L_{\bdim_f}\right)\right)\right] + \eps \notag
\\ &= 2 \cdot \E\left[ \E\left[ \ind_{\hat{G}_{n}(\eps)} \er\left(\alg_{p}\left(\L_{\bdim_f}\right)\right) \Big| |\L_{\bdim_f}|\right] \ind\left[|\L_{\bdim_f}| \geq \Lambda_{p}(\eps,f,\Px)\right]\right] + \eps. \label{eqn:activizer-2}
\end{align}

Note that for any $\ell$ with $\P(|\L_{\bdim_f}| = \ell) > 0$,
the conditional distribution of
$\left\{X_{m} : m \in \U_{n}^{(\bdim_f)}\right\}$ given $\left\{|\L_{\bdim_f}| = \ell\right\}$
is simply the product $\Px^{\ell}$ (i.e., conditionally i.i.d.), which is the same as
the distribution of $\{X_1,X_2,\ldots,X_{\ell}\}$.
Furthermore, on $\hat{G}_{n}(\eps)$, \eqref{eqn:kstar-query-bound}
implies that the $t < \lfloor 2n/3\rfloor$ condition is always satisfied in Step 6 of \BasicActivizer~while $k \leq \bdim_f$,
and \eqref{eqn:kstar-good-labels} implies that
the inferred labels from Step 8 for $k = \bdim_f$ are all correct.
Therefore, for any such $\ell$ with $\ell \geq \Lambda_{p}(\eps,f,\Px)$, we have
\begin{equation*}
\E\left[ \ind_{\hat{G}_{n}(\eps)} \er\left(\alg_{p}\left(\L_{\bdim_f}\right)\right) \Big| \left\{|\L_{\bdim_f}| = \ell\right\}\right]
\leq \E\left[ \er\left(\alg_{p}\left(\Data_{\ell}\right)\right) \right]
\leq \eps.
\end{equation*}
In particular, this means \eqref{eqn:activizer-2} is at most $3\eps$.
This implies that \BasicActivizer, with $\alg_{p}$ as its argument, achieves a label complexity $\Lambda_{a}$ such that
\begin{equation*}
\Lambda_{a}(3\eps,f,\Px) \leq \max\left\{c_1 \ln(c_2 / \eps), 1 + L^{-1}\left(\Lambda_{p}(\eps,f,\Px);\eps\right)\right\}.
\end{equation*}

Since $\Lambda_{p}(\eps,f,\Px) = \omega(\log(1/\eps)) \Rightarrow c_1 \ln(c_2 / \eps) = o\left(\Lambda_{p}(\eps,f,\Px)\right)$,
it remains only to show that $L^{-1}\left(\Lambda_{p}(\eps,f,\Px);\eps\right) = o\left(\Lambda_{p}(\eps,f,\Px)\right)$.
Note that $\forall \eps \in (0,1)$, $L(1;\eps) = 0$ and $L(n;\eps)$ is diverging in $n$.
Furthermore, by Lemma~\ref{lem:label-everything}, we know that for any $\nats$-valued $N(\eps) = \omega(\log(1/\eps))$,
we have $\Delta_{N(\eps)}^{(\gamma/8)}(\eps) = o(1)$,
which implies $L(N(\eps);\eps) = \omega(N(\eps))$.  Thus, since $\Lambda_{p}(\eps,f,\Px) = \omega(\log(1/\eps))$,
Lemma~\ref{lem:inverse-little-o} implies $L^{-1}\left(\Lambda_{p}(\eps,f,\Px);\eps\right) = o\left(\Lambda_{p}(\eps,f,\Px)\right)$, as desired.

This establishes the result for an arbitrary $\gamma \in (0,1)$.
To specialize to the specific procedure stated as \BasicActivizer, we simply take $\gamma = 1/2$.
\end{proof}

\begin{proof}[Theorem~\ref{thm:activizer}]
Theorem~\ref{thm:activizer} now follows immediately from Lemma~\ref{lem:activizer}.
Specifically, we have proven Lemma~\ref{lem:activizer} for
an arbitrary distribution $\Px$ on $\X$, an arbitrary $f \in \cl(\C)$, and an arbitrary passive algorithm $\alg_{p}$.
Therefore, it will certainly hold for every $\Px$ and $f \in \C$,
and since every $(f,\Px) \in \Nontrivial(\Lambda_{p})$ has
$\infty > \Lambda_{p}(\eps,f,\Px) = \omega(\log(1/\eps))$,
the implication that \BasicActivizer~activizes every passive algorithm $\alg_{p}$ for $\C$ follows.
\end{proof}

Careful examination of the proofs above reveals that the ``$3$'' in Lemma~\ref{lem:activizer}
can be set to any arbitrary constant strictly larger than $1$, by an appropriate modification of the ``$7/12$'' 
threshold in $\ActiveSelect$.  In fact, if we were to replace Step 4 of $\ActiveSelect$ by instead 
selecting $\hat{k} = \argmin_{k} \max_{j \neq k} m_{k j}$ (where $m_{k j} = \er_{Q_{k j}}(h_k)$ when $k < j$), 
then we could even make this a certain
$(1+o(1))$ function of $\eps$, at the expense of larger constant factors in $\Lambda_{a}$.

\section{The Label Complexity of \CAL}
\label{app:cal}

As mentioned, Theorem~\ref{thm:cal} is essentially implied by the details of
the proof of Theorem~\ref{thm:sequential-activizer} in Appendix~\ref{app:exponential} below.
Here we present 
a proof of Theorem~\ref{thm:cal-lower}, along with two useful related lemmas.
The first, Lemma~\ref{lem:cal-queries-lower}, lower bounds the expected number of label 
requests \CAL~would make while processing a given number of random unlabeled examples.
The second, Lemma~\ref{lem:cal-dis-lower}, bounds the amount by which each label request
is expected to reduce the probability mass in the region of disagreement.
Although we will only use Lemma~\ref{lem:cal-dis-lower} in our proof of 
Theorem~\ref{thm:cal-lower}, 
Lemma~\ref{lem:cal-queries-lower} may be of independent interest, 
as it provides additional insights into the behavior of disagreement based methods,
as related to the disagreement coefficient, and is included for this reason.

Throughout, we fix an arbitrary class $\C$, a target function $f \in \C$, and a distribution $\Px$,
and we continue using the notational conventions of the proofs above, 
such as $\truV_{m} = \{h \in \C : \forall i \leq m, h(X_i) = f(X_i)\}$ (with $\truV_{0} = \C$).
Additionally, for $t \in \nats$, define the random variable
\begin{equation*}
M(t) = \min\left\{ m \in \nats : \sum_{\ell=1}^{m} \ind_{\DIS\left(\truV_{\ell-1}\right)}\left(X_{\ell}\right) = t \right\},
\end{equation*}
which represents the index of the $t^{{\rm th}}$
unlabeled example \CAL~would request the label of (assuming it has not yet halted).

The two aforementioned lemmas are formally stated as follows.

\begin{lemma}
\label{lem:cal-queries-lower}
For any $r \in (0,1)$,
\begin{equation*}
\E\left[ \sum_{m=1}^{\lceil 1/r \rceil} \ind_{\DIS\left(\truV_{m-1}\right)}\left(X_m\right)\right]
\geq \frac{\Px\left(\DIS\left(\Ball(f,r)\right)\right)}{2 r}.
\end{equation*}
\upthmend{-.85cm}
\end{lemma}

\begin{lemma}
\label{lem:cal-dis-lower}
For any $r \in (0,1)$ and $n \in \nats$,
\begin{equation*}
\E\left[ \Px\left( \DIS\left( \truV_{M(n)} \right) \right) \right]
\geq \Px\left(\DIS\left(\Ball(f,r)\right)\right) - n r.
\end{equation*}
\upthmend{-1.15cm}
\end{lemma}

Before proving these lemmas, let us first mention their relevance to the disagreement coefficient analysis.
Specifically, note that when $\dc_{f}(\eps)$ is unbounded, there exist arbitrarily small 
values of $\eps$ for which 
$\Px(\DIS(\Ball(f,\eps)))/\eps \approx \dc_{f}(\eps)$, 
so that in particular $\Px(\DIS(\Ball(f,\eps)))/\eps \neq o\left(\dc_{f}(\eps)\right)$.
Therefore, Lemma~\ref{lem:cal-queries-lower} implies that the number of label requests
\CAL~makes among the first $\left\lceil 1/\eps \right\rceil$ unlabeled examples
is $\neq o\left(\dc_{f}(\eps)\right)$ (assuming it does not halt first).
Likewise, one implication of Lemma~\ref{lem:cal-dis-lower} is that
arriving at a region of disagreement with expected probability mass less than
$\Px(\DIS(\Ball(f,\eps)))/2$ requires a budget $n$ of at least 
$\Px(\DIS(\Ball(f,\eps)))/(2\eps) \neq o\left(\dc_{f}(\eps)\right)$.

We now present proofs of Lemmas~\ref{lem:cal-queries-lower} and \ref{lem:cal-dis-lower}.

\begin{proof}[Lemma~\ref{lem:cal-queries-lower}]
Since
\begin{align}
\E\left[ \sum_{m=1}^{\lceil 1/r \rceil} \ind_{\DIS\left(\truV_{m-1}\right)}\left(X_m\right)\right]
& = \sum_{m=1}^{\lceil 1/r \rceil} \E\left[\P\left(X_{m} \in \DIS\left(\truV_{m-1}\right) \Big| \truV_{m-1}\right)\right] \notag
\\ & = \sum_{m=1}^{\lceil 1/r \rceil} \E\left[\Px\left(\DIS\left(\truV_{m-1}\right)\right)\right], \label{eqn:cal-number-of-queries}
\end{align}
we focus on lower bounding $\E\left[\Px\left(\DIS\left(\truV_{m}\right)\right)\right]$ for $m \in \nats \cup \{0\}$.
Let $D_m = \DIS\left(\truV_m \cap \Ball(f,r)\right)$.  Note that
for any $x \in \DIS(\Ball(f,r))$, there exists some $h_{x} \in \Ball(f,r)$
with $h_{x}(x) \neq f(x)$, and if this $h_{x} \in \truV_{m}$, then
$x \in D_{m}$ as well.  This means
$\forall x, \ind_{D_{m}}(x) \geq \ind_{\DIS(\Ball(f,r))}(x) \cdot \ind_{\truV_{m}}(h_{x}) = \ind_{\DIS(\Ball(f,r))}(x) \cdot \prod_{\ell=1}^{m} \ind_{\DIS(\{h_{x},f\})^{c}}(X_{\ell})$.
Therefore,
\begin{align}
%\P\left( X_m \in D_{m-1} \right) =
\E\left[\Px\left(\DIS\left(\truV_{m}\right)\right)\right]
& = \P\left( X_{m+1} \in \DIS\left(\truV_{m}\right) \right)
\geq \P\left( X_{m+1} \in D_{m} \right)
= \E\left[ \E\left[ \ind_{D_{m}}\left(X_{m+1}\right) \Big| X_{m+1} \right] \right] \notag
\\ & \geq \E\left[ \E\left[ \ind_{\DIS(\Ball(f,r))}(X_{m+1}) \cdot \prod_{\ell=1}^{m} \ind_{\DIS(\{h_{X_{m+1}},f\})^{c}}(X_{\ell}) \Bigg| X_{m+1} \right] \right] \notag
\\ &= \E\left[ \prod_{\ell=1}^{m} \P\left(h_{X_{m+1}}(X_{\ell}) = f(X_{\ell}) \Big| X_{m+1} \right) \ind_{\DIS(\Ball(f,r))}(X_{m+1}) \right] \label{eqn:cal-dis-size-1}
\\ & \geq \E\left[ (1-r)^{m} \ind_{\DIS(\Ball(f,r))}(X_{m+1}) \right]
= (1-r)^{m} \Px(\DIS(\Ball(f,r))), \label{eqn:cal-dis-size-2}
\end{align}
where the equality in \eqref{eqn:cal-dis-size-1} is by conditional independence of the $\ind_{\DIS(\{h_{X_{m+1}},f\})^{c}}(X_{\ell})$ indicators, given $X_{m+1}$,
and the inequality in \eqref{eqn:cal-dis-size-2} is due to $h_{X_{m+1}} \in \Ball(f,r)$.
This indicates \eqref{eqn:cal-number-of-queries} is at least
\begin{multline*}
\sum_{m=1}^{\lceil 1/r \rceil} \left(1 - r\right)^{m-1} \Px\left(\DIS\left(\Ball(f,r)\right)\right)
\geq \sum_{m=1}^{\lceil 1/r \rceil} \left(1 - (m-1)r\right) \Px\left(\DIS\left(\Ball(f,r)\right)\right)
\\ = \lceil 1 / r \rceil \left(1 - \frac{\lceil 1 / r \rceil - 1}{2} r\right) \Px\left(\DIS\left(\Ball(f,r)\right)\right)
\geq \frac{\Px\left(\DIS\left(\Ball(f,r)\right)\right)}{2 r}.
\end{multline*}
\end{proof}

\begin{proof}[Lemma~\ref{lem:cal-dis-lower}]
For each $m \in \nats \cup \{0\}$, let $D_{m} = \DIS\left(\Ball(f,r) \cap \truV_{m}\right)$.
For convenience, let $M(0) = 0$.
We prove the result by induction.
We clearly have $\E\left[\Px\left(D_{M(0)}\right)\right] = \E\left[\Px\left(D_{0}\right)\right] = \Px(\DIS(\Ball(f,r)))$, which serves as our base case.
Now fix any $n \in \nats$, and take as the inductive hypothesis that
\begin{equation*}
\E\left[\Px\left(D_{M(n-1)}\right)\right] \geq \Px(\DIS(\Ball(f,r))) - (n-1) r.
\end{equation*}
As in the proof of Lemma~\ref{lem:cal-queries-lower},
for any $x \in D_{M(n-1)}$, there exists $h_{x} \in \Ball(f,r) \cap \truV_{M(n-1)}$
with $h_{x}(x) \neq f(x)$; unlike the proof of Lemma~\ref{lem:cal-queries-lower}, here $h_{x}$ is a random variable, determined by $\truV_{M(n-1)}$.
If $h_{x}$ is also in $\truV_{M(n)}$, then $x \in D_{M(n)}$ as well.
Thus,
$\forall x, \ind_{D_{M(n)}}(x) \geq \ind_{D_{M(n-1)}}(x) \cdot \ind_{\truV_{M(n)}}(h_{x}) = \ind_{D_{M(n-1)}}(x) \cdot \ind_{\DIS(\{h_{x},f\})^{c}}(X_{M(n)})$,
where this last equality is due to the fact that every $m \in \{M(n-1)+1,\ldots,M(n)-1\}$ has $X_m \notin \DIS\left(\truV_{m-1}\right)$, so that
in particular $h_{x}(X_m) = f(X_m)$.
Therefore, letting $X \sim \Px$ be independent of the data $\Data$,
\begin{align}
\E\left[ \Px\left(D_{M(n)}\right) \right]
& = \E\left[ \ind_{D_{M(n)}}(X) \right]
\geq \E\left[ \ind_{D_{M(n-1)}}(X) \cdot \ind_{\DIS(\{h_{X},f\})^{c}}(X_{M(n)}) \right] \notag
\\ & = \E\left[ \ind_{D_{M(n-1)}}(X) \cdot \P\left( h_{X}(X_{M(n)}) = f(X_{M(n)}) \Big| X, \truV_{M(n-1)} \right) \right]. \label{eqn:cal-dis-lower-1}
\end{align}
The conditional distribution of $X_{M(n)}$ given $\truV_{M(n-1)}$ is merely $\Px$,
but with support restricted to $\DIS\left(\truV_{M(n-1)}\right)$, and renormalized to a probability measure.
Thus, since any $x \in D_{M(n-1)}$ has $\DIS(\{h_{x},f\}) \subseteq \DIS\left(\truV_{M(n-1)}\right)$,
we have
\begin{equation*}
\P\left( h_{x}(X_{M(n)}) \neq f(X_{M(n)}) \Big| \truV_{M(n-1)}\right)
= \frac{\Px\left(\DIS(\{h_{x}, f\})\right)}{\Px\left(\DIS\left(\truV_{M(n-1)}\right)\right)}
\leq \frac{r}{\Px\left( D_{M(n-1)} \right)},
\end{equation*}
where the inequality follows from $h_{x} \in \Ball(f,r)$ and $D_{M(n-1)} \subseteq \DIS\left(\truV_{M(n-1)}\right)$.
Therefore, \eqref{eqn:cal-dis-lower-1} is at least
\begin{align*}
\E\bigg[ \ind_{D_{M(n-1)}}(X) \cdot & \left( 1 - \frac{r}{\Px(D_{M(n-1)})} \right) \bigg]
%\\ & = \E\left[ \E\left[ \ind_{D_{M(n-1)}}(X) \cdot \left( 1 - \frac{r}{\Px(D_{M(n-1)})} \right) \Big| D_{M(n-1)}\right]\right]
%\\ & = \E\left[ \E\left[ \ind_{D_{M(n-1)}}(X) \Big| D_{M(n-1)}\right] \cdot \left( 1 - \frac{r}{\Px(D_{M(n-1)})} \right) \right]
\\ & = \E\left[ \P\left( X \in D_{M(n-1)} \Big| D_{M(n-1)}\right) \cdot \left( 1 - \frac{r}{\Px(D_{M(n-1)})} \right) \right]
\\ & = \E\left[ \Px\left( D_{M(n-1)}\right) \cdot \left( 1 - \frac{r}{\Px(D_{M(n-1)})} \right) \right]
= \E\left[ \Px\left( D_{M(n-1)}\right)\right] - r.
\end{align*}
By the inductive hypothesis, this is at least $\Px(\DIS(\Ball(f,r))) - n r$.

Finally, noting $\E\left[\Px\left( \DIS\left( \truV_{M(n)} \right) \right)\right] \geq \E\left[\Px\left( D_{M(n)} \right)\right]$ completes the proof.
\end{proof}

With Lemma~\ref{lem:cal-dis-lower} in hand, we are ready for the proof of Theorem~\ref{thm:cal-lower}.

\begin{proof}[Theorem~\ref{thm:cal-lower}]
Let $\C$, $f$, $\Px$, and $\lambda$ be as in the theorem statement.
For $m \in \nats$, let $\lambda^{-1}(m) = \inf\{ \eps > 0 : \lambda(\eps) \leq m\}$, or $1$ if this is not defined.
We define $\alg_p$ as a randomized algorithm such that,
for $m \in \nats$ and $\L \in (\X \times \{-1,+1\})^{m}$,
$\alg_{p}(\L)$ returns $f$ with probability $1-\lambda^{-1}(|\L|)$
and returns $-f$ with probability $\lambda^{-1}(|\L|)$ (independent of
the contents of $\L$).
Note that, for any integer $m \geq \lambda(\eps)$,
$\E\left[ \er\left(\alg_p\left(\Data_m\right)\right) \right] = \lambda^{-1}(m) \leq \lambda^{-1}( \lambda(\eps) ) \leq \eps$.
Therefore, $\alg_p$ achieves some label complexity $\Lambda_p$ with $\Lambda_p(\eps,f,\Px) = \lambda(\eps)$ for all $\eps > 0$.

If $\dc_{f}\left(\lambda(\eps)^{-1}\right) \neq \omega(1)$,
then since every label complexity $\Lambda_{a}$ is $\Omega(1)$, the result clearly holds.
Otherwise, suppose $\dc_{f}\left(\lambda(\eps)^{-1}\right) = \omega(1)$,
and take any sequence of values $\eps_{i} \to 0$ for which each $i$ has $\eps_i \in (0,1/2)$,
$\dc_{f}\left(\lambda(2\eps_{i})^{-1}\right) \geq 12$,
and $2 \eps_i$ a continuity point of $\lambda$;
this is possible, since $\lambda$ is monotone,
and thus has only a countably infinite number of discontinuities.
We have that $\dc_{f}\left(\lambda(2 \eps_{i})^{-1}\right)$ diverges as $i \to \infty$,
and thus so does $\lambda(2 \eps_{i})$.
This then implies that there exist values $r_i \to 0$ such that
each $r_i > \lambda(2 \eps_{i})^{-1}$ and
$\frac{\Px(\DIS(\Ball(f,r_i)))}{r_i} \geq \dc_{f}\left( \lambda(2\eps_{i})^{-1} \right) / 2$. 

Fix any $i \in \nats$ and any $n \in \nats$ with $n \leq \dc_{f}\left( \lambda(2\eps_{i})^{-1} \right) / 4$.
Consider running \CAL~with arguments $\alg_p$ and $n$,
and let $\hat{\L}$ denote the final value of the set $\L$,
and let $\check{m}$ denote the value of $m$ upon reaching Step 6.
Since $2\eps_i$ is a continuity point of $\lambda$,
any $m < \lambda(2\eps_i)$ and $\L \in (\X \times \{-1,+1\})^{m}$
has $\er\left(\alg_{p}(\L)\right) = \lambda^{-1}(m) > 2 \eps_i$.
Therefore, we have
\begin{align}
\E\left[ \er\left( \alg_{p}\left(\hat{\L}\right) \right) \right]
&\geq 2 \eps_{i} \P\left( |\hat{\L}| < \lambda(2 \eps_i) \right)
 = 2 \eps_{i} \P\left( \left\lfloor n / \left(6 \hat{\Delta}\right)\right\rfloor < \lambda(2 \eps_i) \right) \notag
\\ & = 2 \eps_{i} \P\left( \hat{\Delta} > \frac{n}{6 \lambda(2 \eps_i)} \right)
= 2 \eps_{i} \left( 1 - \P\left( \hat{\Delta} \leq \frac{n}{6 \lambda(2 \eps_i)} \right)\right). \label{eqn:cal-lower-1}
\end{align}
Since
$n \leq \dc_{f}\left( \lambda(2\eps_{i})^{-1}\right) / 4 \leq \Px(\DIS(\Ball(f,r_i))) / (2 r_i) < \lambda(2\eps_{i}) \Px(\DIS(\Ball(f,r_i))) / 2$,
we have
\begin{multline}
\label{eqn:cal-lower-2}
\P\left( \hat{\Delta} \leq \frac{n}{6 \lambda(2 \eps_i)} \right)
 \leq \P\left( \hat{\Delta} < \Px(\DIS(\Ball(f,r_i))) / 12\right)
\\ \leq \P\left( \Big\{\Px\left(\DIS\left(\truV_{\check{m}}\right)\right) < \Px(\DIS(\Ball(f,r_i))) / 12\Big\} \cup \left\{\hat{\Delta} < \Px\left(\DIS\left(\truV_{\check{m}}\right)\right)\right\}\right).
\end{multline}
Since $\check{m} \leq M(\lceil n/2 \rceil)$, monotonicity and a union bound imply this is at most
\begin{equation}
\label{eqn:cal-lower-3}
\P\left( \Px\left(\DIS\left(\truV_{M(\lceil n/2 \rceil)}\right)\right) < \Px(\DIS(\Ball(f,r_i))) / 12\right) +  \P\left(\hat{\Delta} < \Px\left(\DIS\left(\truV_{\check{m}}\right)\right)\right).
\end{equation}
Markov's inequality implies
\begin{align*}
& \P\left( \Px\left(\DIS\left(\truV_{M(\lceil n/2 \rceil)}\right)\right) < \Px(\DIS(\Ball(f,r_i))) / 12\right)
\\ & = \P\left( \Px(\DIS(\Ball(f,r_i))) - \Px\left(\DIS\left(\truV_{M(\lceil n/2 \rceil)}\right)\right) > \frac{11}{12} \Px(\DIS(\Ball(f,r_i)))\right)
\\ & \leq \frac{\E\left[\Px(\DIS(\Ball(f,r_i))) - \Px\left(\DIS\left(\truV_{M(\lceil n/2 \rceil)}\right)\right)\right]}{\frac{11}{12} \Px(\DIS(\Ball(f,r_i)))}
= \frac{12}{11} \left( 1 - \frac{ \E\left[ \Px\left(\DIS\left(\truV_{M(\lceil n/2 \rceil)}\right)\right)\right] }{\Px(\DIS(\Ball(f,r_i)))}\right).
\end{align*}
Lemma~\ref{lem:cal-dis-lower} implies this is at most
$\frac{12}{11} \frac{ \lceil n/2 \rceil r_i }{\Px(\DIS(\Ball(f,r_i)))} \leq \frac{12}{11} \left\lceil \frac{\Px(\DIS(\Ball(f,r_i)))}{4 r_i} \right\rceil \frac{r_i}{\Px(\DIS(\Ball(f,r_i)))}$.
Since any $a \geq 3/2$ has $\lceil a \rceil \leq (3/2) a$,
and $\dc_{f}\left(\lambda(2\eps_{i})^{-1}\right) \geq 12$ implies
$\frac{\Px(\DIS(\Ball(f,r_i)))}{4 r_i} \geq 3/2$,
we have $\left\lceil \frac{\Px(\DIS(\Ball(f,r_i)))}{4 r_i} \right\rceil \leq \frac{3}{8}\frac{\Px(\DIS(\Ball(f,r_i)))}{r_i}$,
so that, $\frac{12}{11} \left\lceil \frac{\Px(\DIS(\Ball(f,r_i)))}{4 r_i} \right\rceil \frac{r_i}{\Px(\DIS(\Ball(f,r_i)))} \leq \frac{9}{22}$.
Combining the above, we have
\begin{equation}
\label{eqn:cal-lower-4}
\P\left( \Px\left(\DIS\left(\truV_{M(\lceil n/2 \rceil)}\right)\right) < \Px(\DIS(\Ball(f,r_i))) / 12\right)
\leq \frac{9}{22}.
\end{equation}
Examining the second term in \eqref{eqn:cal-lower-3},
Hoeffding's inequality and the definition of $\hat{\Delta}$ from \eqref{eqn:hatPn3} imply
\begin{equation}
\P\left(\hat{\Delta} < \Px\left(\DIS\left(\truV_{\check{m}}\right)\right)\right)
 = \E\left[ \P\left( \hat{\Delta} < \Px\left(\DIS\left(\truV_{\check{m}}\right)\right) \Big| \truV_{\check{m}}, \check{m} \right)\right]
\leq \E\left[ e^{- 8 \check{m} } \right]
\leq e^{-8}
< 1 / 11. \label{eqn:cal-lower-5}
\end{equation}
Combining \eqref{eqn:cal-lower-1} through \eqref{eqn:cal-lower-5} implies
\begin{equation*}
\E\left[ \er\left( \alg_{p}\left(\hat{\L}\right) \right) \right]
> 2 \eps_{i} \left( 1 - \frac{9}{22} - \frac{1}{11}\right)
= \eps_{i}.
\end{equation*}
Thus, for any label complexity $\Lambda_{a}$ achieved by running \CAL~with $\alg_{p}$ as its argument,
we must have $\Lambda_{a}(\eps_{i},f,\Px) > \dc_{f}\left( \lambda(2\eps_{i})^{-1} \right) / 4$.
Since this is true for all $i \in \nats$, and $\eps_i \to 0$ as $i \to \infty$, this establishes the result.
\end{proof}

\section{The Label Complexity of \Shattering}
\label{app:exponential}

As in Appendix~\ref{app:activizer}, we will assume $\C$ is a fixed VC class, $\Px$ is some arbitrary distribution,
and $f \in \cl(\C)$ is an arbitrary fixed function.
We continue using the notation introduced above: in particular,
$\S^{k}(\H) = \left\{ S \in \X^{k} : \H \text{ shatters } S \right\}$,
$\bar{\S}^{k}(\H) = \X^{k} \setminus \S^{k}(\H)$,
$\bar{\partial}^{k}_{\H} f = \X^{k} \setminus \partial^{k}_{\H} f$,
and $\dprob = \Px^{\bdim_f-1}\left( \partial_{\C}^{\bdim_f -1} f \right)$.
Also, as above, we will prove a more general result replacing the ``$1/2$'' in Steps 5, 9, and 12 of \Shattering~with
an arbitrary value $\gamma \in (0,1)$; thus, the specific result for the stated algorithm will be obtained by taking $\gamma = 1/2$.

For the estimators $\hat{P}_{m}$ in \Shattering, we take precisely the same
definitions as given in Appendix~\ref{app:hatP-definitions} for the estimators in \BasicActivizer.
In particular, the quantities $\hat{\Delta}_{m}^{(k)}(x,W_2,\H)$, $\hat{\Delta}_{m}^{(k)}(W_1,W_2,\H)$, $\hat{\Gamma}_{m}^{(k)}(x,y,W_2,\H)$,
and $M_{m}^{(k)}(\H)$ are all defined as in Appendix~\ref{app:hatP-definitions}, and the $\hat{P}_{m}$ estimators are again defined as
in \eqref{eqn:hatPn1}, \eqref{eqn:hatPn2} and \eqref{eqn:hatPn3}.

Also, we sometimes refer to quantities defined above, such as
$\bar{p}_{\zeta}(k,\ell,m)$
(defined in \eqref{eqn:bar-delta-defn}), as well as the various events from the lemmas of the previous appendix,
such as $H_{\init}(\delta)$, $H^{\prime}$, $H_{\init}^{(i)}$, $H_{\init}^{(ii)}$, $H_{\init}^{(iii)}(\zeta)$, $H_{\init}^{(iv)}$, and $G_{\init}^{(i)}$.

\subsection{Proof of Theorem~\ref{thm:sequential-activizer}}
\label{app:sequential-activizer}

Throughout the proof, we will make reference to the sets $V_m$ defined in \Shattering.
Also let $V^{(k)}$ denote the final value of $V$ obtained for the specified value of $k$ in \Shattering.
Both $V_m$ and $V^{(k)}$ are implicitly functions of the budget, $n$, given to \Shattering.
As above, we continue to denote by $\truV_m = \{h \in \C : \forall i \leq m, h(X_m) = f(X_m)\}$.
One important fact we will use repeatedly below is that if $V_{m} = \truV_{m}$ for some $m$,
then since Lemma~\ref{lem:Vshat-to-Boundaries} implies that $\truV_{m} \neq \emptyset$ on $H^{\prime}$,
we must have that all of the previous $\hat{y}$ values were consistent with $f$,
which means that $\forall \ell \leq m$, $V_{\ell} = \truV_{\ell}$.  In particular, if $V^{(k^{\prime})} = \truV_m$
for the largest $m$ value obtained while $k = k^{\prime}$ in \Shattering, then
$V_{\ell} = \truV_{\ell}$ for all $\ell$ obtained while $k \leq k^{\prime}$ in \Shattering.

Additionally, define $\tilde{m}_{n} = \lfloor n / 24 \rfloor$,
and note that the value $m = \lceil n/6 \rceil$ is obtained while $k=1$ in \Shattering.
We also define the following quantities, which we will show are typically equal to related quantities
in \Shattering.
Define $\hat{m}_{0} = 0$, $\truT_{0} = \lceil 2 n / 3 \rceil$, and $\hat{t}_{0} = 0$,
and for each $k \in \left\{1,\ldots,\vc+1\right\}$,
inductively define

\begin{align*}
\truT_{k} &= \truT_{k-1} - \hat{t}_{k-1},
\\ \truI_{mk} &= \ind_{[\gamma,\infty)}\left(\hat{\Delta}_{m}^{(k)}\left(X_{m}, W_2, \truV_{m-1}\right)\right), \forall m \in \nats,
\\ \check{m}_{k} &= \min\left\{ m \geq \hat{m}_{k-1} : \sum_{\ell = \hat{m}_{k-1}+1}^{m} \truI_{\ell k} = \left\lceil \truT_k / 4 \right\rceil \right\}\cup\left\{\max\left\{k\cdot 2^{n}+1, \hat{m}_{k-1}\right\}\right\},
\\ \hat{m}_{k} &= \check{m}_{k} + \left\lfloor \truT_k / \left(3 \hat{\Delta}_{\check{m}_{k}}^{(k)}\left(W_1,W_2,\truV_{\check{m}_{k}}\right)\right)\right\rfloor,
\\ \check{\U}_{k} &= (\hat{m}_{k-1}, \check{m}_{k}] \cap \nats,
\\ \hat{\U}_{k} &= (\check{m}_{k}, \hat{m}_{k}] \cap \nats,  
\\ \truC_{mk} &= \ind_{\left[0, \left\lfloor 3 \truT_{k} / 4 \right\rfloor\right)} \left( \sum_{\ell = \hat{m}_{k-1}+1}^{m-1} \truI_{\ell k}\right)
\\ \truQ_{k} &= \sum_{m \in \hat{\U}_{k}} \truI_{mk} \cdot \truC_{mk},
\\ \text{and } \hat{t}_{k} &= \truQ_{k} + \sum_{m \in \check{\U}_{k}} \truI_{mk}.
\end{align*}

\noindent The meaning of these values can be understood in the context of \Shattering,
under the condition that $V_m = \truV_m$ for values of $m$ obtained for the respective
value of $k$.
Specifically, under this condition,
$\truT_{k}$ corresponds to $T_k$,
$\hat{t}_{k}$ represents the final value $t$ for round $k$,
$\check{m}_{k}$ represents the value of $m$ upon reaching Step 9 in round $k$,
while $\hat{m}_{k}$ represents the value of $m$ at the end of round $k$,
$\check{\U}_{k}$ corresponds to the set of indices arrived at in Step 4 during round $k$,
while $\hat{\U}_{k}$ corresponds to the set of indices arrived at in Step 11 during round $k$,
for $m \in \check{\U}_{k}$, $\truI_{mk}$ indicates whether the label of $X_{m}$ is requested,
while for $m \in \hat{\U}_{k}$, $\truI_{mk} \cdot \truC_{mk}$ indicates whether the
label of $X_m$ is requested.
Finally $\truQ_{k}$ corresponds to the number of label requests in Step 13 during round $k$.
In particular, note $\check{m}_{1} \geq \tilde{m}_n$.

\begin{lemma}
\label{lem:monotonic-hat-delta}
For any $\init \in \nats$,
on the event $H^{\prime} \cap G_{\init}^{(i)}$,
$\forall k,\ell, m \in \nats$ with
$k \leq \bdim_f$,
$\forall x \in \X$, for any sets $\H$ and $\H^{\prime}$ with
$\truV_{\ell} \subseteq \H \subseteq \H^{\prime} \subseteq \Ball(f,r_{1/6})$,
if either $k = 1$ or $m \geq \init$, then
\begin{equation*}
\hat{\Delta}_{m}^{(k)}\left(x, W_2, \H\right) \leq (3/2) \hat{\Delta}_{m}^{(k)}\left(x, W_2, \H^{\prime}\right).
 \end{equation*}
In particular, for any $\delta \in (0,1)$ and $\init \geq \init(1/6 ; \delta)$, on $H^{\prime} \cap H_{\init}(\delta) \cap G_{\init}^{(i)}$,
$\forall k,\ell, \ell^{\prime}, m \in \nats$ with $m \geq \init$,
$\ell \geq \ell^{\prime} \geq \init$, and $k \leq \bdim_f$, $\forall x \in \X$,
$\hat{\Delta}_{m}^{(k)}\left(x, W_2, \truV_{\ell}\right) \leq (3/2) \hat{\Delta}_{m}^{(k)}\left(x, W_2, \truV_{\ell^{\prime}}\right)$.
\thmend
\end{lemma}
\begin{proof}
First note that $\forall m \in \nats$, $\forall x \in \X$,
\begin{equation*}
\hat{\Delta}_{m}^{(1)}\left(x, W_2,\H\right)
= \ind_{\DIS\left(\H\right)}(x)
\leq \ind_{\DIS\left(\H^{\prime}\right)}(x)
= \hat{\Delta}_{m}^{(1)}\left(x, W_2, \H^{\prime}\right),
\end{equation*}
so the result holds for $k=1$.
Lemma~\ref{lem:Vshat-to-Boundaries},
Lemma~\ref{lem:Mball-core}, and monotonicity of $M_{m}^{(k)}(\cdot)$
imply that
on $H^{\prime} \cap G_{\init}^{(i)}$,
for any $m \geq \init$ and $k \in \left\{2,\ldots,\bdim_f\right\}$,
\begin{equation*}
M^{(k)}_{m}\left(\H\right)
\geq \sum_{i=1}^{\Msize{m}} \ind_{\partial_{\C}^{k-1} f}\left(S_i^{(k)}\right)
\geq (2/3) M^{(k)}_{m}\left(\Ball(f,r_{1/6})\right)
\geq (2/3) M^{(k)}_{m}\left(\H^{\prime}\right),
\end{equation*}
so that $\forall x \in \X$,
\begin{align*}
\hat{\Delta}_{m}^{(k)}\left(x, W_2, \H\right)
& = M^{(k)}_{m}\left(\H\right)^{-1} \sum_{i=1}^{\Msize{m}} \ind_{\S^{k}\left(\H\right)}\left(S_{i}^{(k)} \cup \{x\}\right)
\\ & \leq M^{(k)}_{m}\left(\H\right)^{-1} \sum_{i=1}^{\Msize{m}} \ind_{\S^{k}\left(\H^{\prime}\right)}\left(S_{i}^{(k)} \cup \{x\}\right)
\\ & \leq (3/2) M^{(k)}_{m}\left(\H^{\prime}\right)^{-1} \sum_{i=1}^{\Msize{m}} \ind_{\S^{k}\left(\H^{\prime}\right)}\left(S_{i}^{(k)} \cup \{x\}\right)
 = (3/2) \hat{\Delta}_{m}^{(k)}\left(x, W_2, \H^{\prime}\right).
\end{align*}
The final claim follows from Lemma~\ref{lem:VinB}.
\end{proof}

\begin{lemma}
\label{lem:sequential-T-lower}
For any $k \in \left\{1,\ldots,\vc+1\right\}$,
if $n \geq 3 \cdot 4^{k-1}$, then
$\truT_{k} \geq 4^{1-k} (2n/3)$ and $\hat{t}_{k} \leq \left\lfloor 3 \truT_{k} / 4 \right\rfloor$.
\thmend
\end{lemma}
\begin{proof}
Recall $\truT_{1} = \lceil 2n/3 \rceil \geq 2n/3$.
If $n \geq 2$, we also have $\lfloor 3 \truT_{1} / 4\rfloor \geq \lceil \truT_{1} / 4 \rceil$,
so that (due to the $\truC_{m1}$ factors) $\hat{t}_{1} \leq \lfloor 3 \truT_{1} / 4 \rfloor$.
For the purpose of induction, suppose some $k \in \left\{2,\ldots,\vc + 1\right\}$ has
$n \geq 3 \cdot 4^{k-1}$, $\truT_{k-1} \geq 4^{2-k} (2n/3)$,
and $\hat{t}_{k-1} \leq \lfloor 3 \truT_{k-1} / 4 \rfloor$.
Then $\truT_{k} = \truT_{k-1} - \hat{t}_{k-1} \geq \truT_{k-1} / 4 \geq 4^{1-k} (2n/3)$,
and since $n \geq 3 \cdot 4^{k-1}$,
we also have $\lfloor 3 \truT_{k} / 4 \rfloor \geq \lceil \truT_{k} / 4 \rceil$,
so that $\hat{t}_{k} \leq \lfloor 3 \truT_{k} / 4 \rfloor$ (again, due to the $\truC_{mk}$ factors).
Thus, by the principle of induction, this holds for all $k \in \left\{1,\ldots,\vc+1\right\}$
with $n \geq 3 \cdot 4^{k-1}$.
\end{proof}

The next lemma indicates that the ``$t < \lfloor 3 T_k / 4 \rfloor$''
constraint in Step 12 is redundant for $k \leq \bdim_f$.
It is similar to \eqref{eqn:kstar-query-bound} in Lemma~\ref{lem:empirical-works-too},
but is made only slightly more complicated by the fact that the
$\hat{\Delta}^{(k)}$ estimate is calculated in Step 9 based on a set $V_{m}$
different from the ones used to decide whether or not to request a label
in Step 12.

\begin{lemma}
\label{lem:sequential-t-vs-T}
There exist $(\C,\Px,f,\gamma)$-dependent constants
$\tilde{c}_{1}^{(i)}, \tilde{c}_{2}^{(i)} \in [1,\infty)$
such that, for any $\delta \in (0,1)$,
and any integer $n \geq \tilde{c}_{1}^{(i)} \ln\left( \tilde{c}_{2}^{(i)} / \delta\right)$,
on an event
\begin{equation*}
\tilde{H}_n^{(i)}(\delta) \subseteq G_{\tilde{m}_n}^{(i)} \cap H_{\tilde{m}_n}(\delta) \cap H_{\tilde{m}_n}^{(i)}
\cap H_{\tilde{m}_n}^{(iii)}(\gamma/16) \cap H_{\tilde{m}_n}^{(iv)}
\end{equation*}
with
$\P\left(\tilde{H}_n^{(i)}(\delta)\right) \geq 1 - 2\delta$,
$\forall k \in \left\{1,\ldots,\bdim_f\right\}$,
$\hat{t}_{k} = \sum\limits_{m = \hat{m}_{k-1}+1}^{\hat{m}_{k}} \truI_{mk} \leq 3 \truT_{k} / 4$.
\thmend
\end{lemma}
\begin{proof}
% Chernoff bounds for number of queries
Define the constants
\begin{eqnarray*}
& \tilde{c}_{1}^{(i)} = \max\left\{\frac{192 \vc}{r_{(3/32)}}, \frac{3 \cdot 4^{\bdim_{f} +6}}{\dprob \gamma^2} \right\}, &
\tilde{c}_{2}^{(i)} = \max\left\{\frac{8 e}{r_{(3/32)}}, \left( c^{(i)}+c^{(iii)}(\gamma/16)+125\bdim_{f} \dprob^{-1}\right)\right\},
\end{eqnarray*}
and let $n^{(i)}(\delta) = \tilde{c}_{1}^{(i)} \ln\left( \tilde{c}_{2}^{(i)} / \delta\right)$.
Fix any integer $n \geq n^{(i)}(\delta)$
and consider the event
\begin{equation*}
\tilde{H}_n^{(1)}(\delta) = G_{\tilde{m}_n}^{(i)} \cap H_{\tilde{m}_n}(\delta) \cap H_{\tilde{m}_n}^{(i)} \cap H_{\tilde{m}_n}^{(iii)}(\gamma/16) \cap H_{\tilde{m}_n}^{(iv)}.
\end{equation*}
By Lemma~\ref{lem:monotonic-hat-delta} and the fact that $\check{m}_k \geq \tilde{m}_n$ for all $k \geq 1$,
since $n \geq n^{(i)}(\delta) \geq 24 \init\left(1/6 ; \delta\right)$, on $\tilde{H}_n^{(1)}(\delta)$, $\forall k \in \left\{1,\ldots,\bdim_f\right\}$, $\forall m \in \hat{\U}_{k}$,
\begin{equation}
\label{eqn:monotonic-hat-delta-application}
\hat{\Delta}_{m}^{(k)}\left(X_{m}, W_2, \truV_{m-1}\right) \leq (3/2) \hat{\Delta}_{m}^{(k)}\left(X_{m}, W_2, \truV_{\check{m}_{k}}\right).
\end{equation}
Now fix any $k \in \left\{1,\ldots,\bdim_f\right\}$.
Since $n \geq n^{(i)}(\delta) \geq 27 \cdot 4^{k-1}$, Lemma~\ref{lem:sequential-T-lower} implies
$\truT_{k} \geq 18$, which means that
$3 \truT_{k} / 4 - \lceil \truT_{k} / 4 \rceil \geq 4 \truT_{k} / 9$.
Also note that $\sum_{m \in \check{\U}_{k}} \truI_{mk} \leq \left\lceil \truT_{k} / 4 \right\rceil$.
Let $N_{k} = (4/3)\hat{\Delta}_{\check{m}_{k}}^{(k)}\left(W_1,W_2,\truV_{\check{m}_{k}}\right) \left| \hat{\U}_{k} \right|$;
note that $\left|\hat{\U}_{k} \right| = \left\lfloor \truT_{k} / \left( 3 \hat{\Delta}_{\check{m}_{k}}^{(k)}\left(W_1,W_2,\truV_{\check{m}_{k}}\right)\right)\right\rfloor$,
so that $N_{k} \leq (4/9) \truT_{k}$.  Thus, we have

\begin{align}
\P&\left( \tilde{H}_{n}^{(1)}(\delta) \cap \left\{\sum_{m=\hat{m}_{k-1}+1}^{\hat{m}_{k}} \truI_{mk} > 3 \truT_{k} / 4 \right\} \right) \notag
\\ & \leq \P\left( \tilde{H}_{n}^{(1)}(\delta) \cap \left\{\sum_{m \in \hat{\U}_{k}} \truI_{mk} > 4 \truT_{k} / 9\right\}\right)
\leq \P\left( \tilde{H}_{n}^{(1)}(\delta)\cap\left\{\sum_{m \in \hat{\U}_{k}} \truI_{mk} > N_{k}\right\} \right) \notag
\\ & \leq \P\left( \tilde{H}_{n}^{(1)}(\delta) \cap \left\{ \sum_{m \in \hat{\U}_{k}} \ind_{[2\gamma/3,\infty)}\left( \hat{\Delta}_{m}^{(k)}\left(X_{m}, W_2, \truV_{\check{m}_{k}}\right)\right) > N_{k}\right\} \right), \label{eqn:checkm-pre-chernoff}
\end{align}
where this last inequality is by \eqref{eqn:monotonic-hat-delta-application}.
To simplify notation, define $\tilde{Z}_{k} = \left( \truT_{k}, \check{m}_{k}, W_1, W_2, \truV_{\check{m}_{k}}\right)$.
By Lemmas \ref{lem:basic-bar-delta-bound} and \ref{lem:basic-empirical-bound}
(with $\beta = 3/32$, $\zeta = 2\gamma/3$, $\alpha = 3/4$, and $\xi = \gamma/16$),
since $n \geq n^{(i)}(\delta) \geq 24 \cdot \max\left\{ \init^{(iv)}(\gamma/16;\delta), \init(3/32; \delta)\right\}$,
on $\tilde{H}_n^{(1)}(\delta)$,
$\forall m \in \hat{\U}_{k}$,
\begin{align*}
\bar{p}_{2\gamma/3}(k,\check{m}_{k},m) & \leq \Px\left( x : p_{x}\left(k,\check{m}_{k}\right) \geq \gamma / 2\right) + \exp\left\{-\gamma^2 \tilde{M}(m) / 256\right\}
\\ & \leq \Px\left( x : p_{x}\left(k,\check{m}_{k}\right) \geq \gamma / 2\right) + \exp\left\{-\gamma^2 \tilde{M}(\check{m}_{k}) / 256\right\}
\\ & \leq \hat{\Delta}_{\check{m}_{k}}^{(k)}\left(W_1, W_2, \truV_{\check{m}_{k}}\right).
\end{align*}
Letting $\tilde{G}_{n}^{\prime}(k)$ denote the event that $\bar{p}_{2\gamma/3}(k,\check{m}_{k},m) \leq \hat{\Delta}_{\check{m}_{k}}^{(k)}\left(W_1, W_2, \truV_{\check{m}_{k}}\right)$,
we see that $\tilde{G}_{n}^{\prime}(k) \supseteq \tilde{H}_{n}^{(1)}(\delta)$.
Thus, since the $\ind_{[2\gamma/3,\infty)}\left( \hat{\Delta}_{m}^{(k)}\left(X_m,W_2,\truV_{\check{m}_{k}}\right)\right)$
variables are conditionally independent given
$\tilde{Z}_{k}$
for $m \in \hat{\U}_{k}$,
each with respective conditional distribution ${\rm Bernoulli}\left( \bar{p}_{2\gamma/3}\left(k, \check{m}_{k}, m\right)\right)$,
%given $\tilde{Z}_{k}$,
the law of total probability and a Chernoff bound
imply that \eqref{eqn:checkm-pre-chernoff} is at most
\begin{align*}
&\P\left( \tilde{G}_{n}^{\prime}(k) \cap \left\{ \sum_{m \in \hat{\U}_{k}} \ind_{[2\gamma/3,\infty)}\left( \hat{\Delta}_{m}^{(k)}\left(X_{m}, W_2, \truV_{\check{m}_{k}}\right)\right) > N_{k}\right\} \right)
\\ & = \E\left[\P\left( \sum_{m \in \hat{\U}_{k}} \ind_{[2\gamma/3,\infty)}\left( \hat{\Delta}_{m}^{(k)}\left(X_{m}, W_2, \truV_{\check{m}_{k}}\right)\right) > N_{k} \Bigg| \tilde{Z}_{k} \right) \cdot \ind_{\tilde{G}_{n}^{\prime}(k)} \right]
\\ & \leq \E\left[\exp\!\left\{ - \hat{\Delta}_{\check{m}_{k}}^{(k)}\left(W_1,W_2,\truV_{\check{m}_{k}}\right)\left| \hat{\U}_{k} \right| / 27\right\}\right] \!\leq \E\left[\exp\!\left\{ - \truT_{k} / 162 \right\}\right]  \!\leq \exp\!\left\{ - n / \left( 243 \cdot 4^{k-1}\right)\right\}\!,
\end{align*}
where the last inequality is by Lemma~\ref{lem:sequential-T-lower}.
Thus, there exists $\tilde{G}_{n}(k)$ with
$\P\left(\tilde{H}_{n}^{(1)}(\delta) \setminus \tilde{G}_{n}(k)\right) \leq \exp\left\{ - n / \left(243 \cdot 4^{k-1}\right)\right\}$ such that,
on $\tilde{H}_{n}^{(1)}(\delta) \cap \tilde{G}_{n}(k)$, we have $\sum_{m = \hat{m}_{k-1}+1}^{\hat{m}_{k}} \truI_{mk} \leq 3 \truT_k / 4$.
Defining $\tilde{H}_{n}^{(i)}(\delta) = \tilde{H}_{n}^{(1)}(\delta) \cap \bigcap_{k=1}^{\bdim_f} \tilde{G}_{n}(k)$,
a union bound implies
\begin{equation}
\label{eqn:tildeHi-difference}
\P\left(\tilde{H}_{n}^{(1)}(\delta) \setminus \tilde{H}_{n}^{(i)}(\delta)\right) \leq \bdim_f \cdot \exp\left\{- n / \left( 243 \cdot 4^{\bdim_f -1}\right) \right\},
\end{equation}
and on $\tilde{H}_{n}^{(i)}(\delta)$, every $k \in \left\{1,\ldots,\bdim_f\right\}$ has $\sum_{m=\hat{m}_{k-1}+1}^{\hat{m}_{k}} \truI_{mk} \leq 3 \truT_k / 4$.
In particular, this means the $\truC_{mk}$ factors are redundant in $\truQ_{k}$, so that $\hat{t}_{k} = \sum_{m = \hat{m}_{k-1}+1}^{\hat{m}_{k}} \truI_{mk}$.

To get the stated probability bound, a union bound implies that
\begin{align}
1 - \P\left(\tilde{H}_n^{(1)}(\delta)\right)
&\leq \left(1 - \P\left(H_{\tilde{m}_n}(\delta)\right)\right)
+ \left(1 - \P\left(H_{\tilde{m}_n}^{(i)}\right)\right)
+ \P\left(H_{\tilde{m}_n}^{(i)} \setminus H_{\tilde{m}_n}^{(iii)}(\gamma/16)\right) \notag
\\ & \phantom{\leq }+ \left(1- \P\left(H_{\tilde{m}_n}^{(iv)}\right)\right)
+ \P\left(H_{\tilde{m}_{n}}^{(i)} \setminus G_{\tilde{m}_n}^{(i)}\right) \notag
\\ & \leq
\delta
+ c^{(i)} \cdot \exp\left\{-\tilde{M}\left(\tilde{m}_n\right) / 4\right\} \notag
\\ & \phantom{\leq }+ c^{(iii)}(\gamma/16) \cdot \exp\left\{-\tilde{M}\left(\tilde{m}_n\right) \gamma^2 / 256\right\}
+ 3 \bdim_f \cdot \exp\left\{ - 2 \tilde{m}_n \right\} \notag
\\ & \phantom{\leq} + 121 \bdim_f \dprob^{-1} \cdot \exp\left\{- \tilde{M}\left(\tilde{m}_n\right) / 60 \right\} \notag
\\ & \leq \delta + \left( c^{(i)} + c^{(iii)}(\gamma/16) + 124 \bdim_f \dprob^{-1}\right) \cdot\exp\left\{- \tilde{m}_n \dprob \gamma^2 / 512 \right\}. \label{eqn:tildeH1-prob}
\end{align}
Since $n \geq n^{(i)}(\delta) \geq 24$, we have $\tilde{m}_n \geq n / 48$,
so that summing \eqref{eqn:tildeHi-difference} and \eqref{eqn:tildeH1-prob} gives us
\begin{equation}
\label{eqn:tildeHi-n-bound}
1-\P\left(\tilde{H}_{n}^{(i)}(\delta)\right) \leq
\delta + \left( c^{(i)} + c^{(iii)}(\gamma/16) + 125 \bdim_f \dprob^{-1}\right) \cdot\exp\left\{- n \dprob \gamma^2 / \left(512 \cdot 48 \cdot 4^{\bdim_f - 1}\right)  \right\}.
\end{equation}
Finally, note that we have chosen $n^{(i)}(\delta)$ sufficiently large so that \eqref{eqn:tildeHi-n-bound} is at most $2\delta$.
\end{proof}

The next lemma indicates that the redundancy of the ``$t < \lfloor 3 T_k / 4 \rfloor$'' constraint,
just established in Lemma~\ref{lem:sequential-t-vs-T},
implies that all $\hat{y}$ labels obtained while $k \leq \bdim_f$ are consistent with the
target function.

\begin{lemma}
\label{lem:sequential-good-labels}
Consider running \Shattering~with a budget $n \in \nats$,
while $f$ is the target function and $\Px$ is the data distribution.
There is an event $\tilde{H}_{n}^{(ii)}$
and $(\C,\Px,f,\gamma)$-dependent constants $\tilde{c}_{1}^{(ii)}, \tilde{c}_{2}^{(ii)} \in [1,\infty)$
such that, for any $\delta \in (0,1)$, if $n \geq \tilde{c}_{1}^{(ii)} \ln \left( \tilde{c}_{2}^{(ii)} / \delta \right)$,
then $\P\left(\tilde{H}_{n}^{(i)}(\delta) \setminus \tilde{H}_{n}^{(ii)}\right) \leq \delta$,
and on $\tilde{H}_n^{(i)}(\delta) \cap \tilde{H}_n^{(ii)}$,
we have $V^{(\bdim_f)} = V_{\hat{m}_{\bdim_f}} = \truV_{\hat{m}_{\bdim_f}}$.
\thmend
\end{lemma}
\begin{proof}
Define
$\tilde{c}_{1}^{(ii)}  = \max\left\{ \tilde{c}_{1}^{(i)}, \frac{192 \vc}{r_{(1-\gamma)/6}}, \frac{2^{11}}{\dprob^{1/3}}\right\}$,
$\tilde{c}_{2}^{(ii)}  = \max\left\{\tilde{c}_{2}^{(i)}, \frac{8 e}{r_{(1-\gamma)/6}}, c^{(ii)}, \exp\left\{\init^{*}\right\}\right\}$,
let $n^{(ii)}(\delta) = \tilde{c}_{1}^{(ii)} \ln \left( \tilde{c}_{2}^{(ii)} / \delta \right)$, suppose $n \geq n^{(ii)}(\delta)$,
and define the event
$\tilde{H}_n^{(ii)} = H_{\tilde{m}_n}^{(ii)}$.

By Lemma~\ref{lem:kstar-good-labels}, since $n \geq n^{(ii)}(\delta) \geq 24 \cdot \max\left\{\init((1-\gamma)/6;\delta), \init^{*}\right\}$,
on $\tilde{H}_n^{(i)}(\delta)\cap \tilde{H}_n^{(ii)}$,
$\forall m \in \nats$ and $k \in \left\{1,\ldots,\bdim_f\right\}$ with either $k = 1$ or $m > \tilde{m}_n$,
\begin{equation}
\label{eqn:sequential-good-labels}
\hat{\Delta}_{m}^{(k)}\!\left(X_{m}, W_2, \truV_{m-1}\right) < \gamma \Rightarrow \hat{\Gamma}_{m}^{(k)}\!\left(X_{m}, -f(X_{m}), W_2, \truV_{m-1}\right) < \hat{\Gamma}_{m}^{(k)}\!\left(X_{m}, f(X_{m}), W_2, \truV_{m-1}\right).
\end{equation}
Recall that $\tilde{m}_n \leq \min\left\{ \left\lceil T_1 / 4 \right\rceil, 2^n \right\} = \left\lceil \left\lceil 2n/3 \right\rceil / 4 \right\rceil$.
Therefore, $V_{\tilde{m}_n}$ is obtained purely by $\tilde{m}_n$ executions of Step 8 while $k=1$.
Thus, for every $m$ obtained in \Shattering,
either $k = 1$ or $m > \tilde{m}_n$.
We now proceed by induction on $m$.
We already know $V_{0} = \C = \truV_{0}$, so this serves as our base case.
Now consider some value $m \in \nats$ obtained in \Shattering~while $k \leq \bdim_f$,
and suppose every $m^{\prime} < m$ has $V_{m^{\prime}} = \truV_{m^{\prime}}$.
But this means that $T_k = \truT_k$ and the value of $t$ upon obtaining
this particular $m$ has $t \leq \sum_{\ell = \hat{m}_{k-1}+1}^{m-1} \truI_{\ell k}$.
In particular, if $\hat{\Delta}_{m}^{(k)}\left(X_{m}, W_2, V_{m-1}\right) \geq \gamma$, then $\truI_{m k} = 1$,
so that $t < \sum_{\ell = \hat{m}_{k-1}+1}^{m} \truI_{mk}$;
by Lemma~\ref{lem:sequential-t-vs-T}, on $\tilde{H}_n^{(i)}(\delta) \cap \tilde{H}_n^{(ii)}$,
$\sum_{\ell = \hat{m}_{k-1}+1}^{m} \truI_{mk} \leq \sum_{\ell = \hat{m}_{k-1}+1}^{\hat{m}_{k}} \truI_{mk} \leq 3 \truT_{k} / 4$,
so that $t < 3 \truT_{k} / 4$, and therefore $\hat{y} = Y_{m} = f(X_{m})$; this implies $V_{m} = \truV_{m}$.
On the other hand, on $\tilde{H}_n^{(i)}(\delta) \cap \tilde{H}_n^{(ii)}$,
if $\hat{\Delta}_{m}^{(k)}\left(X_{m}, W_2, V_{m-1}\right) < \gamma$,
then \eqref{eqn:sequential-good-labels} implies
\begin{equation*}
\hat{y} = \argmax\limits_{y \in \{-1,+1\}} \hat{\Gamma}_{m}^{(k)}\left(X_{m}, y, W_2, V_{m-1}\right) = f(X_{m}),
\end{equation*}
so that again $V_{m} = \truV_{m}$.
Thus, by the principle of induction, on $\tilde{H}_n^{(i)}(\delta) \cap \tilde{H}_n^{(ii)}$,
for every $m \in \nats$ obtained while $k \leq \bdim_f$,
we have $V_{m} = \truV_{m}$; in particular, this implies $V^{(\bdim_f)} = V_{\hat{m}_{\bdim_f}} = \truV_{\hat{m}_{\bdim_f}}$.
The bound on $\P\left(\tilde{H}_{n}^{(i)}(\delta) \setminus \tilde{H}_{n}^{(ii)}\right)$ then follows from Lemma~\ref{lem:kstar-good-labels},
as we have chosen $n^{(ii)}(\delta)$ sufficiently large so that \eqref{eqn:Hii-prob-bound} (with $\init = \tilde{m}_n$) is at most $\delta$.
\end{proof}

\begin{lemma}
\label{lem:sequential-exponential}
Consider running \Shattering~with a budget $n \in \nats$, while $f$ is the target function and $\Px$ is the data distribution.
There exist $(\C,\Px,f,\gamma)$-dependent constants $\tilde{c}_1^{(iii)}, \tilde{c}_2^{(iii)} \in [1,\infty)$ such that,
for any $\delta \in (0,e^{-3})$, $\lambda \in [1,\infty)$, and $n \in \nats$,
there is an event $\tilde{H}_{n}^{(iii)}(\delta,\lambda)$
with $\P\left(\tilde{H}_{n}^{(i)}(\delta) \cap \tilde{H}_{n}^{(ii)} \setminus \tilde{H}_{n}^{(iii)}(\delta,\lambda)\right) \leq \delta$
with the property that, if
\begin{equation*}
n \geq \tilde{c}_1^{(iii)} \hdc_{f}(d/\lambda) \ln^{2} \left(\frac{\tilde{c}_{2}^{(iii)} \lambda}{\delta}\right),
\end{equation*}
then on $\tilde{H}_{n}^{(i)}(\delta) \cap \tilde{H}_{n}^{(ii)} \cap \tilde{H}_{n}^{(iii)}(\delta,\lambda)$,
at the conclusion of \Shattering,
$\left| \L_{\bdim_f} \right| \geq \lambda$.
\thmend
\end{lemma}
\begin{proof}
Let
$\tilde{c}_1^{(iii)} = \max\left\{\tilde{c}_{1}^{(i)}, \tilde{c}_{1}^{(ii)}, \frac{\vc \cdot \bdim_{f} \cdot 4^{10+2\bdim_f}}{\gamma^3 \dprob^3}, \frac{192 \vc}{r_{(3/32)}} \right\}$,
$\tilde{c}_2^{(iii)} = \max\left\{\tilde{c}_{2}^{(i)}, \tilde{c}_{2}^{(ii)}, \frac{8 e}{r_{(3/32)}}\right\}$,
fix any $\delta \in (0,e^{-3})$, $\lambda \in [1,\infty)$,
let $n^{(iii)}(\delta,\lambda) = \tilde{c}_1^{(iii)} \hdc_{f}(d / \lambda) \ln^{2}( \tilde{c}_{2}^{(iii)} \lambda / \delta)$,
and suppose $n \geq n^{(iii)}(\delta,\lambda)$.

Define a sequence $\ell_{i} = 2^{i}$ for integers $i \geq 0$, and let
$\hat{\iota} = \left\lceil \log_{2}\left( 4^{2 + \bdim_f} \lambda / \gamma \dprob \right) \right\rceil$.
Also define $\tilde{\phi}(m,\delta, \lambda) = \max\left\{\phi\left(m ; \delta / 2 \hat{\iota} \right), d/\lambda\right\}$,
where $\phi$ is defined in Lemma~\ref{lem:VinB}.
Then define the events
\begin{eqnarray*}
\tilde{H}^{(3)}(\delta,\lambda) = \bigcap_{i=1}^{\hat{\iota}} H_{\ell_{i}} \left( \delta / 2 \hat{\iota} \right),
& \tilde{H}_{n}^{(iii)}(\delta,\lambda) = \tilde{H}^{(3)}(\delta,\lambda) \cap \left\{ \check{m}_{\bdim_f} \geq \ell_{\hat{\iota}}\right\}.
\end{eqnarray*}
Note that $\hat{\iota} \leq n$, so that $\ell_{\hat{\iota}} \leq 2^{n}$,
and therefore the truncation in the definition of $\check{m}_{\bdim_f}$,
which enforces $\check{m}_{\bdim_f} \leq \max\left\{\bdim_f \cdot 2^{n} +1,\hat{m}_{k-1}\right\}$,
will never be a factor in whether or not $\check{m}_{\bdim_f} \geq \ell_{\hat{\iota}}$ is satisfied.

Since $n \geq n^{(iii)}(\lambda,\delta) \geq \tilde{c}_{1}^{(ii)} \ln\left( \tilde{c}_{2}^{(ii)} / \delta\right)$,
Lemma~\ref{lem:sequential-good-labels} implies that on $\tilde{H}_{n}^{(i)}(\delta) \cap \tilde{H}_{n}^{(ii)}$,
$V_{\hat{m}_{\bdim_f}} = \truV_{\hat{m}_{\bdim_f}}$.  Recall that this implies that all $\hat{y}$ values obtained
while $m \leq \hat{m}_{\bdim_f}$ are consistent with their respective $f(X_m)$ values, so that every such
$m$ has $V_m = \truV_m$ as well.  In particular, $V_{\check{m}_{\bdim_f}} = \truV_{\check{m}_{\bdim_f}}$.
Also note that $n^{(iii)}(\delta,\lambda) \geq 24 \cdot \init^{(iv)}(\gamma/16;\delta)$,
so that $\init^{(iv)}(\gamma/16;\delta) \leq \tilde{m}_{n}$, and recall we always have $\tilde{m}_{n} \leq \check{m}_{\bdim_f}$.
Thus, on $\tilde{H}_{n}^{(i)}(\delta) \cap \tilde{H}_{n}^{(ii)} \cap \tilde{H}_{n}^{(iii)}(\delta, \lambda)$,
(taking $\hat{\Delta}^{(k)}$ as in \Shattering)
\begin{align}
\hat{\Delta}^{(\bdim_f)}
& = \hat{\Delta}_{\check{m}_{\bdim_f}}^{(\bdim_f)} \left(W_1,W_2, \truV_{\check{m}_{\bdim_f}}\right) &\text{ (Lemma~\ref{lem:sequential-good-labels})} \notag
\\ & \leq \Px\left( x : p_{x}\left(\bdim_f, \check{m}_{\bdim_f}\right)\geq \gamma / 8\right) + 4 \check{m}_{\bdim_f}^{-1} &\text{ (Lemma~\ref{lem:basic-empirical-bound})}\notag
\\ & \leq
\frac{8\Px^{\bdim_f}\left(\S^{\bdim_f}\left(\truV_{\check{m}_{\bdim_f}}\right)\right)}
{\gamma \Px^{\bdim_f-1}\left(\S^{\bdim_f-1}\left(\truV_{\check{m}_{\bdim_f}}\right)\right)}
+ 4 \check{m}_{\bdim_f}^{-1} &\text{ (Markov's ineq.)} \notag
\\ & \leq
\left(8 / \gamma \dprob\right) \Px^{\bdim_f}\left(\S^{\bdim_f}\left(\truV_{\check{m}_{\bdim_f}}\right)\right)
+ 4 \check{m}_{\bdim_f}^{-1} &\text{ (Lemma~\ref{lem:Vshat-to-Boundaries})} \notag
\\ & \leq 
\left(8 / \gamma \dprob\right) \Px^{\bdim_f}\left(\S^{\bdim_f}\left(\truV_{\ell_{\hat{\iota}}}\right)\right)
+ 4 \ell_{\hat{\iota}}^{-1} &\text{ (defn of $\tilde{H}_n^{(iii)}(\delta,\lambda)$)} \notag
\\ & \leq
\left(8 / \gamma \dprob\right) \Px^{\bdim_f}\left(\S^{\bdim_f}\left(\Ball\left(f, \tilde{\phi}\left(\ell_{\hat{\iota}}, \delta,\lambda\right)\right)\right)\right)
+ 4 \ell_{\hat{\iota}}^{-1} &\text{ (Lemma~\ref{lem:VinB})} \notag
\\ & \leq
\left(8 / \gamma \dprob\right) \hdc_{f}(d/\lambda) \tilde{\phi}\left(\ell_{\hat{\iota}}, \delta,\lambda\right)
+ 4 \ell_{\hat{\iota}}^{-1}
&\text{ (defn of $\hdc_{f}(d/\lambda)$)} \notag
\\ & \leq
\left(12 / \gamma \dprob\right) \hdc_{f}(d/\lambda) \tilde{\phi}\left(\ell_{\hat{\iota}}, \delta,\lambda\right) & \text{ ($\tilde{\phi}\left(\ell_{\hat{\iota}},\delta,\lambda\right) \geq \ell_{\hat{\iota}}^{-1}$)} \notag
\\ & = \frac{12 \hdc_{f}(d/\lambda)}{\gamma \dprob} \max\left\{2\frac{d \ln \left(2 e \max\left\{\ell_{\hat{\iota}},d\right\} / d \right) + \ln \left( 4\hat{\iota} / \delta\right)}{\ell_{\hat{\iota}}}, d/\lambda \right\}. & \label{eqn:exp-hat-delta-bound}
\end{align}
Plugging in the definition of $\hat{\iota}$ and $\ell_{\hat{\iota}}$,
\begin{equation*}
 \frac{d \ln\left( 2 e \max\left\{\ell_{\hat{\iota}}, d\right\} / d\right) + \ln\left( 4 \hat{\iota}/\delta\right)}{\ell_{\hat{\iota}}}
 \leq (d/\lambda) \gamma \dprob 4^{-1-\bdim_f} \ln\left( 4^{1 + \bdim_f} \lambda / \delta \gamma \dprob\right)
\leq (d/\lambda) \ln\left( \lambda / \delta \right).
\end{equation*}
Therefore, \eqref{eqn:exp-hat-delta-bound} is at most
$24 \hdc_{f}(d/\lambda) (d/\lambda) \ln\left(\lambda/\delta\right) / \gamma \dprob$.
Thus, since
\begin{equation*}
n^{(iii)}(\delta,\lambda) \geq \max\left\{ \tilde{c}_{1}^{(i)} \ln\left( \tilde{c}_{2}^{(i)} / \delta\right), \tilde{c}_{1}^{(ii)} \ln\left( \tilde{c}_{2}^{(ii)} / \delta\right)\right\},
\end{equation*}
Lemmas~\ref{lem:sequential-t-vs-T} and \ref{lem:sequential-good-labels} imply that on $\tilde{H}_{n}^{(i)}(\delta) \cap \tilde{H}_{n}^{(ii)} \cap \tilde{H}_{n}^{(iii)}(\delta,\lambda)$,
\begin{align*}
\left| \L_{\bdim_f} \right|
 = \left\lfloor \truT_{\bdim_f} / \left(3 \hat{\Delta}^{(\bdim_f)}\right)\right\rfloor
& \geq \left\lfloor  4^{1-\bdim_f} 2 n / \left(9 \hat{\Delta}^{(\bdim_f)}\right) \right\rfloor
\\ &\geq \frac{4^{1-\bdim_f} \gamma \dprob n}
{9 \cdot 24 \cdot \hdc_{f}(d/\lambda) (d/\lambda) \ln\left(\lambda/\delta\right)}
 \geq \lambda \ln (\lambda/\delta) \geq \lambda.
\end{align*}

Now we turn to bounding $\P\left(\tilde{H}_n^{(i)}(\delta) \cap \tilde{H}_n^{(ii)} \setminus \tilde{H}_n^{(iii)}(\delta, \lambda)\right)$.
By a union bound, we have
\begin{equation}
\label{eqn:tildeH3-union-bound}
1 - \P\left(\tilde{H}^{(3)}(\delta, \lambda)\right) \leq \sum_{i=1}^{\hat{\iota}} \left(1- \P\left(H_{\ell_{i}}\left(\delta / 2 \hat{\iota}\right)\right)\right) \leq \delta / 2.
\end{equation}
Thus, it remains only to bound
$\P\left( \tilde{H}_{n}^{(i)}(\delta) \cap \tilde{H}_{n}^{(ii)} \cap \tilde{H}^{(3)}(\delta,\lambda) \cap \left\{ \check{m}_{\bdim_f} < \ell_{\hat{\iota}} \right\} \right)$.

For each $i \in \{0,1,\ldots,\hat{\iota}-1\}$, let $\check{Q}_{i} = \left|\left\{ m \in (\ell_{i}, \ell_{i+1}] \cap \check{\U}_{\bdim_f} : \truI_{m \bdim_f} = 1\right\}\right|$.
Now consider the set $\mathcal{I}$ of all $i \in \{0,1,\ldots, \hat{\iota}-1\}$ with $\ell_{i} \geq \tilde{m}_n$ and $(\ell_{i}, \ell_{i+1}] \cap \check{\U}_{\bdim_f} \neq \emptyset$.
Note that $n^{(iii)}(\delta,\lambda) \geq 48$, so that $\ell_0 < \tilde{m}_n$.
Fix any $i \in \mathcal{I}$. 
Since $n^{(iii)}(\lambda,\delta) \geq 24 \cdot \init(1/6;\delta)$, we have $\tilde{m}_n \geq \init(1/6;\delta)$,
so that Lemma~\ref{lem:monotonic-hat-delta} implies that on $\tilde{H}_{n}^{(i)}(\delta) \cap \tilde{H}_{n}^{(ii)} \cap \tilde{H}^{(3)}(\delta,\lambda)$,
letting $\bar{Q} = 2 \cdot 4^{6+\bdim_f} \left( d / \gamma^2 \dprob^2\right) \hdc_{f}(d/\lambda) \ln(\lambda / \delta)$,
\begin{multline}
 \P\left( \tilde{H}_{n}^{(i)}(\delta) \cap \tilde{H}_{n}^{(ii)} \cap \tilde{H}^{(3)}(\delta, \lambda) \cap \left\{\check{Q}_{i} > \bar{Q} \right\} \Big| W_2, \truV_{\ell_{i}}\right)
\\  \leq \P\left( \left| \left\{ m \in (\ell_{i},\ell_{i+1}] \cap \nats : \hat{\Delta}_{m}^{(\bdim_f)}\left(X_{m}, W_2, \truV_{\ell_{i}}\right) \geq 2 \gamma / 3 \right\}\right| > \bar{Q}\Bigg| W_2, \truV_{\ell_{i}}\right). \label{eqn:exp-Vli-swap}
\end{multline}
For $m > \ell_{i}$, the variables $\ind_{[2\gamma/3,\infty)}\left(\hat{\Delta}_{m}^{(\bdim_f)}\left(X_{m}, W_2, \truV_{\ell_{i}}\right)\right)$
are conditionally (given $W_2, \truV_{\ell_{i}}$) independent, each with respective conditional distribution Bernoulli with mean
$\bar{p}_{2\gamma/3}\left(\bdim_f,\ell_{i},m\right)$.
Since $n^{(iii)}(\delta,\lambda) \geq 24 \cdot \init(3/32;\delta)$, we have $\tilde{m}_n \geq \init(3/32;\delta)$,
so that Lemma~\ref{lem:basic-bar-delta-bound} (with $\zeta = 2\gamma/3$, $\alpha = 3/4$, and $\beta = 3/32$)
implies that on $\tilde{H}_{n}^{(i)}(\delta) \cap \tilde{H}_{n}^{(ii)} \cap \tilde{H}^{(3)}(\delta,\lambda)$, each of these $m$ values has
\begin{align}
\bar{p}_{2\gamma/3} & \left(\bdim_f, \ell_{i}, m\right)
 \leq \Px\left( x : p_{x}\left(\bdim_f, \ell_{i}\right) \geq \gamma / 2\right) + \exp\left\{- \tilde{M}(m) \gamma^2 / 256\right\} & \notag
\\ & \leq \frac{2 \Px^{\bdim_f}\left(\S^{\bdim_f}\left(\truV_{\ell_{i}}\right)\right)}{\gamma \Px^{\bdim_f-1}\left(\S^{\bdim_f-1}\left(\truV_{\ell_{i}}\right)\right)} + \exp\left\{-\tilde{M}(\ell_{i}) \gamma^2 / 256\right\} & \text{ (Markov's ineq.)} \notag
\\ & \leq \left(2 / \gamma \dprob\right) \Px^{\bdim_f}\left(\S^{\bdim_f}\left(\truV_{\ell_{i}}\right)\right) + \exp\left\{-\tilde{M}(\ell_{i}) \gamma^2 / 256\right\} & \text{ (Lemma~\ref{lem:Vshat-to-Boundaries})} \notag
\\ & \leq \left(2 / \gamma \dprob\right) \Px^{\bdim_f}\left(\S^{\bdim_f}\left( \Ball\left(f, \tilde{\phi}(\ell_{i},\delta,\lambda)\right)\right)\right) + \exp\left\{-\tilde{M}(\ell_{i}) \gamma^2 / 256\right\} & \text{ (Lemma~\ref{lem:VinB})} \notag
\\ & \leq \left(2 / \gamma \dprob\right) \hdc_{f}(d/\lambda) \tilde{\phi}(\ell_{i},\delta,\lambda) + \exp\left\{-\tilde{M}(\ell_{i}) \gamma^2 / 256\right\} & \text{ (defn of $\hdc_{f}(d/\lambda)$)}. \notag
\end{align}
Denote the expression in this last line by $p_i$,
and let $\mathbf{B}(\ell_{i}, p_{i})$ be a ${\rm {Binomial}}(\ell_{i}, p_{i})$ random variable.
Noting that $\ell_{i+1} - \ell_{i} = \ell_{i}$,
we have that on $\tilde{H}_{n}^{(i)}(\delta) \cap \tilde{H}_{n}^{(ii)} \cap \tilde{H}^{(3)}(\delta,\lambda)$,
\eqref{eqn:exp-Vli-swap} is at most
$\P\left( \mathbf{B}(\ell_{i}, p_{i}) > \bar{Q} \right)$.
Next, 
note that
\begin{equation*}
\ell_i p_i  = (2 / \gamma \dprob) \hdc_{f}(d/\lambda) \ell_i \tilde{\phi}(\ell_i,\delta,\lambda) + \ell_i \cdot \exp\left\{- \Msize{\ell_i} \dprob \gamma^2 / 512\right\}.
\end{equation*}
Since $u \cdot \exp\left\{ - u^3 \right\} \leq (3 e)^{-1/3}$ for any $u$, letting $u = \ell_i \dprob \gamma / 8$ we have
\begin{equation*}
\ell_i \cdot \exp\left\{ - \ell_i^3 \dprob \gamma^2 / 512\right\}
\leq \left(8 / \gamma \dprob\right) u \cdot \exp\left\{ - u^3 \right\}
\leq 8 / \left(\gamma \dprob (3 e)^{1/3}\right)
\leq 4 / \gamma \dprob.
\end{equation*}
Therefore, since $\tilde{\phi}(\ell_i,\delta,\lambda) \geq \ell_i^{-1}$, we have that $\ell_i p_i$ is at most
\begin{align*}
\frac{6}{\gamma \dprob} \hdc_{f}(\vc/\lambda) \ell_i \tilde{\phi}(\ell_i,\delta,\lambda)
& \leq \frac{6}{\gamma \dprob} \hdc_{f}(\vc/\lambda) \max\left\{ 2 \vc \ln \left(2 e \ell_{\hat{\iota}}\right) + 2 \ln \left(\frac{4 \hat{\iota}}{\delta}\right), \ell_{\hat{\iota}} \vc/\lambda\right\}
\\ & \leq \frac{6}{\gamma \dprob} \hdc_{f}(\vc/\lambda) \max\left\{ 2 \vc \ln \left(\frac{4^{3+\bdim_f} e\lambda}{\gamma \dprob}\right) + 2 \ln \left(\frac{4^{3+\bdim_f} 2\lambda}{\gamma \dprob \delta}\right), \frac{\vc 4^{3+\bdim_f}}{\gamma \dprob} \right\}
\\ & \leq \frac{6}{\gamma \dprob} \hdc_{f}(\vc / \lambda) \max\left\{4 \vc \ln \left(\frac{4^{3+\bdim_f} \lambda}{\gamma \dprob \delta}\right), \frac{\vc 4^{3+\bdim_f}}{\gamma \dprob} \right\} 
\\ & \leq \frac{6}{\gamma \dprob} \hdc_{f}(\vc / \lambda) \cdot \frac{\vc  4^{4+\bdim_f}}{\gamma \dprob} \ln\left(\frac{\lambda}{\delta}\right)
\leq \frac{4^{6+\bdim_f} \vc}{\gamma^2 \dprob^2} \hdc_{f}(\vc/\lambda)  \ln \left(\frac{\lambda}{\delta}\right) = \bar{Q} / 2.
\end{align*}
Therefore, a Chernoff bound implies
$\P\left(\mathbf{B}(\ell_{i}, p_{i}) > \bar{Q} \right) \leq \exp\left\{- \bar{Q} / 6\right\} \leq \delta / 2 \hat{\iota}$,
so that on $\tilde{H}_{n}^{(i)}(\delta) \cap \tilde{H}_{n}^{(ii)} \cap \tilde{H}^{(3)}(\delta,\lambda)$, \eqref{eqn:exp-Vli-swap} is at most $\delta / 2 \hat{\iota}$.
The law of total probability implies there exists an event $\tilde{H}_{n}^{(4)}(i,\delta,\lambda)$ with
$\P\left(\tilde{H}_{n}^{(i)}(\delta) \cap \tilde{H}_{n}^{(ii)} \cap \tilde{H}^{(3)}(\delta,\lambda) \setminus \tilde{H}_{n}^{(4)}(i, \delta,\lambda)\right)$ $\leq \delta / 2 \hat{\iota}$
such that, on $\tilde{H}_{n}^{(i)}(\delta) \cap \tilde{H}_{n}^{(ii)} \cap \tilde{H}^{(3)}(\delta,\lambda) \cap \tilde{H}_{n}^{(4)}(i,\delta,\lambda)$,
$\check{Q}_{i} \leq \bar{Q}$.

Note that
\begin{align}
\hat{\iota} \bar{Q}
&\leq \log_{2} \left( 4^{2+\bdim_f} \lambda/ \gamma \dprob \right) \cdot 4^{7+\bdim_f} \left( d / \gamma^2 \dprob^2\right) \hdc_{f}(d/\lambda) \ln(\lambda/\delta) \notag
\\ &\leq \left(\bdim_f 4^{9 + \bdim_f} / \gamma^3 \dprob^3\right) d \hdc_{f}(d/\lambda) \ln^{2} \left( \lambda / \delta\right)
\leq 4^{1-\bdim_f} n / 12. \label{eqn:sequential-iQ-query-bound}
\end{align}
Since $\sum_{m \leq 2 \tilde{m}_{n}} \truI_{m \bdim_f} \leq n / 12$,
if $\bdim_f = 1$ then \eqref{eqn:sequential-iQ-query-bound} implies that on
$\tilde{H}_{n}^{(i)}(\delta) \cap \tilde{H}_{n}^{(ii)} \cap \tilde{H}^{(3)}(\delta,\lambda) \cap \bigcap_{i \in \mathcal{I}} \tilde{H}_{n}^{(4)}(i,\delta,\lambda)$,
$\sum_{m \leq \ell_{\hat{\iota}}} \truI_{m 1} \leq n/12 + \sum_{i \in \mathcal{I}} \check{Q}_i \leq n/12 + \hat{\iota} \bar{Q} \leq n/6 \leq \left\lceil \truT_{1} / 4 \right\rceil$,
so that $\check{m}_{1} \geq \ell_{\hat{\iota}}$.  Otherwise, if $\bdim_f > 1$, then every $m \in \check{\U}_{\bdim_f}$ has $m > 2 \tilde{m}_n$, 
so that $\sum_{i \leq \hat{\iota}} \check{Q}_i = \sum_{i \in \mathcal{I}} \check{Q}_i$; thus,
on $\tilde{H}_{n}^{(i)}(\delta) \cap \tilde{H}_{n}^{(ii)} \cap \tilde{H}^{(3)}(\delta,\lambda) \cap \bigcap_{i \in \mathcal{I}} \tilde{H}_{n}^{(4)}(i,\delta,\lambda)$,
$\sum_{i \in \mathcal{I}} \check{Q}_i \leq \hat{\iota} \bar{Q} \leq 4^{1-\bdim_f} n / 12$;
Lemma~\ref{lem:sequential-T-lower} implies $4^{1-\bdim_f} n / 12 \leq \left\lceil \truT_{\bdim_f} / 4\right\rceil$,
so that again we have $\check{m}_{\bdim_f} \geq \ell_{\hat{\iota}}$.
Thus, a union bound implies
\begin{align}
\P&\left( \tilde{H}_{n}^{(i)}(\delta) \cap \tilde{H}_{n}^{(ii)} \cap \tilde{H}^{(3)} (\delta,\lambda) \cap \left\{ \check{m}_{\bdim_f} <  \ell_{\hat{\iota}}\right\} \right) \notag
\\ & \leq \P\left( \tilde{H}_{n}^{(i)}(\delta) \cap \tilde{H}_{n}^{(ii)} \cap \tilde{H}^{(3)}(\delta,\lambda) \setminus \bigcap_{i \in \mathcal{I}} \tilde{H}_{n}^{(4)}(i,\delta,\lambda)\right) \notag
\\ & \leq \sum_{i \in \mathcal{I}} \P\left( \tilde{H}_{n}^{(i)}(\delta) \cap \tilde{H}_{n}^{(ii)} \cap \tilde{H}^{(3)}(\delta,\lambda) \setminus \tilde{H}_{n}^{(4)}(i,\delta,\lambda)\right)
 \leq \delta / 2. \label{eqn:tildeH4-union-bound}
\end{align}
Therefore,
$\P\left(\tilde{H}_{n}^{(i)}(\delta) \cap \tilde{H}_{n}^{(ii)} \setminus \tilde{H}_{n}^{(iii)}(\delta,\lambda)\right) \leq \delta$,
obtained by summing \eqref{eqn:tildeH4-union-bound} and \eqref{eqn:tildeH3-union-bound}.
\end{proof}

\begin{proof}[Theorem~\ref{thm:sequential-activizer}]
If $\Lambda_{p}(\eps/4,f,\Px) = \infty$ then the result trivially holds.
Otherwise, suppose $\eps \in (0,10 e^{-3})$, let $\delta = \eps / 10$, $\lambda = \Lambda_{p}(\eps/4,f,\Px)$,
$\tilde{c}_2 = \max\left\{10\tilde{c}_{2}^{(i)}, 10\tilde{c}_{2}^{(ii)}, 10\tilde{c}_{2}^{(iii)}, 10 e (\vc+1)\right\}$,
and
$\tilde{c}_1 = \max\left\{\tilde{c}_{1}^{(i)}, \tilde{c}_{1}^{(ii)}, \tilde{c}_{1}^{(iii)}, 2\cdot 6^3 (\vc+1) \bdim \ln(e(\vc+1))\right\}$,
and consider running \Shattering~with passive algorithm $\alg_p$ and budget
$n \geq \tilde{c}_1 \hdc_{f}(d/\lambda) \ln^{2}(\tilde{c}_2 \lambda/\eps)$,
while $f$ is the target function and $\Px$ is the data distribution.
On the event $\tilde{H}_{n}^{(i)}(\delta) \cap \tilde{H}_{n}^{(ii)} \cap \tilde{H}_{n}^{(iii)}(\delta,\lambda)$,
Lemma~\ref{lem:sequential-exponential} implies $\left| \L_{\bdim_f} \right| \geq \lambda$,
while Lemma~\ref{lem:sequential-good-labels} implies $V^{(\bdim_f)} = \truV_{\hat{m}_{\bdim_f}}$; recalling that
Lemma~\ref{lem:Vshat-to-Boundaries} implies that $\truV_{\hat{m}_{\bdim_f}} \neq \emptyset$ on this event,
we must have $\er_{\L_{\bdim_f}}(f) = 0$.
Furthermore, if $\hat{h}$ is the classifier returned by \Shattering, then Lemma~\ref{lem:active-select} implies
that $\er(\hat{h})$ is at most $2\er(\alg_{p}(\L_{\bdim_f}))$, on a high probability event (call it $\hat{E}_2$ in this context).
Letting $\hat{E}_{3}(\delta) = \hat{E}_2 \cap \tilde{H}_{n}^{(i)}(\delta) \cap \tilde{H}_{n}^{(ii)} \cap \tilde{H}_{n}^{(iii)}(\delta,\lambda)$,
the total failure probability $1-\P(\hat{E}_3(\delta))$ from all of these events is at most
$4 \delta + e(\vc+1) \cdot \exp\left\{- \lfloor n/3 \rfloor / \left(72 \bdim_f (\vc+1) \ln(e(\vc+1))\right)\right\} \leq 5\delta = \eps / 2$.
Since, for $\ell \in \nats$ with $\P\left(\left| \L_{\bdim_f} \right| = \ell\right) > 0$,
the sequence of $X_{m}$ values appearing in $\L_{\bdim_{f}}$ are conditionally distributed as $\Px^{\ell}$ given $|\L_{\bdim_{f}}| = \ell$,
and this is the same as the (unconditional) distribution of $\{X_1,X_2,\ldots,X_{\ell}\}$,
we have that
\begin{align*}
\E\left[ \er\!\left(\hat{h}\right)\right]
\leq \E\left[ 2 \er\!\left(\alg_{p}\left(\L_{\bdim_{f}}\right)\right) \ind_{\hat{E}_{3}(\delta)} \right] + \eps/2
&= \E\left[ \E\left[2 \er\!\left(\alg_{p}\left(\L_{\bdim_{f}}\right)\right) \ind_{\hat{E}_{3}(\delta)} \Big| |\L_{\bdim_{f}} | \right]\right] \!+ \eps/2
\\ &\leq 2 \sup_{\ell \geq \Lambda_{p}(\eps/4,f,\Px)} \E\left[ \er\!\left(\alg_{p}\left(\Data_{\ell}\right)\right)\right] + \eps/2 \leq \eps.
\end{align*}
To specialize to the specific variant of \Shattering~stated in Section~\ref{subsec:sequential-activizer}, take $\gamma = 1/2$.
\end{proof}

\section{Proofs Related to Section~\ref{sec:agnostic}: Agnostic Learning}
\label{app:agnostic}

\subsection{Proof of Theorem~\ref{thm:agnostic-counterexample}: Negative Result for Agnostic Activized Learning}
\label{app:agnostic-counterexample}

It suffices to show that $\check{\alg}_{p}$ achieves a label complexity $\Lambda_{p}$ such that,
for any label complexity $\Lambda_{a}$ achieved by any active learning algorithm $\alg_{a}$,
there exists a distribution $\PXY$ on $\X \times \{-1,+1\}$ such that
$\PXY \in \Nontrivial(\Lambda_{p};\C)$ and yet
$\Lambda_{a}(\nu + c\eps,\PXY) \neq o\left(\Lambda_{p}(\nu + \eps,\PXY)\right)$
for every constant $c \in (0,\infty)$.
Specifically, we will show that there is a distribution $\PXY$ for which $\Lambda_{p}(\nu + \eps,\PXY) = \Theta(1/\eps)$
and $\Lambda_{a}(\nu + \eps,\PXY) \neq o(1/\eps)$.

Let $\Px(\{0\}) = 1/2$, and for any measurable $A \subseteq (0,1]$,
$\Px(A) = \lambda(A)/2$, where $\lambda$ is Lebesgue measure.
Let $\mathbb{D}$ be the family of distributions $\PXY$ on $\X\times\{-1,+1\}$ characterized by the
properties that the marginal distribution on $\X$ is $\Px$, $\eta(0;\PXY) \in (1/8,3/8)$, and
$\forall x \in (0,1]$,
\begin{equation*}
\eta(x; \PXY) = \eta(0; \PXY) + \left( x/2 \right) \cdot \left(1 - \eta(0; \PXY)\right).
\end{equation*}
Thus, $\eta(x;\PXY)$ is a linear function.
For any $\PXY \in \mathbb{D}$, since the point
$z^* = \frac{1-2\eta(0;\PXY)}{1-\eta(0;\PXY)}$ has $\eta(z^*;\PXY) = 1/2$,
we see that $f = h_{z^*}$ is a Bayes optimal classifier.
Furthermore, for any $\eta_{0} \in [1/8,3/8]$,
\begin{equation*}
\left|\frac{1-2\eta_0}{1-\eta_0} - \frac{1-2\eta(0;\PXY)}{1-\eta(0;\PXY)} \right|
= \frac{\left|\eta(0;\PXY) - \eta_0\right|}{(1-\eta_0)(1-\eta(0;\PXY))},
\end{equation*}
and since $(1-\eta_0)(1-\eta(0;\PXY)) \in (25/64, 49/64) \subset (1/3, 1)$,
the value $z = \frac{1 - 2 \eta_0}{1-\eta_0}$ satisfies
\begin{equation}
\label{eqn:agco-eta-z}
|\eta_{0} - \eta(0;\PXY)| \leq |z - z^*| \leq 3 |\eta_{0} - \eta(0;\PXY)|.
\end{equation}
Also note that under $\PXY$, since $(1-2\eta(0;\PXY)) = (1-\eta(0;\PXY)) z^*$, any $z \in (0,1)$ has
\begin{align*}
\er(h_{z}) - \er(h_{z^*})
& = \int_{z}^{z^*} \!\!\big(1 - 2 \eta(x;\PXY)\big) {\rm d}x
= \int_{z}^{z^*} \!\!\big(1 - 2\eta(0;\PXY) - x (1-\eta(0;\PXY))\big) {\rm d}x
\\ & = \left(1-\eta(0;\PXY)\right) \int_{z}^{z^*} \left( z^* - x \right) {\rm d}x
= \frac{\left(1-\eta(0;\PXY)\right)}{2} \left( z^* - z \right)^2,
\end{align*}
so that
\begin{equation}
\label{eqn:agco-er-z}
\frac{5}{16} (z-z^*)^2 \leq \er(h_{z}) - \er(h_{z^*}) \leq \frac{7}{16} (z - z^*)^2.
\end{equation}
Finally, note that any $x,x^{\prime} \in (0,1]$ with $|x - z^*| < |x^{\prime}-z^*|$
has
\begin{equation*}
|1 - 2\eta(x;\PXY)| = |x-z^*|(1-\eta(0;\PXY)) < |x^{\prime}-z^*|(1-\eta(0;\PXY)) = |1 - 2\eta(x^{\prime};\PXY)|.
\end{equation*}
Thus, for any $q \in (0,1/2]$, there exists $z^{\prime}_{q} \in [0,1]$ such that $z^{*} \in [z^{\prime}_{q}, z^{\prime}_{q}+2q] \subseteq [0,1]$,
and the classifier $h_{q}^{\prime}(x) = h_{z^*}(x) \cdot \left( 1 - 2 \ind_{(z^{\prime}_{q}, z^{\prime}_{q} + 2q]}(x)\right)$
has $\er(h) \geq \er(h_{q}^{\prime})$ for every classifier $h$ with $h(0) = -1$ and $\Px(x : h(x) \neq h_{z^*}(x)) = q$.
Noting that
$\er(h_{q}^{\prime}) - \er(h_{z^*}) = \left(\lim_{z \downarrow z_{q}^{\prime}} \er(h_{z}) - \er(h_{z^*})\right) + \left(\er(h_{z_{q}^{\prime}+2q}) - \er(h_{z^*}) \right)$,
\eqref{eqn:agco-er-z} implies that
$\er(h_{q}^{\prime}) - \er(h_{z^*}) \geq$ $\frac{5}{16} \left( \left( z_{q}^{\prime} - z^* \right)^2 + \left( z_{q}^{\prime} + 2q - z^*\right)^2\right)$,
and since $\max\{ z^* - z_q^{\prime}, z_q^{\prime}+2q - z^*\} \geq q$,
this is at least $\frac{5}{16} q^2$.
In general, any $h$ with $h(0) = +1$ has $\er(h) - \er(h_{z^*}) \geq 1/2 - \eta(0;\PXY) > 1/8 \geq (1/8) \Px(x : h(x) \neq h_{z^*}(x))^2$.
Combining these facts, we see that any classifier $h$ has
\begin{equation}
\label{eqn:agco-tsybakov}
\er(h) - \er(h_{z^*}) \geq (1/8) \Px\left( x : h(x) \neq h_{z^*}(x) \right)^2.
\end{equation}

\begin{lemma}
\label{lem:agnostic-counterexample-passive}
The passive learning algorithm $\check{\alg}_{p}$ achieves a label complexity $\Lambda_{p}$ such that,
for every $\PXY \in \mathbb{D}$, $\Lambda_{p}(\nu+\eps,\PXY) = \Theta(1/\eps)$.
\thmend
\end{lemma}
\begin{proof}
Consider the values $\hat{\eta}_{0}$ and $\hat{z}$ from $\check{\alg}_{p}(\Data_n)$ for some $n \in \nats$.
Combining \eqref{eqn:agco-eta-z} and \eqref{eqn:agco-er-z}, we have
$\er(h_{\hat{z}}) - \er(h_{z^*}) \leq \frac{7}{16} ( \hat{z} - z^*)^2 \leq \frac{63}{16} ( \hat{\eta}_{0} - \eta(0; \PXY))^2 \leq 4 (\hat{\eta}_{0} - \eta(0;\PXY))^2$.
Let $N_n = |\{i \in \{1,\ldots,n\} : X_i = 0\}|$,
and $\bar{\eta}_{0} = N_n^{-1} |\{i \in \{1,\ldots,n\}: X_i = 0, Y_i = +1\}|$ if $N_n > 0$, or $\bar{\eta}_{0} = 0$ if $N_n = 0$.
Note that $\hat{\eta}_{0} = \left( \bar{\eta}_{0} \lor \frac{1}{8} \right) \land \frac{3}{8}$,
and since $\eta(0;\PXY) \in (1/8,3/8)$, we have
$|\hat{\eta}_{0} - \eta(0;\PXY)| \leq |\bar{\eta}_{0} - \eta(0;\PXY)|$.
Therefore, for any $\PXY \in \mathbb{D}$,
\begin{align}
\E\left[ \er(h_{\hat{z}}) - \er(h_{z^*}) \right]
& \leq 4 \E\left[(\hat{\eta}_{0} - \eta(0; \PXY))^2 \right]
\leq 4 \E\left[(\bar{\eta}_{0} - \eta(0; \PXY))^2 \right] \notag
\\ & \leq 4 \E\left[\E\left[ (\bar{\eta}_{0} - \eta(0; \PXY))^2 \Big| N_n\right] \ind_{[n/4,n]}(N_n)\right] + 4 \P(N_n < n/4).
\label{eqn:agnostic-counterexample-passive-1}
\end{align}
By a Chernoff bound, $\P(N_n < n/4) \leq \exp\{- n/16\}$,
and since the conditional distribution of $N_n \bar{\eta}_{0}$ given $N_n$ is ${\rm Binomial}(N_n, \eta(0;\PXY))$,
\eqref{eqn:agnostic-counterexample-passive-1} is at most
\begin{equation*}
4 \E\left[\frac{1}{N_{n} \lor n/4} \eta(0; \PXY) (1 - \eta(0;\PXY)) \right] + 4\cdot\exp\left\{-n/16\right\}
\leq 4 \cdot \frac{4}{n} \cdot \frac{15}{64} + 4 \cdot \frac{16}{n} 
< \frac{68}{n}.
\end{equation*}
For any $n \geq \lceil 68 / \eps \rceil$, this is at most $\eps$.
Therefore, $\check{\alg}_{p}$ achieves a label complexity $\Lambda_{p}$ such that,
for any $\PXY \in \mathbb{D}$, $\Lambda_{p}(\nu + \eps, \PXY) = \lceil 68 / \eps \rceil = \Theta(1/\eps)$.
\end{proof}

Next we establish a corresponding lower bound for any active learning algorithm.
Note that this requires more than a simple minimax lower bound, since we must have an asymptotic lower bound for a \emph{fixed} $\PXY$,
rather than selecting a different $\PXY$ for each $\eps$ value; this is akin to the \emph{strong} minimax lower bounds proven
by \citet*{antos:98} for passive learning in the realizable case.  For this, we proceed by reduction from the task of estimating a binomial mean;
toward this end, the following lemma will be useful.

\begin{lemma}
\label{lem:binomial-lower-bound}
For any nonempty $(a,b) \subset [0,1]$, and any sequence of estimators $\hat{p}_{n} : \{0,1\}^{n} \to [0,1]$,
there exists $p \in (a,b)$ such that,
if $B_1,B_2,\ldots$ are independent ${\rm Bernoulli}(p)$ random variables,
also independent from every $\hat{p}_n$,
then $\E\left[ \left(\hat{p}_n(B_1,\ldots,B_n) - p\right)^2\right] \neq o(1/n)$.
\thmend
\end{lemma}
\begin{proof}
We first establish the claim when $a = 0$ and $b=1$.
For any $p \in [0,1]$, let $B_1(p),B_2(p),\ldots$ be i.i.d. ${\rm Bernoulli}(p)$ random variables, independent from any internal randomness of the $\hat{p}_{n}$ estimators.
We proceed by reduction from hypothesis testing, for which there are known lower bounds.
Specifically, it is known \citep*[e.g.,][]{wald:45,bar-yossef:03}
that for any $p,q \in (0,1)$, $\delta \in (0,e^{-1})$,
any (possibly randomized) $\hat{q} : \{0,1\}^{n} \to \{p,q\}$, and any $n \in \nats$,
\begin{equation*}
n < \frac{(1 - 8 \delta) \ln ( 1 / 8 \delta )}{8 {\rm KL}( p \| q )} \implies \max_{p^{*} \in \{p,q\}} \P\left( \hat{q}(B_1(p^*),\ldots,B_n(p^*)) \neq p^*\right) > \delta,
\end{equation*}
where ${\rm KL}(p \| q) = p \ln (p/q) + (1-p) \ln((1-p)/ (1-q))$.
It is also known \citep*[e.g.,][]{poland:06} that for $p,q \in [1/4,3/4]$, ${\rm KL}(p\| q) \leq (8/3) (p-q)^2$.
Combining this with the above fact, we have that for $p,q \in [1/4,3/4]$,
\begin{equation}
\label{eqn:binomial-test-lower}
\max_{p^* \in \{p,q\}} \P\left(\hat{q}(B_1(p^*),\ldots,B_n(p^*)) \neq p^*\right) \geq (1/16) \cdot \exp\left\{ - 128 (p-q)^2 n / 3\right\}.
\end{equation}
Given the estimator $\hat{p}_{n}$ from the lemma statement, we construct a sequence of hypothesis tests as follows.
For $i \in \nats$, let $\alpha_i = \exp\left\{- 2^{i}\right\}$ and $n_i = \left\lfloor 1 / \alpha_i^2 \right\rfloor$.
Define $p_0^* = 1/4$,
and for $i \in \nats$, inductively define
$\hat{q}_{i}(b_1,\ldots,b_{n_i}) = \argmin_{p \in \{p_{i-1}^{*}, p_{i-1}^{*}+\alpha_{i}\}} \left| \hat{p}_{n_i}(b_1,\ldots,b_{n_i}) - p\right|$
for $b_1,\ldots,b_{n_i} \in \{0,1\}$,
and
$p^{*}_{i} = \argmax_{p \in \{ p^{*}_{i-1}, p^{*}_{i-1} + \alpha_{i}\}} \P\left( \hat{q}_{i}(B_1(p),\ldots,B_{n_i}(p)) \neq p\right)$.
Finally, define $p^{*} = \lim_{i \to \infty} p_{i}^{*}$.
Note that $\forall i \in \nats$, $p_{i}^{*} < 1/2$, $p_{i-1}^{*}, p_{i-1}^{*}+\alpha_i \in [1/4,3/4]$,
and $0 \leq p^{*} - p^{*}_{i} \leq \sum_{j = i+1}^{\infty} \alpha_{j} < 2 \alpha_{i+1} = 2 \alpha_{i}^2$.
We generally have
\begin{align*}
\E\left[ \left( \hat{p}_{n_{i}}(B_1(p^{*}),\ldots,B_{n_i}(p^{*})) - p^{*}\right)^2 \right]
& \geq \frac{1}{3} \E\left[ \left( \hat{p}_{n_{i}}(B_1(p^{*}),\ldots,B_{n_i}(p^{*})) - p_{i}^{*}\right)^2 \right] - \left(p^{*} - p_{i}^{*}\right)^2
\\ & \geq \frac{1}{3} \E\left[ \left( \hat{p}_{n_{i}}(B_1(p^{*}),\ldots,B_{n_i}(p^{*})) - p_{i}^{*}\right)^2 \right] - 4 \alpha_{i}^{4}.
\end{align*}
Furthermore, note that for any $m \in \{0,\ldots,n_i\}$,
\begin{align*}
\frac{(p^*)^m (1-p^*)^{n_i-m}}{(p_{i}^{*})^m (1-p_{i}^{*})^{n_i-m}}
& \geq \left( \frac{1-p^*}{1-p_i^*}\right)^{n_i}
\geq \left( \frac{ 1 - p_i^* - 2 \alpha_{i}^{2}}{1 - p_i^*} \right)^{n_i}
\\ & \geq \left( 1 - 4 \alpha_{i}^2 \right)^{n_i}
\geq \exp\left\{ - 8 \alpha_i^2 n_i \right\}
\geq e^{ - 8 },
\end{align*}
so that the probability mass function of $(B_1(p^*),\ldots,B_{n_i}(p^*))$ is never smaller than $e^{-8}$ times that of $(B_1(p_i^*),\ldots,B_{n_i}(p_i^*))$,
which implies (by the law of the unconscious statistician)
\begin{equation*}
\E\left[ \left( \hat{p}_{n_{i}}(B_1(p^{*}), \ldots, B_{n_i}(p^*)) - p_i^*\right)^2\right]
\geq e^{-8} \E\left[ \left( \hat{p}_{n_{i}}( B_{1}(p_{i}^{*}),\ldots,B_{n_i}(p_{i}^{*})) - p_{i}^{*}\right)^2 \right].
\end{equation*}
By a triangle inequality, we have
\begin{equation*}
\E\left[ \left( \hat{p}_{n_i}(B_1(p^{*}_{i}),\ldots,B_{n_i}(p^{*}_{i})) - p^{*}_{i} \right)^{2} \right] \geq \frac{\alpha_{i}^{2}}{4} \P\left( \hat{q}_{i}(B_1(p^{*}_{i}),\ldots,B_{n_i}(p^{*}_{i})) \neq p^{*}_{i}\right).
\end{equation*}
By \eqref{eqn:binomial-test-lower}, this is at least
\begin{equation*}
\frac{\alpha_{i}^{2}}{4} (1/16) \cdot \exp\left\{ - 128 \alpha_i^{2} n_{i} / 3 \right\}
\geq 2^{-6} e^{-43} \alpha_{i}^{2}.
\end{equation*}
Combining the above, we have
\begin{equation*}
\E\left[ \left( \hat{p}_{n_{i}}(B_1(p^{*}),\ldots,B_{n_i}(p^{*})) - p^{*}\right)^2 \right]
\geq 3^{-1} 2^{-6} e^{-51} \alpha_{i}^{2} - 4 \alpha_{i}^{4}
\geq 2^{-9} e^{-51} n_{i}^{-1} - 4 n_{i}^{-2}.
\end{equation*}
For $i \geq 5$, this is larger than $2^{-11} e^{-51} n_{i}^{-1}$.
Since $n_i$ diverges as $i \to \infty$, we have that
\begin{equation*}
\E\left[ \left( \hat{p}_{n_{i}}(B_1(p^{*}),\ldots,B_{n_i}(p^{*})) - p^{*}\right)^2 \right] \neq o(1/n),
\end{equation*}
which establishes the result for $a=0$ and $b=1$.

To extend this result to general nonempty ranges $(a,b)$, we proceed by reduction from the above problem.
Specifically, suppose $p^{\prime} \in (0,1)$, and consider the following
independent random variables (also independent from the $B_i(p^{\prime})$ variables and $\hat{p}_n$ estimators).
For each $i \in \nats$, $C_{i1} \sim {\rm Bernoulli}(a)$, $C_{i2} \sim {\rm Bernoulli}((b-a)/(1-a))$.
Then for $b_{i} \in \{0,1\}$, define $B_{i}^{\prime}(b_{i}) = \max\{C_{i1}, C_{i2} \cdot b_{i} \}$.
For any given $p^{\prime} \in (0,1)$, the random variables $B_{i}^{\prime}(B_{i}(p^{\prime}))$ are i.i.d. ${\rm Bernoulli}\left(p\right)$,
with $p = a + (b-a) p^{\prime} \in (a,b)$ (which forms a bijection between $(0,1)$ and $(a,b)$).
Defining
$\hat{p}_n^{\prime}(b_1,\ldots, b_n) = (\hat{p}_n(B_{1}^{\prime}(b_{1}),\ldots,B_{n}^{\prime}(b_n)) - a)/(b-a)$, we have
\begin{equation}
\label{eqn:binomial-reduced}
\E\left[\left(\hat{p}_n(B_1(p),\ldots,B_n(p)) - p\right)^2\right] = (b-a)^2 \cdot \E\left[ \left(\hat{p}_n^{\prime}(B_1(p^{\prime}),\ldots,B_n(p^{\prime})) - p^{\prime}\right)^2\right].
\end{equation}
We have already shown there exists a value of $p^{\prime} \in (0,1)$ such that the right side of \eqref{eqn:binomial-reduced} is not $o(1/n)$.
Therefore, the corresponding value of $p = a + (b-a) p^{\prime} \in (a,b)$ has the left side of \eqref{eqn:binomial-reduced} not $o(1/n)$,
which establishes the result.
\end{proof}

We are now ready for the lower bound result for our setting.

\begin{lemma}
\label{lem:agnostic-counterexample-active}
For any label complexity $\Lambda_{a}$ achieved by any active learning algorithm $\alg_{a}$,
there exists a $\PXY \in \mathbb{D}$ such that $\Lambda_{a}(\nu + \eps, \PXY) \neq o(1/\eps)$.
\thmend
\end{lemma}
\begin{proof}
The idea here is to reduce from the task of estimating the mean of iid Bernoulli trials, corresponding to the $Y_i$ values.
Specifically, consider any active learning algorithm $\alg_a$; we use $\alg_a$ to construct an estimator for the
mean of iid Bernoulli trials as follows.  Suppose we have $B_1,B_2,\ldots, B_n$ i.i.d. ${\rm Bernoulli}(p)$, for some $p \in (1/8,3/8)$ and $n \in \nats$.
We take the sequence of $X_1,X_2,\ldots$ random variables i.i.d. with distribution $\Px$ defined above (independent from the $B_j$ variables).
For each $i$, we additionally have a random variable $C_i$ with
conditional distribution ${\rm Bernoulli}(X_i/2)$ given $X_i$, where the $C_i$ are conditionally independent given the $X_i$
sequence, and independent from the $B_i$ sequence as well.

We run $\alg_a$ with this sequence of $X_i$ values.  For the $t^{\rm th}$ label request made by the algorithm, say for the $Y_i$ value corresponding to some $X_i$,
if it has previously requested this $Y_i$ already, then we simply repeat the same answer for $Y_i$ again,
and otherwise we return to the algorithm the value $2\max\{B_t, C_i\} -1$ for $Y_i$.
Note that in the latter case, the conditional distribution of $\max\{B_t,C_i\}$ is ${\rm Bernoulli}(p + (1-p) X_i/2)$,
given the $X_i$ that $\alg_a$ requests the label of; thus, the $Y_i$ response has the same conditional distribution given $X_i$ as it would have for the
$\PXY \in \mathbb{D}$ with $\eta(0;\PXY) = p$ (i.e., $\eta(X_i; \PXY) = p + (1-p) X_i / 2$).
Since this $Y_i$ value is conditionally (given $X_i$) independent from the previously returned labels and $X_j$ sequence,
this is distributionally equivalent to running
$\alg_a$ under the $\PXY \in \mathbb{D}$ with $\eta(0; \PXY) = p$.

Let $\hat{h}_n$ be the classifier returned by $\alg_a(n)$ in the above context, and let $\hat{z}_n$ denote
the value of $z \in [2/5,6/7]$ with minimum $\Px(x : h_{z}(x) \neq \hat{h}_{n}(x))$. %%% it exists by compactness of [2/5,6/7]
Then define $\hat{p}_n = \frac{1-\hat{z}_n}{2-\hat{z}_n} \in [1/8,3/8]$ and $z^* = \frac{1-2p}{1-p} \in (2/5,6/7)$.
By a triangle inequality, we have
$|\hat{z}_n - z^*| = 2\Px(x : h_{\hat{z}_n}(x) \neq h_{z^*}(x)) \leq 4 \Px(x : \hat{h}_{n}(x) \neq h_{z^*}(x))$.
Combining this with \eqref{eqn:agco-tsybakov}  and \eqref{eqn:agco-eta-z} implies that
\begin{equation}
\label{eqn:agnostic-counterexample-active-1}
\er(\hat{h}_n) - \er(h_{z^*})
\geq \frac{1}{8} \Px\left( x : \hat{h}_{n}(x) \neq h_{z^*}(x)\right)^2
\geq \frac{1}{128} \left( \hat{z}_n - z^* \right)^2
\geq \frac{1}{128} \left( \hat{p}_n - p \right)^2.
\end{equation}
In particular, by Lemma~\ref{lem:binomial-lower-bound}, we can choose $p \in (1/8,3/8)$
so that $\E\left[ \left( \hat{p}_n - p \right)^2 \right] \neq o(1/n)$,
which, by \eqref{eqn:agnostic-counterexample-active-1}, implies $\E\left[ \er(\hat{h}_n) \right] - \nu \neq o(1/n)$.
This means there is an increasing infinite sequence of values $n_k \in \nats$, and a constant $c \in (0,\infty)$
such that $\forall k \in \nats$, $\E\left[ \er(\hat{h}_{n_k}) \right] - \nu \geq c/n_k$.  Supposing $\alg_a$ achieves
label complexity $\Lambda_a$, and taking the values $\eps_k = c/(2n_k)$,
we have $\Lambda_{a}(\nu+\eps_k,\PXY) > n_k = c / (2 \eps_k)$.
Since $\eps_k > 0$ and approaches $0$ as $k \to \infty$, we have $\Lambda_{a}(\nu+\eps,\PXY) \neq o(1/\eps)$.
\end{proof}

\begin{proof}[of Theorem~\ref{thm:agnostic-counterexample}]
The result follows from Lemmas~\ref{lem:agnostic-counterexample-passive} and \ref{lem:agnostic-counterexample-active}.
\end{proof}

\subsection{Proof of Lemma~\ref{lem:robust-tsybakov}: Label Complexity of \RobustShattering}
\label{app:robust-tsybakov}

The proof of Lemma~\ref{lem:robust-tsybakov} essentially runs parallel to that of Theorem~\ref{thm:sequential-activizer},
with variants of each lemma from that proof adapted to the noise-robust \RobustShattering.

As before, in this section we will fix a particular joint distribution $\PXY$ on $\X \times \{-1,+1\}$
with marginal $\Px$ on $\X$, and then analyze the label complexity achieved by \RobustShattering~for that
particular distribution.  For our purposes, we will suppose $\PXY$ satisfies Condition~\ref{con:tsybakov}
for some finite parameters $\mu$ and $\kappa$.  We also fix any $f \in \bigcap\limits_{\eps > 0} \cl( \C(\eps) )$.
Furthermore, we will continue using the notation of Appendix~\ref{app:activizer}, such as $\S^{k}(\H)$, etc.,
and in particular we continue to denote $\truV_{m} = \{h \in \C : \forall \ell \leq m, h(X_{\ell}) = f(X_{\ell})\}$ 
(though note that in this case, we may sometimes have $f(X_{\ell}) \neq Y_{\ell}$, so that $\truV_{m} \neq \C[\Data_{m}]$).
As in the above proofs, we will prove a slightly more general result in which the ``$1/2$'' threshold
in Step 5 can be replaced by an arbitrary constant $\gamma \in (0,1)$.

For the estimators $\hat{P}_{4m}$ used in the algorithm, we take the same definitions as in
Appendix~\ref{app:hatP-definitions}.  To be clear, we assume the sequences $W_1$ and $W_2$ mentioned
there are independent from the entire $(X_1,Y_1), (X_2,Y_2),\ldots$ sequence of data points; this is consistent
with the earlier discussion of how these $W_1$ and $W_2$ sequences can be constructed in a preprocessing step.

We will consider running \RobustShattering~with label budget $n \in \nats$ and confidence parameter $\conf \in (0,e^{-3})$,
and analyze properties of the internal sets $V_{i}$.
We will denote by $\hat{V}_{i}$, $\hat{\L}_{i}$, and $\hat{i}_k$,
the \emph{final} values of $V_i$, $\L_i$, and $i_k$, respectively, for each $i$ and $k$ in \RobustShattering.
We also denote by $\hat{m}^{(k)}$ and $\hat{V}^{(k)}$ the final values of $m$ and $V_{i_k + 1}$, respectively,
obtained while $k$ has the specified value in \RobustShattering;
$\hat{V}^{(k)}$ may be smaller than $\hat{V}_{\hat{i}_k}$ when $\hat{m}^{(k)}$ is not a power of $2$.
Additionally, define
$\truL_i = \{(X_{m},Y_{m})\}_{m=2^{i-1}+1}^{2^i}$.
After establishing a few results concerning these, we will show that for $n$ satsifying the condition in
Lemma~\ref{lem:robust-tsybakov}, the conclusion of the lemma holds.  First, we have a few auxilliary definitions.
For $\H \subseteq \C$, and any $i \in \nats$, define
\begin{align*}
\phi_{i}(\H) & = \E \sup\limits_{h_1,h_2 \in \H} \left| \left(\er(h_1) - \er_{\truL_i}(h_1)\right) - \left( \er(h_2) - \er_{\truL_i}(h_2)\right)\right|\\
\text{and }~~\tilde{U}_{i}(\H, \conf) & = \min\left\{\tilde{K} \left( \phi_{i}(\H) + \sqrt{\diam(\H) \frac{\ln(32 i^2 /\conf)}{2^{i-1}}} + \frac{\ln(32 i^2 /\conf)}{2^{i-1}}\right), 1\right\},
\end{align*}
where for our purposes we can take $\tilde{K} = 8272$.
It is known \citep*[see e.g.,][]{massart:06,gine:06} that for some universal constant $c^{\prime} \in [2,\infty)$,
\begin{equation}
\label{eqn:modulus-bound}
\phi_{i+1}(\H) \leq c^{\prime} \max\left\{\sqrt{\diam(\H) 2^{-i} d \log_{2} \frac{2}{\diam(\H)}}, 2^{-i} d i\right\}.
\end{equation}
We also generally have $\phi_{i}(\H) \leq 2$ for every $i \in \nats$.
The next lemma is taken from the work of \citet*{koltchinskii:06}
on data-dependent Rademacher complexity bounds on the excess risk.

\begin{lemma}
\label{lem:original-koltchinskii}
For any $\conf \in (0,e^{-3})$, any $\H \subseteq \C$ with $f \in \cl(\H)$, and any $i \in \nats$,
on an event $K_i$ with $\P(K_i) \geq 1-\delta / 4 i^2$, $\forall h \in \H$,
\begin{align*}
\er_{\truL_{i}}(h) - \min_{h^{\prime} \in \H} \er_{\truL_{i}}(h^{\prime}) &\leq \er(h) - \er(f) + \hat{U}_{i}(\H, \conf)
\\ \er(h) - \er(f) & \leq \er_{\truL_{i}}(h) - \er_{\truL_{i}}(f) + \hat{U}_{i}(\H, \conf)
\\ \min\left\{\hat{U}_{i}(\H,\conf),1\right\} &\leq \tilde{U}_{i}(\H,\conf).
\end{align*}
\upthmend{-1.15cm}
\end{lemma}

Lemma~\ref{lem:original-koltchinskii} essentially follows from a version of Talagrand's inequality.
The details of the proof may be extracted from the proofs of \citet*{koltchinskii:06},
and related derivations have previously been presented by \citet*{hanneke:11a,koltchinskii:10}.
The only minor twist here is that $f$ need only be in $\cl(\H)$, rather than in $\H$ itself,
which easily follows from Koltchinskii's original results,
since the Borel-Cantelli lemma implies that with probability one,
every $\eps > 0$ has some $g \in \H(\eps)$ (very close to $f$) with $\er_{\truL_{i}}(g) = \er_{\truL_{i}}(f)$.

For our purposes, the important implications of Lemma~\ref{lem:original-koltchinskii} are summarized by the following lemma.

\begin{lemma}
\label{lem:koltchinskii}
For any $\conf \in (0,e^{-3})$ and any $n \in \nats$,
when running \RobustShattering~with label budget $n$ and confidence parameter $\conf$,
on an event $J_{n}(\conf)$ with $\P(J_{n}(\conf)) \geq 1-\conf/2$,
$\forall i \in \{0,1,\ldots,\hat{i}_{d+1}\}$, if $\truV_{2^{i}} \subseteq \hat{V}_i$ then $\forall h \in \hat{V}_i$,
\begin{align}
\er_{\truL_{i+1}}(h) - \min_{h^{\prime} \in \hat{V}_{i}} \er_{\truL_{i+1}}(h^{\prime}) &\leq \er(h) - \er(f) + \hat{U}_{i+1}(\hat{V}_i, \conf) \label{eqn:kolt-empirical-bound}
\\ \er(h) - \er(f) & \leq \er_{\truL_{i+1}}(h) - \er_{\truL_{i+1}}(f) + \hat{U}_{i+1}(\hat{V}_i, \conf) \label{eqn:kolt-tru-bound}
\\ \min\left\{\hat{U}_{i+1}(\hat{V}_i,\conf),1\right\} &\leq \tilde{U}_{i+1}(\hat{V}_i,\conf). \label{eqn:kolt-U-bound}
\end{align}
\thmend
\end{lemma}
\begin{proof}
For each $i$, consider applying Lemma~\ref{lem:original-koltchinskii} under the conditional distribution given $\hat{V}_i$.
The set $\truL_{i+1}$ is independent from $\hat{V}_i$, as are the Rademacher variables in the definition of $\hat{R}_{i+1}(\hat{V}_i)$.
Furthermore, by Lemma~\ref{lem:Vshat-to-Boundaries}, on $H^{\prime}$, $f \in \cl\left( \truV_{2^i} \right)$, so that
the conditions of Lemma~\ref{lem:original-koltchinskii} hold.
The law of total probability then implies the existence of an event $J_i$ of probability $\P(J_i) \geq 1 - \delta / 4 (i+1)^2$,
on which the claimed inequalities hold for that value of $i$ if $i \leq \hat{i}_{d+1}$.  A union bound over values of $i$ then implies the
existence of an event $J_n(\conf) = \bigcap_{i} J_i$ with probability $\P(J_n(\conf)) \geq 1 - \sum_{i} \delta / 4(i+1)^2 \geq 1 - \delta / 2$
on which the claimed inequalities hold for all $i \leq \hat{i}_{d+1}$.
\end{proof}

\begin{lemma}
\label{lem:robust-good-labels}
For some $(\C,\PXY,\gamma)$-dependent constants $c,c^* \in [1,\infty)$,
for any $\conf \in (0,e^{-3})$ and integer $n \geq c^* \ln(1/\conf)$,
when running \RobustShattering~with label budget $n$ and confidence parameter $\conf$,
on event $J_n(\conf) \cap H_{n}^{(i)} \cap H_{n}^{(ii)}$,
every $i \in \{0,1,\ldots, \hat{i}_{\bdim_f}\}$ satisfies
\begin{equation*}
\truV_{2^{i}} \subseteq \hat{V}_{i} \subseteq \C\left( c  \left(\frac{d i + \ln(1/\conf)}{2^{i}}\right)^{\frac{\kappa}{2\kappa - 1}}\right),
\end{equation*}
and furthermore $\truV_{\hat{m}^{(\bdim_f)}} \subseteq \hat{V}^{(\bdim_f)}$.
\thmend
\end{lemma}
\begin{proof}
Define
$c = \left(24 \tilde{K} c^{\prime} \sqrt{\mu}\right)^{\frac{2\kappa}{2\kappa - 1}}$,
$c^* = \max\left\{\init^*, 8 \vc \left(\frac{\mu c^{1/\kappa}}{r_{(1-\gamma)/6}} \right)^{\frac{1}{2\kappa-1}} \log_{2}\left(\frac{4 \mu c^{1/\kappa}}{r_{(1-\gamma)/6}} \right)\right\}$,
and suppose $n \geq c^* \ln(1/\conf)$.
We now proceed by induction.  As the right side equals $\C$ for $i=0$, the claimed inclusions are certainly true for $\hat{V}_0 = \C$, which serves as our base case.
Now suppose some $i \in \{0,1,\ldots,\hat{i}_{\bdim_f}\}$ satisfies
\begin{equation}
\label{eqn:robust-good-labels-ih}
\truV_{2^{i}} \subseteq \hat{V}_{i} \subseteq \C\left( c  \left(\frac{d i + \ln(1/\conf)}{2^{i}}\right)^{\frac{\kappa}{2\kappa - 1}}\right).
\end{equation}
In particular, Condition~\ref{con:tsybakov} implies
\begin{equation}
\label{eqn:robust-good-labels-diam}
\diam(\hat{V}_i) \leq \diam\left( \C\left( c  \left(\frac{d i + \ln(1/\conf)}{2^{i}}\right)^{\frac{\kappa}{2\kappa - 1}}\right) \right) \leq \mu c^{\frac{1}{\kappa}} \left(\frac{d i + \ln(1/\conf)}{2^{i}}\right)^{\frac{1}{2\kappa - 1}}.
\end{equation}
If $i < \hat{i}_{\bdim_f}$, then let $k$ be the integer for which $\hat{i}_{k-1} \leq i < \hat{i}_{k}$,
and otherwise let $k = \bdim_f$.
Note that we certainly have $\hat{i}_1 \geq \left\lfloor \log_{2} (n / 2)\right\rfloor$,
since $m = \lfloor n/2 \rfloor \geq 2^{\lfloor \log_{2}(n/2) \rfloor}$ is obtained while $k=1$.
Therefore, if $k>1$,
\begin{equation*}
\frac{d i + \ln(1/\conf)}{2^{i}} \leq \frac{4 d \log_{2}(n) + 4 \ln(1/\conf)}{n},
\end{equation*}
so that \eqref{eqn:robust-good-labels-diam} implies
\begin{equation*}
\diam\left( \hat{V}_i \right) \leq \mu c^{\frac{1}{\kappa}} \left( \frac{ 4 d \log_{2}(n) + 4\ln(1/\conf)}{n} \right)^{\frac{1}{2\kappa-1}}.
\end{equation*}
By our choice of $c^*$, the right side is at most $r_{(1-\gamma)/6}$.
Therefore, since Lemma~\ref{lem:Vshat-to-Boundaries} implies $f \in \cl\left(\truV_{2^{i}}\right)$ on $H_n^{(i)}$,
we have $\hat{V}_i \subseteq \Ball\left(f,r_{(1-\gamma)/6}\right)$ when $k > 1$.
Combined with \eqref{eqn:robust-good-labels-ih}, we have that
$\truV_{2^{i}} \subseteq \hat{V}_{i}$,
and either $k=1$, or
$\hat{V}_{i} \subseteq \Ball(f,r_{(1-\gamma)/6})$ and $4 m > 4\lfloor n / 2 \rfloor \geq n$.
Now consider any $m$ with $2^{i}+1 \leq m \leq \min\left\{2^{i+1}, \hat{m}^{(\bdim_f)}\right\}$,
and for the purpose of induction suppose $\truV_{m-1} \subseteq V_{i+1}$ upon reaching Step 5 for that value of $m$ in \RobustShattering.
Since $V_{i+1} \subseteq \hat{V}_{i}$ and $n \geq \init^{*}$,
Lemma~\ref{lem:kstar-good-labels} (with $\ell = m-1$)
implies that on $H^{(i)}_{n} \cap H^{(ii)}_{n}$,
\begin{equation}
\label{eqn:robust-tsybakov-good-labels}
\hat{\Delta}_{4m}^{(k)}\left(X_m, W_2, V_{i+1}\right) < \gamma \implies \hat{\Gamma}_{4m}^{(k)}\left(X_m, -f(X_m), W_2, V_{i+1}\right) < \hat{\Gamma}_{4m}^{(k)}\left(X_m, f(X_m), W_2, V_{i+1}\right),
\end{equation}
so that after Step 8 we have $\truV_{m} \subseteq V_{i+1}$.
Since \eqref{eqn:robust-good-labels-ih} implies that the $\truV_{m-1} \subseteq V_{i+1}$ condition holds
if \RobustShattering~reaches Step 5 with $m = 2^{i}+1$ (at which time $V_{i+1} = \hat{V}_{i}$),
we have by induction that on $H^{(i)}_{n} \cap H^{(ii)}_{n}$,
$\truV_{m} \subseteq V_{i+1}$ upon reaching Step 9 with $m = \min\left\{2^{i+1}, \hat{m}^{(\bdim_f)}\right\}$.
This establishes the final claim of the lemma, given that the first claim holds.
For the remainder of this inductive proof, suppose $i < \hat{i}_{\bdim_f}$.
Since Step 8 enforces that, upon reaching Step 9 with $m = 2^{i+1}$,
every $h_1,h_2 \in V_{i+1}$ have
$\er_{\hat{\L}_{i+1}}(h_1) - \er_{\hat{\L}_{i+1}}(h_2) = \er_{\truL_{i+1}}(h_1) - \er_{\truL_{i+1}}(h_2)$,
on $J_n(\conf) \cap H_{n}^{(i)} \cap H_{n}^{(ii)}$ we have
\begin{align}
\hat{V}_{i+1} & \subseteq \left\{ h \in \hat{V}_i : \er_{\truL_{i+1}}(h) - \min\limits_{h^{\prime} \in \truV_{2^{i+1}}} \er_{\truL_{i+1}}(h^{\prime}) \leq \hat{U}_{i+1}\left(\hat{V}_{i},\conf\right)\right\} \notag
\\ & \subseteq \left\{ h \in \hat{V}_i : \er_{\truL_{i+1}}(h) - \er_{\truL_{i+1}}(f) \leq \hat{U}_{i+1}\left(\hat{V}_{i},\conf\right)\right\} \notag
\\ & \subseteq \hat{V}_i \cap \C\left( 2\hat{U}_{i+1}\left(\hat{V}_{i}, \conf\right)\right) 
 \subseteq \C\left( 2\tilde{U}_{i+1}\left(\hat{V}_{i}, \conf\right)\right), \label{eqn:robust-Vip-bound}
\end{align}
where the second line follows from Lemma~\ref{lem:Vshat-to-Boundaries} and the last two inclusions follow from Lemma~\ref{lem:koltchinskii}.
Focusing on \eqref{eqn:robust-Vip-bound},
combining \eqref{eqn:robust-good-labels-diam} with \eqref{eqn:modulus-bound} (and the fact that $\phi_{i+1}(\hat{V}_i) \leq 2$), we can bound $\tilde{U}_{i+1}\left(\hat{V}_i,\conf\right)$ as follows.
\begin{align*}
\sqrt{\diam(\hat{V}_i) \frac{\ln(32 (i+1)^2/\conf)}{2^{i}}} & \leq \sqrt{\mu} c^{\frac{1}{2\kappa}} \left(\frac{d i + \ln(1/\conf)}{2^{i}}\right)^{\frac{1}{4 \kappa - 2}} \left(\frac{\ln(32 (i+1)^2/\conf)}{2^{i}}\right)^{\frac{1}{2}} \\
& \leq \sqrt{\mu} c^{\frac{1}{2\kappa}} \left(\frac{2 d i + 2\ln(1/\conf)}{2^{i+1}}\right)^{\frac{1}{4 \kappa - 2}} \left(\frac{ 8(i+1) + 2\ln(1 / \conf)}{2^{i+1}}\right)^{\frac{1}{2}} \\
&                                                                                  \leq 4 \sqrt{\mu} c^{\frac{1}{2\kappa}} \left(\frac{d (i+1) + \ln(1/\conf)}{2^{i+1}}\right)^{\frac{\kappa}{2 \kappa - 1}},
\end{align*}
\begin{align*}
\phi_{i+1}(\hat{V}_i) &
\leq c^{\prime} \sqrt{\mu} c^{\frac{1}{2\kappa}} \left( \frac{d i + \ln(1/\conf)}{2^{i}} \right)^{\frac{1}{4\kappa - 2}} \left(\frac{ d (i+2) }{2^{i}}\right)^{\frac{1}{2}} \\
&                                \leq 4 c^{\prime} \sqrt{\mu} c^{\frac{1}{2\kappa}} \left( \frac{d (i+1) + \ln(1/\conf)}{2^{i+1}} \right)^{\frac{\kappa}{2\kappa - 1}},
\end{align*}
and thus
\begin{align*}
\tilde{U}_{i+1}(\hat{V}_i,\conf) & \leq \min\left\{8 \tilde{K} c^{\prime} \sqrt{\mu} c^{\frac{1}{2\kappa}} \left( \frac{d (i+1) + \ln(1/\conf)}{2^{i+1}} \right)^{\frac{\kappa}{2\kappa-1}} + \tilde{K} \frac{\ln(32 (i+1)^2 / \conf)}{2^{i}}, 1\right\}
\\ & \leq 12 \tilde{K} c^{\prime} \sqrt{\mu} c^{\frac{1}{2\kappa}} \left( \frac{d (i+1) + \ln(1/\conf)}{2^{i+1}} \right)^{\frac{\kappa}{2\kappa-1}}
= (c / 2) \left( \frac{d (i+1) + \ln(1/\conf)}{2^{i+1}} \right)^{\frac{\kappa}{2\kappa-1}}.
\end{align*}
Combining this with \eqref{eqn:robust-Vip-bound} now implies
\begin{equation*}
\hat{V}_{i+1} \subseteq \C\left( c \left( \frac{d (i+1) + \ln(1/\conf)}{2^{i+1}} \right)^{\frac{\kappa}{2\kappa-1}}\right).
\end{equation*}

To complete the inductive proof, it remains only to show $\truV_{2^{i+1}} \subseteq \hat{V}_{i+1}$.
Toward this end, recall we have shown above that on $H^{(i)}_{n} \cap H^{(ii)}_{n}$,
$\truV_{2^{i+1}} \subseteq V_{i+1}$ upon reaching Step 9 with $m = 2^{i+1}$,
and that every $h_1,h_2 \in V_{i+1}$ at this point have
$\er_{\hat{\L}_{i+1}}(h_1) - \er_{\hat{\L}_{i+1}}(h_2) = \er_{\truL_{i+1}}(h_1) - \er_{\truL_{i+1}}(h_2)$.
Consider any $h \in \truV_{2^{i+1}}$, and note that any other $g \in \truV_{2^{i+1}}$ has $\er_{\truL_{i+1}}(g) = \er_{\truL_{i+1}}(h)$.
Thus, on $H^{(i)}_{n} \cap H^{(ii)}_{n}$,
\begin{multline}
\er_{\hat{\L}_{i+1}}(h) - \min\limits_{h^{\prime} \in V_{i+1}} \er_{\hat{\L}_{i+1}}(h^{\prime})
  = \er_{\truL_{i+1}}(h) - \min\limits_{h^{\prime} \in V_{i+1}} \er_{\truL_{i+1}}(h^{\prime})
\\  \leq \er_{\truL_{i+1}}(h) - \min\limits_{h^{\prime} \in \hat{V}_{i}} \er_{\truL_{i+1}}(h^{\prime})
= \inf\limits_{g \in \truV_{2^{i+1}}} \er_{\truL_{i+1}}(g) - \min\limits_{h^{\prime} \in \hat{V}_{i}} \er_{\truL_{i+1}}(h^{\prime}).
\label{eqn:robust-active-inf-er-core}
\end{multline}
Lemma~\ref{lem:koltchinskii} and \eqref{eqn:robust-good-labels-ih} imply that on $J_n(\conf) \cap H^{(i)}_{n} \cap H^{(ii)}_{n}$,
the last expression in \eqref{eqn:robust-active-inf-er-core} is at most
$\inf_{g \in \truV_{2^{i+1}}} \er(g) - \er(f) + \hat{U}_{i+1}( \hat{V}_{i}, \conf)$,
and Lemma~\ref{lem:Vshat-to-Boundaries} implies $f \in \cl\left( \truV_{2^{i+1}} \right)$ on $H_{n}^{(i)}$,
so that $\inf_{g \in \truV_{2^{i+1}}} \er(g) = \er(f)$.
We therefore have
\begin{equation*}
\er_{\hat{\L}_{i+1}}(h) - \min\limits_{h^{\prime} \in V_{i+1}} \er_{\hat{\L}_{i+1}}(h^{\prime}) \leq \hat{U}_{i+1}( \hat{V}_{i}, \conf ),
\end{equation*}
so that $h \in \hat{V}_{i+1}$ as well.  Since this holds for any $h \in \truV_{2^{i+1}}$, we have $\truV_{2^{i+1}} \subseteq \hat{V}_{i+1}$.
The lemma now follows by the principle of induction.
\end{proof}

\begin{lemma}
\label{lem:robust-label-complexity}
There exist $(\C,\PXY,\gamma)$-dependent constants $c_1^*, c_2^* \in [1,\infty)$
such that, for any $\eps,\conf \in (0,e^{-3})$ and integer
\begin{equation*}
n \geq c_1^* + c_2^* \hdc_f \left( \eps^{\frac{1}{\kappa}}\right) \eps^{\frac{2}{\kappa}-2} \log_{2}^{2} \left(\frac{1}{\eps\conf}\right),
\end{equation*}
when running \RobustShattering~with label budget $n$ and confidence parameter $\conf$,
on an event $J_n^{*}(\eps,\conf)$ with $\P(J_n^{*}(\eps,\conf)) \geq 1-\conf$, we have
$\hat{V}_{\hat{i}_{\bdim_f}} \subseteq \C(\eps)$.
\thmend
\end{lemma}
\begin{proof}
Define
\begin{equation*}
c_1^* = \max\left\{2^{\bdim_f+5} \left(\frac{\mu c^{1/\kappa}}{r_{(1-\gamma)/6}}\right)^{\!\!\!2\kappa-1} \!\!\!\!\!\!d\log_{2} \frac{d \mu c^{1/\kappa}}{r_{(1-\gamma)/6}}, \frac{2}{\dprob^{1/3}}\ln \left(8c^{(i)}\right), \frac{120}{\dprob^{1/3}}\ln \left(8 c^{(ii)}\right)\right\}
\end{equation*}
and
\begin{equation*}
c_2^* = \max\left\{c^*, 2^{\bdim_f+5} \cdot \left(\frac{\mu c^{1/\kappa}}{r_{(1-\gamma)/6}}\right)^{2\kappa-1}, 2^{\bdim_f+15} \cdot \frac{\mu c^2 d}{\gamma \dprob} \log_{2}^{2}(4dc) \right\}.
\end{equation*}
Fix any $\eps,\conf \in (0,e^{-3})$ and integer
$n \geq c_1^* + c_2^* \hdc_{f}\left(\eps^{\frac{1}{\kappa}}\right) \eps^{\frac{2}{\kappa}-2} \log_{2}^{2}\left(\frac{1}{\eps\conf}\right)$.

For each $i \in \{0,1,\ldots\}$, let
$\tilde{r}_i = \mu c^{\frac{1}{\kappa}} \left(\frac{d i + \ln(1/\conf)}{2^{i}}\right)^{\frac{1}{2\kappa-1}}$.
Also define
\begin{equation*}
\tilde{i} = 
\left\lceil \left(2-\frac{1}{\kappa}\right)\log_{2}\frac{c}{\eps} + \log_{2}\left[8 d \log_{2}\frac{2d c}{\eps\conf}\right]\right\rceil.
\end{equation*}
and let
$\check{i} = \min\left\{ i \in \nats : \sup_{j \geq i} \tilde{r}_j < r_{(1-\gamma)/6}\right\}$.
For any $i \in \left\{ \check{i},\ldots, \hat{i}_{\bdim_{f}}\right\}$,
let
\begin{equation*}
\Q_{i+1} = \left\{ m \in \left\{2^{i} +1, \ldots, 2^{i+1}\right\} : \hat{\Delta}_{4m}^{(\bdim_f)}\left(X_m, W_2, \Ball\left(f,\tilde{r}_i\right)\right) \geq 2\gamma/3\right\}.
\end{equation*}
Also define
\begin{equation*}
\tilde{\Q} = 
\frac{96}{\gamma \dprob} \hdc_{f}\left(\eps^{\frac{1}{\kappa}}\right) \cdot 2 \mu c^{2} \cdot \left(8 d \log_{2} \frac{2 d c}{\eps \conf} \right) \cdot \eps^{\frac{2}{\kappa}-2}.
\end{equation*}
By Lemma~\ref{lem:robust-good-labels} and Condition~\ref{con:tsybakov}, on $J_n(\conf) \cap H_n^{(i)} \cap H_n^{(ii)}$, if $i \leq \hat{i}_{\bdim_{f}}$,
\begin{equation}
\label{eqn:robust-ball-inclusion}
\hat{V}_i
\subseteq \C\left( c \left(\frac{d i + \ln(1/\conf)}{2^{i}}\right)^{\frac{\kappa}{2\kappa-1}}\right)
\subseteq \Ball\left(f, \tilde{r}_i\right).
\end{equation}
Lemma~\ref{lem:robust-good-labels} also implies that, on $J_n(\conf) \cap H_n^{(i)} \cap H_n^{(ii)}$,
for $i$ with $\hat{i}_{\bdim_{f}-1} \leq i \leq \hat{i}_{\bdim_{f}}$,
all of the sets $V_{i+1}$ obtained in \RobustShattering~while $k = \bdim_f$ and $m \in \left\{2^{i}+1,\ldots,2^{i+1}\right\}$ satisfy
$\truV_{2^{i+1}} \subseteq V_{i+1} \subseteq \hat{V}_{i}$.
Recall that $\hat{i}_{1} \geq \lfloor \log_{2}(n/2) \rfloor$, so that we have either $\bdim_f = 1$ or else every $m \in \left\{2^{i}+1,\ldots,2^{i+1}\right\}$ has $4m > n$.
Also recall that Lemma~\ref{lem:monotonic-hat-delta} implies that when the above conditions are satisfied,
and $i \geq \check{i}$, 
on $H^{\prime} \cap G_{n}^{(i)}$, $\hat{\Delta}_{4m}^{(\bdim_{f})}\left(X_{m},W_2,V_{i+1}\right) \leq (3/2)\hat{\Delta}_{4m}^{(\bdim_{f})}\left(X_m,W_2,\Ball\left(f,\tilde{r}_i\right)\right)$,
so that $|\Q_{i+1}|$ upper bounds the number of $m \in \left\{2^{i}+1,\ldots,2^{i+1}\right\}$ for which \RobustShattering~requests the label $Y_{m}$
in Step 6 of the $k=\bdim_{f}$ round.  Thus, on $J_n(\conf) \cap H_n^{(i)} \cap H_n^{(ii)}$,
$2^{\check{i}} + \sum_{i=\max\left\{\check{i},\hat{i}_{\bdim_f-1}\right\}}^{\hat{i}_{\bdim_{f}}} |\Q_{i+1}|$ upper bounds
the total number of label requests by \RobustShattering~while $k=\bdim_{f}$;
therefore, by the constraint in Step 3, we know that either this quantity is at least as big as $\left\lfloor 2^{-\bdim_{f}} n \right\rfloor$,
or else we have $2^{\hat{i}_{\bdim_{f}}+1} > \bdim_{f} \cdot 2^{n}$.
In particular, on this event, if we can show that
\begin{equation}
\label{eqn:i-tilde-conditions}
2^{\check{i}} + \sum\limits_{i=\max\left\{\check{i},\hat{i}_{\bdim_f-1}\right\}}^{\min\left\{\hat{i}_{\bdim_{f}},\tilde{i}\right\}} |\Q_{i+1}| < \left\lfloor 2^{-\bdim_f} n \right\rfloor \text{ and } 2^{\tilde{i}+1} \leq \bdim_{f} \cdot 2^n,
\end{equation}
then it must be true that $\tilde{i} < \hat{i}_{\bdim_{f}}$.  Next, we will focus on establishing this fact.

Consider any $i \in \left\{\max\left\{\check{i},\hat{i}_{\bdim_f-1}\right\},\ldots,\min\left\{\hat{i}_{\bdim_{f}}, \tilde{i}\right\}\right\}$ and any $m \in \left\{2^{i}+1,\ldots,2^{i+1}\right\}$.
If $\bdim_f = 1$, then
\begin{equation*}
\P\left( \hat{\Delta}_{4m}^{(\bdim_f)}\left(X_m,W_2,\Ball\left(f,\tilde{r}_i\right)\right) \geq 2\gamma/3 \Big| W_2\right)
= \Px^{\bdim_f}\left( \S^{\bdim_f}\left( \Ball\left(f,\tilde{r}_i\right)\right)\right).
\end{equation*}
Otherwise, if $\bdim_f > 1$, then by Markov's inequality and the definition of $\hat{\Delta}_{4m}^{(\bdim_f)}\left(\cdot,\cdot,\cdot\right)$ from \eqref{eqn:hat-delta-defn},

\begin{align*}
& \P\left( \hat{\Delta}_{4m}^{(\bdim_f)} \left(X_m,W_2,\Ball\left(f,\tilde{r}_i\right)\right) \geq 2\gamma/3 \Big| W_2\right)
 \leq \frac{3}{2\gamma} \E\left[ \hat{\Delta}_{4m}^{(\bdim_f)}\left(X_m, W_2, \Ball\left(f,\tilde{r}_i\right)\right) \Big| W_2 \right]
\\ & = \frac{3}{2\gamma} \frac{1}{M_{4m}^{(\bdim_f)}\left(\Ball\left(f,\tilde{r}_i\right)\right)} \sum_{s=1}^{\Msize{(4m)}} \P\left(S_s^{(\bdim_f)} \cup \left\{X_m\right\} \in \S^{\bdim_f}\left(\Ball\left(f,\tilde{r}_i\right)\right) \Big| S_{s}^{(\bdim_f)} \right).
\end{align*}
By Lemma~\ref{lem:basic-Mk-lower-bound}, Lemma~\ref{lem:robust-good-labels}, and \eqref{eqn:robust-ball-inclusion},
on $J_n(\conf) \cap H_n^{(i)} \cap H_n^{(ii)}$, this is at most 
\begin{align*}
& \frac{3}{\dprob \gamma} \frac{1}{\Msize{(4m)}} \sum_{s=1}^{\Msize{(4m)}} \P\left(S_s^{(\bdim_f)} \cup \left\{X_m\right\} \in \S^{\bdim_f}\left(\Ball\left(f,\tilde{r}_i\right)\right) \Big| S_{s}^{(\bdim_f)} \right)
\\ & \leq \frac{24}{\dprob \gamma} \frac{1}{4^3 2^{3i+3}} \sum_{s=1}^{4^3 2^{3i+3}} \P\left(S_s^{(\bdim_f)} \cup \left\{X_m\right\} \in \S^{\bdim_f}\left(\Ball\left(f,\tilde{r}_i\right)\right) \Big| S_{s}^{(\bdim_f)} \right). 
\end{align*}
Note that this value is invariant to the choice of $m \in \left\{2^{i}+1,\ldots,2^{i+1}\right\}$.
By Hoeffding's inequality, on an event $J_n^*(i)$ of probability $\P\left(J_n^*(i)\right) \geq 1-\conf / (16 i^2)$, this is at most
\begin{equation}
\label{eqn:robust-label-complexity-Hoeffding}
\frac{24}{\dprob \gamma} \left( \sqrt{\frac{\ln(4 i / \conf)}{4^3 2^{3i+3}}}+ \Px^{\bdim_f}\left(\S^{\bdim_f}\left(\Ball\left(f,\tilde{r}_i\right)\right)\right) \right).
\end{equation}
Since $i \geq \hat{i}_{1} > \log_{2}(n/4)$ and $n \geq \ln(1/\conf)$, we have
\begin{equation*}
\sqrt{\frac{\ln(4 i / \conf)}{4^3 2^{3i+3}}}
\leq 2^{-i} \sqrt{\frac{\ln( 4 \log_{2}(n/4) / \conf)}{128 n}}
\leq 2^{-i} \sqrt{\frac{\ln( n / \conf)}{128 n}}
\leq 2^{-i}.
\end{equation*}
Thus, \eqref{eqn:robust-label-complexity-Hoeffding} is at most
\begin{equation*}
\frac{24}{\dprob \gamma} \left( 2^{-i} + \Px^{\bdim_f}\left(\S^{\bdim_f}\left(\Ball\left(f,\tilde{r}_i\right)\right)\right) \right).
\end{equation*}
In either case ($\bdim_f = 1$ or $\bdim_f > 1$), by definition of $\hdc_f\left(\eps^{\frac{1}{\kappa}}\right)$,
on $J_n(\conf) \cap H_n^{(i)} \cap H_n^{(ii)} \cap J_n^*(i)$,
$\forall m \in \left\{2^{i}+1,\ldots,2^{i+1}\right\}$
we have
\begin{equation}
\label{eqn:robust-label-complexity-bernoulli-mean-bound}
\P\left( \hat{\Delta}_{4m}^{(\bdim_f)}\left( X_m, W_2, \Ball\left(f,\tilde{r}_i\right)\right) \geq 2\gamma/3 \Big| W_2\right)
\leq \frac{24}{\dprob \gamma} \left( 2^{-i} + \hdc_f\left(\eps^{\frac{1}{\kappa}}\right) \cdot \max\left\{\tilde{r}_i, \eps^{\frac{1}{\kappa}}\right\}\right).
\end{equation}
Furthermore, the
$\ind_{[2\gamma/3,\infty)}\left(\hat{\Delta}_{4m}^{(\bdim_f)}\left(X_m,W_2,\Ball\left(f,\tilde{r}_i\right)\right)\right)$
indicators are conditionally independent given $W_2$,
so that we may bound $\P\left(|Q_{i+1}| > \tilde{Q} \Big| W_2\right)$
via a Chernoff bound.
Toward this end, note that
on $J_n(\conf) \cap H_n^{(i)} \cap H_n^{(ii)} \cap J_n^*(i)$,
\eqref{eqn:robust-label-complexity-bernoulli-mean-bound} implies
\begin{align}
& \E\left[ \left|\Q_{i+1}\right| \big| W_2 \right]
 = \sum_{m=2^{i}+1}^{2^{i+1}} \P\left( \hat{\Delta}_{4m}^{(\bdim_f)}\left(X_m, W_2, \Ball\left(f,\tilde{r}_i\right)\right) \geq 2\gamma/3 \Big| W_2 \right) \notag
\\ & \leq 2^{i} \cdot \frac{24}{\dprob \gamma} \left( 2^{-i} +  \hdc_{f}\left(\eps^{\frac{1}{\kappa}}\right) \cdot \max\left\{ \tilde{r}_i, \eps^{\frac{1}{\kappa}}\right\}\right)
\leq \frac{24}{\dprob \gamma} \left( 1 + \hdc_{f}\left(\eps^{\frac{1}{\kappa}}\right) \cdot \max\left\{ 2^{i} \tilde{r}_i,  2^{\tilde{i}} \eps^{\frac{1}{\kappa}}\right\} \right). \label{eqn:robust-expectation-1}
\end{align}
Note that
\begin{align*}
2^{i} \tilde{r}_i
& = \mu c^{\frac{1}{\kappa}} \left( d i + \ln(1/\conf) \right)^{\frac{1}{2\kappa - 1}} \cdot 2^{i\left(1 - \frac{1}{2\kappa - 1}\right)}
\\ &\leq \mu c^{\frac{1}{\kappa}} \left( d \tilde{i} + \ln(1/\conf) \right)^{\frac{1}{2\kappa - 1}} \cdot 2^{\tilde{i}\left(1 - \frac{1}{2\kappa - 1}\right)}
\leq \mu c^{\frac{1}{\kappa}} \left( 8 d \log_{2}\frac{2 d c}{\eps \conf} \right)^{\frac{1}{2\kappa - 1}} \cdot 2^{\tilde{i}\left(1 - \frac{1}{2\kappa - 1}\right)}.
\end{align*}
Then since $2^{-\tilde{i}\frac{1}{2\kappa-1}} \leq \left(\frac{\eps}{c}\right)^{\frac{1}{\kappa}} \cdot \left(8 d \log_{2} \frac{2 d c}{\eps \conf}\right)^{- \frac{1}{2\kappa - 1}}$,
we have that the rightmost expression in \eqref{eqn:robust-expectation-1} is at most
\begin{equation*}
 \frac{24}{\gamma \dprob} \left( 1 + \hdc_{f}\left(\eps^{\frac{1}{\kappa}}\right) \cdot \mu \cdot 2^{\tilde{i}} \eps^{\frac{1}{\kappa}}\right)
 \leq \frac{24}{\gamma \dprob} \left( 1 + \hdc_{f}\left(\eps^{\frac{1}{\kappa}}\right) \cdot 2 \mu c^{2} \cdot \left(8 d \log_{2} \frac{2 d c}{\eps \conf} \right) \cdot \eps^{\frac{2}{\kappa}-2}\right)
\leq \tilde{\Q} / 2.
\end{equation*}
Therefore, a Chernoff bound implies that on $J_n(\conf) \cap H_n^{(i)} \cap H_n^{(ii)} \cap J_n^*(i)$, we have
\begin{align*}
\P\left( \left|\Q_{i+1}\right| > \tilde{\Q} \Big| W_2 \right)
& \leq \exp\left\{ - \tilde{\Q} / 6 \right\}
\leq \exp\left\{ - 8 \log_{2}\left( \frac{2 \vc c}{\eps \conf}\right)\right\}
\\ & \leq \exp\left\{ - \log_{2}\left( \frac{ 48 \log_{2}\left(2 \vc c / \eps \conf\right)}{\conf}\right)\right\}
\leq \conf / (8 \tilde{i}).
\end{align*}
Combined with the law of total probability and a union bound over $i$ values, 
this implies there exists an event $J_n^*(\eps,\conf) \subseteq J_n(\conf) \cap H_n^{(i)} \cap H_n^{(ii)}$
with $\P\left( J_n(\conf) \cap H_n^{(i)} \cap H_n^{(ii)} \setminus  J_n^*(\eps,\conf) \right) \leq \sum_{i=\check{i}}^{\tilde{i}} \left(\conf / (16 i^2) + \conf / (8 \tilde{i})\right) \leq \conf / 4$,
on which 
every $i \in \left\{\max\left\{\check{i},\hat{i}_{\bdim_{f}-1}\right\},\ldots,\min\left\{\hat{i}_{\bdim_{f}},\tilde{i}\right\}\right\}$ has $\left| \Q_{i+1} \right| \leq \tilde{\Q}$.

We have chosen $c_1^*$ and $c_2^*$ large enough that
$2^{\tilde{i}+1} < \bdim_f \cdot 2^{n}$
and $2^{\check{i}} < 2^{-\bdim_f-2} n$.  In particular, this means that on $J_n^*(\eps,\conf)$,
\begin{equation*}
2^{\check{i}} + \sum_{i = \max\left\{\check{i},\hat{i}_{\bdim_{f}-1}\right\}}^{\min\left\{\tilde{i},\hat{i}_{\bdim_{f}}\right\}} |\Q_{i+1}| < 2^{-\bdim_f-2} n + \tilde{i}\tilde{\Q}.
\end{equation*}
Furthermore, since $\tilde{i} \leq 3 \log_{2} \frac{4 d c}{\eps \conf}$, we have
\begin{align*}
\tilde{i} \tilde{\Q} &\leq \frac{2^{13} \mu c^2 d}{\gamma \dprob} \hdc_{f}\left(\eps^{\frac{1}{\kappa}}\right) \cdot \eps^{\frac{2}{\kappa}-2} \cdot \log_{2}^{2} \frac{4 d c}{\eps \conf}
\\ & \leq \frac{2^{13} \mu c^2 d \log_{2}^{2} (4 d c)}{\gamma \dprob} \hdc_{f}\left(\eps^{\frac{1}{\kappa}}\right) \cdot \eps^{\frac{2}{\kappa}-2} \cdot \log_{2}^{2} \frac{1}{\eps \conf} \leq 2^{-\bdim_f-2} n.
\end{align*}
Combining the above, we have that \eqref{eqn:i-tilde-conditions}
is satisfied on $J_n^*(\eps,\conf)$, so that
$\hat{i}_{\bdim_f} > \tilde{i}$.
Combined with Lemma~\ref{lem:robust-good-labels}, this implies that on $J_n^*(\eps,\conf)$,
\begin{equation*}
\hat{V}_{\hat{i}_{\bdim_f}} \subseteq \hat{V}_{\tilde{i}} \subseteq \C\left( c \left( \frac{d \tilde{i} + \ln(1/\conf)}{2^{\tilde{i}}}\right)^{\frac{\kappa}{2\kappa-1}}\right),
\end{equation*}
and by definition of $\tilde{i}$ we have
\begin{align*}
c \left( \frac{d \tilde{i} + \ln(1/\conf)}{2^{\tilde{i}}}\right)^{\frac{\kappa}{2\kappa-1}}
& \leq c \left( 8 d \log_{2} \frac{2dc}{\eps\conf}\right)^{\frac{\kappa}{2\kappa-1}} \cdot 2^{-\tilde{i} \frac{\kappa}{2\kappa-1}}
\\ & \leq c \left( 8 d \log_{2} \frac{2dc}{\eps\conf}\right)^{\frac{\kappa}{2\kappa-1}} \cdot \left(\eps / c\right) \cdot \left(8 d \log_{2}\frac{2dc}{\eps\conf}\right)^{-\frac{\kappa}{2\kappa-1}} = \eps,
\end{align*}
so that $\hat{V}_{\hat{i}_{\bdim_f}} \subseteq \C(\eps)$.

Finally, to prove the stated bound on $\P(J_n^*(\eps,\conf))$, we have
\begin{align*}
1-\P\left(J_n^*(\eps,\conf)\right)
& \leq \left(1-\P(J_n(\conf))\right)
+ \left(1-\P\left(H_n^{(i)}\right)\right) + \P\left(H_n^{(i)} \setminus H_n^{(ii)}\right)
\\ &\phantom{\leq }+ \P\left(J_n(\conf) \cap H_n^{(i)} \cap H_n^{(ii)} \setminus J_n^*(\eps,\conf)\right)
\\ & \leq 3\conf/4 + c^{(i)} \cdot \exp\left\{ - n^3 \dprob / 8\right\} + c^{(ii)} \cdot \exp\left\{- n \dprob^{1/3} / 120\right\}
\leq \conf.
\end{align*}
\end{proof}

Finally, we are ready for the proof of Lemma~\ref{lem:robust-tsybakov}.

\begin{proof}[Lemma~\ref{lem:robust-tsybakov}]
First, note that because we break ties in the $\argmax$ of Step 7 in
favor of a $\hat{y}$ value with $V_{i_k+1}[(X_m,\hat{y})] \neq \emptyset$,
if $V_{i_k+1} \neq \emptyset$ before Step 8, then this remains true after Step 8.  Furthermore,
the $\hat{U}_{i_k+1}$ estimator is nonnegative, and thus the update in Step 10 never removes from
$V_{i_k+1}$ the minimizer of $\er_{\hat{\L}_{i_k+1}}(h)$ among $h \in V_{i_k+1}$.
Therefore, by induction we have $V_{i_k} \neq \emptyset$ at all times in \RobustShattering.
In particular, $\hat{V}_{\hat{i}_{d+1}+1} \neq \emptyset$ so that the return classifier
$\hat{h}$ exists.
Also, by Lemma~\ref{lem:robust-label-complexity}, for $n$ as in Lemma~\ref{lem:robust-label-complexity}, on $J_n^*(\eps,\conf)$,
running \RobustShattering~with label budget $n$ and confidence parameter $\conf$ results in $\hat{V}_{\hat{i}_{\bdim_f}} \subseteq \C(\eps)$.
Combining these two facts implies that for such a value of $n$, on $J_n^*(\eps,\conf)$,
$\hat{h} \in \hat{V}_{\hat{i}_{d+1}+1} \subseteq \hat{V}_{\hat{i}_{\bdim_f}} \subseteq \C(\eps)$,
so that $\er\left(\hat{h}\right) \leq \nu + \eps$.
\end{proof}

\subsection{The Misspecified Model Case}
\label{app:misspecified-model-trivial}

Here we present a proof of Theorem~\ref{thm:misspecified-model-trivial}, including a specification of the
method $\alg_a^{\prime}$ from the theorem statement.

\begin{proof}[Theorem~\ref{thm:misspecified-model-trivial}]
Consider a weakly universally consistent passive learning algorithm $\alg_{u}$ \citep*{devroye:96}.
Such a method must exist in our setting;
for instance, Hoeffding's inequality and a union bound imply that it suffices to take
$\alg_{u}(\L) = \argmin_{\ind^{\pm}_{B_i}} \er_{\L}(\ind^{\pm}_{B_i}) + \sqrt{\frac{\ln\left(4 i^2 |\L| \right)}{2|\L|}}$,
where $\{B_1,B_2,\ldots\}$ is a countable algebra that generates $\mathcal{F}_{\X}$.
%it easily follows from second-countability of the generating topology of standard Borel spaces that such an algebra exists.

Then $\alg_{u}$ achieves a label complexity $\Lambda_{u}$ such that for any distribution $\PXY$ on $\X \times \{-1,+1\}$,
$\forall \eps \in (0,1)$,
$\Lambda_u(\eps + \nu^*(\PXY), \PXY) < \infty$.
In particular, if $\nu^*(\PXY) < \nu(\C ; \PXY)$,
then $\Lambda_u((\nu^*(\PXY) + \nu(\C ; \PXY))/2, \PXY) < \infty$.

Fix any $n \in \nats$, and describe the execution of $\alg_a^{\prime}(n)$ as follows.
In a preprocessing step, withhold the first $m_{un} = n - \lfloor n/2\rfloor - \lfloor n/3 \rfloor \geq n/6$ examples
$\{X_1,\ldots,X_{m_{un}}\}$ and request their labels $\{Y_1,\ldots,Y_{m_{un}}\}$.
Run $\alg_{a}(\lfloor n/2 \rfloor)$ on the remainder of the sequence $\{X_{m_{un}+1}, X_{m_{un}+2}, \ldots\}$ (i.e., shift any index references in the algorithm by $m_{un}$),
and let $h_a$ denote the classifier it returns.  
Also request the labels $Y_{m_{un}+1},\ldots Y_{m_{un} + \lfloor n/3 \rfloor}$,
and let 
\begin{equation*}
h_u = \alg_{u}\left( \left\{(X_{m_{un}+1},Y_{m_{un}+1}),\ldots,(X_{m_{un}+\lfloor n/3\rfloor},Y_{m_{un}+\lfloor n/3 \rfloor})\right\}\right).
\end{equation*}
If $\er_{m_{un}}(h_a) - \er_{m_{un}}(h_u) > n^{-1/3}$, return $\hat{h} = h_u$; otherwise, return $\hat{h} = h_a$.
This method achieves the stated result, for the following reasons.

First, let us examine the final step of this algorithm.  By Hoeffding's inequality,
with probability at least $1 - 2 \cdot \exp\left\{ - n^{1/3} / 12 \right\}$,
\begin{equation*}
\left| \left(\er_{m_{un}}(h_a) - \er_{m_{un}}(h_u)\right) - \left(\er(h_a) - \er(h_u)\right)\right| 
\leq n^{-1/3}.
\end{equation*}
When this is the case, a triangle inequality implies 
$\er(\hat{h}) \leq \min\{\er(h_a), \er(h_u) + 2 n^{-1/3}\}$.

If $\PXY$ satisfies the benign noise case,
then for any
\begin{equation*}
n \geq 2 \Lambda_a(\eps/2 + \nu(\C ; \PXY),\PXY),
\end{equation*}
we have $\E[\er(h_a)] \leq \nu(\C ; \PXY) + \eps/2$,
so $\E[\er(\hat{h})] \leq \nu(\C ; \PXY) + \eps/2 + 2 \cdot \exp\{- n^{1/3} / 12\}$,
which is at most $\nu(\C ; \PXY) + \eps$ if $n \geq 12^3\ln^3(4/\eps)$.
So in this case, we can take $\lambda(\eps) = \left\lceil 12^3 \ln^3(4/\eps) \right\rceil$.

On the other hand, if $\PXY$ is not in the benign noise case (i.e., the misspecified model case),
then for any
$n \geq 3 \Lambda_u((\nu^*(\PXY) + \nu(\C ; \PXY))/2,\PXY)$,
$\E\left[ \er(h_u) \right] \leq (\nu^*(\PXY) + \nu(\C;\PXY))/2$,
so that
\begin{align*}
\E[\er(\hat{h})]
& \leq \E[\er(h_u)] + 2 n^{-1/3} + 2 \cdot \exp\{-n^{1/3}/12\}
\\ &\leq (\nu^*(\PXY) + \nu(\C ; \PXY))/2 + 2 n^{-1/3} + 2 \cdot \exp\{-n^{1/3}/12\}.
\end{align*}
Again, this is at most $\nu(\C ; \PXY) + \eps$ if
$n \geq \max\left\{12^3 \ln^3 \frac{2}{\eps}, 64(\nu(\C ; \PXY) - \nu^*(\PXY))^{-3}\right\}$.
So in this case, we can take
\begin{equation*}
\lambda(\eps) = \left\lceil \max\left\{
12^3 \ln^3 \frac{2}{\eps},
3 \!\Lambda_u\left(\frac{\nu^*(\PXY)+\nu(\C;\PXY)}{2},\PXY\right)\!,
\frac{64}{(\nu(\C ; \PXY) - \nu^*(\PXY))^{3}}
\right\}\right\rceil\!.
\end{equation*}

In either case, we have $\lambda(\eps) \in \Polylog(1/\eps)$.
\end{proof}

\section*{Acknowledgments}
I am grateful to
Nina Balcan, Rui Castro, Sanjoy Dasgupta, Carlos Guestrin,
Vladimir Koltchinskii,
John Langford, Rob Nowak, Larry Wasserman, and Eric Xing for insightful discussions.

\bibliography{learning}

\end{document}